\RequirePackage[l2tabu,orthodox]{nag}
\documentclass
[letterpaper,11pt,]
{article}

\usepackage{etex}
\usepackage{verbatim}
\usepackage{xspace,enumerate}
\usepackage[dvipsnames]{xcolor}
\usepackage[T1]{fontenc}
\usepackage[full]{textcomp}
\usepackage[american]{babel}
\usepackage{mathtools}
\usepackage{tablefootnote}
\usepackage{amsthm}
\usepackage[
letterpaper,
top=1in,
bottom=1in,
left=1in,
right=1in]{geometry}
\usepackage{newpxtext} 
\usepackage{textcomp} 
\usepackage[varg,bigdelims]{newpxmath}
\usepackage[scr=rsfso]{mathalfa}
\usepackage{bm} 
\let\mathbb\varmathbb
\usepackage{microtype}
\usepackage[pagebackref,colorlinks=true,urlcolor=blue,linkcolor=blue,citecolor=OliveGreen]{hyperref}
\usepackage[capitalise,nameinlink]{cleveref}
\crefname{lemma}{Lemma}{Lemmas}
\crefname{fact}{Fact}{Facts}
\crefname{theorem}{Theorem}{Theorems}
\crefname{corollary}{Corollary}{Corollaries}
\crefname{claim}{Claim}{Claims}
\crefname{example}{Example}{Examples}
\crefname{algorithm}{Algorithm}{Algorithms}
\crefname{problem}{Problem}{Problems}
\crefname{definition}{Definition}{Definitions}
\crefname{exercise}{Exercise}{Exercises}
\usepackage{amsthm}

\newtheorem{theorem}{Theorem}[section]
\newtheorem*{theorem*}{Theorem}
\newtheorem{lemma}[theorem]{Lemma}
\newtheorem*{lemma*}{Lemma}
\newtheorem{fact}[theorem]{Fact}
\newtheorem*{fact*}{Fact}
\newtheorem{proposition}[theorem]{Proposition}
\newtheorem*{proposition*}{Proposition}
\newtheorem{corollary}[theorem]{Corollary}
\newtheorem*{corollary*}{Corollary}

\newtheorem*{hypothesis*}{Hypothesis}
\newtheorem{conjecture}[theorem]{Conjecture}
\newtheorem*{conjecture*}{Conjecture}
\theoremstyle{definition}
\newtheorem{definition}[theorem]{Definition}
\newtheorem*{definition*}{Definition}

\newtheorem*{construction*}{Construction}

\newtheorem*{example*}{Example}

\newtheorem*{question*}{Question}
\newtheorem{algorithm}[theorem]{Algorithm}
\newtheorem*{algorithm*}{Algorithm}

\newtheorem*{assumption*}{Assumption}
\newtheorem{problem}[theorem]{Problem}
\newtheorem*{problem*}{Problem}

\newtheorem*{openquestion*}{Open Question}
\theoremstyle{remark}

\newtheorem*{claim*}{Claim}
\newtheorem{remark}[theorem]{Remark}
\newtheorem*{remark*}{Remark}

\newtheorem*{observation*}{Observation}
\usepackage{paralist}
\let\originalleft\left
\let\originalright\right
\renewcommand{\left}{\mathopen{}\mathclose\bgroup\originalleft}
\renewcommand{\right}{\aftergroup\egroup\originalright}
\usepackage{turnstile}
\usepackage{mdframed}
\usepackage{tikz}
\usetikzlibrary{positioning}
\usepackage{caption}
\DeclareCaptionType{Algorithm}
\usepackage{newfloat}
\usepackage{float}
\restylefloat{table}
\usepackage{array}
\usepackage{subfig}
\usepackage{enumitem}
\usepackage{xparse}
\usepackage{amsthm} 
\makeatletter
\let\latexparagraph\paragraph
\RenewDocumentCommand{\paragraph}{som}{%
  \IfBooleanTF{#1}
    {\latexparagraph*{#3}}
    {\IfNoValueTF{#2}
       {\latexparagraph{\maybe@addperiod{#3}}}
       {\latexparagraph[#2]{\maybe@addperiod{#3}}}%
  }%
}
\newcommand{\maybe@addperiod}[1]{%
  #1\@addpunct{.}%
}
\makeatother

\newcommand{\Authornote}[2]{}
\newcommand{\Authornotecolored}[3]{}
\newcommand{\Authorcomment}[2]{}
\newcommand{\Authorfnote}[2]{}
\newcommand{\Tnote}{\Authornote{T}}
\newcommand{\Tcomment}{\Authorcomment{T}}

\newcommand{\Gnote}{\Authornote{G}}

\usepackage{boxedminipage}
\newcommand{\paren}[1]{(#1)}
\newcommand{\Paren}[1]{\left(#1\right)}

\newcommand{\brac}[1]{[#1]}
\newcommand{\Brac}[1]{\left[#1\right]}

\newcommand{\abs}[1]{\lvert#1\rvert}
\newcommand{\Abs}[1]{\left\lvert#1\right\rvert}

\newcommand{\card}[1]{\lvert#1\rvert}
\newcommand{\Card}[1]{\left\lvert#1\right\rvert}

\newcommand{\set}[1]{\{#1\}}
\newcommand{\Set}[1]{\left\{#1\right\}}

\newcommand{\norm}[1]{\lVert#1\rVert}
\newcommand{\Norm}[1]{\left\lVert#1\right\rVert}

\newcommand{\snorm}[1]{\norm{#1}^2}
\newcommand{\Snorm}[1]{\Norm{#1}^2}

\newcommand{\Normo}[1]{\Norm{#1}_1}

\newcommand{\normi}[1]{\norm{#1}_\infty}
\newcommand{\Normi}[1]{\Norm{#1}_\infty}

\newcommand{\iprod}[1]{\langle#1\rangle}
\newcommand{\Iprod}[1]{\left\langle#1\right\rangle}

\newcommand{\Esymb}{\mathbb{E}}
\newcommand{\Psymb}{\mathbb{P}}
\newcommand{\Vsymb}{\mathbb{V}}
\DeclareMathOperator*{\E}{\Esymb}
\DeclareMathOperator*{\Var}{\Vsymb}
\DeclareMathOperator*{\ProbOp}{\Psymb}
\renewcommand{\Pr}{\ProbOp}
\newcommand{\given}{\mathrel{}\middle\vert\mathrel{}}
\newcommand{\suchthat}{\;\middle\vert\;}
\newcommand{\tensor}{\otimes}

\newcommand{\sge}{\succeq}

\renewcommand{\ij}{{ij}}

\newcommand{\tensorpower}[2]{#1^{\tensor #2}}

\newcommand{\from}{\colon}

\newcommand{\mper}{\,.}
\newcommand{\mcom}{\,,}
\newcommand\bdot\bullet

\DeclareMathOperator{\Tr}{Tr}

\DeclareMathOperator{\argmax}{argmax}
\DeclareMathOperator{\polylog}{polylog}
\DeclareMathOperator{\supp}{supp}

\DeclareMathOperator{\sign}{sign}

\DeclareMathOperator{\tsum}{{\textstyle \sum}}

\DeclareMathOperator{\tprod}{{\textstyle \prod}}

\newcommand{\Z}{\mathbb Z}
\newcommand{\N}{\mathbb N}
\newcommand{\R}{\mathbb R}

\newcommand{\cA}{\mathcal A}
\newcommand{\cB}{\mathcal B}
\newcommand{\cC}{\mathcal C}
\newcommand{\cD}{\mathcal D}
\newcommand{\cE}{\mathcal E}

\newcommand{\cG}{\mathcal G}
\newcommand{\cH}{\mathcal H}

\newcommand{\cM}{\mathcal M}
\newcommand{\cN}{\mathcal N}

\newcommand{\cS}{\mathcal S}

\newcommand{\cU}{\mathcal U}

\newcommand{\cX}{\mathcal X}

\newcommand{\cZ}{\mathcal Z}

\newcommand{\bbP}{\mathbb P}

\renewcommand{\leq}{\leqslant}
\renewcommand{\le}{\leqslant}
\renewcommand{\geq}{\geqslant}
\renewcommand{\ge}{\geqslant}
\let\epsilon=\varepsilon
\numberwithin{equation}{section}
\newcommand\MYcurrentlabel{xxx}
\newcommand{\MYstore}[2]{%
  \global\expandafter \def \csname MYMEMORY #1 \endcsname{#2}%
}
\newcommand{\MYload}[1]{%
  \csname MYMEMORY #1 \endcsname%
}
\newcommand{\MYnewlabel}[1]{%
  \renewcommand\MYcurrentlabel{#1}%
  \MYoldlabel{#1}%
}
\newcommand{\MYdummylabel}[1]{}
\newcommand{\torestate}[1]{%
  \let\MYoldlabel\label%
  \let\label\MYnewlabel%
  #1%
  \MYstore{\MYcurrentlabel}{#1}%
  \let\label\MYoldlabel%
}
\newcommand{\restatetheorem}[1]{%
  \let\MYoldlabel\label
  \let\label\MYdummylabel
  \begin{theorem*}[Restatement of \cref{#1}]
    \MYload{#1}
  \end{theorem*}
  \let\label\MYoldlabel
}
\newcommand{\restatelemma}[1]{%
  \let\MYoldlabel\label
  \let\label\MYdummylabel
  \begin{lemma*}[Restatement of \cref{#1}]
    \MYload{#1}
  \end{lemma*}
  \let\label\MYoldlabel
}
\newcommand{\restateprop}[1]{%
  \let\MYoldlabel\label
  \let\label\MYdummylabel
  \begin{proposition*}[Restatement of \cref{#1}]
    \MYload{#1}
  \end{proposition*}
  \let\label\MYoldlabel
}
\newcommand{\restatefact}[1]{%
  \let\MYoldlabel\label
  \let\label\MYdummylabel
  \begin{fact*}[Restatement of \cref{#1}]
    \MYload{#1}
  \end{fact*}
  \let\label\MYoldlabel
}
\newcommand{\restate}[1]{%
  \let\MYoldlabel\label
  \let\label\MYdummylabel
  \MYload{#1}
  \let\label\MYoldlabel
}

\newcommand{\eps}{\epsilon}

\allowdisplaybreaks
\newcommand*{\Id}{\mathrm{Id}}

\newcommand*{\normf}[1]{\norm{#1}_{\mathrm{F}}}

\newcommand{\Sdiff}{S_\text{diff}}
\newcommand{\ind}[1]{\textbf{1}_{\Brac{#1}}}
\providecommand{\todo}{{\color{red}{\textbf{TODO}}}}

\providecommand{\yij}{y_{i,j}}
\providecommand{\wij}{w_{ij}}
\providecommand{\nulld}{\nu}
\providecommand{\planted}{\mu}
\providecommand{\Ep}{\E_{\planted}}
\providecommand{\En}{\E_{\nulld}}
\providecommand{\Hermitepolys}[1]{\cH_{\leq #1}}
\providecommand{\hermitepoly}[2]{H_{#2}\Paren{#1}}
\providecommand{\lowdegpolys}[1]{\R[Y]_{\leq#1}}
\providecommand{\multilinearpoly}[1]{\cM_{\leq#1}}
\providecommand{\Hermitemlpolys}[1]{\cH\cM_{\leq #1}}
\providecommand{\Tnote}{\Authornote{T}}
\providecommand{\Tcomment}{\Authorcomment{T}}

\providecommand{\Gnote}{\Authornote{G}}

\DeclareMathOperator{\pE}{\tilde{\mathbb{E}}}

\newcommand*{\transpose}[1]{{#1}{}^{\mkern-1.5mu\mathsf{T}}}
\newcommand*{\dyad}[1]{#1#1{}^{\mkern-1.5mu\mathsf{T}}}

\usepackage{etoolbox}
\newtoggle{shortv}
\togglefalse{shortv}

\title{
  Sparse PCA: Algorithms, Adversarial Perturbations and Certificates
}

\author{
  Tommaso d'Orsi\thanks{ETH Z\"urich. Supported by Steurer's ERC Consolidator Grant.}
  \and
  Pravesh K. Kothari\thanks{Carnegie-Mellon University. Part of this work done while at Princeton University and the Institute for Advanced Study.}
  \and
  Gleb Novikov\thanks{ETH Z\"urich.}
  \and
  David Steurer\thanks{ETH Z\"urich. Supported by an ERC Consolidator Grant.}
}

\begin{document}

\pagestyle{empty}


\maketitle
\thispagestyle{empty} 


\begin{abstract}

  We study efficient algorithms for Sparse PCA in standard statistical models (spiked covariance in its Wishart form).
Our goal is to achieve optimal recovery guarantees while being resilient to small perturbations.
Despite a long history of prior works, including explicit studies of perturbation resilience, the best known algorithmic guarantees for Sparse PCA are fragile and break down under small adversarial perturbations.

We observe a basic connection between perturbation resilience and \emph{certifying algorithms} that are based on certificates of upper bounds on sparse eigenvalues of random matrices.
In contrast to other techniques, such certifying algorithms, including the brute-force maximum likelihood estimator, are automatically robust against small adversarial perturbation.

We use this connection to obtain the first polynomial-time algorithms for this problem that are resilient against additive adversarial  perturbations by obtaining new efficient certificates for upper bounds on sparse eigenvalues of random matrices.
Our algorithms are based either on basic semidefinite programming or on its low-degree sum-of-squares strengthening depending on the parameter regimes.
Their guarantees either match or approach the best known guarantees of \emph{fragile} algorithms in terms of sparsity of the unknown vector, number of samples and the ambient dimension.

To complement our algorithmic results, we prove rigorous lower bounds matching the gap between fragile and robust polynomial-time algorithms in a natural computational model based on low-degree polynomials (closely related to the pseudo-calibration technique for sum-of-squares lower bounds) that is known to capture the best known guarantees for related statistical estimation problems. The combination of these results provides formal evidence of an inherent price to pay to achieve robustness.

Beyond these issues of perturbation resilience, our analysis also leads to new algorithms for the fragile setting, whose guarantees improve over best previous results in some parameter regimes (e.g. if the sample size is polynomially smaller than the dimension).

\end{abstract}

\clearpage


\microtypesetup{protrusion=false}
\tableofcontents{}
\microtypesetup{protrusion=true}

\clearpage

\pagestyle{plain}
\setcounter{page}{1}


\section{Introduction}\label{sec:introduction-new}
\emph{Sparse principal component analysis (sparse PCA)} is a fundamental primitive in high-dimensional statistics.
Given a collection of vectors \(y_1,\ldots,y_n\in\R^d\), we seek a ``structured'' direction \(v_0\in\R^d\) with \(\norm{v_0}=1\) maximally correlated with the vectors, commonly measured by the empirical variance of \(\set{\iprod{y_i,v_0}}_{i\in [n]}\).
The structure we impose on \(v_0\) is sparsity, that is, an upper bound on the number of its non-zero entries.

\paragraph{Spiked covariance model} A widely studied statistical model for sparse PCA is the \emph{spiked covariance model} (also called \textit{Wishart model}).
Here, \(y_1,\ldots,y_n\) are independent draws from the distribution \(N(0,\Id_d + \beta\cdot v_0 v_0^T)\) for an unknown \(k\)-sparse unit vector \(v_0\in \R^d\).
(For simplicity, we will assume that the sparsity parameter \(k\) is known.)
The goal is to compute an estimate \(\hat v\) for \(v_0\) with correlation\footnote{
	Instead of asking for the correlation to be bounded away from \(0\), we could also ask for it to approach \(1\).
	Alternatively, we could ask to recover the support of \(v_0\).
	At the granularity of our discussion here, these measures of success are equivalent in most regards.} bounded away from \(0\) so that  \(\norm{\hat v}=1\) and \(\iprod{\hat v, v_0}^2\ge c\) for an absolute constant \(c>0\).
(Here, we square the inner product because \(v_0\) is identifiable only up to sign.)

In order to simplify our discussion, we hide multiplicative factors logarithmic in \(d\) using the notation \(\tilde O(\cdot)\). Similarly, we hide absolute constant multiplicative factors using the standard notations \(\lesssim\), \(O(\cdot)\), $\Omega(\cdot)$ and $\Theta(\cdot)$.

If we ignore computational efficiency, we can achieve  optimal statistical guarantees for sparse PCA in the spiked covariance model by the following kind of exhaustive search: among all \(k\)-by-\(k\) principal submatrices of the empirical covariance matrix of the vectors \(\set{y_i}_{i\in[n]}\), find one with maximum eigenvalue and output a corresponding eigenvector (e.g., \cite{amini2009,DBLP:journals/corr/abs-1304-0828,berthet2013}).
This procedure achieves constant correlation with high probability as long as \(n\ge \tilde O(k/\min\set{\beta,\beta^2})\). However, the running-time is exponential in \(k\).  When the number of samples $n$ is significantly smaller than the ambient dimension $d$ as well as the sparsity parameter $k$, an alternative approach is to find a unit vector $u$ such that $\transpose{u} Y$ is close to a $k$-sparse vector. This procedure also works for \(n\ge \tilde O(k/\min\set{\beta,\beta^2})\) and the running time is exponential in $n$. 

The spiked covariance model exhibit a sharp transition in the top eigenvalue for $n \gtrsim \frac{d}{\beta^2}$ (called \textit{BPP transition} \cite{baik2005} in reference to the authors' names).
In this regime, called\textit{ strong-signal regime}, the following spectral algorithm matches the optimal statistical guarantees of exhaustive search: compute the top right singular vector of $Y$ and restrict it to the \(k\) largest entries \cite{krauthgamer2015semidefinite}.
We refer to this algorithm as \textit{SVD with thresholding}.

Whenever $n \lesssim \frac{d}{\beta^2}$, principal component analysis of \(\set{y_i}_{i\in[n]}\) cannot be used to recover $v_0$.
One of the best known polynomial-time algorithms for this regime (called \textit{weak-signal-regime}) is \emph{diagonal thresholding} \cite{johnstone-lu-dt}: restrict the empirical covariance matrix to the principal submatrix that contains the \(k\) largest diagonal entries and output the top eigenvector of this submatrix.
This algorithm succeeds with high probability whenever $n \gtrsim \frac{k^2}{\beta^2}{\log\frac{d}{k}}$ --- almost quadratically worse than exhaustive search.%
\footnote{We remark that \cite{ding2019subexponentialtime} provides an algorithm that interpolates between Diagonal Thresholding and brute force search.
  Concretely, given any natural number $t\le n/\log d$, the algorithm recover the sparse vector in time $d^{O(t)}$ if $\beta \gtrsim \frac{k}{\sqrt{tn}}\sqrt{\log d}$.
  Whenever our discussion will revolve around polynomial time algorithms, we will simply talk about Diagonal Thresholding.} Similar guarantees were shown to be achievable in polynomial time through a semidefinite relaxation \cite{d2005direct,amini2009} (which we refer to as the \textit{basic SDP}) \iftoggle{shortv}{}{(see \cref{sec:basic-sdp} for a precise formulation)}.
A large and diverse body of work \cite{amini2009, cai2013, DBLP:journals/corr/abs-1304-0828, krauthgamer2015semidefinite, DBLP:journals/corr/abs-1710-05017, ding2019subexponentialtime}  has been dedicated to the question of understanding if this quadratic gap between the sample sizes required for computationally efficient and inefficient methods is inherent or if better polynomial-time algorithms exist for this problem.
Hardness results addressing this question take two forms:
either reductions from conjecturally hard problems, such as \textit{planted clique} \cite{DBLP:journals/corr/abs-1304-0828} or concrete lower bounds against restricted classes of algorithm such as the sum-of-squares \cite{DBLP:conf/nips/MaW15,DBLP:conf/focs/HopkinsKPRSS17} or low-degree polynomials \cite{ding2019subexponentialtime}.

While these results provide evidence that a quadratic gap between polynomial-time algorithms and exhaustive search is inherent in the weak signal regime, it turns out that a logarithmic improvement over diagonal thresholding is possible (for a broad parameter range):
in the regime $k\leq \sqrt{d}/2$, a more sophisticated algorithm called \textit{  covariance thresholding} \cite{krauthgamer2015semidefinite, DBLP:conf/nips/DeshpandeM14} succeeds for $n \gtrsim \max\Set{\frac{k^2}{\beta^2}\log\frac{d}{k^2}, k^2}$. \iftoggle{shortv}{}{See  \cref{lem:spectral-norm-thresholded-gaussian}.}
This turns into an asymptotic improvement over diagonal thresholding in the settings $d^{1-o(1)}\leq k^2 \leq o(d)$, but requires the constraint $n\geq k^2$. For example, if $\beta = 1$ and $k^2 = \eps d$ for some small enough $\varepsilon > 0$, covariance thresholding works with $n \gtrsim k^2\log(1/\eps)$, while SVD requires $n \gtrsim k^2 / \eps$ and diagonal thresholding requires $n \gtrsim k^2 \log d$.

\paragraph{Adversarial entry-wise perturbations}

In a seminal work, Huber \cite{huber1981robust} asked how the guarantees of estimators---designed to work under the assumption of observing Gaussian noise---would change if the data were roughly normal, but not exactly so, thus broadening the circumstances under which the performance of an estimator should be judged. 
This is especially relevant if we consider that in many real world problems, data may be preprocessed, or the precision of an individual input may be limited. For example, digital images may use few bits to encode a pixel and discard all residual information. For these reasons, it is not desirable for an estimator to drastically change its response as the input changes between $Y$ and $Y+E$ for a small perturbation matrix $E$.
In this sense, the robustness of an estimator is an important aspect for understanding its performances in real-world environments \cite{survey-robust-statistics,Fawzi2016RobustnessOC}.

It turns out that the algorithmic landscape for sparse PCA changes drastically in the presence of adversarial perturbations, where an adversary may change each entry of the input vectors $y_1,\ldots,y_n$ by a small amount.
On the one hand, exhaustive search and the basic SDP continue to give the same guarantees as in the vanilla single-spike model.
On the other hand, all aforementioned thresholding algorithms are highly sensitive to small adversarial perturbations.

Concretely, in the strong signal regime \(\beta \lesssim  d/n\), it is possible to adversarially perturb the vectors \(y_1,\ldots,y_n\) by at most \(\tilde O(1/\sqrt {n\,})\) per entry such that SVD with thresholding achieves only vanishing correlation.
Indeed an adversarial perturbation with this effect can be viewed as a whitening transformation and corresponds to a natural generative process for \(y_1,\ldots,y_n\), where the vectors are chosen randomly from an \(n\)-dimensional subspace containing an approximately sparse vector \iftoggle{shortv}{}{(see \cref{sec:lowerbounds})}.
We also show that adversarial perturbations of this magnitude can fool diagonal thresholding and covariance thresholding\iftoggle{shortv}{}{ (see \cref{sec:thresholding-algorithms})}.

\paragraph{Sparse eigenvalues certificates}

It is remarkable to notice the stark contrast that appears when instead adversarial perturbations are used against the basic SDP\footnote{
  We remark that a certain informal notion of robustness to entry-wise perturbations of the basic SDP program was already argued in \cite{d2005direct}.
  Additionally, in \cite{DBLP:journals/corr/abs-1304-0828} the authors observed that the algorithm is robust to small perturbations of the empirical covariance matrix.
  We allow here more general perturbations.
}, indeed it is easy to show that the algorithm succeeds whenever adversarial perturbations are bounded (in absolute value) by $\sqrt{\frac{\beta }{k}}\cdot \min\set{\sqrt{\beta},1}$.
If, for example, we assume $\beta \ge 1$ and consider the regime in which diagonal thresholding works, that is $\beta \geq \tilde{O}(k/\sqrt{n})$, this bound means the algorithm can afford perturbations bounded by $O(1/n^{1/4})$.
This is even more remarkable when one notices that for perturbations larger than $\tilde{O}(1/n^{1/4})$ an adversary could plant a matrix with $k$-sparse norm greater than $\beta n$, thus fooling even the exhaustive search algorithm\iftoggle{shortv}{}{(see \cref{sec:non-robust-algorithms})} (moreover, this adversary can completely remove the signal from $Y$).

Considering these observations, it is only natural to ask  what is the reason that makes some algorithms robust\footnote{In this paper we will interchangeably use the terms robust and resilient.} to corruptions while others turn out to be highly susceptible to small perturbations in the samples. This lead us to  the central questions of this paper:
\begin{quote}
	\itshape Is there some inherent property that makes an algorithm resilient to adversarial perturbations?
\end{quote} 
In the context of Sparse PCA, we answer this question showing how algorithms that come with \textit{certificates} of sparse quadratic forms\footnote{For a matrix $M\in \R^{d\times d}$ we study the values of the quadratic form $\Snorm{Mv}$ at $k$-sparse vectors $v$. We define the $k$-sparse norm of $M$ as $\underset{\substack{\|v\|=1, v \text{ $k$-sparse}}}{\max}\Norm{Mv}$. We sometimes refer to the $k$-sparse unit vector $v$ that maximizes $\Norm{Mv}$ as a sparse eigenvector, and to the corresponding value as a sparse eigenvalue.} are intrinsically better in the sense that small perturbations -- which by virtue of being small cannot significantly change the sparse eigenvalues of the instance -- cannot be used to fool them. In contrast, fragile algorithms -- which do not produce such certificates -- may be fooled by adversarial perturbations into outputting an estimation uncorrelated with the sparse vector $v_0$.

We remark that the insight obtained in this analysis also led us to new improvements in the single spiked covariance model.

\paragraph{Certification and the cost of resilience}  The robustness of semidefinite programs had  already been noted in the literature. For the stochastic block model, efficient spectral algorithms (see \cite{spectral_algorithm_SBM}) are known to recover the partitions up to the (conjectured) computational threshold.\footnote{Called the Kesten-Stigum threshold.} However, few  adversarial edge deletions and additions can fool such estimators. On the other hand, algorithms based on semidefinite programming were shown to be resilient to adversarial perturbations\cite{DBLP:journals/jcss/FeigeK01, SDP_SBM, Moitra-semirandom, DBLP:conf/stoc/MontanariS16, DBLP:conf/colt/MakarychevMV16, raghavendra_local_statistics}, albeit far from the Kesten-Stigum thresold in general settings.\footnote{Another qualitative difference between the semidefinite programs studied in the paper above and other families of algorithms is the resilience to monotonic perturbations (see \cite{DBLP:journals/jcss/FeigeK01, Moitra-semirandom} ).}
The underlying question of this line of work is whether the additional property of resilience  comes "for free".  

In the context of this paper, with the idea of certification mechanisms being a sufficient algorithmic property for adversarial resilience, it becomes relevant to look into the limitations of certification algorithms as well. For the \textit{Sherrington-Kirkpatric} problem \cite{sherrington1975solvable} -- the problem of maximizing the quadratic form $\transpose{x}Wx$ where $x\in \Set{\pm 1/\sqrt{n}}^n$ and $W$ is  a symmetric random matrix with iid Gaussian entries above the diagonal -- \cite{montanari2019optimization} showed (modulo a reasonable conjecture) that for any $\eps >0$ there exists a polynomial-time optimization algorithm returning a value $\eps$-close to the optimum. Conversely, \cite{BandeiraKW20} proved that no low-degree polynomial can obtain an $\eps$-close certificate for the problem. Thus suggesting that certification may be a inherently harder task than optimization. 

For sparse PCA in the strong signal regime, we observe a strikingly steep \textit{statistical price to pay for robustness}, in the form of a lower bound on the guarantees of low-degree polynomials. That is a \textit{fundamental separation between the power of fragile and resilient algorithms}.

\subsection{Results}\label{sec:results}
So far, we have generically said that an algorithm is "robust" if it recovers the planted signal even in the presence of malicious noise. However, several issues arise if one tries to make this vague definition more concrete.  At first,  one could say that robust algorithms achieve comparable guarantees both  in the presence and the absence of adversarial corruptions. Yet, in general, this interpretation makes little sense. Malicious perturbations may remove part of the signal, making the guarantees of the fragile settings statistically impossible to achieve or --as we will see for the  sparse PCA in certain regimes-- they might make the goal of achieving such guarantees \textit{computationally much harder}, thus at the very least forcing us to spend a significantly higher amount of time to obtain the same aforementioned guarantees.

For this reason, in many settings it will make sense to say that an algorithm is resilient if it recovers the sparse signal in the presence of adversarially chosen perturbations \textit{even though} its guarantees may not fair well when compared to those achievable  in the fragile settings. 

The second fundamental aspect concerns the desirable degree of robustness that an algorithm should possess.
Indeed, any reasonable algorithm can likely tolerate sufficiently small adversarial perturbations.
Therefore, it is important to quantify the magnitude of the perturbations we ask algorithms to tolerate. 
Here, we also expect this magnitude to decrease monotonically with the signal strength $\beta$. 
A natural concrete way to formalize this idea is the following:
\textit{the algorithm should be expected to obtain correlation bounded away from zero, as long as $v_0$ remains  the principal sparse component}. That is, as long as the vector maximizing the $k$-sparse norm of $Y$ is correlated with $v_0$, then the algorithm should be able to output an estimator correlated with $v_0$. 

Concretely, these observations lead us to the following problem formulation.

\begin{problem}[Robust sparse PCA]\label{def:wishart-matrix-model}
	Given a matrix of the form
	\begin{align}\label{eq:matrix-wishart-model}
	Y = W + \sqrt{\beta}u_0 \transpose{v_0} + E,\, \text{ where}
	\end{align}
	\begin{itemize}
		\item  $v_0\in \R^d$ is  a unit $k$-sparse vector,
		\item  $u_0\sim N(0,\Id_n)$ is a standard Gaussian vector,
		\item  $W\sim N(0,1)^{n\times d}$  is a Gaussian matrix and  $W,u_0,v_0$ are distributionally independent,
		\item $E\in \R^{n \times d}$ is an arbitrary perturbation matrix satisfying\footnote{	In non-robust settings, we simply enforce the constraint $\Normi{E}=0$.} 
		\begin{align}\label{eq:perturbation-resilient}
		\Normi{E} \lesssim \sqrt{\beta/k}\cdot \min\set{\sqrt{\beta},1}\,.
		\end{align} 
	\end{itemize}
	Return a unit vector $\hat{v}$ having  non-vanishing correlation with $v_0$.

\end{problem}

To get an intuition why bound \cref{eq:perturbation-resilient} is canonical, observe that for $\beta\geq \Omega(1)$ adversarial perturbations of magnitude $\tilde{O}\paren{\sqrt{\beta/k}}$ could remove all information about $v_0$ (see \cref{sec:perturbation-resilience-overview}). With this formalization of the problem we can now unambiguously define robust algorithms. Specifically, we say that an algorithm is $(n,d,k,\beta,\delta, p)$--\textit{perturbation resilient} if, for parameters $(n,d,k,\beta)$, with probability at least $p$ it outputs a unit vector $\hat{v}$ such that $1-\iprod{\hat{v},v_0}^2\le \delta$.

Note that the exhaustive search algorithms described in introduction can also recover $v_0$ in the presence of the adversarial matrix $E$ from \cref{def:wishart-matrix-model}. 
So we can assume that $n > \log d$ since otherwise \cref{def:wishart-matrix-model} can be solved in time $d^{O(1)}$ using exhaustive search if  $n\ge \tilde O(k/\min\set{\beta,\beta^2})$.
	
	To better keep track of the multiple results presented in the section, we provide  three tables summarizing the results of this works, each result is then individually discussed in the paragraphs below.
	
	\begin{table}[H]
		\begin{tabular}{|m{5.2cm}|m{6.5cm}|m{2.1cm}|m{1.3cm}|}
			\hline
			\multicolumn{4}{|c|}{\textbf{Strong Signal Regime}}\\\hline
			\textit{Algorithm} & \textit{Succeeds if} & \textit{Running Time} & \textit{Resilient}\\\hline
			SVD with thresholding & $\beta \gtrsim \sqrt{\frac{d}{n}} + \frac{k\log d}{n}$ & $O\Paren{nd\log n}$ & No \\\hline
			{\color{BrickRed} Sum of squares, \cref{thm:sos-algorithm-robust-spca-small-n}} &  {\color{BrickRed} $\beta \gtrsim \frac{k}{\sqrt{n}}\Paren{\frac{d}{k}}^{1/t}$ for $d\gtrsim n^t\cdot t^t \cdot \log^{t+1} n$}&{\color{BrickRed}  $d^{O(t)}$ }& {\color{BrickRed} Yes} \\\hline
			{\color{BrickRed} Spectral algorithm, \cref{thm:main-fast-spectral}} &  {\color{BrickRed}$\beta \gtrsim \frac{k}{\sqrt{n}}\Paren{\frac{d}{k}}^{1/3}$ for $d\gtrsim n^3\log d \log n$} & {\color{BrickRed}  $O\Big({nd\log n}\Big)$ } & {\color{BrickRed} *\tablefootnote{Resilient to the distribution of \cref{theorem:lower_bound_main}}} \\\hline
		\end{tabular}
		\caption{Algorithmic landscape in the strong signal regime.
			The spectral algorithm is provably resilient to the adversary used to fool SVD with thresholding but we do not expect it to be resilient to arbitrary adversaries.}
	\end{table}
	
	\begin{table}[H]
		\centering
		\begin{tabular}{|m{5.0cm}|m{7.1cm}|m{1.8cm}|m{1.3cm}|}
			\hline
			\multicolumn{4}{|c|}{\textbf{Weak Signal Regime}}\\\hline
			\textit{Algorithm} & \textit{Succeeds if} & \textit{Running Time} & \textit{Resilient}\\\hline
			(Generalized) diagonal thresholding & $\beta \gtrsim \frac{k}{\sqrt{n\cdot t}}\sqrt{\log d}$ for $t\le \frac{1}{\ln d}\min\set{d,n}$ & $n^{O(1)}d^{O(t)}$ & No \\\hline
			Covariance thresholding & $\beta \gtrsim \frac{k}{\sqrt{n}}\sqrt{\log\frac{d}{k^2}}$ for $k\lesssim  \sqrt{d}$ and $k \lesssim \sqrt{n}$ & $n^{O(1)} d^{O(1)}$ & No \\\hline
			{\color{BrickRed} Basic SDP, \cref{thm:weak-signal-regime-sdp}} &  {\color{BrickRed} 
				$\beta \gtrsim \min\Set{\frac{k}{\sqrt{n}}\sqrt{\log\Paren{2 + \frac{d}{k^2} + \frac{d}{n}}}, \frac{d}{n} + \sqrt{\frac{d}{n}}}$ }&{\color{BrickRed}  $n^{O(1)}d^{O(1)}$ }& {\color{BrickRed} Yes} \\\hline
			{\color{BrickRed} Sum of squares, \cref{thm:weak-signal-regime-sos}} &  {\color{BrickRed} $\beta \gtrsim \frac{k}{\sqrt{n\cdot t}}\sqrt{\log d}$ for $t\le \frac{1}{\ln d}\min\set{d,n}$} &{\color{BrickRed}  $n^{O(1)}d^{O(t)}$ }& {\color{BrickRed} Yes} \\\hline
			{\color{BrickRed} Low-degree polynomials, \cref{thm:results-weak-signal-regime}} &  {\color{BrickRed} $\beta \gtrsim \frac{k}{\sqrt{n}}\sqrt{\log \frac{d}{k^2} +\frac{\log d}{\log n}}$ for  $d^{1-o(1)}\lesssim k^2\lesssim d$ and $n\gtrsim \log^5 d$}&{\color{BrickRed}  $n^{O(1)}d^{O(1)}$ }& {\color{BrickRed} No} \\\hline
		\end{tabular}
		\caption{Algorithmic landscape in the weak signal regime.}
	\end{table}
	
	\begin{table}[H]
		\centering
		\begin{tabular}{|m{2.5cm}|m{2.1cm}|m{7.7cm}|m{2.5cm}|}
			\hline
			\multicolumn{4}{|c|}{\textbf{Computational Lower Bounds for Polynomials}}\\\hline
			\textit{Settings} & \textit{Work} & \textit{Polynomials of degree $D$ cannot distinguish if} & \textit{Up to degree}\\\hline
			Fragile & \cite{ding2019subexponentialtime} & 
			$\beta \lesssim \Set{\sqrt{\frac{d}{n}}, \frac{k}{\sqrt{Dn}}}$ & $D\leq o(n)$ \\\hline
			{\color{BrickRed} Fragile }& \cref{thm:fragile-lower-bound-informal}  &{\color{BrickRed}  $\beta \lesssim \Set{\sqrt{\frac{d}{n}}, \frac{k\log \Paren{2+\frac{Dd}{k^2}}}{\sqrt{Dn}}  }$ }& {\color{BrickRed} $D\leq \frac{n}{\log^2 n}$} \\\hline
			{\color{BrickRed} Resilient} & \cref{thm:lowerbound-robust-spca-informal} & {\color{BrickRed}$\beta \le O\Paren{\frac{k}{\sqrt{n}}\Paren{\frac{d}{k}}^{1/t}}$ for $\beta n/k\leq n^{0.49}$ and $d \le n^{0.99t-1}$} & {\color{BrickRed} $D\leq n^{0.001}$} \\\hline
		\end{tabular}
		\caption{Computational landscape for low-degree polynomials.}
	\end{table}

\paragraph{Resilient algorithms in the strong signal regime}
With the above discussion in mind, one may ask whether the same guarantees known for the single spike covariance model may also be achieved in the presence of adversarial perturbations.
In the strong signal regime $\beta \gtrsim \sqrt{d/n}$, this amounts to finding a robust and efficient algorithm that achieves the same guarantees as SVD with thresholding.
As we will see however, this is most likely impossible.
That is, we will provide compelling evidence that \textit{resilient algorithms cannot match the guarantees of fragile algorithms in the strong signal regime.}

Since for $\sqrt{d/n}\lesssim\beta \lesssim d/n$ adversarial perturbations of the order $\tilde{O}(1/\sqrt{n})$ can change the top eigenvalue of the covariance matrix, PCA arguments cannot be used to obtain resilient algorithms.
Thus intuitively, this suggests that different kinds of certificates are needed. 

We provide a Sum-of-Squares algorithm that recovers in time $d^{O(t)}$ the sparse vector whenever $n \gtrsim \frac{k}{\beta}\cdot t\Paren{\frac{d}{k}}^{1/t}$ and $d^{1/t}\geq \tilde{\Omega}\Paren{n}$.  
The key contribution is indeed an efficient algorithm to certify upper bounds on random quadratic forms. For subgaussian\footnote{Formally we require a stronger property, we need matrices to be \textit{certifiably subgaussian.}} low-rank quadratic forms, these upper bounds approach information-theoretically optimal bounds. 

Concretely, for an $n$-by-$d$ matrix $W$ with i.i.d. Gaussian entries, with high probability the degree-$t$ sum-of-squares algorithm (with running time $d^{O(t)}$) certifies an upper bound of $O(k \cdot (k/d)^{-1/t} \cdot t)$ on the quadratic form $Q(x)=\|W x\|^2$ over all $k$-sparse unit vectors $x$ if $d^{1/t}\geq \tilde{\Omega}\Paren{n}$.
With these certificates, a robust algorithm for Sparse PCA follows then as a specific corollary.

It is important to notice how this result for sparse PCA is interesting regardless of its resilience properties. As $t$ approaches $\log(d/k)$, the algorithm approaches the information theoretic optimal bound $O(\frac{k}{\beta} \cdot \log(d/k))$. For example, consider the case $n = 2^{\Theta\Paren{\sqrt{\log d}}}$. If also $\frac{d}{k} = 2^{\Theta\Paren{\sqrt{\log d}}}$, the Sum of Squares algorithm works in time $d^{O\Paren{\sqrt{\log d}}} = n^{O\Paren{\log^2 n}}$ with information theoretically optimal guarantees, while exhaustive search takes time exponential in $n$. 

The specific algorithmic result is shown in the following theorem.

\begin{theorem}[Perturbation Resilient Algorithm in the Strong Signal Regime]
	\label{thm:sos-algorithm-robust-spca-small-n}
	Given an  $n$-by-$d$ matrix $Y$  of the form,
	\begin{align*}
	Y =\sqrt{\beta}\cdot u_0\transpose{v_0}+ W + E\,,
	\end{align*} 
	for $\beta >0$, a unit $k$-sparse vector $v_0\in \R^d$, a Gaussian matrix $W\sim N(0,1)^{n\times d}$, a vector $u_0\in \R^n$ independent of $W$ with $\norm{u_0}^2 = \Theta(n)$, and a matrix $E\in \R^{n\times d}$ satisfying $\Normi{E} \lesssim \sqrt{\beta/k}\cdot \min\set{\sqrt{\beta},1}$.
	
	For $t \in \N$ suppose that $d \gtrsim  n^t \log^{t+1}{(n)} t^t$  and
	\[
	\beta \gtrsim \frac{k}{n}\cdot t\cdot \Paren{\tfrac d k}^{1/t}\,.
	\]
	Then, there exists an algorithm
	that computes in time $d^{O(t)}$ a unit vector $\hat{v}\in\R^d$ such that
	\begin{align*}
	1-\iprod{\hat{v},v_0}^2\le 0.01
	\end{align*}
	with probability at least $0.99$.
\end{theorem}

In any case, the fundamental limitation of the above algorithm is the requirement $d^{1/t}\geq \tilde{\Omega}\Paren{n}$. This constraint makes it impossible to match the guarantees of SVD+ thresholding in most regimes, but a priori it remains unclear why better robust algorithms could not be designed. To provide formal evidence that without the requirement $d^{1/t}\geq \tilde{\Omega}\Paren{n}$ achieving the kind of guarantees of \cref{thm:sos-algorithm-robust-spca-small-n} may be computationally intractable, we make use of a remarkably simple method (sometimes called analysis of the \textit{low degree likelihood ratio}), developed in a recent line of work on the the sum of squares hierarchy \cite{DBLP:conf/focs/BarakHKKMP16, DBLP:conf/focs/HopkinsS17, DBLP:conf/focs/HopkinsKPRSS17, hopkins2018statistical}. That is, we show that in the restricted computational model of low-degree polynomials\iftoggle{shortv}{}{\footnote{As we will argue in Section \ref{sec:low-degree-likelihood-ratio}, being indistinguishable with respect to low degree polynomials is an important indication of computational hardness.}}, there is no efficient algorithm that can improve over the Sum-of-Squares algorithm. This hardness results suggests a \textit{fundamental separation between fragile and resilient algorithms, in other words, an inherent cost to pay in exchange for perturbation-resilience.}

Concretely, we construct $n\times d$ matrices  of the form $Y=\sqrt{\beta} u_0 \transpose{v_0 }+ W +E$ where $E$ is a perturbation matrix with entries bounded by $\tilde{O}\Paren{1/\sqrt{n}}$ such that, whenever $d$ is significantly smaller than $n^t$, multilinear polynomials of degree at most $n^{0.001}$ cannot distinguish these $Y$'s from $n\times d$ Gaussian matrices (in the sense that w.h.p. every such polynomial takes roughly the same values under both distributions). These ideas are formalized in the theorem below.

\begin{theorem}[Lower Bound for Resilient Algorithms in the Strong Signal Regime, Informal]
	\label{thm:lowerbound-robust-spca-informal}
	Let $t$ be a constant and let $d \le n^{0.99 t - 1}$. Suppose that
	\[
	\beta \le O\Paren{\frac{k}{n}\cdot t\cdot (d/k)^{1/t}}\,.
	\] 
	and\footnote{This  constraint is used to ensure that inequalities of the form $\beta\gtrsim \frac{k}{\sqrt{n \cdot D}}$ for any $D\le n^{0.001}$ are never satisfied. Informally speaking, we restrict our statement to the settings where algorithms with guarantees similar to diagonal thresholding do not work.} 
	$\beta n/k\le n^{0.49}$.
	Then, there exists a distribution $\mu$ over $n\times d$ matrices $Y$ of the form $Y= \sqrt{\beta}u_0 v_0^T+ W + E$ where $\Normi{E}\leq\tilde{O}\Paren{1/\sqrt{n}}$, with the following properties:
	\begin{itemize}
		\item $\mu$ is indistinguishable from the Gaussian distribution $N(0,1)^{d\times n}$ with respect to all multilinear polynomials of degree at most $n^{0.001}$ \iftoggle{shortv}{}{in the sense described in Section \ref{sec:baground-low-degree-method}},
		\item the jointly-distributed random variables $W$, $u_0$, $v_0$ are independent,
		\item the marginal distribution of $v_0$ is supported on  unit vectors with entries in $\Set{-1/\sqrt{k},0,1/\sqrt{k}}$,
		\item the marginal distribution of $u_0$ is uniform over $\set{-1,1}^n$,
		\item the marginal distribution of $W$ is $N(0,1)^{n\times d}$.
	\end{itemize}
\end{theorem}


Informally speaking, \cref{thm:lowerbound-robust-spca-informal} conveys the following message. Any resilient algorithm for Sparse PCA can also distinguish the distribution $\mu$ over $n$-by-$d$ matrices $Y$ from the Gaussian distribution $N(0,1)^{n\times d}$. Therefore, if an estimator returned by this algorithm can be approximated by low-degree polynomials, then this algorithm cannot certify upper bounds of sparse eigenvalues of Gaussian matrices that are sharp enough to significantly improve the guarantees of \cref{thm:sos-algorithm-robust-spca-small-n}.

Sparse principal component analysis is intimately related to the problem of learning Gaussian mixtures.
Indeed, for a vector $v_0$ with entries in $\set{\pm 1/\sqrt{k},0}$, sparse PCA can be rephrased as the problem of learning a non-uniform mixture $M$ of three subgaussian distributions, one centered at zero, one centered at $\sqrt{\beta/k}\cdot u_0$ and the last at $-\sqrt{\beta/k}\cdot u_0$.  As we will see, this is true even for the distribution $\mu$ used in \cref{thm:lowerbound-robust-spca-informal} Thus, from this perspective
the result also provides interesting insight on the  complexity of this problem. The theorem  suggests that to distinguish between $M$ and a standard Gaussian $W\sim N(0,1)^{n\times d}$, an algorithm would either need $d\gtrsim n^t$ samples or should not be computable by polynomials of degree at most $n^{0.001}$ \iftoggle{shortv}{}{(see \cref{sec:relationship-gmm})}.

\paragraph{Resilient algorithms in the weak signal regime}
Having cleared the picture for efficient algorithms in the strong signal regime, we may focus our attention to the weak signal settings $\beta\lesssim \sqrt{d/n}$.
Surprisingly, in these settings adversarial perturbations do not change the computational landscape of the problem.
As a matter of fact, a robust algorithm was already known.
In fragile settings, the basic SDP program was proved (e.g. see \cite{berthet2013}) to have the same guarantees as diagonal thresholding.
But as the algorithm can certify the upper  bounds $\Snorm{Mx}\leq k\cdot \Normi{M}^2$ and $\Snorm{Wx}\leq n+Ck\sqrt{n \log d}$ over $k$-sparse unit vectors $x\in \R^d$ and matrices $M\in \R^{n\times d}$, $W\sim N(0,1)^{n\times d}$ (where $C > 0$ is some absolute constant), it is therefore resilient to adversarial corruptions.
We improve this latter upper bound showing that the algorithm can also certify the inequality $\Snorm{Wx}\leq n + Ck\sqrt{n\log(d/\min\set{k^2,n})}$, thus matching the guarantees of covariance thresholding and leading us to the following result.

\begin{theorem}[Perturbation Resilient Algorithm in the Weak Signal Regime]\label{thm:weak-signal-regime-sdp}
	Given an  $n$-by-$d$ matrix $Y$  of the form,
	\begin{align*}
	Y =\sqrt{\beta}\cdot u_0\transpose{v_0}+ W + E\,,
	\end{align*} 
	for $\beta > 0$, a unit $k$-sparse vector $v_0\in \R^d$, a Gaussian matrix $W\sim N(0,1)^{n\times d}$, a vector $u_0\in \R^n$ independent of $W$ with $\norm{u_0}^2 = \Theta(n)$, and a matrix $E\in \R^{n\times d}$ satisfying $\Normi{E} \lesssim \sqrt{\beta/k}\cdot \min\set{\sqrt{\beta},1}$. 
	
	Suppose that
	\[
	\beta \gtrsim \min\Set{\frac{k}{\sqrt{n}}\sqrt{\log\Paren{2 + \frac{d}{k^2} + \frac{d}{n}}}, \frac{d}{n} + \sqrt{\frac{d}{n}}}\,.
	\]
	
	Then, there exists an algorithm that uses the basic SDP program for sparse PCA, 
	and computes in polynomial time a unit vector $\hat{v}\in\R^d$ such that
	\begin{align*}
	1-\iprod{\hat{v},v_0}^2\le 0.01
	\end{align*}
	with probability at least $0.99$.
\end{theorem}

\cref{thm:weak-signal-regime-sdp} says that among polynomial time algorithms, in  the weak signal regime or whenever $\beta < 1$, the basic SDP achieves the best known guarantees. Furthermore, in contrast to thresholding and PCA algorithms, it works even in the presence of adversarial corruptions.

\paragraph{High degree certificates in the weak signal regime} A consequential observation of the previous paragraphs is that, perhaps, the  Sum-of-Squares algorithm of larger degree can improve over the guarantees of the basic SDP even in the weak signal regime.
Indeed in many settings, these guarantees can be improved observing that the (degree $t$) Sum-of-Squares algorithm can certify upper bounds of the form $\Snorm{Wx}\leq n + k\sqrt{(n/t)\log d}$ in time $d^{O(t)}$.
Hence offering a smooth trade-off between sample complexity and running time.

\begin{theorem}[Perturbation Resilient Algorithm via Limited Exhaustive Search]\label{thm:weak-signal-regime-sos}
	Given an  $n$-by-$d$ matrix $Y$ of the form,
	\begin{align*}
	Y =\sqrt{\beta}\cdot u_0\transpose{v_0}+ W + E\,,
	\end{align*} 
	for $\beta >0$, a unit $k$-sparse vector $v_0\in \R^d$, a Gaussian matrix $W\sim N(0,1)^{n\times d}$, a vector $u_0\in \R^n$ independent of $W$ with $\norm{u_0}^2 = \Theta(n)$  and a matrix $E\in \R^{n\times d}$ satisfying $\Normi{E} \lesssim \sqrt{\beta/k}\cdot \min\set{\sqrt{\beta},1}$. 
	
	Suppose that for some positive integer $t\le \frac{1}{\ln d}\min\set{d,n}$,
	\[
	\beta \gtrsim \frac{k}{\sqrt{nt}}\sqrt{\log d}\,.
	\]
	
	Then, there exists an algorithm 
	that computes in time $n^{O(1)}d^{O(t)}$ a unit vector $\hat{v}\in\R^d$ such that
	\begin{align*}
	1-\iprod{\hat{v},v_0}^2\le 0.01
	\end{align*}
	with probability $0.99$.
\end{theorem}

Whenever $k^2\leq d^{1-\Omega(1)}$, \cref{thm:weak-signal-regime-sos} provides better guarantees than \cref{thm:weak-signal-regime-sdp} (with worse running time).

It is also interesting to compare this result with the bound of \cref{thm:sos-algorithm-robust-spca-small-n}. For some $t$, we can determine the parameter regimes when one theorem provides better guarantees then the  other for running time $d^{O(t)}$.
Assume that
$(t+1)^{t+1} n^{t+1} \Paren{\log n}^{t+2} \gtrsim d \gtrsim t^t n^t \Paren{\log n}^{t+1}$. 
Then there exist constants $0 < C < C' $ such that:
\begin{itemize}
	\item If $k^2 \le d\cdot \Paren{Ct}^t$, we get $t\cdot\Paren{\frac{d}{k}}^{1/t} > \sqrt{\frac{n}{t}\log d}$, so in this case the guarantees in \cref{thm:weak-signal-regime-sos} are better.
	\item  If $k^2 \ge d\cdot \Paren{n\log^2 n}^2\cdot \Paren{C't}^t$, we get 
	$t\cdot\Paren{\frac{d}{k}}^{1/t} < \sqrt{\frac{n}{t}\log d}$, so in this case the guarantees in \cref{thm:sos-algorithm-robust-spca-small-n} are better.
\end{itemize}
Informally speaking, these conditions show that the guarantees in \cref{thm:sos-algorithm-robust-spca-small-n} are better when the vector is only mildly sparse: $k^2 \gg d$, and the number of samples $n$ is very small.

\cref{thm:weak-signal-regime-sos}, along with \cref{thm:weak-signal-regime-sdp} and \cref{thm:sos-algorithm-robust-spca-small-n} provides also a nice consequence, namely it allows us to state that f\textit{or the problem of Sparse PCA, the Sum-of-Squares algorithm achieves the best known guarantees among perturbation resilient polynomial time algorithms. Furthermore, under the restrict computational model of low-degree polynomials, these guarantees are nearly optimal.}

\subsubsection{Sharp bounds for the Wishart model} 
In the regime where $k\leq \sqrt{d}$, covariance thresholding succeeds for $\beta \gtrsim \frac{k}{\sqrt{n}}\sqrt{\log\frac{d}{k^2}}$. 
This turns into an asymptotic improvement over diagonal thresholding in the settings $d^{1-o(1)}\leq k^2 \leq o(d)$ but requires a constraint on the sample complexity of the form $n\geq k^2$, for which there is no evidence in the known lower bounds.
This picture raises the following questions: can we obtain guarantees of the form $\beta\gtrsim \frac{k}{\sqrt{n}}\sqrt{\log\frac{d}{k^2}}$   even for $n \leq k^2$? And furthermore, can we improve over this logarithmic factor?

Studying low-degree polynomials we improve over this incomplete picture providing a new algorithm which succeed in recovering the sparse vector in polynomial time whenever $\beta\gtrsim \frac{k}{\sqrt{n}}\sqrt{\log\frac{d}{k^2}}$ and $n\gtrsim d^{1/\log\Paren{\frac{d}{k^2}}}+\log^5d$. Thus obtaining an asymptotic improvement over diagonal thresholding in a significantly large set of parameters. 

Concretely, the algorithm improves over the state-of-the-art whenever $d^{1/\log\frac{d}{k^2}}+\log^5d\lesssim n \lesssim d^{1-\Omega\Paren{1}}$.
In other words, the algorithm requires much fewer samples than covariance thresholding. 
This result is captured by the theorem below.

\begin{theorem}[Polynomials based Algorithm for the Strong Signal Regime]\label{thm:results-weak-signal-regime}
	Given an  $n$-by-$d$ matrix $Y$  of the form,
	\begin{displaymath}
	Y=\sqrt \beta \cdot u_0v_0^T +W\,,
	\end{displaymath}
	for a  unit vector $v_0\in\R^d$ with entries in $\set{\pm 1/\sqrt{k},0}$, a vector $u_0$ with \textit{i.i.d.} entries satisfying $\E{u_i}=0$, $\E {u_i}^2=1$, $\E u_i^4 \le O(1)$  and a matrix $W\in \R^{n \times d}$ with \textit{i.i.d.} entries satisfying $\E\Brac{W_\ij}=0$, $\E\Brac{W^2_\ij}=1$, such that $W$ and $u_0$ are independent; suppose that $n \gtrsim\log^5 d$, $d^{1-o(1)}\le k^2\le d/2$, and
	\[
	\beta \gtrsim \frac{k}{\sqrt{n}}\sqrt{\log\Paren{\frac{d}{k^2}} + \frac{\log{d}}{\log n}}\,.
	\]
	Then,  there exists a probabilistic algorithm
	that computes in polynomial time a unit vector $\hat v\in\R^d$ such that
	\[
	1-\iprod{\hat{v},v_0}^2\le 0.01
	\]
	with probability at least $0.99$.
\end{theorem}

Along with \cref{thm:results-weak-signal-regime}, we provide a fine-grained lower bound  that in many settings matches the known algorithmic guarantees for the single spiked model.  Some relevant lower bounds were already known. In \cite{DBLP:journals/corr/abs-1304-0828} the authors used a reduction to the planted clique problem to provide evidence that in the weak signal regime\footnote{Actually the parameter regime they considered is a proper subset of the weak signal regime.} efficient algorithms cannot recover the sparse vector if $\beta \ll \frac{k}{\sqrt{n}}$. In
\cite{ding2019subexponentialtime} similar lower bound was obtained: in the weak signal regime low-degree polynomials cannot succeed if $\beta\lesssim  \frac{k}{\sqrt{n}}$.
This lower bounds fall short of matching the guarantees of diagonal thresholding by a logarithmic factor. Here, we show that whenever $k^2\le d^{1-\Omega\Paren{1}}$, polynomials of degree $O(\log d)$ cannot recover the sparse vector for $\beta \lesssim \min \Set{\sqrt{\frac{d}{n}},\frac{k}{\sqrt{n}}\sqrt{\log d}}$. In particular, we provide strong evidence that in the weak signal regime, in the settings where our polynomials based algorithm does not improve over the state-of-the-art, the known efficient algorithms (diagonal thresholding, basic SDP) are optimal up to constant factors. 

\begin{theorem}[Lower Bound for Standard Sparse PCA, Informal]\label{thm:fragile-lower-bound-informal}
	There exists a distribution $\mu_k$ over $k$-sparse $d$-dimensional unit vectors such that
	if $Y$ is an $n$-by-$d$ matrix of the form
	\begin{align*}
	Y =\sqrt{\beta}\cdot u_0\transpose{v_0}+ W\,,
	\end{align*} 
	for a vector $v_0$ sampled from $\mu_k$, a Gaussian matrix $W\sim N(0,1)^{n\times d}$ and a Gaussian vector $u_0\sim N(0,\Id_n)$ such that $v_0,u_0,W$ are distributionally independent, 
	then the distribution of $Y$ is indistinguishable from the Gaussian distribution $N(0,1)^{n\times d}$ with respect to all polynomials of degree $D\leq n/{\log^2 n}$ \iftoggle{shortv}{}{in the sense described in Section \ref{sec:baground-low-degree-method}}, whenever
	\begin{align*}
	\beta \lesssim \min\Set{\sqrt{\frac{d}{n}},\frac{k}{\sqrt{Dn}}{\log\Paren{2 + \frac{D\cdot d}{k^2}}}}.
	\end{align*}
\end{theorem}

\subsubsection{Additional Results: Practical Algorithms and Experiments}

From a practical perspective, the main issue with the results of \cref{thm:sos-algorithm-robust-spca-small-n} is the reliance on solving large semidefinite programs, something that is often computationally too expensive to do in practice for the large-scale problems that arise in machine learning. In the same fashion of \cite{DBLP:conf/stoc/HopkinsSSS16}, from the insight of the SoS analysis we develop a \textit{fast} spectral algorithm (which we will call \textit{SVD-t}) with guarantees matching  \cref{thm:sos-algorithm-robust-spca-small-n} for degree $ t\leq 3$ for some interesting family of adversaries.  Our algorithm runs in time $O\Paren{n d \log n}$, which for high dimensional settings, can be considerably faster than algorithms that rely on computing the covariance matrix\footnote{While computing $\transpose{Y}Y$ is necessary for covariance thresholding, the diagonal thresholding algorithm can run in time $O(k^2+dn)$.}. Furthermore, while not showing robustness of the algorithm (indeed the algorithm cannot certify upper bounds), we prove that SVD-t succeeds under the adversarial perturbations which are enough to prove \cref{thm:lowerbound-robust-spca-informal}.
Such adversarial settings are especially interesting since the problem has a nice geometric description in which the objective is to recover an approximately sparse vector planted in a random subspace. \iftoggle{shortv}{}{(see \cref{sec:lowerbounds-subspace})} We remark that it is not known how to generalize the algorithm for larger $t$. Finally, we complement this result with experiments on synthetic data which highlights how in many practical settings the algorithm outperforms (\textit{and outruns}) diagonal thresholding. The following theorem presents the guarantees of the algorithm in the spiked covariance model.

\begin{theorem}[Fast Spectral Algorithm for the Strong Signal Regime, Informal]\label{thm:main-fast-spectral}
	Given an  $n$-by-$d$ matrix $Y$  of the form,
	\begin{align*}
	Y=\sqrt{\beta}u_0\transpose{v_0} + W + E\,,
	\end{align*}
	for $\beta >0$,  a unit $k$-sparse vector $v_0\in \R^d$, a Gaussian matrix $W\sim N(0,1)^{n\times d}$, a Gaussian vector $u_0\sim N(0,\Id_n)$ such that $v_0,u_0,W$ are distributionally independent, and $E\in \R^{n\times d}$ is a matrix from \cref{thm:lowerbound-robust-spca-informal} for $t=3$.\footnote{More precisely, to prove  \cref{thm:lowerbound-robust-spca-informal} we consider a specific distribution over matrices $E$ (this distribution depends on $v_0,u_0$ and $W$), and here we mean that $E$ is sampled from this distribution.}
	
	Suppose that $d\gtrsim n^3\log d\log n$, $k\gtrsim n\log n$ and
	\[
	\beta \gtrsim \frac{k}{\sqrt{n}}\Paren{\frac{d}{k}}^{1/3}\,.
	\]
	Then there exits an algorithm that computes in time $O(nd\log n)$ a unit vector $\hat{v}\in \R^d$ such that
	\begin{align*}
	1-\iprod{v_0,\hat{v}}\leq 0.01
	\end{align*}
	with probability at least $0.99$.
\end{theorem}

\iftoggle{shortv}{
	We conclude our introduction with some notation.
}{
	\subsection*{Outline and Notation}
	We conclude our introduction with  an outline of the structure of the paper and some notation.
}

\iftoggle{shortv}{}{
	In \cref{sec:techniques} we give an overview of the techniques and the ideas required to obtain the results. Preliminary information are presented in \cref{sec:preliminaries}. \cref{sec:basic-sdp} and \cref{sec:sos} contains the results for the basic SDP and the Sum-of-Squares algorithms. In \cref{sec:lowerbounds} we show our lower bounds on polynomials and in \cref{sec:polynomials} we use such polynomials to prove \cref{thm:results-weak-signal-regime}. 
	Finally, in \cref{sec:detection-vs-estimation} and \cref{sec:experiments} we describe our fast algorithms, formally prove some of their properties and compare them with known algorithms through experiments.
	
	Additionally, we discuss the relationship with the problem of clustering mixture of Gaussians in \cref{sec:relationship-gmm}. In \cref{sec:wigner-model} we describe the picture for the Wigner model.
	\cref{sec:thresholding-algorithms} contains formal proofs that thresholding algorithms are not robust and \cref{sec:covariance-thresholding-fails-small-n} shows why Covariance Thresholding fails for small sample size. We also
	provide an information theoretic bound in Appendix \ref{section:estimation-lower-bound}.
}

\paragraph{Notation} 
We say that a unit vector $v\in \R^d$ is \textit{flat} if its entries are in $\Set{\pm\frac{1}{\sqrt{t}},0}$ for some $t$. For a matrix $M\in \R^{n\times d}$, we will denote its entry $\ij$ with $M_\ij$. Depending on the context we may refer to  the $i$-th row or the $i$-th column of $M$ with $M_i$ or $m_i$, we will specify it each time to avoid ambiguity. We call $\Norm{M}_1=\underset{i,j\in [d]}{\sum}\Abs{M_\ij}$ the "absolute norm" of $M$. For a Gaussian matrix $W\sim N(0,1)^{n\times d}$, we denote with $w_1,\ldots,w_d$ its columns. For a vector $v\in \R^n$, we denote its $j$-th entry as $v_j$. We hide absolute constant multiplicative factors using the standard notations \(\lesssim\), \(O(\cdot)\), $\Omega(\cdot)$ and $\Theta(\cdot)$, we hide multiplicative factors logarithmic in \(d\) using the notation \(\tilde O(\cdot)\). For a set $S\subseteq [d]\times [d]$, and a  matrix $M\in \R^{d\times d}$, we denote by $M\Brac{S}$ the matrix with entries $M\Brac{S}_{ij}=M_\ij$ if $(i,j)\in S$, and $M\Brac{S}_{ij}=0$ otherwise. For a matrix $M \in \R^{d\times d}$ and $\tau \in \R$, we define $\eta_\tau\Paren{M}\in \R^{d\times d}$ to be the matrix with entries 
\begin{align*}
\eta_\tau\Paren{M}_\ij =\begin{cases}
M_\ij &\text{ if $\Abs{M_\ij}\geq \tau$}\\
0 & \text{ otherwise.}
\end{cases}
\end{align*}
\iftoggle{shortv}{}{
	Furthermore, we define $\zeta_\tau\Paren{M}\in \R^{d\times d}$ to be the matrix with entries
	\begin{align*}
	\zeta_\tau\Paren{M}_\ij =\begin{cases}
	M_\ij -\sign\Paren{M_\ij}\cdot \tau &\text{ if $\Abs{M_\ij}\geq \tau$}\\
	0 & \text{ otherwise.}
	\end{cases}
	\end{align*}
	Additional notation will be introduced when needed.
}
\begin{remark}[Strong and weak signal regimes in robust settings]
	The attentive reader may have noticed how the notions of strong and weak signal regime should differ in the robust settings. Indeed there is no easy algorithm that looks at the spectrum of $Y$ and begins to work as $\beta$ approaches $\sqrt{\frac{d}{n}}$. In this sense, in the presence of an adversary the bound $\beta \lesssim \sqrt{\frac{d}{n}}$ looses significance. However we will continue using these terms to orientate ourselves and implicitly describe which are the desirable guarantees an algorithm should possess in a given regime. For this reason, when talking about weak-signal regime, our discussion will implicitly revolve around settings in which $ \beta \gtrsim \frac{k}{\sqrt{n}}\sqrt{\log \frac{d}{k^2}}$.
\end{remark}

\section{Techniques}\label{sec:techniques}
\subsection{Perturbation-resilience from Sparse Eigenvalue Certificates}\label{sec:perturbation-resilience-overview}
Here we outline the structure of our Sum-of-Squares algorithm and the basic SDP. \iftoggle{shortv}{}{We assume the reader to be familiar with the knowledge in \cref{sec:preliminaries}.}
\paragraph{How robust should an algorithm be?}  In light of our discussion in \cref{sec:results}, we would like efficient algorithms to be as resilient as exhaustive search. In order for such brute-force algorithm to recover the sparse vector $v_0$, there must be no other sparse vector $x$ far from $v_0$ such that $\Norm{Yx}\approx\Norm{Yv}$. This also means that the adversary should not be able to plant a  $k$-sparse vector $z$ far from $v_0$ such that $\Norm{Yx}\gtrsim \Norm{Yv_0}$. To see what bound to enforce on the adversarial matrix, first observe that if $E$ were the zero matrix then 

\begin{align*}
\Norm{Yv_0}=\Norm{Wv_0+ \sqrt{\beta}u_0}\gtrsim \sqrt{n+\beta n}.
\end{align*}
Now consider the following adversarial matrix, let $x$ be a $k$-sparse unit vector with entries in $\Set{0,\pm 1/\sqrt{k}}$ and such that the intersection between $\supp\Set{x}$ and $\supp\Set{v_0}$ is the empty set. With high probability $\Norm{Wx}\approx\sqrt{n}$. 
So let $z=\frac{1}{\Norm{Wx}}Wx$ and define $E$ as the matrix with entries 
$E_\ij = b\cdot z_i \cdot \sign (x_j)$, 
where  $b>0$ is some parameter that we will choose later. 
Then
\begin{align*}
\Norm{Yx}&=\Norm{\Paren{W+E}x} =\Norm{\Paren{\Norm{Wx}+b\sqrt{k}}z}
\approx \sqrt{n}+b\sqrt{k}.
\end{align*} Consequently, $\Norm{Yx}\ge \Norm{Yv_0}$ whenever $\sqrt{n}+b\sqrt{k}\ge \sqrt{n+\beta n}$. The inequality is true for 
$b\gtrsim \sqrt{\frac{\beta n}{k}}\cdot \min\Set{\sqrt{\beta},1}$. 
In other words, the perturbation matrix must satisfy the bound:
\begin{align*}\label{eq:bound-wishart-infinity}
\Normi{E}\le 
\tilde{\Omega}\Paren{\sqrt{\frac{\beta}{k}}\cdot \min\Set{\sqrt{\beta},1} }. \tag{Bound-1}
\end{align*}
For a set of parameters $d,n,k,\beta$, we call an algorithm \textit{perturbation resilient} if it can successfully recover the sparse vector for any adversarial perturbation satisfying bound \ref{eq:bound-wishart-infinity}.

\iftoggle{shortv}{}{
\begin{remark}
	In the proofs presented in the paper, we will measure adversarial corruptions with the norm $\Norm{E}_{1\rightarrow 2}$, which denotes the largest norm of a column of $E$. Clearly this choice allows for a larger class of adversaries. There are a two main reasons behind our choice. The first one being that the adversarial matrices we consider are more naturally described using such norm. Furthermore, this norm has  a direct correspondence with the infinity norm of the adversarial perturbation in the covariance matrix. Indeed, $\Normi{\transpose{E}E} =\Snorm{E}_{1\rightarrow 2}$. This also will allow us to draw a better comparison between the Wishart and the Wigner model. We remark that the reasoning above can be used as well to show the bound: 
	\begin{align*}\label{eq:bound-wishart}
	\Norm{E}_{1\rightarrow 2}\lesssim \min\Set{\sqrt{\frac{\beta n}{k}}\cdot \min\Set{\sqrt{\beta},1}}. \tag{Bound-2}
	\end{align*}
\end{remark}
}

\subsection{Algorithms that Certify Sparse Eigenvalues}
For simplicity of the discussion we illustrate the idea of sparse eigenvaluex certificates for the Wigner model: $Y = \gamma v_0\transpose{v_0}  + W + E$, where $\gamma > 0$, $v_0\in \R^d$ is a $k$-sparse unit vector, $W \sim N(0,1)^{d\times d}$ and $E$ is some matrix with small entries. Denote the set of $k$-sparse unit vectors by $S_k$. The starting idea is to turn the following intuition into an identifiability proof and then a Sum of Squares program: if $\hat{v}$ is a $k$-sparse unit vector which maximizes $\transpose{v}Yv$ over $S_k$  and $\gamma$ is large enough, then with high probability $\iprod{\hat{v},v_0}^2\geq 0.99$. 

Concretely, observe that
\begin{align*}
	&\text{on one side }\qquad \transpose{v_0}Yv_0 = 
	\gamma + \transpose{v_0}Wv_0 + \transpose{v_0}Ev_0,\\
	&\text{on the other }\qquad \transpose{\hat{v}}Y\hat{v} = 
	\gamma\iprod{\hat{v},v_0}^2 + \transpose{\hat{v}}W\hat{v} +\transpose{\hat{v}}E\hat{v}.
\end{align*}
Combining the two and rearranging we obtain the inequality
\begin{align*}
\iprod{\hat{v},v_0}^2\geq 
1-\frac{1}{\gamma}O\Paren{ 
	\max_{v\in S_k}\transpose{v}Wv +  
	\max_{v\in S_k}\transpose{v}Ev}.
\end{align*}
Now, this is where certified upper bounds come in to the picture.
There is an easy certificate (capture by SoS and the basic SDP) of the fact that for any matrix $M$, $\max_{v\in S_k}\transpose{v}Mv\leq \Normi{M}k$
Using such bound we get
\begin{equation}\label{eq:sos-correlation-no-certificate}
\iprod{\hat{v},v_0}^2
\geq 
1-\frac{1}{\gamma}O\Paren{ 
	\max_{v\in S_k}\transpose{v}Wv +  
	 k\Normi{E}} .
\end{equation}
\cref{eq:sos-correlation-no-certificate} already shows how an algorithm that can certify sparse eigenvalues is perturbation resilient (in the sense of the previous paragraph). Indeed for $\Normi{E} = \eps \cdot {\gamma/k}$, the inequality becomes
\begin{equation}
\iprod{\hat{v},v_0}^2\geq 1-O(\eps)  -\frac{1}{\gamma}
O\Paren{ \max_{v\in S_k}\transpose{v}Wv}.
\end{equation}
At this point, the guarantees of the algorithm depend only on the specific certified upper bound on $\max_{v\in S_k}\transpose{v}Wv $ it can obtain. 

For the Wishart Model 	$Y = \sqrt{\beta}u_0 \transpose{v_0} + W + E$, the reasoning is essentially the same. However we need to work with $\transpose{Y}Y - n\Id$ and carefully bound the cross terms. Similar to the Wigner model,  the guarantees of the algorithm depend only on  the certified upper bound on $\max_{v\in S_k}\transpose{v}\Paren{\transpose{W}W-n\Id}v$ it can obtain. For the rest of our preliminary discussion we go back to the Wishart model.

\subsection{New Certificates via basic SDP}
For a matrix $M\in \R^{d\times d}$, the basic SDP program\footnote{Recall $\Norm{X}_1=\underset{i,j\in [d]}{\sum}\Abs{X_\ij}$ is the "absolute norm".} 
\begin{align}
\argmax\Set{\iprod{\transpose{Y}Y,X}\given X\sge 0, \Tr X=1, \Norm{X}_1\leq k}
\end{align}
can certify two types of upper bound:
\begin{align}\label{eq:sdp-infinity}
	\iprod{M,X}&\leq \normi{M}\cdot k\\
	\iprod{M,X}&\leq \Norm{M}\label{eq:sdp-spectral}\,.
\end{align}
The first follows using $ \Norm{X}_1\leq k$ and the second applying $X\sge 0, \Tr X=1$.
These are enough to capture standard principal component analysis as well as diagonal and covariance thresholding.

Specifically, \cref{eq:sdp-spectral} can be used to certify the upper bound $\iprod{\transpose{W}W - n\Id,X}\le O\Paren{d+\sqrt{dn}}$ -- obtaining the guarantees of PCA -- and \cref{eq:sdp-infinity} the bound $\iprod{\transpose{W}W-n\Id,X}\leq O\Paren{k\cdot \sqrt{n\log d}}$, as in diagonal thresholding\footnote{A more careful analysis can get $k\cdot \sqrt{n\log (d/k)}$, but we ignore it here.}. Now these results were already known, but surprisingly a  combination of the two bounds can also be used to show $\iprod{\transpose{W}W-n\Id,X}\leq k\cdot \sqrt{n\log(d/k^2)}$. Thus allowing us to match the guarantees of covariance thresholding.

Concretely, using the notation from the introduction,
\begin{align*}
	\iprod{\transpose{W}W - n\Id,X } &= \iprod{\eta_\tau\Paren{\transpose{W}W - n\Id },X} +\\ &\;\;\;\;\,\iprod{\transpose{W}W-\eta_t\Paren{\transpose{W}W},X}.
\end{align*}

Here $\transpose{W}W-\eta_t\Paren{\transpose{W}W}$ is a matrix with entries bounded (in absolute value) by $\tau$ for which we can plug in \cref{eq:sdp-infinity} and get 
\begin{align*}
\iprod{\transpose{W}W-\eta_t\Paren{\transpose{W}W},X}\leq \tau\cdot k
\end{align*} 

The same argument cannot be used for $\eta_\tau\Paren{\transpose{W}W}$, but note that this matrix is suspiciously close (up to an addition of $n\cdot \Id$) to the thresholded covariance matrix obtained in covariance thresholding. Hence, taking $\tau = \sqrt{n\log(d/k^2)}$ and using \cref{eq:sdp-spectral}, we get
\begin{align*}
	\iprod{\eta_\tau\Paren{\transpose{W}W - n\Id}, X} \leq O\Paren{k\sqrt{n\log\frac{d}{k^2}}}\,,
\end{align*}
where we get the spectral bound (almost) for free by the analysis in \cite{DBLP:conf/nips/DeshpandeM14}.

\subsection{New certificates via higher-level Sum-of-Squares}
\subsubsection{Certificates via Certifiable Subgaussianity}
The Sum-of-Squares algorithm can certify more refined bounds on sparse eigenvalues of $W\sim N(0,1)^{n\times d}$. 
In particular we can exploit Gaussian moments bound $\E \iprod{W_i,u}^{2t}\leq t^t\cdot \Norm{u}^{2t}$ for all $t\in \N$, $u \in \R^d$.

Concretely let's see how to use such property  to obtain an identifiability proof of a bound on the $k$-sparse norm of $W$. To this end let $v$ be a $k$-sparse vector and let $s\in \Set{0,1}^d$ be the indicator vector of its support (here we drop the subscript $v_0$ to ease the notation). Using Cauchy-Schwarz,
\iftoggle{shortv}{
\begin{align*}
	\|Wv\|^{4}&= \left( \sum_{i\le d} v_i \langle W_i, Wv\rangle \right)^2\\
	& \leq\Paren{\underset{i \leq d}{\sum} v_i^2}\Paren{\underset{i \leq d}{\sum} s_i^2 \iprod{W_i,Wv}^2}\\
	&\leq \Paren{\underset{i \leq d}{\sum} s_i^2 \iprod{W_i,Wv}^2}.
\end{align*}
}{
\begin{align*}
	\|Wv\|^{4}&= \left( \sum_{i\le d} v_i \langle W_i, Wv\rangle \right)^2 \leq\Paren{\underset{i \leq d}{\sum} v_i^2}\Paren{\underset{i \leq d}{\sum} s_i^2 \iprod{W_i,Wv}^2}\leq \Paren{\underset{i \leq d}{\sum} s_i^2 \iprod{W_i,Wv}^2}.
\end{align*}
}
Then applying Holder's inequality with $1/p+1/t=1$, and using the fact that $s$ is binary with norm $k$,
\iftoggle{shortv}{
\begin{align*}
	\Paren{\underset{i \leq d}{\sum} s_i^2 \iprod{W_i,Wv}^2}\leq& \Paren{\underset{i \leq d}{\sum}s_i^{2p}}^{1/p}\Paren{\underset{i \leq d}{\sum}\iprod{W_i,Wv}^{2t}}^{1/t}\\
	\leq&\Norm{Wv}^{2}\cdot k^{1-1/t}\cdot\\
	&\cdot\Paren{\underset{i \leq d}{\sum}\iprod{W_i,\frac{1}{\Norm{Wv}}  Wv}^{2t}}^{1/t}.
\end{align*}
}{
\begin{align*}
	\Paren{\underset{i \leq d}{\sum} s_i^2 \iprod{W_i,Wv}^2}\leq \Paren{\underset{i \leq d}{\sum}s_i^{2p}}^{1/p}\Paren{\underset{i \leq d}{\sum}\iprod{W_i,Wv}^{2t}}^{1/t}\leq\Norm{Wv}^{2}\cdot k^{1-1/t}\Paren{\underset{i \leq d}{\sum}\iprod{W_i,\frac{1}{\Norm{Wv}}  Wv}^{2t}}^{1/t}.
\end{align*}
}
This gets us to, 
\begin{align}\label{eq:identifiability-conclusion}
	\Snorm{Wv}\leq k^{1-1/t}\cdot \Paren{\underset{i \leq d}{\sum}
		\iprod{W_i,\frac{1}{\Norm{Wv}}  Wv}^{2t}}^{1/t}\,.
\end{align}
Now, whenever $d\gtrsim n^t t^t\log ^tn$ , the $t$-moment of the column vectors $W_1\ldots,W_d$ converges with high probability. That is, for any unit vector $u$,
\begin{align}
\frac{1}{d}\underset{i \leq d}{\sum} \iprod{W_i,u}^{2t}\leq O(t^t)\,.\label{eq:subgaussianity}
\end{align}
Thus, combining \cref{eq:identifiability-conclusion} and  \cref{eq:subgaussianity} we can conclude
\begin{align*}
	\Snorm{Wv}\lesssim k^{1-1/t}\cdot d^{1/t}\cdot t\,.
\end{align*}

The catch is that all the steps taken can be written as polynomial inequalities of degree at most $O(t)$. So we can certify the same bound through the Sum-of-Squares proof system.

\subsubsection{Certificates via Limited Brute Force}
Whenever the sparse vector $v_0$ is almost flat, that is when  for all $i \in \supp\Set{v_0}$ we have $\Abs{v_{0i}}\in \Brac{\frac{1}{C\sqrt{k}},\frac{C}{\sqrt{k}}}$, the guarantees of diagonal thresholding can be improved at the cost of increasing its running time (see \cite{ding2019subexponentialtime}). 

Diagonal thresholding can be viewed as selecting the $k$ vectors of the standard basis $e_1,\ldots,e_d$ maximizing $\Snorm{Ye_i}$, and then returning a top eigenvector of the covariance matrix projected onto the span of such vectors. Indeed this formulation has an intuitive generalization, namely instead of looking at $1$-sparse vectors, the algorithm could look into $t$-sparse vectors $u$ with entries in $\Set{\pm 1/\sqrt{t},0}$, pick the top $\binom{k}{t}$ and use them to recover $v_0$.

This idea can be translated into a certified upper bound for the sparse eigenvalues of $W\sim N(0,1)^{n \times d}$. Although we will be able to recover general sparse vectors, for the sake of this discussion we assume $v_0$ is flat.\footnote{So the Sum-of-Squares algorithm works in more general settings than the algorithm from \cite{ding2019subexponentialtime}.} Let's denote the set of  $t$-sparse flat vectors by $\cN_t$. Let $v_0\in\R^d$ be a $k$-sparse vector and denote with $D$ the uniform distribution over the vectors in $\cN_t$ such that $\iprod{u,v_0}= \sqrt{t/k}$. That is, the set of vectors $u$ such that $\supp\Set{u}\subseteq \supp\Set{v_0}$ and with sign pattern matching the sign pattern of $v$ restricted to $\supp\Set{u}$.

Note  that for any matrix $M\in\R^{d\times d}$,
\begin{align*}
	\transpose{v_0}Mv_0= \frac{k}{t}\E_{u\sim D}\E_{u'\sim D}  \transpose{u}M u'\,.
\end{align*}
This equality \textit{per se} is not interesting, but for a Gaussian matrix $W\sim N(0,1)^{n\times d}$, with high probability,
\begin{align*}
	\underset{u,u' \in \cN_t}{\max} \Abs{\transpose{u}\Paren{\transpose{W}W-n\Id}u'}  \le  O\Paren{\sqrt{n t\log d}}.
\end{align*}
Thus, combining the two we get
\begin{align*}
	\transpose{v_0}\Paren{\transpose{W}W-n \Id}v_0 =& \frac{k}{t}\E_D \transpose{u} \Paren{\transpose{W}W-n\Id}u'\\
	\leq& \frac{k}{t}\underset{u,u' \in \cN_t}{\max} \Abs{\transpose{u}\Paren{\transpose{W}W-n\Id}u'} \\
	\leq& \frac{k}{\sqrt{t}}\sqrt{n \log d}\,,
\end{align*}
which allows us to conclude that $\Snorm{Wv_0}\leq n + \frac{k}{\sqrt{t}}\sqrt{n \log d}$. This certificates can be proved using Sum-of-Squares, hence  allowing us to improve over the basic SDP by a factor $t$ in the settings $k^2\leq d^{1-\Omega(1)}$.

\subsection{Concrete lower bounds for perturbation-resilient algorithms}\label{sec:lower-bounds-robust-settings}
Sparse principal component analysis is what we often call a \textit{planted problem}.
These are problems that ask to recover some signal hidden  by random or adversarial noise. The easiest way one could formulate a planted problem is its \textit{distinguishing} version: where given two distributions, a  \textit{null} distribution without structure and a \textit{planted} distribution containing the hidden signal, the objective is to determine with high probability whether a given instance was sampled from one distribution or the other.

A common strategy to provide evidence for \textit{information-computation gap} in a certain planted problem is to prove that powerful classes of efficient algorithms are unable to solve it in the (conjecturally) hard regime.
Indeed our goal here will be that of constructing two distributions under which low-degree polynomials take roughly the same values and hence cannot distinguish \iftoggle{shortv}{}{(in the sense of \cref{sec:low-degree-likelihood-ratio})} from which distribution the instance $Y$ was sampled. Since low-degree polynomials cannot tell if $Y$ has indeed the form $W+\sqrt{\beta}u_0\transpose{v_0}+E$ (and therefore cannot solve the problem), this would mean they cannot be used to improve over the guarantees of \cref{thm:sos-algorithm-robust-spca-small-n}. 

Our null distribution $\nu$ will be the standard Gaussian $N(0,1)^{n\times d}$. However, the main question is how to design the planted distribution $\mu$. Recall $Y$ takes the form $W+\sqrt{\beta}u_0\transpose{v_0}+E$. If we set $E=0$, then our planted distribution corresponds to the single spike covariance model. We could get a lower bound for such problem \iftoggle{shortv}{}{(see \cref{sec:lowerbounds})} but this would not help us in showing that the guarantees of \cref{thm:sos-algorithm-robust-spca-small-n} are tight. On the other hand, if for example we choose $E$ with the goal of planting a large eigenvalue, then the problem of distinguishing between $\nu$ and $\mu$ may become even easier than without adversarial perturbations. 

This suggests that we should choose $E$ very carefully, in particular we should design $E$ so that $Y=W+\sqrt{\beta}u_0\transpose{v_0}+E$ appears -- to the eyes of a low-degree polynomial estimator -- as a Gaussian distribution. 
Our approach will be that of constructing $E$ so that the first few moments of $\mu$ will be Gaussian. This will lead us to \cref{thm:lowerbound-robust-spca-informal} through two basic observations: first, given two distributions with same first $2t$ moments, computing those first $2t$ moments won't help distinguishing between the two distributions. Second, for a Gaussian distribution $N(0,\Id_n)$,  at least $n^{t}$ samples are required in order for the $2t$-th moment of the empirical distribution to converge to $\E\Brac{\tensorpower{w}{2t}}$. 

Concretely, we consider the following model: we choose iid gaussian vectors $z_1,\ldots,z_{n-1}\sim N(0,1)^d$, and a random vector $z_0\in \R^d$ with iid symmetric (about zero) coordinates that satisfies the following properties:
\begin{enumerate}
	\item $z_0$ has approximately $k$ large coordinates (larger than $\lambda\approx \sqrt{\beta n /k}$ by absolute value).
	\item  For any coordinate of $z_0$ its first $2t-2$ moments coincide with moments of $N(0,1)$, and its higher $r$-moments (for even $r$) are close to $\frac{k}{d}\lambda^r$.
\end{enumerate}
Then we obtain the matrix $Y\in \R^{n\times d}$ applying a random rotation $R\in \R^{n\times n}$ to the $n\times d$ matrix with rows $z_0^\top,z_1^\top,\ldots,z_{n-1}^\top$.  It is not difficult to see that indeed such $Y$ can be written as $Y=W+\sqrt{\beta}u_0\transpose{v_0}+E$, as in the model of \cref{def:wishart-matrix-model}. 

Now, assume for simplicity that $t$ is constant and denote the distribution of $Y$ described above by $\mu$ and the standard Gaussian distribution $N(0,1)^{n\times d}$ by $\nu$. An immediate consequence of our construction is that for any polynomial $p$  of degree at most $2t-2$, $\E_{Y\sim \mu}\Brac{p(Y)}=\E_{Y\sim\nu}\Brac{p(Y)}$.  Furthermore,  in order to reliably tell the difference between $\E_\mu\Brac{p'(Y)}$ and $\E_{\nu}\Brac{p'(W)}$ for a polynomial of even degree $r\geq 2t$  (say up to $r = n^{0.001}$), we will need a precise estimate of such $r$-th moments and \textit{hence} at least $n^{r/2\geq t}$ samples.  This effect is then shown by proving that for multilinear polynomials $p(Y)$ of degree $D \le n^{0.001}$, if $d \le n^{0.99t - 1}$ and $\beta n / k \le n^{0.49}$, then the low-degree analogue of $\chi^2$-divergence $\underset{p(Y)\text{ of degree } \le D } {\max}\frac{\Paren{\E_\nu p(Y)-\E_\mu p(Y)}^2}{{\Var_\nu P(Y)}}$ is close to zero. Note that for technical reasons our analysis is restricted to the multilinear polynomials.  As shown in \cite{DBLP:conf/focs/BarakHKKMP16, DBLP:conf/focs/HopkinsS17, hopkins2018statistical} \iftoggle{shortv}{}{and as it will be evident from the single spike model lower bound in \cref{sec:lowerbounds-spike},} this restricted model of computation captures the best known algorithms for many planted problems.

\subsection{Beyond limitations of CT via low-degree polynomials}
An important aspect of the computation of lower bounds for low-degree polynomials is that they may provide valuable insight on how to construct an optimal algorithm. Indeed low-degree polynomials capture many spectral properties of linear operators; for example, the largest singular value of a $d$-dimensional linear operator with a spectral gap can be approximated by $\lesssim \log d$ degree polynomial in its entries.

We discuss here how they can be used to improve over the guarantees of Covariance Thresholding

\paragraph{Why Covariance Thresholding doesn't work with small sample size}
In order to improve over Covariance Thresholding, the first question we need to understand is whether the algorithm could actually work in a larger set of parameters than the one currently known. The answer is no. Recall that for $k^2 \le d/2$ and $n \le d$ Covariance Thresholding  (with an appropriate choice of thresholding parameter $\tau$) works if $\beta \gtrsim \frac{k}{\sqrt{n}}\sqrt{\log \frac{d}{k^2} + \log \frac{d}{n}}$, 
and so for $n \ge k^2$ and $d^{1-o(1)} \le k^2 \le o(d)$ this is asymptotically better than the guarantees of SVD, SVD+Thresholding and Diagonal Thresholding.

It is not difficult to see that Covariance Thresholding with $\tau \ge \Omega(\sqrt{n\log d})$ cannot have better guarantees than Diagonal Thresholding. So we consider $\tau \le o(\sqrt{n\log d})$.

Note that $d^{1-o(1)} \le k^2 \le o(d)$ and $n \ge k^2$ imply $n > d^{1-o(1)}$. 
The assumption $n > d^{1-o(1)}$ is crucial for Covariance Thresholding. 
To show this, it is enough to prove that for some unit $x\in \R^d$, $\transpose{x}\eta_{\tau}\Paren{Y^\top Y-n \Id}x> d^{1-o(1)}$. Indeed, as on the other hand $\Abs{\transpose{v_0}\eta_{\tau}\Paren{Y^\top Y-n \Id}v_0} \approx \beta n$, this would mean that for $\beta \ll d^{1-o(1)}/n$ the top eigenvectors of $\eta_{\tau}\Paren{Y^\top Y-n \Id}$ are uncorrelated with $v_0$. Additionally, since $\sqrt{\frac{d}{n}} \ll \Paren{ \frac{ d^{1-o(1) } } {n} }$ in these settings SVD+Thresholding has significantly better guarantees.
 
An $x$ satisfying our inequality is easy to find, for example any row $W_1,\ldots,W_n \in \R^d$ satisfies $\transpose{W_i}\eta_{\tau}\Paren{Y^\top Y-n \Id}W_i> d^{1-o(1)}\Snorm{W_i}$ with high probability \iftoggle{shortv}{}{(see \cref{sec:covariance-thresholding-fails-small-n})}.

Hence Covariance Thresholding doesn't provide better guarantees than SVD or Diagonal Thresholding if $n \le d^{1-\Omega(1)}$ (for example, if $n = d^{0.99}$).

\subsubsection{Polynomials based algorithm}

Theorem \ref{thm:fragile-lower-bound-informal} shows that if 
$k^2 \le d^{1-\Omega(1)}$ and
$\beta \le o\Paren{\frac{k}{\sqrt{n}}\sqrt{\log d}}$, it is unlikely that polynomial time algorithms can solve the problem. So to get an asymptotic improvement over Diagonal Thresholding we need $k^2 \ge d^{1-o(1)}$. 

However, note that there is no condition $n \ge d^{1-o(1)}$ in our lower bound. This suggests that there might be an algorithm that is asymptotically better than SVD and Diagonal Thresholding for small $n$, for example $n = d^{0.99}$ or $n = d^{0.01}$. Indeed, we show that there exists a polynomial time algorithm that can recover the sparse vector $v_0$ with entries in $\set{0, \pm 1/\sqrt{k}}$ as long as $d^{1-o(1)} \le k^2 \le d/2$, $\beta \gtrsim \frac{k}{\sqrt{n}}\sqrt{\log \frac{d}{k^2} + \frac{\log d}{\log n}}$ and $n\gtrsim \log^5 d$. In particular, if $d^{1-o(1)} \le k^2 \le o(d)$ and $d^{0.01} \le  n \le d^{0.99}$, this algorithm has asymptotically better guarantees than Diagonal Thresholding, SVD, SVD+Thresholding, and Covariance Thresholding.

Our algorithm is based on the approach introduced in \cite{DBLP:conf/focs/HopkinsS17} for commutinity detection in stochastic block model.
An informal description of the algorithm is as follows: 
we compute some symmetric matrix $P(Y) \in \R^{d\times d}$ whose entries are polynomials 
$P_{jj'}(Y)$ in the entries of $Y$ of degree $O(\log d)$. 
The algorithm outputs a top eigenvector of this matrix, which we prove to be highly correlated with $v_0$.
Note that since the degrees of involved polynomials are $O(\log d)$, 
simple evaluation takes time $(nd)^{O\Paren{\log d}}$. 
However, we can compute a very good approximation to the values of these polynomials in time 
$(nd)^{O(1)}$ using a \emph{color coding} technique (this part of the algorithm uses internal randomness).

More precisely, for $j,j'\in [d]$ we compute multilinear
polynomials $P_{jj'}(Y)$ of degree $O(\log d)$ such that 
for every $j\neq j'$, $\E P_{jj'}(Y) = v_0(j) v_0(j')$, and for every $j\in [d]$, $P_{jj}(Y)= 0$. 
Then we show that variance of $P_{jj'}(Y)$ is small so that 
$\E\normf{P(Y) - v_0\transpose{v_0}}^2 < o\Paren{1}$. 
This implies that with probability $1-o(1)$, $\normf{P(Y) - v_0\transpose{v_0}}^2 < o(1)$, 
so the top eigenvector of $P(Y)$ is highly correlated with either $v_0$ or $-v_0$.

To bound the variance, we represent each monomial as a bipartite multigraph $G = (R, C, E)$, 
with bipartition $R\subset [n]$ 
which corresponds to rows of $Y$ and $C\subset [d]$ which correspond to columns of $Y$.
Since the variance is a sum of monomials, we compute the contribution of each monomial and bound the number of corresponding multigraphs. Finally, we show that there exists a polynomial such that in the parameter regime $d^{1-o(1)} \le k^2 \le d/2$, $\beta \gtrsim \frac{k}{\sqrt{n}}\sqrt{\log \frac{d}{k^2} + \frac{\log d}{\log n}}$ and $n\gtrsim \log^5 d$, there is no group of monomials with large contribution in the variance, so we can conclude that this polynomial is a good estimator.

After showing that there are good polynomial estimators of degree $O(\log d)$, 
we approximately compute them using color coding. All monomials of the polynomials $P_{jj'}$ that we consider have the same structure (in the sence that the  graphs corresponding to these monomials are isomorphic). Each of them has the same number $r$ of vertices which correspond to rows and the same number $c$ of vertices which correspond to columns. We show that for each coloring of $[n]$ in $r$ color and each coloring of $[d]$ in $c$ colors, we can in time $(nd)^{O(1)}$ 
compute the sum of monomials of $P_{jj'}$ colored exactly in colores from $[r]$ and $[c]$. 
If we average these values over large enough set of random colorings (of size $(nd)^{O(1)}$), we get a value very close to $P_{jj'}(Y)$.

One important advantage of this polynomial-based algorithm is that we only need the following assumptions on $W$: that the entries of $W$ are i.i.d., $\E W_{ij} =0$ and $\E W_{ij}^2 = 1$.\footnote{Indeed, prior work \cite{dingsteurer2020} observed that polynomial-based algorithms require only first and second moment conditions on the noise entries for a broad range of matrix and tensor estimation problems.}
All previously known algorithms require bounds on entries or the spectral norm of $\transpose{W}W$ (or related matrices, e.g. thresholded $\transpose{W}W$), so they require $\chi^2$ tail bounds.

\iftoggle{shortv}{
\section*{Acknowledgment}
The authors would like to thank the anonymous reviewers whose suggestions greatly improved the exposition of the paper.
}{}


\section{Preliminaries}
\label{sec:preliminaries}


In this section, we introduce preliminary notions which will be used in the rest of the paper. We start by defining pseudo-distributions and sum-of-squares proofs (see the lecture notes \cite{barak2016proofs} for more details and the appendix in \cite{DBLP:journals/corr/MaSS16} for proofs of the propositions appearing here). Then we introduce the low-degree likelihood ratio (see \cite{hopkins2018statistical} for details).

Let $x = (x_1, x_2, \ldots, x_n)$ be a tuple of $n$ indeterminates and let $\R[x]$ be the set of polynomials with real coefficients and indeterminates $x_1,\ldots,x_n$.
We say that a polynomial $p\in \R[x]$ is a \emph{sum-of-squares (sos)} if there are polynomials $q_1,\ldots,q_r$ such that $p=q_1^2 + \cdots + q_r^2$.

\subsection{Pseudo-distributions}

Pseudo-distributions are generalizations of probability distributions.
We can represent a discrete (i.e., finitely supported) probability distribution over $\R^n$ by its probability mass function $D\from \R^n \to \R$ such that $D \geq 0$ and $\sum_{x \in \mathrm{supp}(D)} D(x) = 1$.
Similarly, we can describe a pseudo-distribution by its mass function.
Here, we relax the constraint $D\ge 0$ and only require that $D$ passes certain low-degree non-negativity tests.

Concretely, a \emph{level-$\ell$ pseudo-distribution} is a finitely-supported function $D:\R^n \rightarrow \R$ such that $\sum_{x} D(x) = 1$ and $\sum_{x} D(x) f(x)^2 \geq 0$ for every polynomial $f$ of degree at most $\ell/2$.
(Here, the summations are over the support of $D$.)
A straightforward polynomial-interpolation argument shows that every level-$\infty$-pseudo distribution satisfies $D\ge 0$ and is thus an actual probability distribution.
We define the \emph{pseudo-expectation} of a function $f$ on $\R^d$ with respect to a pseudo-distribution $D$, denoted $\pE_{D(x)} f(x)$, as
\begin{equation}
  \pE_{D(x)} f(x) = \sum_{x} D(x) f(x) \,\mper
\end{equation}
The degree-$\ell$ moment tensor of a pseudo-distribution $D$ is the tensor $\E_{D(x)} (1,x_1, x_2,\ldots, x_n)^{\otimes \ell}$.
In particular, the moment tensor has an entry corresponding to the pseudo-expectation of all monomials of degree at most $\ell$ in $x$.
The set of all degree-$\ell$ moment tensors of probability distribution is a convex set.
Similarly, the set of all degree-$\ell$ moment tensors of degree $d$ pseudo-distributions is also convex.
Key to the algorithmic utility of pseudo-distributions is the fact that while there can be no efficient separation oracle for the convex set of all degree-$\ell$ moment tensors of an actual probability distribution, there's a separation oracle running in time $n^{O(\ell)}$ for the convex set of the degree-$\ell$ moment tensors of all level-$\ell$ pseudodistributions.

\begin{fact}[\cite{MR939596-Shor87,parrilo2000structured,MR1748764-Nesterov00,MR1846160-Lasserre01}]
  \label[fact]{fact:sos-separation-efficient}
  For any $n,\ell \in \N$, the following set has a $n^{O(\ell)}$-time weak separation oracle (in the sense of \cite{MR625550-Grotschel81}):
  \begin{equation}
    \Set{ \pE_{D(x)} (1,x_1, x_2, \ldots, x_n)^{\otimes d} \mid \text{ degree-d pseudo-distribution $D$ over $\R^n$}}\,\mper
  \end{equation}
\end{fact}
This fact, together with the equivalence of weak separation and optimization \cite{MR625550-Grotschel81} allows us to efficiently optimize over pseudo-distributions (approximately)---this algorithm is referred to as the sum-of-squares algorithm.

The \emph{level-$\ell$ sum-of-squares algorithm} optimizes over the space of all level-$\ell$ pseudo-distributions that satisfy a given set of polynomial constraints---we formally define this next.

\begin{definition}[Constrained pseudo-distributions]
  Let $D$ be a level-$\ell$ pseudo-distribution over $\R^n$.
  Let $\cA = \{f_1\ge 0, f_2\ge 0, \ldots, f_m\ge 0\}$ be a system of $m$ polynomial inequality constraints.
  We say that \emph{$D$ satisfies the system of constraints $\cA$ at degree $r$}, denoted $D \sdtstile{r}{} \cA$, if for every $S\subseteq[m]$ and every sum-of-squares polynomial $h$ with $\deg h + \sum_{i\in S} \max\set{\deg f_i,r}\leq \ell$,
  \begin{displaymath}
    \pE_{D} h \cdot \prod _{i\in S}f_i  \ge 0\,.
  \end{displaymath}
  We write $D \sdtstile{}{} \cA$ (without specifying the degree) if $D \sdtstile{0}{} \cA$ holds.
  Furthermore, we say that $D\sdtstile{r}{}\cA$ holds \emph{approximately} if the above inequalities are satisfied up to an error of $2^{-n^\ell}\cdot \norm{h}\cdot\prod_{i\in S}\norm{f_i}$, where $\norm{\cdot}$ denotes the Euclidean norm\footnote{The choice of norm is not important here because the factor $2^{-n^\ell}$ swamps the effects of choosing another norm.} of the cofficients of a polynomial in the monomial basis.
\end{definition}

We remark that if $D$ is an actual (discrete) probability distribution, then we have  $D\sdtstile{}{}\cA$ if and only if $D$ is supported on solutions to the constraints $\cA$.

We say that a system $\cA$ of polynomial constraints is \emph{explicitly bounded} if it contains a constraint of the form $\{ \|x\|^2 \leq M\}$.
The following fact is a consequence of \cref{fact:sos-separation-efficient} and \cite{MR625550-Grotschel81},

\begin{fact}[Efficient Optimization over Pseudo-distributions]
There exists an $(n+ m)^{O(\ell)} $-time algorithm that, given any explicitly bounded and satisfiable system\footnote{Here, we assume that the bitcomplexity of the constraints in $\cA$ is $(n+m)^{O(1)}$.} $\cA$ of $m$ polynomial constraints in $n$ variables, outputs a level-$\ell$ pseudo-distribution that satisfies $\cA$ approximately. 
\end{fact}

\subsection{Sum-of-squares proofs}

Let $f_1, f_2, \ldots, f_r$ and $g$ be multivariate polynomials in $x$.
A \emph{sum-of-squares proof} that the constraints $\{f_1 \geq 0, \ldots, f_m \geq 0\}$ imply the constraint $\{g \geq 0\}$ consists of  sum-of-squares polynomials $(p_S)_{S \subseteq [m]}$ such that
\begin{equation}
g = \sum_{S \subseteq [m]} p_S \cdot \Pi_{i \in S} f_i
\mper
\end{equation}
We say that this proof has \emph{degree $\ell$} if for every set $S \subseteq [m]$, the polynomial $p_S \Pi_{i \in S} f_i$ has degree at most $\ell$.
If there is a degree $\ell$ SoS proof that $\{f_i \geq 0 \mid i \leq r\}$ implies $\{g \geq 0\}$, we write:
\begin{equation}
  \{f_i \geq 0 \mid i \leq r\} \sststile{\ell}{}\{g \geq 0\}
  \mper
\end{equation}

Sum-of-squares proofs satisfy the following inference rules.
For all polynomials $f,g\colon\R^n \to \R$ and for all functions $F\colon \R^n \to \R^m$, $G\colon \R^n \to \R^k$, $H\colon \R^{p} \to \R^n$ such that each of the coordinates of the outputs are polynomials of the inputs, we have:

\begin{align}
&\frac{\cA \sststile{\ell}{} \{f \geq 0, g \geq 0 \} } {\cA \sststile{\ell}{} \{f + g \geq 0\}}, \frac{\cA \sststile{\ell}{} \{f \geq 0\}, \cA \sststile{\ell'}{} \{g \geq 0\}} {\cA \sststile{\ell+\ell'}{} \{f \cdot g \geq 0\}} \tag{addition and multiplication}\\
&\frac{\cA \sststile{\ell}{} \cB, \cB \sststile{\ell'}{} C}{\cA \sststile{\ell \cdot \ell'}{} C}  \tag{transitivity}\\
&\frac{\{F \geq 0\} \sststile{\ell}{} \{G \geq 0\}}{\{F(H) \geq 0\} \sststile{\ell \cdot \deg(H)} {} \{G(H) \geq 0\}} \tag{substitution}\mper
\end{align}

Low-degree sum-of-squares proofs are sound and complete if we take low-level pseudo-distributions as models.

Concretely, sum-of-squares proofs allow us to deduce properties of pseudo-distributions that satisfy some constraints.

\begin{fact}[Soundness]
  \label{fact:sos-soundness}
  If $D \sdtstile{r}{} \cA$ for a level-$\ell$ pseudo-distribution $D$ and there exists a sum-of-squares proof $\cA \sststile{r'}{} \cB$, then $D \sdtstile{r\cdot r'+r'}{} \cB$.
\end{fact}

If the pseudo-distribution $D$ satisfies $\cA$ only approximately, soundness continues to hold if we require an upper bound on the bit-complexity of the sum-of-squares $\cA \sststile{r'}{} B$  (number of bits required to write down the proof).

In our applications, the bit complexity of all sum of squares proofs will be $n^{O(\ell)}$ (assuming that all numbers in the input have bit complexity $n^{O(1)}$).
This bound suffices in order to argue about pseudo-distributions that satisfy polynomial constraints approximately.

The following fact shows that every property of low-level pseudo-distributions can be derived by low-degree sum-of-squares proofs.

\begin{fact}[Completeness]
  \label{fact:sos-completeness}
  Suppose $d \geq r' \geq r$ and $\cA$ is a collection of polynomial constraints with degree at most $r$, and $\cA \vdash \{ \sum_{i = 1}^n x_i^2 \leq B\}$ for some finite $B$.

  Let $\{g \geq 0 \}$ be a polynomial constraint.
  If every degree-$d$ pseudo-distribution that satisfies $D \sdtstile{r}{} \cA$ also satisfies $D \sdtstile{r'}{} \{g \geq 0 \}$, then for every $\epsilon > 0$, there is a sum-of-squares proof $\cA \sststile{d}{} \{g \geq - \epsilon \}$.
\end{fact}

We will repeatedly use the following SoS version of Cauchy-Schwarz inequality and its generalization, Hölder's inequality:
\begin{fact}[Sum-of-Squares Cauchy-Schwarz]
Let $x,y \in \R^d$ be indeterminites. Then,
\[ 
\sststile{4}{x,y} \Set{\Paren{\sum_i x_i y_i}^2 \leq \Paren{\sum_i x_i^2} \Paren{\sum_i y_i^2}} 
\]
 \label{fact:sos-cauchy-schwarz}
\end{fact} 

We will also use the following fact that shows that spectral certificates are captured within the SoS proof system.

\begin{fact}[Spectral Certificates] \label{fact:spectral-certificates}
For any $m \times m$ matrix $A$, 
\[
\sststile{2}{u} \Set{ \iprod{u,Au} \leq \Norm{A} \Norm{u}_2^2}\mper
\]
\end{fact}

We will also use the following Cauchy-Schwarz inequality for pseudo-distributions. 

\begin{fact}[Cauchy-Schwarz for Pseudo-distributions]
Let $f,g$ be polynomials of degree at most $d$ in indeterminate $x \in \R^d$. Then, for any degree d pseudo-distribution $D$,
$\pE_{D}[fg] \leq \sqrt{\pE_{D}[f^2]} \sqrt{\pE_{D}[g^2]}$.
 \label{fact:pseudo-expectation-cauchy-schwarz}
\end{fact} 

\subsection{Low-degree likelihood Ratio} \label{sec:low-degree-likelihood-ratio}
The low-degree likelihood ratio is a proxy to model efficiently computable functions. It is closely related to the pseudo-calibration technique and it has been developed in a recent line of work on the Sum-of-Squares hierarchy \cite{DBLP:conf/focs/BarakHKKMP16, DBLP:conf/focs/HopkinsS17, DBLP:conf/focs/HopkinsKPRSS17,hopkins2018statistical}. Our description is also based on \cite{conf/innovations/BandeiraKW20}.

The objects of study are distinguishing versions of planted problems, in which given two distributions and an instance, the goal is to decide from which distribution the instance was sampled. For example, in the context of Sparse PCA, the distinguishing formulation takes the form of deciding whether the matrix $Y$ was sampled according to the (planted) distribution as described in \ref{def:wishart-matrix-model}, or if it was sampled from the (null) Gaussian distribution $N(0,1)^{n\times d}$. In general, we denote with $\nu$ the null distribution and with $\mu$ the planted  distribution with the hidden structure.

\subsubsection{Background on Classical Decision Theory}\label{sec:background-decision-theory}
From the point of view of classical Decision Theory, the optimal algorithm to distinguish between two distribution is well-understood. Given distributions $\nu$ and $\mu$ on a measurable space $\cS$, the likelihood ratio $L(Y):=d\bbP_\mu(Y)/d\bbP_\nu(Y)$\footnote{The Radon-Nikodym derivative} is the optimal function to distinguish whether $Y\sim \nu$  or $Y\sim \mu$ in the following sense.
\begin{proposition}\cite{neymanpearson}
	If $\mu$ is absolutely continuous with respect to $\nu$, then the unique solution of the optimization problem
	\begin{align*}
	\max \E_\mu\Brac{f(Y)} \qquad \text{subject to }\E_\nu\Brac{f(Y)^2}=1
	\end{align*}
	is the normalized likelihood ratio $L(Y)/\E_\nu\Brac{L(Y)^2}$ and the value of the optimization problem is $\E_\nu\Brac{L(Y)^2}$.
\end{proposition}
Similarly, arguments about statistical distinguishability are known as well. Unsurprisingly, the likelihood ratio plays a major role here as well. The key concept is the Le Cam's contiguity. 
\begin{definition}\cite{leCam}
	Let $\underline{\mu}=\Paren{\mu_n}_{n\in \N}$ and $\underline{\nu}=\Paren{\nu_n}_{n\in \N}$ be sequences of probability measures on a common probability space $\cS_n$. Then $\underline{\mu}$ and $\underline{\nu}$ are \textit{contiguous}, written $\underline{\mu} \triangleleft \underline{\nu}$, if as $n\rightarrow \infty $, whenever for $A_n\in \cS_n$, $\bbP_{\underline{\mu}}(A_n)\rightarrow 0$ then $\bbP_{\underline{\nu}}(A_n)\rightarrow 0$.
\end{definition}
Contiguity allows us to capture the idea of indistinguishability of probability measures. Indeed two contiguous sequences $\underline{\mu},\underline{\nu}$ of probability measures are indistinguishable in the sense than there is no function $f:\cS_n\rightarrow\Set{0,1}$ such that $f(Y)=1$ with high probability whenever $Y\sim \underline{\mu}$ and $f(Y)=0$ with high probability whenever $Y\sim \underline{\nu}$. The key tool now is the so called \textit{Second Moment Method}, which allows us to establish contiguity through the likelihood ratio.
\begin{proposition}
	If $\E_\nu\Brac{L_n(Y)^2}$ remains bounded as $n\rightarrow \infty$, then $\underline{\mu} \triangleleft\underline{\nu}$.
\end{proposition}
This discussion allows us to argue whether a given function can be used to distinguish between our planted and null distributions. 

\subsubsection{Background on the Low-degree Method}\label{sec:baground-low-degree-method}
The main problem with the likelihood ratio is that it is in general hard to compute, thus we need to restrict these classical analysis to the space of efficiently computable functions. Concretely, we use low-degree multivariate polynomials in the entries of the observation $Y$ as a proxy for efficiently computable functions. Denoting with $\R_{\leq D}[Y]$ the space of polynomials in $Y$ of degree at most $D$ we can establish a low-degree version of the Neyman-Pearson lemma.
\begin{proposition}[e.g. \cite{hopkins2018statistical}]
	The unique solution of the optimization problem
	\begin{align*}
	\underset{f\in \R_{\leq D}[Y]}{\max} \E_\mu\Brac{f(Y)} \qquad \text{subject to }\E_\nu\Brac{f(Y)^2}=1
	\end{align*}
	is the normalized orthogonal projection $L^{\leq D}(Y)/\E_\nu\Brac{L^{\leq D}(Y)^2}$ of the likelihood ratio $L(Y)$ onto $\R_{\leq D}[Y]$ and the value of the optimization problem is $\E_\nu\Brac{L^{\leq D}(Y)^2}$.
\end{proposition}
It is important to remark that at the heart of our discussion, there is the belief that in the study of planted problems, low-degree polynomials capture the computational power of efficiently computable functions. This can be phrased as the following conjecture.
\begin{conjecture}[Informal]\label{con:low-degree-polynomials}\cite{DBLP:conf/focs/BarakHKKMP16, DBLP:conf/focs/HopkinsS17, DBLP:conf/focs/HopkinsKPRSS17,hopkins2018statistical}
	For "nice" sequences of probability measures $\underline{\mu}$ and $\underline{\nu}$, if there exists $D=D(d)\geq \omega\Paren{\log d}$ for which $\E_\nu\Brac{L^{\leq D}(Y)^2}$ remains bounded as $d\rightarrow \infty$, then there is no polynomial-time algorithm that distinguishes in the sense described in \ref{sec:background-decision-theory}.\footnote{We do not explain what "nice" means and direct the reader to \cite{hopkins2018statistical}.}
\end{conjecture}

A large body of work provide support for this conjecture (see any of the citations above), mostly in the form of evidence of an intimate relation between polynomials and Sum of Squares algorithms and lower bounds. For a more in detail discussion we point the interested reader to \cite{ DBLP:conf/focs/HopkinsKPRSS17,hopkins2018statistical}.

\section{Resilience of the basic SDP and Certified Upper Bounds}\label{sec:basic-sdp}
In this section we show the guarantees of the basic SDP algorithm \cite{d2005direct,amini2009}, thus proving \cref{thm:weak-signal-regime-sdp}. 

We will first prove that for any matrix $M\in \R^{d\times d}$  the basic SDP  can certify an upper bound $\|Mx\|^2\leq  k \cdot\Normi{M}^2$ on $k$-sparse quadratic forms over $M$. Furthermore we will show that for random Gaussian matrices  $W\sim N(0,1)^{n\times d}$ this bound can be significantly improved in various ways,  depending on the regime. Most notably, we will show that the basic SDP can certify a bound $\Snorm{Wx}\leq n+k\sqrt{n\log(d/\min \Set{k^2, n})}$, thus matching the guarantees of Covariance Thresholding. As a corollary, we also get that for $\beta < 1$ the algorithm achieves the best known guarantees among polynomial time algorithms in \textit{both} the fragile and the robust settings.

Formally the Sparse PCA problem can be defined as follows. 

\begin{problem}\label{problem:l0-sparse-pca}
	Given an instance $Y$ of \ref{def:wishart-matrix-model} let $\hat{\Sigma} = \transpose{Y}Y$. Then the Sparse PCA problem is defined by
	\begin{align*}
	\argmax \Set{\transpose{v}\hat{\Sigma} v\given \Snorm{v}=1,\Norm{v}_0\leq k}
	\end{align*}
	where $\Norm{v}_0$ is the number of non-zero entries in $v$.
\end{problem}

Solving Problem \ref{problem:l0-sparse-pca} is NP-hard in general \cite{conf/icml/MoghaddamWA06, journals/siamcomp/Natarajan95, krauthgamer2015semidefinite}, however the following concrete SDP relaxation \cite{d2005direct}  can be efficiently solved
\begin{align}
\argmax\Set{\iprod{\hat{\Sigma},X}\given X\sge 0, \Tr X=1, \Norm{X}_1\leq k}\tag{SDP-1} \label{eq:sparse-pca-sdp-relaxation}
\end{align}
where $\Norm{X}_1=\underset{i,j\in [d]}{\sum}\Abs{X_\ij}$ is the "absolute norm". We will show how to recover $v_0$ using such program. 

We start by restating some of the notation from the introduction. For a set $S\subseteq [d]\times [d]$, and a  matrix $M\in \R^{d\times d}$, we denote by $M\Brac{S}$ the matrix with entries $M\Brac{S}_{ij}=M_\ij$ if $(i,j)\in S$, and $M\Brac{S}_{ij}=0$ otherwise. For a matrix $M \in \R^{d\times d}$ and $\tau \in \R$, we define $\eta_\tau\Paren{M}\in \R^{d\times d}$ to be the matrix with entries 
\begin{align*}
\eta_\tau\Paren{M} =\begin{cases}
M_\ij &\text{ if $\Abs{M_\ij}\geq \tau$}\\
0 & \text{ otherwise.}
\end{cases}
\end{align*}
Furthermore, we define $\zeta_\tau\Paren{M}\in \R^{d\times d}$ to be the matrix with entries
\begin{align*}
\zeta_\tau\Paren{M} =\begin{cases}
M_\ij -\sign\Paren{M_\ij}\cdot \tau &\text{ if $\Abs{M_\ij}\geq \tau$}\\
0 & \text{ otherwise.}
\end{cases}
\end{align*}

\subsection{Basic Certificates for Sparse Quadratic Forms}\label{sec:basic-sdp-certificates}
We show here what certificates over sparse quadratic forms \ref{eq:sparse-pca-sdp-relaxation} can provide. These certificates are already enough to match the best known guarantees in the weak signal regime. 
The first observation is that it is straightforward to  bound the  product between $X$ and matrices with small infinity norm. By construction of $X$ this is indeed a certificate of an upper bound over $k$-sparse quadratic forms. 

\begin{lemma}\label{lem:basic-sdp-iprod-matrix-with-small-entries}
	For $k \in \N$, let $X\in \R^{d\times d}$ such that $\Normo{X}\leq k$. Then for any matrix $M\in \R^{d\times d}$
	\begin{align*}
		\Abs{\iprod{M,X}}\leq k \cdot \Normi{M}.
	\end{align*}
	\begin{proof}
		The Lemma follows immediately by choice of $X$,
		\begin{align*}
		\Abs{\iprod{X,M}} = 
		\Abs{\underset{i,j \in [d]}{\sum}M_{\ij}X_{\ij}}
		\leq \underset{i,j \in [d]}{\sum}\Abs{M_{\ij}X_{\ij}}
		\leq  \Normi{M} \underset{i,j \in [d]}{\sum}\Abs{X_\ij}
		\leq  k\cdot\Normi{M}.
		\end{align*}
	\end{proof}
\end{lemma}

Now we improve this bound for random matrices. In particular we look into the Hilbert-Schmidt inner product $\iprod{\eta_{\tau}\Paren{\transpose{W}W-n \Id}, X}$. 

\begin{lemma}\label{lem:basic-sdp-bound-gaussian}
	Let $X\in \R^{d\times d}$ be a positive semidefinite matrix such that $\Tr X = 1$ and $\Normo{X}\leq k$. 
	Let $W\sim N(0,1)^{n\times d}$, then with probability $1-o(1)$
	\begin{align*}
	\Abs{\iprod{\transpose{W}W- n\Id,X}}\leq 
	O\Paren{\min\Set{k\sqrt{n\log\Paren{1 + \frac{d}{k^2} + \frac{d}{n+k\sqrt{n}}}},\, d + \sqrt{dn}}}.
	\end{align*}
	\begin{proof}
		By \cref{thm:bound-covariance-gaussian}, $\norm{\transpose{W}W- n\Id} \le O\Paren{d + \sqrt{dn}}$  with probability $1-d^{-10}$, so by \cref{lem:basic-sdp-iprod-matrix-spectral-norm}
		\[
		\Abs{\iprod{\transpose{W}W- n\Id,X}}\leq 
		O\Paren{d + \sqrt{dn}}.
		\]
		
		Let $D\subseteq [d]\times [d]$ be the set of diagonal entries of $\Paren{\transpose{W}W-n \Id}$ and $\bar{D}$ its complement. For any $\tau \ge 0$ we can rewrite the matrix $\Paren{\transpose{W}W-n \Id}$ as
		\begin{align*}
		\transpose{W}W-n \Id = \Paren{\transpose{W}W-n \Id }\Brac{D}+ \eta_\tau\Paren{\transpose{W}W-n \Id }\Brac{\bar{D}}+ \Paren{\transpose{W}W-n \Id -\eta_\tau\Paren{\transpose{W}W-n \Id }}\Brac{\bar{D}}.
		\end{align*} 
		Now, by \cref{fact:chi-squared-tail-bounds} with probability $1-o(1)$, $\Norm{\Paren{\transpose{W}W-n \Id }\Brac{D}}\leq 10\sqrt{n\log d}$.
		Furthermore,
		\begin{align*}
		\eta_\tau\Paren{\transpose{W}W-n \Id}[\bar{D}] = \zeta_\tau\Paren{\transpose{W}W-n \Id}[\bar{D}] + M\,,
		\end{align*}
		where $M\in \R^{d\times d}$ is a matrix with $\Normi{M}\leq \tau$ and by \cref{lem:spectral-norm-thresholded-gaussian} there is a constant 
		$C\ge 1$ such that 
		$\Norm{\zeta_\tau\Paren{\transpose{W}W-n \Id}[\bar{D}]}
		\leq C\Paren{d + \sqrt{dn}}\exp\Brac{-\frac{\tau^2}{Cn}}$ 
		with probability $1-o(1)$. 
		Let $\tau =10C\cdot \sqrt{n\log\Paren{1 + \frac{d}{k^2} + \frac{d}{n+k\sqrt{n}}}}$. If $d \le n$,
		\[
		\Norm{\zeta_\tau\Paren{\transpose{W}W-n \Id}[\bar{D}]} \le 
		3Ck\sqrt{n+k\sqrt{n}} \cdot \Paren{\frac{2\sqrt{dn}}{k\sqrt{n+k} + \sqrt{d(n+k)} +k\sqrt{d}}}
		\le 10Ck\sqrt{n}\,.
		\]
		If $k^2 \le n \le d$,
	     \[
		\Norm{\zeta_\tau\Paren{\transpose{W}W-n \Id}[\bar{D}]} \le 
		Ck^2\cdot 
		\Paren{\frac{2d}{k^2 +d}}
		\le 2Ck\sqrt{n}\,.
		\]
		And if $n\le \max\set{k^2, d}$,
		\[
		\Norm{\zeta_\tau\Paren{\transpose{W}W-n \Id}[\bar{D}]} \le 
		C(n+k\sqrt{n})\cdot 
		\Paren{\frac{2d}{n+k\sqrt{n}+d}}
		\le 4Ck\sqrt{n}\,.
		\]
		So, applying \cref{lem:basic-sdp-iprod-matrix-spectral-norm}, we get
		\[
		\Abs{\iprod{\eta_\tau\Paren{\transpose{W}W-n \Id}\Brac{\bar{D}},X}}
		\leq \Norm{\zeta_\tau\Paren{\transpose{W}W-n\Id}\Brac{\bar{D}}} + \Abs{\iprod{M, X}}
		\leq 2k \tau.
		\]
		Since $X$ is $k$-bounded,
		\[
		\Abs{\iprod{\Paren{\transpose{W}W-n \Id -
					\eta_\tau\Paren{\transpose{W}W-n \Id }}\Brac{\bar{D}}, X}}
			\leq k\tau.
			\]
		Hence with probability $1-o(1)$
		\begin{align*}
			\Abs{\iprod{\transpose{W}W- n\Id,X}}
			&\le
			30Ck \sqrt{n\log\Paren{2 + \frac{d}{k^2} + \frac{d}{n+k\sqrt{n}}}} + 10\sqrt{n\log d} 
			\\&\le 
			100Ck\sqrt{n\log\Paren{2 + \frac{d}{k^2} + \frac{d}{n+k\sqrt{n}}}}\,,
		\end{align*}
		since if $k \le \log d$, $\log\Paren{2 + \frac{d}{k^2}} \ge \frac{1}{2}\log d$.
	\end{proof}
\end{lemma}

\subsection{The basic SDP Algorithm}\label{sec:basic-sdp-as-algorithm}
Having providing certificates on sparse quadratic form, we can now use \cref{eq:sparse-pca-sdp-relaxation} to obtain a robust algorithm for Sparse PCA. 

\begin{mdframed}
	\begin{algorithm}[SDP-based Algorithm]
		\label[algorithm]{alg:spca-sdp-relaxation}\mbox{}
		\begin{description}
			\item[Input:]
			Sample matrix $Y=\sqrt \beta \cdot u_0v_0^T +W+E \in \R^{n \times d}$ from \ref{def:wishart-matrix-model}.
			\item[Estimate:]
			The sparse vector $v_0$.
			\item[Operation:]\mbox{}
			\begin{enumerate}
				\item 
				Compute matrix $X \in \R^{d\times d}$ solving program \ref{eq:sparse-pca-sdp-relaxation}.
				\item
				Output top eigenvector $\hat{v}$ of $X$. 
			\end{enumerate}
		\end{description}    
	\end{algorithm}
\end{mdframed}

Indeed  we will show that \cref{alg:spca-sdp-relaxation} is perturbation resilient (in the sense of \cref{sec:non-robust-algorithms}) and its guarantees matches those of the state-of-the-art \textit{fragile} algorithms such as SVD, Diagonal Thresholding and Covariance Thresholding. The following theorem formalize this result.

\begin{theorem}\label{thm:standard-sdp-estimation}
	Let $Y$ be a $n$-by-$d$ matrix of the form,
	\begin{align*}
	Y =\sqrt{\beta}\cdot u_0\transpose{v_0}+ W + E\,,
	\end{align*} 
	for a unit $k$-sparse vector $v_0\in \R^d$, a standard Gaussian vector $u_0\sim N(0,\Id_n)$, an arbitrary matrix $E\in \R^{n \times d}$ and a Gaussian matrix $W\sim N(0,1)^{n\times d}$ such that $W,u_0,$ are distributionally independent.
	Then
	algorithm \ref{alg:spca-sdp-relaxation}  outputs a unit vector $\hat{v}\in \R^d$ such that with probability $1-o(1)$, 
	\begin{align*}
		1 - \iprod{v_0, \hat{v}}^2
		\lesssim \frac{k}{\beta n}\cdot q +\sqrt{\frac{k}{\beta n}}\Paren{\sqrt{\log \frac{d}{k}}+\Norm{E}_{1\rightarrow 2}}\cdot \Paren{1+\frac{1}{\sqrt{\beta}}}\,.
	\end{align*}
	where $q:=	\min\Set{\sqrt{n\log\Paren{2 + \frac{d}{k^2} + \frac{d}{n+k\sqrt{n}}}}, \frac{d+\sqrt{dn}}{k}}$ and $\Norm{E}_{1\rightarrow 2}$ denotes the largest norm of a column of $E$. Furthermore, the same kind of guarantees hold if $u_0$ is a vector with $\norm{u_0}^2 = \Theta(n)$ independent of $W$.
\end{theorem}

We prove \cref{thm:standard-sdp-estimation} through the  result below, which will be  useful in the Sum-of-Squares proofs as well.

\begin{theorem}[Meta-theorem]\label{thm:meta-theorem}
	Let $Y$ be a $n$-by-$d$ matrix of the form,
	\begin{align*}
	Y =\sqrt{\beta}\cdot u_0\transpose{v_0}+ W + E\,,
	\end{align*} 
	for a unit $k$-sparse vector $v_0\in \R^d$, a standard Gaussian vector $u_0\sim N(0,\Id_n)$, an arbitrary matrix $E\in \R^{n \times d}$ and a Gaussian matrix $W\sim N(0,1)^{n\times d}$ such that $W,u_0,v_0$ are distributionally independent.
	Let $X$ be a  feasible solution of \ref{eq:sparse-pca-sdp-relaxation}  satisfying $\iprod{\hat{\Sigma},X}\geq \iprod{\hat{\Sigma}, v_0 \transpose{v_0}}$. Then with probability $1-o(1)$, 
	\begin{align*}
	1 - \iprod{v_0\transpose{v_0},X}
	\lesssim \frac{1}{\beta n}\cdot \Abs{\iprod{\transpose{W}W- n\Id,X}} 
	+\sqrt{\frac{k}{\beta n}}\Paren{\sqrt{\log \frac{d}{k}}+
		\Norm{E}_{1\rightarrow 2}}\cdot \Paren{1+\frac{1}{\sqrt{\beta}}}\,,
	\end{align*}
	where  $\Norm{E}_{1\rightarrow 2}$ denotes the largest norm of a column of $E$. Furthermore, the same kind of guarantees hold if $u_0$ is a vector with $\norm{u_0}^2 = \Theta(n)$ independent of $W$.
\end{theorem}

Indeed \cref{thm:meta-theorem} immediately implies \cref{thm:standard-sdp-estimation}.
\begin{proof}[Proof of \cref{thm:standard-sdp-estimation}]
	Assume \cref{thm:meta-theorem} is true. By definition $X$ satisfies its premises. By \cref{lem:basic-sdp-bound-gaussian}  
	\begin{align*}
		\Abs{\iprod{\transpose{W}W- n\Id,X}}\leq O\Paren{\min\Set{k\sqrt{n\log\Paren{1 + \frac{d}{k^2} + \frac{d}{n+k\sqrt{n}}}},\, d + \sqrt{dn}}}.
	\end{align*}
	Applying \cref{lem:linear-algebra-correlation-eigenverctor-large-quadratic-form} the result follows.
\end{proof}

Now let's prove \cref{thm:meta-theorem}. First we look into cross-terms containing the signal.

\begin{lemma}\label{lem:basic-sdp-bound-signal-with-gaussian}
	Let $Y$ be as in \cref{thm:meta-theorem} and suppose $E\in\R^{n\times d}$ is a matrix with maximal column norm $\Norm{E}_{1\rightarrow 2}\leq b$. Let $X$ be a feasible solution to \ref{eq:sparse-pca-sdp-relaxation}.
	Then with probability $1 - o(1)$,
	\[
		\Abs{\iprod{\transpose{W}\sqrt{\beta}u_0\transpose{v_0},X}}\leq O\Paren{\sqrt{\beta n k\log \frac{d}{k}}}\,.
	\]
	\begin{proof}
		With probability $1-o(1)$, $\norm{u_0} \le O(\sqrt{n})$.
		Let $g = \frac{1}{\norm{u_0}}\transpose{W}u_0$.
		Since $u_0$ and $W$ are independent, $g\sim N(0, 1)$.
		Let $S$ be the set of $k$ largest coordinates in $g$, and let $g' = g[S]$. 
		Then $g = g' + g''$, 
		where vector $g''$ has entries bounded by $O\Paren{\sqrt{\log\frac{d}{k}}}$ and 
		$\norm{g'} \le O\Paren{\sqrt{k\log\frac{d}{k}}}$ with probability
		$1-o(1)$ (by \cref{lemma:large-coordinates-gauusian-vector}). Hence by \cref{lem:basic-sdp-iprod-matrix-spectral-norm},
		\begin{align*}
		\Abs{\iprod{\transpose{W}\sqrt{\beta}u_0\transpose{v_0},X}}
		&\leq
		O\Paren{\sqrt{n\beta}
			\Abs{\iprod{g'\transpose{v_0},X}} + \sqrt{n\beta}\Abs{\iprod{g''\transpose{v_0},X}}}
		\\&\leq
		 O\Paren{\sqrt{\beta nk\log\frac{d}{k}}} + 
		O\Paren{\sqrt{n\beta}\Abs{\iprod{g''\transpose{v_0},X}}}\,.
		\end{align*}
		By \cref{lem:basic-sdp-c-s} and \cref{lem:basic-sdp-iprod-matrix-spectral-norm},
		\[
		\Abs{\iprod{g''\transpose{v_0},X}}\le
		 \sqrt{\iprod{g''\transpose{\Paren{g''}},X}\cdot\iprod{v_0\transpose{v_0},X}}
		 \le
		 \sqrt{\iprod{g''\transpose{\Paren{g''}},X}}\,.
		\]
		The desired bound follows from \cref{lem:basic-sdp-iprod-matrix-with-small-entries}, 
		since the entries of $g''\transpose{\Paren{g''}}$ are bounded by $O(\log\frac{d}{k})$ with probability $1-o(1)$.
	 \end{proof}
\end{lemma}

\begin{lemma}\label{lem:basic-sdp-bound-signal-with-adversary}
	Let $Y$ be as in \cref{thm:meta-theorem} and suppose $E\in\R^{n\times d}$ is a matrix with maximal column norm $\Norm{E}_{1\rightarrow 2}\leq b$. Let $X$ be a feasible solution to \ref{eq:sparse-pca-sdp-relaxation}.
	Then with probability  $1 - o(1)$,
	\[
	\Abs{\iprod{\transpose{E}\sqrt{\beta}u_0\transpose{v_0},X}}\leq O\Paren{b\sqrt{\beta n k}}.
	\]
	\begin{proof}
		With probability $1-o(1)$, $\norm{u_0} \le O(\sqrt{n})$.
		Let $z = \transpose{E} u_0$. With probability $1-o(1)$ the entries of $z$ are bounded by $O\Paren{b\sqrt{n}}$.
		By \cref{lem:basic-sdp-c-s} and \cref{lem:basic-sdp-iprod-matrix-spectral-norm},
		\[
		\Abs{\iprod{z\transpose{v_0},X}}\le
		\sqrt{\iprod{z\transpose{z},X}\cdot\iprod{v_0\transpose{v_0},X}}
		\le
		\sqrt{\iprod{z\transpose{z},X}}\le O\Paren{b\sqrt{nk}}\,.
		\]
	\end{proof}
\end{lemma}
The following lemma shows how to bound the remaining cross-terms.

\begin{lemma}\label{lem:basic-sdp-bound-gaussian-with-adversary}
	Let $Y$ be as in \cref{thm:meta-theorem} and suppose $E\in\R^{n\times d}$ is a matrix with maximal column norm $\Norm{E}_{1\rightarrow 2}\leq b$. Let $X$ be a feasible solution to \ref{eq:sparse-pca-sdp-relaxation}. Then
	\begin{align*}
	\Abs{\iprod{\transpose{E}W+\transpose{W}E,X}}
	\leq
		2b\sqrt{kn} + b^2k + \Abs{\iprod{\transpose{W}W- n\Id,X}}\,.
	\end{align*}
	\begin{proof}
		Applying \cref{fact:product-of-psd-matrices} with setting $A=\transpose{\Paren{W-c\cdot E}}\Paren{W-c\cdot E}$ for some  $c>0$ and $B=X$ we immediately get 
		\begin{align*}
		c\Abs{\iprod{\transpose{E}W+\transpose{W}E,X}}\leq& \iprod{\transpose{W}W,X}+c^2\iprod{\transpose{E}E,X} =
		n+\iprod{\transpose{W}W-n\Id,X}+c^2\iprod{\transpose{E}E,X}\,.
		\end{align*}
		By  \cref{lem:basic-sdp-iprod-matrix-with-small-entries}
		\begin{align*}
		\Abs{\iprod{\transpose{E}W+\transpose{W}E,X}}
			\leq \frac{1}{c}\Paren{n + \Abs{\iprod{\transpose{W}W- n\Id,X}}} + c\cdot b^2 k\,.
		\end{align*}
		Minimizing over $c$, we get
		\begin{align*}
	\Abs{\iprod{\transpose{E}W+\transpose{W}E,X}}
		&\leq
		2b\sqrt{kn + k\cdot\Abs{\iprod{\transpose{W}W- n\Id,X}}}
		\\&\leq 
		2b\sqrt{kn} + 2b\sqrt{k\cdot \Abs{\iprod{\transpose{W}W- n\Id,X}}}
		\\&\leq
		2b\sqrt{kn} + b^2k + \Abs{\iprod{\transpose{W}W- n\Id,X}}\,.
		\end{align*}
	\end{proof}
\end{lemma}

We are now ready to prove \cref{thm:meta-theorem}.

\begin{proof}[Proof of \cref{thm:meta-theorem}]
	Opening up the product,
	\begin{align*}
	\iprod{\hat{\Sigma},X}=&\iprod{\hat{\Sigma}-n \Id +n\Id,X}\\
	= &\beta \Snorm{u_0}\iprod{v_0\transpose{v_0},X}+n\\
	&+\iprod{\transpose{W}W-n \Id, X}\\
    &+\iprod{\transpose{E}E, X}\\
	&+\iprod{\transpose{E}W+\transpose{W}E,X}\\
	&+\sqrt{\beta}\iprod{v_0\transpose{u_0}W+\transpose{W}u_0\transpose{v_0}+v_0\transpose{u_0}E+\transpose{E}u_0\transpose{v_0},X}.
	\end{align*}
	
	Applying Lemmata \ref{lem:basic-sdp-iprod-matrix-with-small-entries}, \ref{lem:basic-sdp-bound-gaussian}, \ref{lem:basic-sdp-bound-gaussian-with-adversary},
	  \ref{lem:basic-sdp-bound-signal-with-adversary},  \ref{lem:basic-sdp-bound-signal-with-gaussian} and we get
    \begin{align*}
	\iprod{\hat{\Sigma},X}
	\leq
	\beta \Snorm{u_0}  \iprod{v_0\transpose{v_0},X}+n
	+ 
	2\Abs{\iprod{\transpose{W}W- n\Id,X}} + 2b^2k + 2b\sqrt{kn}
	+O\Paren{\Paren{\sqrt{\log \frac{d}{k}} + b}\sqrt{\beta n k}}.
	\end{align*}
	
	Furthermore, by choice of $X$,
	\begin{align*}
	\iprod{\hat{\Sigma},X}\geq &\iprod{\hat{\Sigma},v_0\transpose{v_0}}\\
	=&\iprod{\hat{\Sigma}+n\Id-n\Id,v_0\transpose{v_0}}\\
	\geq
	&\beta \Snorm{u_0} +n\\
	&-\Abs{\iprod{\transpose{W}W-n \Id, v_0\transpose{v_0}}}\\
	&-\Abs{\iprod{\transpose{E}E, v_0\transpose{v_0}}}\\
	&-\Abs{\iprod{\transpose{E}W+\transpose{W}E,v_0\transpose{v_0}}}\\
	&-\Abs{\sqrt{\beta}\iprod{v_0\transpose{u_0}W+\transpose{W}u_0\transpose{v_0}+v_0\transpose{u_0}E+\transpose{E}u_0\transpose{v_0},v_0\transpose{v_0}}}\\
	\geq& \beta \Snorm{u_0} + n - 
	2\Abs{\iprod{\transpose{W}W- n\Id,v_0\transpose{v_0}}} - 2b^2k - 2b\sqrt{kn}
	-O\Paren{\Paren{\sqrt{\log \frac{d}{k}} + b}\sqrt{\beta n k}}\\
	\geq& \beta \Snorm{u_0} + n - 
	O\Paren{\Abs{\iprod{\transpose{W}W- n\Id,v_0\transpose{v_0}}} - b^2k - b\sqrt{kn}
    -\Paren{\sqrt{\log \frac{d}{k}} + b}\sqrt{\beta n k}}\,.
	\end{align*}
	Now by \cref{theorem:k-sparse-norm-gaussian} $\Abs{\iprod{\transpose{W}W- n\Id,v_0\transpose{v_0}}}\le 10k\log (d/k)+20\sqrt{nk\log (d/k)}$ with probability $1-o(1)$.
	Let $m = \Abs{\iprod{\transpose{W}W- n\Id,X}} + 
	\Abs{\iprod{\transpose{W}W- n\Id,v_0\transpose{v_0}}}$.
	Combining the two inequalities and rearranging, we get
	\begin{align*}
	\beta \Snorm{u_0}  \cdot \Paren{1 - \iprod{v_0\transpose{v_0},X}}
	\le O\Paren{m+ b^2k + b\sqrt{kn} + \Paren{\sqrt{\log \frac{d}{k}} + b}\sqrt{\beta n k}}\,. 
	\end{align*}
	With probability $1-o(1)$, $\Snorm{u_0} \ge n/2$. Recall that $b \le \sqrt{\frac{\beta n}{k}}$. Hence
	\begin{align*}
	{1 - \iprod{v_0\transpose{v_0},X}}
	\le \frac{1}{\beta n}O\Paren{m+ \sqrt{\beta n k\log \frac{d}{k}} + \Paren{1+\beta} b\sqrt{kn}}\,.
	\end{align*}
	The result follows rearranging and observing that with probability $1-o(1)$, 
	$\Abs{\iprod{\transpose{W}W- n\Id,v_0\transpose{v_0}}}\lesssim k\log(d/k)+\sqrt{kn\log(d/k)}$ by Lemma \ref{theorem:k-sparse-norm-gaussian}.
\end{proof}

\section{Resilience of SoS and Stronger Certified Upper Bounds}\label{sec:sos}
In this section we prove \cref{thm:sos-algorithm-robust-spca-small-n} and \cref{thm:weak-signal-regime-sos}. We will show that the Sum-of-Squares algorithm can certify various upper bounds on sparse eigenvalues. 
In \cref{sec:sos-certificate-few-samples} we will prove increasingly stronger certified upper bounds on sparse eigenvalues of subgaussian matrices.
These certified upper bounds will require increasingly stronger assumptions on $d$ and $n$, but for degree $\log (d/k)$ will approach information theoretic guarantees. 
In \cref{sec:basic-sos-weak-signal-regime} we will prove alternative certified upper bounds fo sparse eigenvalues of Gaussian matrices.  These bounds will not require any additional assumption on $d$ and $n$. We will then use these bounds in  \cref{sec:sos-algorithms} to obtain  maximally robust algorithms for Sparse PCA.

\subsection{SoS Certificates for Sparse Eigenvalues via Certifiable Subgaussianity}\label{sec:sos-certificate-few-samples}
Let $\cA_{s,v}$ be the following system of quadratic constraints. 
Observe for any $(s,v)$ satisfying $\cA_{s,v}$, $v$ is a $k$-sparse unit vector supported on coordinates $i$ such that $s_i =1$.
\begin{equation}
  \cA_{s,v}\colon
  \left \{
    \begin{aligned}
      &&
      \textstyle\sum_{i=1}^d s_i
      &= k\\
      &\forall i\in [d].
      & s_i^2
      & =s_i \\
      &\forall i\in [d].
      & s_i \cdot v_i 
      & = v_i\\
      &
      & \sum_{i = 1}^d v_i^2
      &=1\\
    \end{aligned}
  \right \}
\end{equation}

We prove a certified upper bound for sparse eigenvalues of random rectangular matrices $W \in \R^{n \times d}$ with independent subgaussian entries. This upper bound differs considerably from the one obtained using \ref{eq:sparse-pca-sdp-relaxation}. Let us recall the definition of subgaussian random variables before proceeding.

\begin{definition}[$C$-Subgaussian Random Variables]
A $\R$-valued random variable $x$ is said to be $C$-subgaussian if for every $t$, $\E \abs{x}^{t} \leq C^{t/2} t^{t/2}$. 
\end{definition}

Let $W_1, W_2, \ldots, W_d$ be the columns of $W$. We will use the following lemma:
\begin{lemma}\label{lem:subgaussianity-sampling}
Let $W_1, W_2, \ldots, W_d \in \R^n$ be independently drawn from a product distribution with each $1$-subgaussian coordinates with mean $0$ and variance $1$. Then, with probability at least $0.99$ over the draw of $W_1,W_2,\ldots, W_d$, 
\[
\sststile{2t}{u} \Set{ \frac{1}{d} \sum_{i \leq d} \Iprod{W_i,u}^{2t} \leq  \Norm{u}_2^{2t} 
	\Paren{t^{t} + \frac{n^{t/2} \log^{(t+1)/2}{(n)} \Paren{C't}^{t}}{\sqrt{d}}}}\mper 
\]
for some absolute constant $C' > 0$.
\end{lemma}

We will prove the lemma whenever the columns of $W$ are \emph{certifiably subgaussian}.
Informally, certifiably subgaussianity means that a random variable has its moments upper-bounded as in the the definition above and that this bound has a SoS proof. Formally, we have:

\begin{definition}[Certifiable Subgaussianity]
A $\R^n$-valued random variable $Y$ is said to be $t$-certifiably $C$-subgaussian if for all $t' \leq t$, $\sststile{2t}{u} \Set{\E \iprod{Y,u}^{2t} \leq C^t t^t \Paren{\E \iprod{Y,u}^2}^t}$. A matrix $W \in \R^{n \times d}$ is said to be $t$-certifiably $C$-subgaussian if the uniform distribution on the columns of $W$ is $t$-certifiably $C$-subgaussian. 
\end{definition}

Certifiable subgaussianity has, by now, appeared in several works~\cite{DBLP:journals/corr/abs-1711-11581,DBLP:journals/corr/abs-1711-07465,DBLP:conf/stoc/Hopkins018,DBLP:conf/colt/KlivansKM18} that employ the sum-of-squares method for statistical estimation problems. 

Given the above lemma, to prove Lemma~\ref{lem:subgaussianity-sampling}, we need to show certified subgaussianity of $W$ when $W$ is a random matrix in $\R^{n \times d}$. To show this, we will use the following fact:

\begin{fact}[Certifiable Subgaussianity of Product Subgaussians, Lemma~5.9, Page 25 of~\cite{DBLP:journals/corr/abs-1711-11581}] \label{fact:certifiable-subgaussianity-product}
Let $Y$ be a $\R^d$-valued random variable with independent, $C$-subgaussian coordinates of mean $0$ and variance $1$. Then, $Y$ is $t$-certifiably $C$-subgaussian for every $t$. 
\end{fact}

We are now ready to prove Lemma~\ref{lem:subgaussianity-sampling}.
\begin{proof}[Proof of Lemma~\ref{lem:subgaussianity-sampling}]
We have:
\[
\sststile{2t}{u} \Set{ \frac{1}{d} \Paren{\sum_{i \leq d} \Iprod{W_i,u}^{2t} - \E \Iprod{W_i,u}^{2t}} = \Iprod{u^{\otimes t}, \Paren{\frac{1}{d} \sum_{i \leq d} \Paren{W_i^{\otimes t}}\Paren{W_i^{\otimes t}}^{\top} - \E\Paren{W_i^{\otimes t}}\Paren{W_i^{\otimes t}}^{\top}} u^{\otimes t}}} \mper
\]

Using Fact~\ref{fact:spectral-certificates} and $\Norm{u^{\otimes t}}_2^2 = \Norm{u}_2^{2t}$, we have:
\begin{equation}
\sststile{2t}{u} \Set{ \frac{1}{d} \sum_{i \leq d} \Iprod{W_i,u}^{2t} - \E \Iprod{W_i,u}^{2t} \leq \Norm{u}_2^{2t} \cdot \Norm{\Paren{\frac{1}{d} \Paren{W_i^{\otimes t}}\Paren{W_i^{\otimes t}}^{\top} - \E\Paren{W_i^{\otimes t}}\Paren{W_i^{\otimes t}}^{\top}}} } \mper
\label{eq:sampling-bound-1}
\end{equation}
From Lemma~\ref{lem:matrix-concentration}, we know that with probability at least $0.99$ over the draw of $W_1, W_2,\ldots, W_d$, it holds that:

\begin{equation} \label{eq:sampling-bound-2}
\Norm{\Paren{\frac{1}{d} \Paren{W_i^{\otimes t}}\Paren{W_i^{\otimes t}}^{\top} - \E\Paren{W_i^{\otimes t}}\Paren{W_i^{\otimes t}}^{\top}}} \leq 
\frac{n^{t/2} \log^{(t+1)/2}{(n)} \Paren{C't}^{t}}{\sqrt{d}} \mper
\end{equation}

Using Fact~\ref{fact:certifiable-subgaussianity-product}, 
\begin{equation} \label{eq:sampling-bound-3}
\sststile{2t}{u} \Set{\E \Iprod{W_i,u}^{2t} \leq t^{t} \Norm{u}_2^{2t}}\mper
\end{equation}

Combining \eqref{eq:sampling-bound-1}, \eqref{eq:sampling-bound-2} and \eqref{eq:sampling-bound-3}, we have:

\[
\sststile{2t}{u} \Set{ \frac{1}{d} \sum_{i \leq d} \Iprod{W_i,u}^{2t}\leq 
	\Norm{u}_2^{2t} 
	\Paren{t^{t} + \frac{n^{t/2} \log^{(t+1)/2}{(n)} \Paren{C't}^{t}}{\sqrt{d}}}}\mper 
\]
\end{proof}

\cref{lem:subgaussianity-sampling} implies the following lemma:

\begin{lemma} \label{lem:W-bound-partial}
	Let $W$ satisfy the assumptions of \cref{lem:subgaussianity-sampling}. Suppose that $d\ge t^t n^t\log^{(t+1)}(n)$. Then with probability at least $0.99$,
	
	\[
	\cA_{s,v}  \sststile{2t}{s,v} 
	\Set{\Norm{Wv}_2^{4t} \leq  dk^{t-1}(C't)^t \Norm{Wv}_2^{2t}} \mper
	\]
	for some absolute constant $C' > 0$.
\end{lemma}
\begin{proof}
	
	
	For $u = Wv$, using $\cA_{s,v} \sststile{2t}{s} \set{s_iv_i = v_i\mid \forall i}$ and Cauchy-Schwarz inequality, we have:
	\begin{align*}
	\cA_{s,v} &\sststile{2t}{s,v,u} \Set{ \Paren{\sum_{i\leq d} s_i v_i \iprod{W_i,u}}^{2t} \leq \Paren{\sum_{i\leq d} v_i^2}^t \Paren{\sum_{i \leq d} s_i^2 \iprod{W_i,u}^2}^{t}}
	\end{align*}
	
	Using $\cA_{s,v} \sststile{2t}{s} \set{s_i^{t-1} = s_i^2 \mid \forall i}$, we have:
	\begin{align*}
	\cA_{s,v} &\sststile{t}{s,v,u} \Set{ \Paren{\sum_{i\leq d} s_i v_i \iprod{W_i,u}}^{2t} \leq \Paren{\sum_{i\leq d} v_i^2}^t \Paren{\sum_{i \leq d} s_i^2 \iprod{W_i,u}^2}^{t}}
	\end{align*}
	Now, using $\cA_{s,v} \sststile{2}{s,v}\set{\sum_i s_i = k}$ and \cref{lem:subgaussianity-sampling}, we have:

	\begin{align}
	\cA_{s,v} \sststile{t}{s,v,u}
	\begin{aligned}[t]
	\Bigl\{& \Paren{\sum_{i\leq d} s_i \iprod{W_i,u}^2}^{t} = \Paren{\sum_{i\leq d} s^{t-1}_i \iprod{W_i,u}^2}^{t} \leq \Paren{\sum_{i\leq d} s^{t}_i}^{t-1} \Paren{\sum_{i\leq d}\iprod{W_i,u}^{2t}}\\
	&\leq k^{t-1} d
	\Norm{u}_2^{2t} 
	\Paren{t^{t} + \frac{n^{t/2} \log^{(t+1)/2}{(n)} \Paren{C't}^{t}}{\sqrt{d}}} 
	\Bigr\} \label{eq:W-bound-part-2}
	\end{aligned}
	\end{align}
	Plugging back $u = Wv$, we get the desired bound.
\end{proof}

Now we are ready to derive the certified upper bound on $\Snorm{Wv}_2$.

\begin{lemma} \label{lem:bounding-errors}
	Suppose that $d\ge C^* t^t n^t\log^t(n)$ for large enough absolute constant $C^*$.
	Let $D$ be a pesudo-distribution satisfying $\cA_{s,v}$.
	Let $W \in \R^{n \times d}$ with i.i.d. $1$-subgaussian entries with mean $0$ and variance $1$.
	Then, with probability at least $0.99$ over the draw of $W_1, W_2, \ldots, W_d$, 
	\[
	\pE_{D} \Norm{Wv}^2_2 \leq C' \cdot d^{1/t} k^{1-\frac{1}{t}}  t\,,
	\]
	for some absolute constant $C'>0$.
\end{lemma}
\begin{proof}
	Using \cref{lem:W-bound-partial} and taking pseudo-expectations with respect to $D$ that satisfies $\cA_{s,v}$, we have:
	\[
	\pE_D \Norm{Wv}_2^{4t} \leq d k^{t-1} (C't)^t \pE_D \Norm{Wv}_2^{2t}\mper
	\]
	By Cauchy-Schwarz inequality for pseudo-distributions,  
	$\pE_D \Norm{Wv}_2^{2t} \leq \Paren{\pE_D \Norm{Wv}_2^{4t}}^{1/2}$, and by H\"older's indequality
	$\Paren{\pE_D \Norm{Wv}_2^{2}}^{2t} \leq \pE_D \Norm{Wv}_2^{4t}$.
	Thus, we have:
	\[
	\Paren{\pE_D \Norm{Wv}_2^2}^{t} \leq d k^{t-1} (C't)^{t} \mper
	\]
	Taking $t$-th roots gives: $\pE_D \Norm{Wv}_2^2 \leq C'\cdot d^{1/t} k^{1-\frac{1}{t}} t$. 
	%
	%
\end{proof}






\subsection{SoS Certificates for Sparse Eigenvalues via Limited Brute Force}\label{sec:basic-sos-weak-signal-regime}
We show here that, using additional constraints over the system $\cA_{s,v}$, we can provide different certified upper bounds on the sparse eigenvalues of Gaussian matrices $W$.

%
%
Let $\cS_t$ be a set of all vectors with values in $\set{0,1}$ that have exactly $t$ nonzero coordinates.

%
%
We start with a definition.

\begin{definition}
	For any $u\in \cS_t$ we define a polynomial in variables $s_1,\ldots,s_d =: s$
	\[
	p_u(s) = \binom{k}{t}^{-1}\cdot \underset{i \in \supp\Set{u}}{\prod} s_i\,.
	\]
\end{definition}

Note that if $v$ denotes a $k$-sparse vector and $s$ is the indicator of its support, then for any
$u\in \cS_t$,
\begin{align*}
p_u(s)=\begin{cases}
\binom{k}{t}^{-1}&\text{ if} \supp\Set{u}\subseteq\supp\Set{v} \\
0& \text{ otherwise}
\end{cases}
\end{align*}

Now consider the following system $\cB_{s,v}$ of polynomial constraints.  

\begin{equation}
\cB_{s,v}\colon
\left \{
\begin{aligned}
&\forall i\in [d],
& s_i^2
& =s_i \\
&&\textstyle \underset{i \in [d]}{\sum}s_i&=k\\
&\forall i \in [d], &s_i\cdot v_i &=v_i\\
&&\textstyle \underset{i \in [d]}{\sum}v_i^2&=1\\
&&\underset{u \in \cS_t}{\sum} p_u(s) &=1\\
&\forall i\in [d],
&\underset{u \in \cS_t}{\sum} u_ip_u(s) &=\frac{t}{k} \cdot s_i
\end{aligned}
\right \}
\end{equation}


We will use the following preliminary fact.
\begin{fact}\label{fact:principal-submatrices-norm-wishart}
	Let $W\sim N(0,1)^{n\times d}$, let $n\geq \log d$ and let $t\le \frac{1}{\log d}\min\set{d,n}$. 
	Then with probability $1-o(1)$ all principle submatrices of $\transpose{W}W - n\Id$ of size $t\times t$ have spectral norm bounded by $O\Paren{\sqrt{nt\log d}}$.
	\begin{proof}
		Fix a $t \times t$ principal submatrix $N$. By \cref{thm:bound-covariance-gaussian} there exists a constant $C>0$, such that $\Norm{N}\leq C \sqrt{nt\log d}$
		with probability at least $1-d^{10t}$.
		The fact follows taking a union bound over all possible $\binom{d}{t}$ submatrices.
	\end{proof}
\end{fact}

We are now ready to show the upper bound on quadratic forms of sparse vectors.

\begin{theorem}\label{theorem:ceritificate-gaussian-quadratic-form}
	Let $W\sim N(0,1)^{n\times d}$. Then there exists a constant $C>0$ such that with probability at least $1-o(1)$,
	\begin{align*}
	\cB_{s,v} &\sststile{4t}{s,v}
	\Set{-C\cdot{\frac{k}{\sqrt{t}}\sqrt{n\log d}}
			\le \transpose{v}\Paren{\transpose{W}W-n \Id}v \leq C\cdot{\frac{k}{\sqrt{t}}\sqrt{n\log d}}}\,.
	\end{align*}
	\begin{proof}
		Note that 
		\[
		\cB_{s,v} \sststile{2t}{s,v} \Set{s\transpose{s} = \frac{k^2}{t^2}\underset{u,u' \in \cS_t}{\sum} u'\transpose{u}p_{u'}(s)p_u(s)}\,.
		\]
		For vectors $x,y\in \R^d$ we denote the vector with entries $x_i\cdot y_i$ by $\Paren{xy}$.
		It follows that
		\begin{align*}
		\cB_{s,v} &\sststile{4t}{s,v} \Set{v\transpose{v} = \Paren{v s}\transpose{\Paren{v s}} }\\
		\cB_{s,v} &\sststile{4t}{s,v} \Set{v\transpose{v} =
			\frac{k^2}{t^2}\underset{u ,u'\in \cS_t}{\sum} 
			\Paren{v u'}\transpose{\Paren{v u}}p_{u'}(s)p_u(s)}\,.
		\end{align*}
		
		Let $M=\Paren{\transpose{W}W-n \Id}$. Then 
		\begin{align*}
		\cB_{s,v} &\sststile{4t}{s,v}\Set{{\transpose{v}Mv}  = \iprod{M,v\transpose{v}}}
		\\
		&\sststile{4t}{s,v}\Set{{\transpose{v}Mv}  = \iprod{M,
				\frac{k^2}{t^2}\underset{u ,u'\in \cS_t}{\sum} 
				\Paren{v u'}\transpose{\Paren{v u}}p_{u'}(s)p_u(s)}}		
		\\
		&\sststile{4t}{s,v}\Set{{\transpose{v}Mv}  = 
			\frac{k^2}{t^2}\underset{u ,u'\in \cS_t}{\sum} 
			\transpose{\Paren{v u}}M\Paren{v u'}p_{u'}(s)p_u(s)}
		\end{align*}
		
		Now for any $u,u' \in \cS_t$,
		\begin{align*}
		\cB_{s,v} &\sststile{2}{s,v}\Set{\transpose{\Paren{v u}}M\Paren{v u'} = 
			\transpose{\Paren{v u}}\transpose{W}W\Paren{v u'} - n{\transpose{\Paren{v u}}\Paren{v u'}}}
		\\
		&\sststile{2}{s,v}\Set{2\transpose{\Paren{v u}}M\Paren{v u'} \le
			\transpose{\Paren{v u}}\transpose{W}W\Paren{v u} +
			\transpose{\Paren{v u'}}\transpose{W}W\Paren{v u'}
			- 2n{\transpose{\Paren{v u}}\Paren{v u'}}}		
		\\
		&\sststile{2}{s,v}\Set{2\transpose{\Paren{v u}}M\Paren{v u'} \le
			\transpose{\Paren{v u}}M\Paren{v u} +\transpose{\Paren{v u'}}M\Paren{v u'} + 
			n\cdot\Paren{\snorm{\Paren{v u}} + \snorm{\Paren{v u'}} 
				- 2{\transpose{\Paren{v u}}\Paren{v u'}}}}
		\\
		&\sststile{2}{s,v}\Set{\transpose{2\Paren{v u}}M\Paren{v u'} \le
			\transpose{\Paren{v u}}M\Paren{v u} +\transpose{\Paren{v u'}}M\Paren{v u'}}\,.
		\end{align*}
		where the first equality follows by definition, the second using the fact that for any $N\sge 0$, and $a,b\in \R^d$, $\sststile{2}{s,v}\Set{\iprod{\Paren{a-b}\transpose{\Paren{a-b}}, N}\geq 0  } $. The last follows from the fact that $\sststile{2}{s,v}\Set{\Snorm{a-b}\geq 0  } $.
		Similar derivation shows that $\cB_{s,v} \sststile{4t}{s,v} \Set{-\transpose{2\Paren{v u}}M\Paren{v u'} \le
			\transpose{-\Paren{v u}}M\Paren{v u} -\transpose{\Paren{v u'}}M\Paren{v u'}}$.
		
		Now let $q$ be the maximal norm of any $t\times t$ principal submatrices of $M$. Note that for any $u\in\cS_t$, 
		$\card{\supp\set{\Paren{vu}}}\le t$. 
		Since $\cB_{s,v} \sststile{4t}{s,v} \Set{s_i \ge 0}$,  	$\cB_{s,v} \sststile{4t}{s,v} \Set{p_u(s) \ge 0}$,
		\begin{align*}
		\cB_{s,v} &\sststile{4t}{s,v}\Set{\transpose{v}Mv 
			\le
			\frac{k^2}{2t^2}\underset{u ,u'\in \cS_t}{\sum} 
			\Paren{	\transpose{\Paren{v u}}M\Paren{v u} +\transpose{\Paren{v u'}}M\Paren{v u'}}p_{u'}(s)p_u(s)}
		\\
		&\sststile{4t}{s,v}\Set{\transpose{v}Mv
			\le
			\frac{k^2}{2t^2}\underset{u ,u'\in \cS_t}{\sum} 
			\Paren{	q\cdot \snorm{\Paren{v u}}+
				q\cdot\snorm{\Paren{v u'}}}p_{u'}(s)p_u(s)}
		\\
		&\sststile{4t}{s,v}\Set{\transpose{v}Mv
			\le
			q\frac{k^2}{2t^2}\Paren{
				\underset{u\in \cS_t}{\sum} \snorm{\Paren{v u}} p_u(s) \Paren{\underset{u' \in \cS_t}{\sum}p_{u'}(s)}
				+  	\underset{u'\in \cS_t}{\sum} \snorm{\Paren{v u'}} p_{u'}(s)  \Paren{\underset{u\in \cS_t}{\sum}p_{u}(s)}
		}}
		\\
		&\sststile{4t}{s,v}\Set{\transpose{v}Mv
			\le
			q\frac{k^2}{2t^2}\Paren{
				\underset{u\in \cS_t}{\sum} \snorm{\Paren{v u}} p_u(s) 
				+  	\underset{u'\in \cS_t}{\sum} \snorm{\Paren{v u'}} p_{u'}(s) 
		}}
		\\
		&\sststile{4t}{s,v}\Set{\transpose{v}Mv
			\le
			q\frac{k^2}{t^2}
			\underset{u\in \cS_t}{\sum} \snorm{\Paren{v u}} p_u(s).
		}
		\end{align*}
		Here the second inequality follows from choice of $q$, the third uses the fact that $\cB_{s,v} \sststile{4t}{s,v}\Set{\Paren{\underset{u\in \cS_t}{\sum}p_{u}(s)}=1}$. Finally observe that
		\begin{align*}
		\cB_{s,v} &\sststile{4t}{s,v}\Set{
			\underset{u\in \cS_t}{\sum} \snorm{\Paren{v u}} p_u(s) =
			\underset{u \in \cS_t}{\sum} \sum_{i=1}^d v_i^2 u_i^2 \cdot p_u(s)}
		\\
		&\sststile{4t}{s,v}\Set{
			\underset{u\in \cS_t}{\sum} \snorm{\Paren{v u}} p_u(s) =
			\sum_{i=1}^d v_i^2 \underset{u \in \cS_t}{\sum} u_i \cdot p_u(s)}
		\\
		&\sststile{4t}{s,v}\Set{
			\underset{u\in \cS_t}{\sum} \snorm{\Paren{v u}} p_u(s) =
			\frac{t}{k}\sum_{i=1}^d v_i^2 s_i}
		\\
		&\sststile{4t}{s,v}\Set{
			\underset{u\in \cS_t}{\sum} \snorm{\Paren{v u}} p_u(s) =
			\frac{t}{k}}\,,
		\end{align*}
		where we used the facts $cB_{s,v} \sststile{2}{s,v}\Set{v_i = v_i\cdot s_i}$ and $cB_{s,v} \sststile{4t}{s,v}\Set{\underset{u \in \cS_t}{\sum} u_ip_u(s) =\frac{t}{k} \cdot s_i}$.
		By \cref{fact:principal-submatrices-norm-wishart}, there exists an absolute constant $C>0$ such that with probability $1-o(1)$, 
		$q \le C\sqrt{nt\log t}$. Hence with probability  $1-o(1)$, 
		\[
		\cB_{s,v} \sststile{4t}{s,v}\Set{\transpose{v}Mv  \le C{\frac{k}{\sqrt{t}}\sqrt{n\log d}}}\,.
		\]
		Similar derivation shows that
		\[
		\cB_{s,v} \sststile{4t}{s,v}\Set{-\transpose{v}Mv  \le C{\frac{k}{\sqrt{t}}\sqrt{n\log d}}}\,.
		\]
	\end{proof}
\end{theorem}

\subsection{SoS Algorithms}\label{sec:sos-algorithms}
We now use the certified upper bounds from the previous sections to obtain efficient algorithms for Sparse PCA with adversarial errors, thus proving \cref{thm:sos-algorithm-robust-spca-small-n} and \cref{thm:weak-signal-regime-sos} which we formally restate.

\begin{theorem}
  \label{thm:algorithm-robust-spca-technical}
  Suppose $d \gtrsim  n^t \log^{t}{(n)} t^t$ for $t\in\N$.
  Let $Y$ be an $n$-by-$d$ matrix of the form,
  \begin{align*}
  Y =\sqrt{\beta}\cdot u_0\transpose{v_0}+ W + E\,,
  \end{align*} 
  for a unit $k$-sparse vector $v_0\in \R^d$, a standard Gaussian vector $u_0\sim N(0,\Id_n)$, an arbitrary matrix $E\in \R^{n \times d}$ and a Gaussian matrix $W\sim N(0,1)^{n\times d}$ such that $W,u_0$ are distributionally independent.
  Then we can compute in time $d^{O(t)}$ a unit vector $\hat v\in\R^d$ such that with probability at least $0.99$,
  \begin{displaymath}
    1-\iprod{\hat v,v_0}^2
    \lesssim \frac{k}{\beta n}  \cdot t\cdot \Paren{\tfrac d k}^{1/t} + \frac{1}{\beta}
     +\sqrt{\frac{k}{\beta n}}\Paren{\sqrt{\log \frac{d}{k}}+\Norm{E}_{1\rightarrow 2}}\cdot \Paren{1+\frac{1}{\sqrt{\beta}}}\,,
  \end{displaymath}
  where $\Norm{E}_{1\rightarrow 2}$ denotes the largest norm of a column of $E$.
  Furthermore, the same kind of guarantees hold if $u_0$ is a vector with $\norm{u_0}^2 = \Theta(n)$ independent of $W$.
\end{theorem}

\begin{theorem}\label{thm:main-sos-weak-signal}
	Suppose $n\gtrsim \log d$ and $t\leq k$.
	Let $Y$ be an $n$-by-$d$ matrix of the form,
	\begin{align*}
	Y =\sqrt{\beta}\cdot u_0\transpose{v_0}+ W + E\,,
	\end{align*} 
	for a unit $k$-sparse vector $v_0\in \R^d$, a standard Gaussian vector $u_0\sim N(0,\Id_n)$, an arbitrary matrix $E\in \R^{n \times d}$ and a Gaussian matrix $W\sim N(0,1)^{n\times d}$ such that $W,u_0$ are distributionally independent.
	Then we can compute in time $n^{O(1)}d^{O(t)}$ a unit vector $\hat v\in\R^d$ such that with probability $1-o(1)$,
	\begin{align*}
	1-\iprod{\hat v,v_0}^2
	\lesssim \frac{k}{\beta}  \cdot\sqrt{\frac{\log d}{nt}} +\sqrt{\frac{k}{\beta n}}\Paren{\sqrt{\log \frac{d}{k}}+\Norm{E}_{1\rightarrow 2}}\cdot \Paren{1+\frac{1}{\sqrt{\beta}}}\,,
	\end{align*}
	where $\Norm{E}_{1\rightarrow 2}$ denotes the largest norm of a column of $E$.
	Furthermore, the same kind of guarantees hold if $u_0$ is a vector with $\norm{u_0}^2 = \Theta(n)$ independent of $W$.
\end{theorem}

We will prove \cref{thm:algorithm-robust-spca-technical} and \cref{thm:main-sos-weak-signal} using Algorithm~\ref{alg:sparse-pca-program}.

\begin{mdframed}
  \begin{algorithm}[Algorithm for Sparse PCA with Adversarial Corruptions]
    \label[algorithm]{alg:sparse-pca-program}\mbox{}
    \begin{description}
    \item[Given:]
       Sample matrix $Y=\sqrt \beta \cdot u_0v_0^T +W+E \in \R^{n \times d}$ from model \ref{def:wishart-matrix-model}, system $\cC_{sv}\in \Set{\cA_{s,v}, \cB_{s,v}}$
    \item[Estimate:]
      The sparse vector $v_0$.
    \item[Operation:]\mbox{}
      \begin{enumerate}
      \item 
        find a level-$4t$ pseudo-distribution $D$ that satisfies $\cC_{s,v}$ and maximizes $\pE \Norm{Yv}_2^2$.
      \item
        Output a top eigenvector $\hat{v}$ of $\pE v\transpose{v}$. 
      \end{enumerate}
    \end{description}    
  \end{algorithm}
\end{mdframed}

Let us analyze the algorithm. The first observation is that any pseudo-distribution satisfying $\cB_{s,v}$ also satisfies $\cA_{s,v}$. 
Next we show that any pseudo-distribution satisfying $\cA_{s,v}$ is a feasible solution to \ref{eq:sparse-pca-sdp-relaxation}. This will allows us to use \cref{thm:meta-theorem} and conclude the proofs of \cref{thm:algorithm-robust-spca-technical} and \cref{thm:main-sos-weak-signal}.

\begin{lemma}\label{lem:sos-feasible-sdp}
	Let $D$ be any pseudo-distribution of degree $\ge 4$ satisfying $\cA_{s,v}$. Then $\pE_D v \transpose{v}$ is a feasible solution to \ref{eq:sparse-pca-sdp-relaxation}.
	\begin{proof}
		Since $D$ satisfies $\cA_{s,v}$, $\Tr\pE_D v \transpose{v} = \pE_D\sum_{i\leq d} v_i^2 =1 $.
		Now, there exists a vector $x\in \R^d$ with entries in $\Set{-1,+1}$ such that 		
		$\Normo{\pE_D v \transpose{v}}= \iprod{x \transpose{x}, \pE_D v \transpose{v}}$.	
		By Cauchy-Schwarz inequality for pseudo-distributions,  
		\begin{align*}
		\iprod{x \transpose{x}, \pE_D v \transpose{v}}
			=&
			 \pE_D \sum_{i,j\le d} x_i s_iv_i x_js_jv_j
			 \\
			 \le&
			 \sqrt{\pE_D \sum_{i,j\le d} x_i^2x_j^2s_i^2s_j^2}\;	\cdot 
			 \sqrt{\pE_D \sum_{i,j\le d} v_i^2v_j^2}
			 \\
			 =&
			 \pE_D\sum_{i\le d} s_i^2
			 \\
			 =&
			 k\,.
		\end{align*}
		The result follows as $\pE v \transpose{v}\sge 0$.
	\end{proof}
\end{lemma}
We can now finish the analyses using the certified upper bounds from the previous sections.

\begin{proof}[Proof of Theorem~\ref{thm:algorithm-robust-spca-technical}]
	Let $D$ be the pseudo-distribution in \cref{alg:sparse-pca-program}.
	By \cref{lem:bounding-errors}, with probability at least $0.99$, $\pE_{D}\Norm{Wv}^2_2 \leq O\Paren{ d^{1/t} k^{1-\frac{1}{t}}  t}$.
	Note that since the pseudo-distribution that outputs $v_0$ satisfies $\cA_{s,v}$, by  \cref{lem:sos-feasible-sdp}, $\pE_D v\transpose{v}$ satisfies the premises of  \cref{thm:meta-theorem}. Then we immediately get,
	\begin{align*}
		1-\pE_D\iprod{v,v_0}^2
		\lesssim \frac{k}{\beta n} \cdot t\cdot \Paren{\tfrac d k}^{1/t} 
		+ \frac{1}{\beta}
		+\sqrt{\frac{k}{\beta n}}\Paren{\sqrt{\log \frac{d}{k}}+
			\Norm{E}_{1\rightarrow 2}}\cdot \Paren{1+\frac{1}{\sqrt{\beta}}}\,.
	\end{align*}
	The result follows applying \cref{lem:linear-algebra-correlation-eigenverctor-large-quadratic-form}.
\end{proof}

Similarly,
\begin{proof}[Proof of Theorem~\ref{thm:main-sos-weak-signal}]
	Let $D$ be the pseudo-distribution in \cref{alg:sparse-pca-program}.
	By Theorem \ref{theorem:ceritificate-gaussian-quadratic-form}, with probability $1-o(1)$, $\Abs{\pE_{D} \transpose{v}\Paren{\transpose{W}W-n\Id}v}
	 \leq O\Paren{\frac{k}{\sqrt{t}}\sqrt{n \log d}}$. 
	Note that since the pseudo-distribution that outputs $v_0$ with probability $1$ satisfies $\cB_{s,v}$, by  \cref{lem:sos-feasible-sdp}, $\pE_D v\transpose{v}$ satisfies the premises of  \cref{thm:meta-theorem}. Then we immediately get,
	\begin{align*}
	1-\pE_D\iprod{v,v_0}^2\lesssim \frac{k}{\beta \sqrt{nt}}\log d
	+\sqrt{\frac{k}{\beta n}}\Paren{\sqrt{\log \frac{d}{k}}+
	\Norm{E}_{1\rightarrow 2}}\cdot \Paren{1+\frac{1}{\sqrt{\beta}}}\,.
	\end{align*}
	The result follows applying \cref{lem:linear-algebra-correlation-eigenverctor-large-quadratic-form}.
\end{proof}

\section{Unconditional lower bounds for distinguishing}
\label{sec:lowerbounds}

\subsection{Low-degree polynomials}

The goal of this section is to formalize our lower bounds. In light of the discussions in \cref{sec:lower-bounds-robust-settings} and \cref{sec:low-degree-likelihood-ratio} we study distinguishing problems between two distributions over matrices: the null distribution $\nulld$, which in our case is a standard Gaussian, and the planted distribution $\planted$ that contains some sparse signal hidden in random (and adversarial) noise. That is, given an instance $Y$ sampled either from the null or from the planted distribution, the goal is to determine whether $Y$ contains a planted signal. We will show that a large class of polynomial time algorithms (capturing the best known algorithms) cannot distinguish between the null and the planted case even when information-theoretically possible. Specifically, we will show that low degree polynomial estimators cannot solve these problem. Similarly to \cite{DBLP:conf/focs/HopkinsKPRSS17, DBLP:conf/focs/HopkinsS17, ding2019subexponentialtime}, we study the low degree analogue of the $\chi^2$-divergence between probability measures.

\begin{definition}
Let $\mu$ and $\nu$ be probability distributions over $\R^{n\times d}$, and denote by $F$ the set of all functions $f : \R^{n\times d} \to \R$ such that $\Abs{\E_\mu f} < \infty$ and $0 < \Var_\nu f < \infty$.
The
\emph{$\chi^2$-divergence} of $\mu$ with respect to $\nu$ is defined as
\[
\chi^2(\mu\parallel\nu) =
\underset{f \in F}{\sup}\frac{\Paren{\Ep f-\En f}^2}{\Var_\nu f}\,.
\]
\end{definition}

Note that this value is related to the likelihood ratio $L$ described in \cref{sec:low-degree-likelihood-ratio}: the fraction in the right hand side is maximized for $f=L$, and  $\chi^2(\mu\parallel\nu) = \E_{\nu}L^2 -1$.

Recall that,  if $\chi^2(\mu\parallel\nu)$ is bounded, then $\mu$ and $\nu$ are information-theoretically indistinguishable in the sense of \cref{sec:background-decision-theory}. The low-degree analogue of  $\chi^2$-divergence is defined similarly. Denote by $\lowdegpolys{D}$ the set of polynomials of degree at most $D$ in $\R[Y]$ (where $\R[Y]$ is the space of polynomials of $n\cdot d$ variables 
corresponding to the entries of $Y$).
\begin{definition}
	Let $D > 0$ and let $\mu$ and $\nu$ be probability distributions over $\R^{n\times d}$ 
	such that $\nu$ is absolutely continious and for all $p\in \lowdegpolys{D}$, 
	$\Abs{\E_\mu p} < \infty$ and $\Var_\nu p < \infty$. The \emph{degree-$D$ $\chi^2$-divergence} of $\mu$ with respect to $\nu$ is defined as
	\[
	\chi^2_{\le D}(\mu\parallel\nu) =
	\underset{p \in \lowdegpolys{D}}{\sup}\frac{\Paren{\Ep p-\En p}^2}{\Var_\nu p}\,,
	\]
	where we assume that $0/0 = 0$.
\end{definition}
Note that since $\nu$ is absolutely continuous, the denominator $\Var_\nu p$ is zero if and only if $p$ is constant (and in this case the numerator is also zero).

\subsection{Spiked covariance model with sparsity}\label{sec:single-spike-model-lower-bound}

The first problem we will look into is a variant of the standard sparse spiked covariance model which we use to prove the lower bound in \cref{thm:results-weak-signal-regime}.

\begin{problem}(Spiked Covariance Model with Sparsity)\label{problem:spiked_covariance_model_lb}
	Given a matrix $Y \in\R^{n\times d}$, decide whether:
	\begin{itemize}
		\item[$H_0$:] $Y= W$ where $W\sim N(0,1)^{n\times d}$ is a standard Gaussian matrix.
		\item[$H_1$:]  $Y= W + \sqrt{\beta}u_0\transpose{v_0}$, where $W$, $u$, and $v$ are mutually independent, $W\sim N(0,1)^{n\times d}$ is a standard Gaussian matrix, $u\in \R^n$ is a random vector with i.i.d. $1$-subgaussian coordinates  
		$u_i$ distributed symmetrically around zero such that $|u_i| \le R$ for some $R \ge 1$, 
		and $v \in \R^d$ is a random vector with i.i.d. coordinates $v_i$ that take values
		\begin{align*}
			v_i = 
			\begin{cases}
			-\frac{1}{\sqrt{k}}& \text{with probability } k/2d,\\
			\frac{1}{\sqrt{k}}& \text{with probability } k/2d,\\
			0& \text{ otherwise.}
			\end{cases}
		\end{align*}
	\end{itemize}
\end{problem}

Let's take a moment to compare the model in $H_1$ with model \ref{def:wishart-matrix-model}. First, note that we require $\abs{u_i} \le R$, which formally does not hold for Gaussian distribution for any $R\in \R$.
So to get lower bounds for the single spike model we should consider not $H_1$, but $H'_1$ such that $u$ from $H_1$ is replaced by a Gaussian vector $u' \sim \cN\Paren{0, \Id_n}$.
However, with high probability $\abs{u'_i} \lesssim \sqrt{\log n}$ for all $1\le i\le n$. If $u$ is drawn from a truncated Gaussian distribution 
$u_i = \sign\paren{u'_i}\min\{\abs{u'_i}, C\sqrt{\log n}\}$, then  all $u_i$ are $1$-subgaussian and
$u_i\le R = C\sqrt{\log n}$. Moreover, with high probability over $u'$, $u_i = u'_i$ for all $i$. Hence if it is hard to distinguish between $H_0$ and $H_1$ with $R = C\sqrt{\log n}$, it is also hard to distinguish between $H_0$ and $H'_1$. 
Second, if $k \to \infty$, then with high probability  $v$ is $\tilde{k} = k (1 + o(1))$-sparse. Also note that $v$ might not be a unit vector, but with high probability over $v$ its norm is $1 + o(1)$.  So with high probability over $v$, $Y$ from $H_1$ is equal to $W + \sqrt{\tilde{\beta}}u\transpose{\tilde{v}}$, where $\tilde{v}$ is a $\tilde{k} =  k (1 + o(1))$-sparse unit vector and $\tilde{\beta} = \beta/\norm{v} = \beta(1+o(1))$. 

Recall now that if $\beta \gtrsim \sqrt{d/n}$, then the algorithm that just computes the top singular value of $Y$ can distinguish between $H_0$ and $H_1$.
On the other hand if  $\beta \gtrsim\frac{k}{\sqrt{n}}\sqrt{\log d}$, then Diagonal Thresholding can distinguish between $H_0$ and $H_1$.
Under \cref{con:low-degree-polynomials} (see\cite{hopkins2018statistical, conf/innovations/BandeiraKW20} for a  formal discussion), the following theorem provides formal evidence that there is no polynomial time algorithm that can improve over the guarantees of Diagonal Thresholding whenever $\sqrt{\frac{d}{\log d}}\leq k\leq d^{1/2-\Omega(1)}$.
Furthermore the theorem also implies that SVD with Thresholding is optimal in regimes where $\sqrt{d}\leq k\sqrt{\log\frac{d}{k^2}}$.
Finally we remark that in settings where $k=d^{1-o(1)}$ the lower bound does not match the best know guarantees only by a factor $o(\log d)$. The proof of theorem can be found in Section \ref{section:lowerbound-spike-proof}.

%

\begin{theorem}\label{theorem:spike-general}
	Let $\nu$ and $\planted$ be the null and the planted distributions of Problem \ref{problem:spiked_covariance_model_lb} respectively.
	If $D \le n/R^4$ and
	\[
	\beta = \varepsilon\cdot \min\Big\{\sqrt{\frac{d}{n}},\;
	\frac{k}{\sqrt{Dn}}\cdot
	\Paren{\,\Abs{\,\ln\Paren{\frac{D\cdot d}{k^2}}\,}+1}\Big\}\,,
	\]
	where $0 < \varepsilon < \frac{1}{1000}$, then for any nonconstant polynomial 
	$p: \R^{n\times d} \to \R$ of degree at most $D$,
	\[
	\frac{\Paren{\Ep p(Y)-\En p(Y)}^2}{\Var_\nu p(Y)}
	\le \chi^2_{\le D}(\mu\parallel\nu) \lesssim \varepsilon^2\,.
	\]
\end{theorem}

Theorem \ref{theorem:spike-general} implies that polynomials of degree $D \le n / \log^2{n}$ 
cannot distinguish between a sample from the single spike model and a sample of standard Gaussian vectors if $\beta \ll \sqrt{d/n}$ and 
$\beta \ll \frac{k}{\sqrt{Dn}}\log\Paren{\frac{D\cdot d}{k^2}}$. 
In particular, if $k\le d^{1/2 -\Omega(1)}$, then polynomials of degree $\lesssim \log d$ cannot solve the problem for $\beta \ll \frac{k}{n}\sqrt{\log d}$.

We remark that  in \cite{ding2019subexponentialtime} the authors provided a similar hardness result for low-degree polynomials. Their lower bound works for all $D \le o(n)$ (not only for $D \le n/\log^2{n}$), but it does not contain a $\log{d}$ factor. In particular,
if $k \le d^{1/-\Omega(1)}$ and $D\le n/\log^2 n$, our lower bound is 
$\beta \ll \frac{k}{\sqrt{Dn}}\sqrt{\log d}$, and their lower bound is $\beta \ll \frac{k}{\sqrt{Dn}}$. Hence we are able to show a tight bound for Diagonal Thresholding.

In \cite{DBLP:conf/focs/HopkinsKPRSS17} a lower bound for a similar model is presented.  More concretely, the authors obtained exponential lower bounds for the Sum-of-Square Hierarchy in the Wigner model $Y=W+\beta v\transpose{v}$. Their results do not directly apply in our settings as the covariance matrix $\transpose{Y}Y$ is far from being Gaussian.

The same reasoning used in Theorem \ref{theorem:spike-general} can be used to obtain an information theoretic lower bound. 
\begin{theorem}\label{theorem:spike-information-theoretical}
	Let $\nu$ and $\planted$ be the null and the planted distributions of Problem \ref{problem:spiked_covariance_model_lb} respectively. 
	If
	\[
	\beta = \frac{\varepsilon}{R^2}\cdot \min\Big\{\sqrt{\frac{d}{n}},\;
	\frac{k}{n}\cdot
	\Paren{\,\Abs{\,\ln\Paren{\frac{n\cdot d}{k^2}}\,}+1}\Big\}\,,
	\]
	where $0 < \varepsilon < \frac{1}{1000}$, 
	then for any $f: \R^{n\times d} \to  \R$ such that $0 < \Var_\nu f(Y) < \infty$,
	\[
	\frac{\Paren{\Ep f(Y)-\En f(Y)}^2}{\Var_\nu f(Y)}
	\le \chi^2(\mu\parallel\nu) \lesssim \varepsilon^2\,.
	\]
\end{theorem}

We remark that the term $R^2$ is a consequence of our analysis of $\chi^2$-divergence  and can be avoided using different techniques, indeed for $R\in \omega(1)$, Theorem \ref{theorem:estimation_lower_bound} provides tighter guarantees. Additionally, we point out that for $\beta < 1$ a bound of $\sqrt{\frac{k}{n}}\log \frac{d}{k^2}$ can be achieved as shown in \cite{DBLP:journals/corr/abs-1304-0828}.

\subsection{Almost Gaussian vector in random subspace}
In this section we prove \cref{thm:lowerbound-robust-spca-informal}. Concretely, we will show that in the presence of adversarial corruptions, whenever  $t\cdot \Paren{\frac{d}{k}}^{1/t}\gtrsim n^{0.499}$ and $ d\geq  \tilde{\Omega}\Paren{n^{t}t^{t}}$,  so that  the degree $t$ SoS \cref{alg:sparse-pca-program} outperforms other known algorithms,  no multilinear polynomial of degree $\lesssim n^{0.001} $can obtain similar guarantees unless $d\gtrsim \tilde{\Omega}\Paren{\frac{n}{\ln^2 t}}^t$. 

Similarly to \cref{sec:single-spike-model-lower-bound} we design a specific distinguishing problem. In order to prove a lower bound in the presence of adversarial corruptions, we need to carefully chose the adversarial matrix.

%
%

\begin{problem} (Almost-Gaussian vector in a random subspace)
	\label{problem:almost_gaussian_model_special_case}
	Given a matrix Y in $\R^{n\times d}$, decide whether:
	\begin{itemize}
		\item[$H_0$:] $Y= W$ where $W\sim N(0,1)^{n\times d}$ is a standard Gaussian matrix.
		\item[$H_1$:]
	    $Y = \lambda u\transpose{\tilde{v}} + W + E$,
		where $W\sim N(0,1)^{n\times d}$ is a standard Gaussian matrix, $u\in\R^n$ is a unit vector with i.i.d. coordinates that take values $\pm 1/\sqrt{n}$ with probability $1/2$ each, and $\tilde{v}\in \R^d$ is a vector with i.i.d. coordinates that take values
		\begin{align*}
			v_i=\begin{cases}
			-1& \text{with probability }\delta/2\,,\\
			1& \text{with probability }\delta/2\,,\\
			0& \text{otherwise,}
		\end{cases}
		\end{align*}
		for some $\delta\in \Brac{0,1}$.
		 Furthermore $E = u\transpose{\Paren{v'- \transpose{W}u}}$, where  $v'$ is sampled according to the following distribution.  Let $s\ge0$ be the largest even number such that $\delta\lambda^s\leq 2^{-10s}$. For all $j \in [d]$,
		 \begin{itemize}
		 	\item  if $\tilde{v}_j\neq 0$, then $v'_j=0$ ;
		 	\item otherwise $v'_j$ is sampled from the distribution $\eta$ that has finite support $\supp\Paren{\eta} \subseteq [-10s, 10s]$ and moments: 
		 	\[\E_{x\sim\eta} x^r=
		 	\begin{cases}
		 		\frac{(r-1)!!-\lambda^r\delta}{1-\delta} &\text{if $0\le r\le s$ and even}\\
		 		\le (10s)^r  & \text{if $r \ge s+2$ and even} \\
		 		0 &\text{if $r$ is odd.}
		 	\end{cases}\]
		 \end{itemize} 
	\end{itemize}
\end{problem}

Proposition \ref{proposition:existence_of_distribution} shows that if 
$\delta\lambda^s\leq 2^{-10s}$, then such $\eta$ exists. 
Note that for $s=0$ the condition $\delta\lambda^s\leq 2^{-10s}$ is always satisfied, so $s$ in the problem description is well-defined. Also note that if $s = 0$, $v'$ is just a zero vector.

If $\delta d \to \infty$,  
then with high probability $\tilde{v}$ is  $\delta d\Paren{1-o(1)}$-sparse. So $\lambda u\transpose{\tilde{v}} = \sqrt{\beta}u_0\transpose{v_0}$ for $k$-sparse unit vector $v_0 = \frac{1}{\norm{\tilde{v}}}\tilde{v}$ (where $k:=\delta d\Paren{1-o(1)})$,
$\beta = \frac{k\lambda^2}{n}(1+o(1))$ and $u_0 = \sqrt{n} u$.

We will use the notation $v = \lambda \tilde{v} + v'$. Note that the coordinates of $v$ are independent, have Gaussian moments up to $s$, and with high probability $v$ has at least 
 $\delta d\Paren{1-o(1)}$ coordinates $v_j \in \set{\pm\lambda}$.

\paragraph{Geometric description}
	The planted distribution can be also described in geometric terms, where the problem becomes that of distinguishing between a subspace spanned by independent Gaussian vectors or a subspace spanned by independent Gaussian vectors and the planted vector $v$. 
	
	The construction is the following: at first we sample a signal vector $v \in \R^d$ that has at least $\delta d\Paren{1-o(1)}$ coordinates with absolute values at least $\lambda$ (using the construction described above).
	Then we sample $n-1$ i.i.d. standard Gaussian vectors 
	$\tilde{w}_1,\ldots,\tilde{w}_{n-1}\in \R^d$,
	and perform a random rotation $U \in \R^{n\times n}$ with first column vector
	$u$
	(such that $U$ is independent of $v,\tilde{w}_1,\ldots,\tilde{w}_{n-1}$) 
	on $v, \tilde{w}_1, \ldots, \tilde{w}_{n-1}$.
	That is, 
	\[ Y = U\cdot\begin{pmatrix}
	v^T \\
	\tilde{w}_1^T \\
	\vdots \\
	\tilde{w}_{n-1}^T
	\end{pmatrix}\;.
	\]
	This formulation is equivalent to the one described above. Indeed,
	\[
	Y = uv^T + \sum_{i=1}^{n-1} u_i\tilde{w}_i^T,
	\]
	where $u,u_1,\ldots,u_{n-1}$ are the columns of $U$. Note that $\sum_{i=1}^{n-1} u_i\tilde{w}_i^T$ is distributed as a standard (singular) Gaussian supported in the hyperplane orthogonal to $u$, and $(\Id - uu^T)W$ is also a standard Gaussian supported in the same hyperplane.

The theorem below provides a lower bound for the Problem 
\ref{problem:almost_gaussian_model_special_case}. The proof is in section \ref{section:lowerbound-subspace-proof}.

\begin{theorem}\label{theorem:lower_bound_main}
	Suppose that $n\le d$, $1 \le D \le n^{0.33}$, $0<\delta<1$, $k = \delta d$, $\lambda\geq 2$, and let $s \ge 2$ be the maximal even number 
such that $\delta \lambda^{s}\leq 2^{-10s}$, and $t=s/2 + 1$. Let $\nulld$ and $\planted$ denote respectively the null and planted distribution (with parameters $\delta,\lambda,s$) of Problem \ref{problem:almost_gaussian_model_special_case}. Suppose that $\lambda \ge 1000\sqrt{t\ln t}$ and that 
	$\lambda^4D t^2\ln^2 t=o\Paren{n\cdot\Paren{\log^2\Paren{\frac{d}{k^2}}+1}}$ as $n\to \infty$.
	If
	\[
	d = o\Paren{\frac{1}{\lambda^4}\cdot\Paren{\frac{n}{C\cdot \ln^2 t \cdot D}}^{t}}
	\]
	as $n\to \infty$ (for some constant $C$ that does not depend on $n$, $d$, $\delta$, $\lambda$, $s$ and $D$), then for any non-constant multilinear polynomial 
	$p: \R^{n\times d} \to \R$ of degree at most $D$,
	\[\label{equation:lower_bound_main}
      \frac{\Paren{\Ep p(Y)-\En p(Y)}^2}
	{\Var_\nu p(Y)} \to 0\,,
	\]
	as $n\to \infty$.
\end{theorem}

Let's try to illustrate the meaning of Theorem \ref{theorem:lower_bound_main}. If $\lambda \ge B{\sqrt{\log d}}$ for sufficiently large $B$, then $\delta$ is so that $\delta \lambda^{s}= 2^{-10s}$ for an even constant $s$ and $\delta d \to \infty$. Here \cref{alg:sparse-pca-program} can distinguish between the null and the planted distribution if $d \gtrsim n^t\log^t(n) t^t$ in polynomial time. Indeed, in this case with probability at least $0.99$ the algorithm \ref{alg:sparse-pca-program} outputs $\hat{v}$ such that 
\[
1 - \iprod{\hat v,v_0}^2
\lesssim \frac{1}{\lambda^2}\Paren{t\cdot \Paren{\frac{1}{\delta}}^{1/t} + \norm{E}_{1\to 2}^2}\le
t\cdot\Paren{\frac{1}{\lambda^{s+1} \delta}}^{1/t} + \Paren{\frac{\norm{E}_{1\to 2}}{\lambda}}^2
= \frac{2^{20}t}{\lambda^{1/t}} + \Paren{\frac{\norm{E}_{1\to 2}}{\lambda}}^2
\,.
\]
The first term tends to $0$ and $\norm{E}_{1\to 2}$ can be bounded as follows:
\[
\norm{E}_{1\to 2} \le \norm{u\transpose{(v')}}_{1\to 2} +  \norm{u\transpose{u}W}_{1\to 2} \le
\max_{1\le i\le d}{v'_i} + \max_{1\le i\le d}{\Paren{\transpose{u}W}_i}\lesssim s + \sqrt{\log d}\,,
\]
since $\transpose{u}W$ is a standard Gaussian vector. Hence for sufficiently large  $B$, 
$\iprod{\hat v,v_0} \ge 0.99$.

If in addition $D\le n^{0.001}$ and $\lambda \le n^{0.24}$,  then the conditions of Theorem \ref{theorem:lower_bound_main} are satisfied. Hence in this case for $d \le n^{0.999t - 1}$ no multilinear polynomial of degree at most $n^{0.001}$ can distinguish between the planted and the null distribution as $n \to \infty$. Furthermore note that if $\lambda^4 \gtrsim n\log d$, then Diagonal thresholding can distinguish between the planted and the null distribution in polynomial time (even if $d \ll  n^{0.999t - 1}$ ). Finally, it easy to see that exhaustive search works as long as $\lambda \gtrsim \sqrt{\log d/k}$.

\subsection{Chi-squared-divergence and orthogonal polynomials}
Recall that given a hypothesis testing problem with null distribution $\nulld$ and planted distribution $\planted$, we say a polynomial $p(Y)\in \lowdegpolys{D}$ cannot distinguish between $\mu$ and $\nu$ if  
\begin{equation}\label{equation:performance_polynomial_estimator}
\frac{\Abs{\Ep p(Y)-\En p(Y)}}{\sqrt{\Var_\nu p(Y)}}\le o(1)\,.
\end{equation}
So, if for some distinguishing problem this ratio is small for all $p\in \lowdegpolys{D}$, then polynomial estimators of degree at most $D$ cannot solve this distinguishing problem.
The key observation used to prove bounds for low degree polynomials is the fact that the polynomial which maximizes the ratio \eqref{equation:performance_polynomial_estimator} has a convenient characterization in terms of orthogonal polynomials with respect to the null distribution.

Formally, for any linear subspace of polynomials $\cS_{\le D} \subseteq \lowdegpolys{D}$ and any absolutely continuous probability distribution $\nu$ such that all polynomials of degree at most $2D$ are $\nu$-integrable, one can define an inner product in the space 
$\cS_{\le D}$ as follows
\[
\forall p, q \in \cS_{\le D}\quad \Iprod{p,q} = \E_{Y\sim\nu} p(Y)q(Y)\,.
\]
Hence we can talk about orthonormal basis in $\cS_{\le D}$ with respect to this inner product.
\begin{proposition}\label{proposition:optimal_polynomial_statistic}
	Let $\cS_{\le D} \subseteq \lowdegpolys{D}$ be a linear subspace of polynomials of dimension $N$.
	Suppose that $\nu$ and $\mu$ are probability distributions over $Y \in \R^{n\times d}$ such that any polynomial of degree at most $D$ is $\mu$-integrable and any polynomial of degree at most $2D$ is $\nu$-integrable.
	Suppose also that $\nu$ is absolutely continuous. 
	Let $\{\psi_i(Y)\}_{i=1}^{N}$ be an orthonormal basis in
	$\cS_{\le D}[Y]$  with respect to $\nu$. Then
	\begin{equation*}
	\underset{p \in\cS_{\le D}}{\max}
	\frac{\Paren{\Ep p(Y)}^2}{\En p^2(Y)}=
	\sum_{i=1}^{N}\Paren{\Ep\psi_i}^2.
	\end{equation*}
	\begin{proof}
		For any $p\in \cS_{\le D}$
		\[
		\Ep p(Y)= 
		\E\limits_{\mu}\sum_{i=1}^{N} p_i \psi_i(Y)=
		\sum_{i=1}^{N} p_i \E\limits_{\mu}\psi_i(Y)\le
		\Paren{\sum_{i=1}^{N} p^2_i}^{1/2}
		\Paren{\sum_{i=1}^{N}\Paren{\E\limits_{\mu} \psi_i(Y)}^2}^{1/2}\,.
		\]
		Since the system $\{\psi_i(Y)\}_{i=1}^{N}$ 
		is orthonormal with respect to $\nu$,
		\[
		\En p^2(Y)= 
		\sum_{i=1}^{N} p^2_i\,.
		\]
		Hence we get
		\[
		\frac{\Ep p(Y)}{\Paren{\En p^2(Y)}^{\frac{1}{2}}}
		\le\Brac{ \sum_{i=1}^{N}\Paren{\Ep \psi_i}^2}^{1/2}\,.
		\]
		Note that the polynomial $\sum_{i=1}^{N}
		\E_{Y'\sim \mu} \Brac{\psi_i(Y')}\psi_i(Y)$ maximizes the ratio.
	\end{proof}
\end{proposition}

From now on we assume that the distribution $\nu$ is Gaussian.
In this case a  useful orthonormal basis in $\lowdegpolys{D}$ is the system of Hermite polynomials. 

To work with Hermite polynomials we introduce some useful notation. 
For a multi-index $\alpha$ over $[n]\times [d]$, let $I_\alpha:=\{ i\in[n] : \paren{i,j}\in\alpha \text{ for some }j\in[d] \}$  and similarly  $J_\alpha:=\{ j\in[d] : \paren{i,j}\in\alpha \text{ for some }i\in[n] \}$. For $j \in [d]$, let $I_{\alpha,j}:=\{ i\in[n] : \paren{i,j}\in \alpha \}$, and similarly let $J_{\alpha,i}:=\{j \in [d]:\paren{i,j}\in \alpha \}$. 
We will use the notation 
$\alpha! := \prod_{(i ,j)\in \alpha}  \alpha_{ij}!$ and for a matrix $X\in\R^{n\times d}$, 
$X^{\alpha} :=  \prod_{(i ,j)\in \alpha}  X_{ij}^{\alpha_{ij}}$.
Note that every multi-index $\alpha$ over $[n]\times [d]$
can be represented as a bipartite multigraph 
$G_\alpha = \Paren{I_\alpha \bigcup J_\alpha, E_\alpha}$ 
such that each edge $\{i,j\}$ has multiplicity $\alpha_{ij}$. In this representation
the set $J_{\alpha,i}$ corresponds to the neighborhood of the vertex $i$ and the set ${I_{\alpha,j}}$ corresponds to the neighborhood of $j$. If $\alpha$ is multilinear, 
$G_\alpha$ is just a graph (i.e. multiplicity of each edge is 1).

For a multi-index 
$\alpha$ over $[n] \times [d]$ the corresponding Hermite polynomial is
\[
H_\alpha(Y) = \underset{j \in J_{\alpha}}{\prod}\underset{i \in I_{\alpha,j}}{\prod}{H}_{\alpha_{ij}}(Y_{ij})\,,
\]
where $H_{l}$ for $l \in \Z$ is a degree $l$ one variable Hermite polynomial, defined
as follows
\[
H_l(x) = \sum_{\substack{{0\le r\le l} \\  \text{$l-r$ is even}}}\;
\Paren{-\frac{1}{2}}^{\frac{l-r}{2}}\frac{1}{r! \Paren{\frac{l-r}{2}}!}\; x^r\,.
\]
Note that $H_\emptyset(Y) = 1$. Hence by applying Proposition \ref{proposition:optimal_polynomial_statistic}
to the subspace of polynomials such that $\En p(Y) = 0$, we get
\begin{corollary}\label{corollary:optimal_polynomial_statistic_hermite}
	Let $\nu$ be Gaussian.
	Suppose that the distribution $\mu$ is so that any polynomial of degree at most $D$ is $\mu$-integrable. Then
	\[
	\underset{p \in \lowdegpolys{D}}{\max}\;
	\frac{\Paren{\Ep p(Y)-\En p(Y)}^2}
	{\Var_\nu p(Y)}\;
	=
	\underset{0<\card{\alpha}\le D}{\sum}\Paren{\Ep\hermitepoly{Y}{\alpha}}^2\,.
	\]
\end{corollary}

Denote by $\multilinearpoly{D}$  the space of multilinear polynomials of degree at most $D$
(we do not include constant polynomials in $\multilinearpoly{D}$).
Note that multilinear Hermite polynomials $H_\alpha$ (which correspond to multilinear multiindices $\alpha$) are exactly
\[
H_\alpha(Y) =  \underset{j \in J_{\alpha}}{\prod}\underset{i \in I_{\alpha,j}}{\prod}\yij\,.
\]
They form a basis in the space  $\multilinearpoly{D}$ (for $0 < \card{\alpha} \le D$).
Let's denote $\Hermitemlpolys{D} := \Hermitepolys{D} \bigcap \multilinearpoly{D}$.
Applying Proposition \ref{proposition:optimal_polynomial_statistic} to 
the space $\multilinearpoly{D}$  we get
\begin{corollary}\label{corollary:optimal_polynomial_statistic_multilinear}
	Let $\nu$ be Gaussian.
	Suppose that the distribution $\mu$ is so that any polynomial of degree at most $D$ is $\mu$-integrable. Then
	\[
	\underset{p \in \multilinearpoly{D}}{\max}\;
	\frac{\Paren{\Ep p(Y)-\En p(Y)}^2}
	{\Var_\nu p(Y)}\;
	=
	\underset{p \in\multilinearpoly{D}}{\max}
	\frac{\Paren{\Ep p(Y)}^2}{\En p^2(Y)}
	=
	\sum_{\hermitepoly{Y}{\alpha}\in \Hermitemlpolys{D}}\Paren{\Ep\hermitepoly{Y}{\alpha}}^2\,.
	\]
\end{corollary}
Hence the key part of proving lower bounds for low degree polynomial estimators is bounding $\Ep\hermitepoly{Y}{\alpha}$. 


\subsection{Spiked covariance model with sparsity (proof)}\label{section:lowerbound-spike-proof}

\label{sec:lowerbounds-spike}

\providecommand{\qij}{q_{ij}}
\providecommand{\zij}{z_{ij}}
\providecommand{\lsubscript}[2]{{#1}_{1},\ldots,{#1}_{#2}}

In this section we prove Theorems \ref{theorem:spike-general} and \ref{theorem:spike-information-theoretical} .

The proofs will be based on two steps: first, we compute the expectation of Hemrite polynomials under the planted distribution, then we bound their total contribution.

We will need the following fact about Hermite polynomials:
\begin{fact}\label{fact:expectation_hermite_non_centerd}
     For any $c \in \R$ and any $l \in \N$
     \[
     \E\limits_{w\sim\cN(0,1)} H_l(w+c) = \frac{c^l}{l!}\,.
     \]
\end{fact}

The following Lemma is a generalization (and a corrollary) of Fact \ref{fact:expectation_hermite_non_centerd}.
\begin{lemma}\label{lemma:hermite_indep_decomposition}
	Let $X \in \R^{n\times d}$ be a random matrix such that all moments of $X$ exist. Let $W\in \R^{n\times d}$ be a standard Gausian matrix independent of $X$. Then for any multi-index $\alpha$ over $[n]\times[d]$,
	\begin{equation*}
		\E H_\alpha(W + X) = \frac{\E X^\alpha}{\alpha!}\,.
	\end{equation*}
\end{lemma}
\begin{proof}
By Fact \ref{fact:expectation_hermite_non_centerd},
\begin{align*}
\E H_\alpha(W + X) 
&= 
\E_X \E_W \,[H_\alpha(W + X)\,|\,X] \\
&=
\E_X \prod_{(i ,j)\in \alpha}  \E_W\,\big[H_{\alpha_{ij}}(W_{ij} + X_{ij})\,\big|\, X\big] \\
&=
\E_X \prod_{(i ,j)\in \alpha}  \frac{X_{ij}^{\alpha_{ij}}}{\alpha_{ij}!} \\
&=
\frac{\E X^\alpha}{\alpha!}\,.
\end{align*}
\end{proof}

Now we can exactly compute the expectation of Hermite polynomials under the planted distribution.

\begin{lemma}\label{lemma:single-graph-spike-general}
	Let $\alpha$ be a multi-index over $[n]\times [d]$ and let $\planted$ be the planted distribution of Problem \ref{problem:spiked_covariance_model_lb}. If every vertex of the multigraph $G_\alpha$ has even degree, then
	\[
	\Big|\E_{Y\sim \planted} \hermitepoly{Y}{\alpha}\Big|
	=
	\frac{1}{\alpha!} \Paren{\frac{\beta}{k}}^{\card{\alpha}/2} \Paren{\frac{k}{d}}^{\card{J_\alpha}}\cdot
	{\prod_{i\in I_\alpha} \E u_i^{\deg_{G_\alpha}(i)}} 
	\]
	and if at least one vertex of $G_\alpha$ has odd degree, then $\Ep \hermitepoly{Y}{\alpha} = 0$.
\end{lemma}
\begin{proof}
	By lemma \ref{lemma:hermite_indep_decomposition},
	\[
	\E_{Y\sim \planted} \hermitepoly{Y}{\alpha} 
	= 
	\frac{\E(\sqrt{\beta}u v^T)^\alpha }{\alpha!}
	=
	\frac{\beta^{\card{\alpha}/2}}{\alpha!} 
	\E\prod_{\substack{i\in I_\alpha \\ j \in J_\alpha}} u_i^{\alpha_{ij}} v_j^{\alpha_{ij}} 
	=
	\frac{\beta^{\card{\alpha}/2}}{\alpha!}  
	\Paren{\prod_{i\in I_\alpha} \E u_i^{\sum_{j\in J_{\alpha, i}}\alpha_{ij}}} 
	\Paren{\prod_{j\in J_\alpha} \E v_j^{\sum_{i\in I_{\alpha, j}}\alpha_{ij}}}\,.
	\]
	Notice that $\deg_{G_\alpha}(i) = \sum_{j\in I_{\alpha, i}}\alpha_{ij}$ and 
	$\deg_{G_\alpha}(j) = \sum_{i\in J_{\alpha, j}}\alpha_{ij}$. By symmetry of each $u_i$ and $v_j$,
	if at least one vertex of $G_\alpha$ has odd degree, then
	$\E_{Y\sim \planted} \hermitepoly{Y}{\alpha} = 0$.
	
	Now assume that each vertex has even degree. Then 
	$\E v_j^{\deg_{G_\alpha}(j)} = \frac{k}{d} \Paren{\frac{1}{\sqrt{k}}}^{\deg_{G_\alpha}(j)}$ and
	\begin{align*}
	\Big|\E_{Y\sim \planted} \hermitepoly{Y}{\alpha}\Big|
	&=
	\frac{{\beta}^{\card{\alpha}/2} }{\alpha!} 
	{\prod_{i\in I_\alpha} \E u_i^{\deg_{G_\alpha}(i)}} \cdot
	{\prod_{j\in J_\alpha}\E v_j^{\deg_{G_\alpha}(j)}} \\
	&=
	\frac{1}{\alpha!} \Paren{\frac{\beta}{k}}^{\card{\alpha}/2} \Paren{\frac{k}{d}}^{\card{J_\alpha}}\cdot
	{\prod_{i\in I_\alpha} \E u_i^{\deg_{G_\alpha}(i)}}\,.
	\end{align*}
\end{proof}

\begin{lemma}\label{lemma:truncated-moment-bound}
Let $G_\alpha = (I_\alpha, J_\alpha, E_\alpha)$ be a bipartite multigraph and let $u$ be the same as in the Problem \ref{problem:spiked_covariance_model_lb}. Then
\[
	{\prod_{i\in I_\alpha} \E u_i^{\deg_{G_\alpha}(i)}} \le 
             2^{\card{\alpha}}\cdot R^{{\card{\alpha}-2\card{I_\alpha}}}\,.
\]
\end{lemma}
\begin{proof}
	If $R = 1$ the statement is true. Assume that $R > 1$. 
	Let's denote $\deg_{G_\alpha}(i) = x_i$.
    Notice that
    \[
    \prod_{i\in I_\alpha} \E u_i^{x_i}  
    \le 
    \prod_{i\in I_\alpha} \min \{R^{x_i}, \Paren{\frac{x_i}{2}}^{x_i / 2} \}\,.
    \]
    Let's denote $X^\prime= \set{x_i \,|\, i \in I_\alpha,\, x_i \ge 2R^2}$ and $X''= \set{x_i \,|\, i \in I_\alpha,\, 2 \le x_i < 2R^2}$ . Let's show that if the value $\prod_{i\in I_\alpha} \min \{R^{x_i}, \Paren{\frac{x_i}{2}}^{x_i / 2} \}$ is maximal (for fixed $\card{\alpha}$ and $\card{I_\alpha}$), then either all $x_i \in X''$ are equal to $2$, or ${X'}$ is empty and there can be only one $x_i\in X''$ that is greater than $2$. 
    Indeed, if $x_a \ge x_b > 2$ for some $x_a \in I_\alpha, x_b \in X''$, we can increase $x_a$ by $2$ and decrease $x_b$ by $2$. This operation increases $\prod_{i\in I_\alpha} \min \{R^{x_i}, \Paren{\frac{x_i}{2}}^{x_i / 2} \}$.  
    
    If ${X'}$ is empty, then
    	\[
    \prod_{i\in I_\alpha} \min \{R^{x_i}, \Paren{\frac{x_i}{2}}^{x_i / 2} \}
    \le
    \Paren{\frac{s}{2}}^{s/2}
    \le
    \Paren{\frac{s}{2}}^{\card{\alpha}/2 - \card{I_\alpha} + 1}
    \le
    \frac{\card{\alpha}}{2} R^{\card{\alpha} - 2\card{I_\alpha}}\le
     2^{\card{\alpha}}\cdot R^{{\card{\alpha}-2\card{I_\alpha}}}\,,
    \]
    where $s = \max\{x_i \in X''\}$.
    
    Now assume that all $x_i \in X''$ are equal to $2$ and $X'$ is nonempty. In this case $2R^2 \le \card{\alpha}$, so
    \[
      \prod_{i\in I_\alpha} \min \{R^{x_i}, \Paren{\frac{x_i}{2}}^{x_i / 2} \}
    \le
    R^{\card{\alpha}-2\Paren{\card{X''}-1}}
    \le
     \frac{\card{\alpha}}{2} \cdot
    R^{\card{\alpha}-2{\card{X''}}}\,.
    \]
     Notice that
    \[
    \card{\alpha} \ge 2R^2\Paren{\card{I_\alpha} - \card{X''}} + \card{X''} = 
    2\card{X''}\Paren{1-R^2} +  2R^2\card{I_\alpha}\,,
    \]
    Hence
    \[
    \card{\alpha} - 2\card{X''} \le \card{\alpha} - \card{\alpha}\frac{1}{1-R^2}  + 2\card{I_\alpha}\frac{R^2}{1-R^2}=
    \Paren{1 + \frac{1}{R^2-1}}\Paren{\card{\alpha}-2\card{I_\alpha}}\,.
    \]
    It follows that
    \[
    \prod_{i\in I_\alpha} \E u_i^{x_i}  
            \le
              \frac{\card{\alpha}}{2} \cdot
             R^{\Paren{1 + \frac{1}{R^2-1}}\Paren{\card{\alpha}-2\card{I_\alpha}}}
             \le
              \frac{\card{\alpha}}{2} \cdot
             e^{\Paren{\card{\alpha}-2\card{I_\alpha}}/2} \cdot 
             R^{{\card{\alpha}-2\card{I_\alpha}}}
             \le
             2^{\card{\alpha}}\cdot R^{{\card{\alpha}-2\card{I_\alpha}}}\,,
    \]
    where we used the fact that for $R > 1$, $\frac{\ln{R}}{R^2 - 1} < \frac{1}{2}$.
\end{proof}

\begin{lemma}\label{lemma:fixed_graphs}
	Let $\cG(E)$ be a set of all bipartate multigraphs 
	$G_\alpha = (I_\alpha, J_\alpha, E_\alpha)$ such that $I_\alpha \subseteq [n]$, $J_\alpha \subseteq [d]$, $\card{I_\alpha} = A$, $\card{J_\alpha} = B$, $\card{E_\alpha} = E$ and 
	each vertex of $G_\alpha$ has even nonzero degree. Then
		\[
\sum_{\alpha \,:\, G_\alpha \in \cG(A, B, E)}
\Paren{\E_{Y\sim \planted}\hermitepoly{Y}{\alpha}}^2
\le
\Paren{\frac{n}{AR^4}}^{A}\cdot 
\Paren{\frac{{{60\cdot A R^2 \beta}}}
	{\min\Big\{{\sqrt{E\cdot d}, \;
			{k \cdot \Paren{\Abs{\ln{\Paren{{d\cdot E}/{k^2}}}} + 1}}}\Big\}}}^E\,.
\]
\end{lemma}
\begin{proof}
	We can assume that $A \le E/2$ and $B \le E/2$ (otherwise $ \cG(A, B, E)$ is empty). We have to choose $A$ vertices from $[n]$ and $B$ vertices from $[d]$. Then we choose $\alpha_{ij}$ so that $|\alpha| = E$. By Lemma \ref{lemma:single-graph-spike-general} and Lemma 
	\ref{lemma:truncated-moment-bound}
	\begin{align*}
	\sum_{\alpha \,:\, G_\alpha \in \cG(A, B, E)}
	\Paren{\E_{Y\sim \planted}\hermitepoly{Y}{\alpha}}^2 
	&\le
	\binom{n}{A}\binom{n}{B}\sum_{|\alpha| = E} 
	\Paren{\frac{1}{\alpha!}}^2
	\Paren{{\frac{\beta }{k}}}^{E}\Paren{\frac{k}{d}}^{2B} 
	 2^{2E}\cdot  R^{2E-4A}\\
	&\le
	4^E\cdot\binom{n}{A}\binom{n}{B}\sum_{|\alpha| = E} 
	{\frac{1}{\alpha!}}
	\Paren{{\frac{\beta }{k}}}^{E}\Paren{\frac{k}{d}}^{2B} \cdot R^{2E-4A}\,.
	\end{align*}
	By the multinomial theorem,
	\[
	\sum_{|\alpha| = E} 
	\frac{1}{\alpha!} = \frac{(A\cdot B)^E}{E!}\,.
	\]
	Therefore,
	\begin{align*}
	\sum_{\alpha \,:\, G_\alpha \in \cG(A, B, E)}
	\Paren{\E_{Y\sim \planted}\hermitepoly{Y}{\alpha}}^2 
	&\le
	4^E\cdot
	\Paren{\frac{e\cdot n}{A}}^A \Paren{\frac{e\cdot  d}{B}}^B \frac{(e\cdot A\cdot B)^E}{E^E} 
	\Paren{{\frac{\beta }{k}}}^{E}\Paren{\frac{k}{d}}^{2B} \cdot R^{2E - 4A}  \\
	&\le
	4^Ee^{A+B+E}\cdot	  {\frac{B^{E}}{E^E\cdot B^B}}\cdot
	\Paren{\frac{k^2}{d}}^B
	\Paren{{\frac{\beta }{k}}}^{E}
	\cdot R^{2E - 4A} \cdot\Paren{\frac{n}{A}}^A A^E
	\\
	&\le
	30^E\cdot
	\Paren{\frac{n}{AR^4}}^A \cdot \Paren{\frac{B}{E}\cdot\Paren{\frac{E}{B}}^{B/E}
		\Paren{\frac{k^2}{d\cdot E}}^{B/E}\cdot A R^2 \cdot{\frac{\beta}{k}}}^{E}\\
	&\le
	60^E\cdot\Paren{\frac{n}{AR^4}}^A \cdot \Paren{\frac{B}{E}\cdot
		\Paren{\frac{k^2}{d\cdot E}}^{B/E}\cdot A R^2 \cdot{\frac{\beta}{k}}}^{E}\,,
	\end{align*}
	where we used the inequality $x^{1/x} \le 2$.
	
	If $k^2 \ge d\cdot E$, then
	\[
	{\frac{B}{E}\cdot
		\Paren{\frac{k^2}{d\cdot E}}^{B/E}\cdot A R^2\cdot{\frac{\beta}{k}}}
	\le
	\frac{1}{2}{\Paren{\frac{k^2}{d\cdot E}}^{1/2} \cdot A R^2\cdot{\frac{\beta}{k}}}
	\le
	  \frac{AR^2\beta}{\sqrt{E\cdot d}}
	  \,,
	\]
	and if $k^2 < d\cdot E$, then
    \[
	{\frac{B}{E}\cdot
	\Paren{\frac{k^2}{d\cdot E}}^{B/E}\cdot A R^2\cdot{\frac{ \beta}{k}}}
     \le
    \min\Big\{
	\frac{1}{e\ln{\Paren{\frac{d\cdot E}{k^2}}}}
	,\frac{1}{2}\Big\}\cdot A R^2\cdot{\frac{\beta}{k}}
	 \le 
	  \frac{AR^2\beta}{k\cdot\Paren{\ln{\Paren{\frac{d\cdot E}{k^2}}} + 1}}
	 \,,
	\]
	since $x a^x \le \frac{1}{e\ln(1/a)}$ for all $x > 0$ and $0 < a < 1$. 
	Hence
		\[
\sum_{\alpha \,:\, G_\alpha \in \cG(A, B, E)}
\Paren{\E_{Y\sim \planted}\hermitepoly{Y}{\alpha}}^2
\le
\Paren{\frac{n}{AR^4}}^{A}\cdot 
\Paren{\frac{{{60\cdot A R^2 \beta}}}
	{\min\Big\{{\sqrt{E\cdot d}, \;
			{k \cdot \Paren{\Abs{\ln{\Paren{{d\cdot E}/{k^2}}}} + 1}}}\Big\}}}^E\,.
\]
\end{proof}
\begin{proof}[Proof of Theorem \ref{theorem:spike-general}]
	If $A \le E/2\le \frac{n}{R^4}$, the function
	\[
	\Paren{\frac{n}{AR^4}}^{A}\cdot {A}^E
	\]
	is a monotone in $A$. Hence
	\[
	\Paren{\frac{n}{AR^4}}^{A}\cdot \Paren{AR^2}^E\le \Paren{nE}^{E/2}\,.
	\]
	Therefore, by Lemma \ref{lemma:fixed_graphs},
	\begin{align*}
	\sum_{0 < |\alpha| \le D} \Paren{\E_{Y\sim \planted}\hermitepoly{Y}{\alpha}}^2 
	&=
	\sum_{2 \le E \le D} \; \sum_{\substack{1 \le A \le E/2 \\ 1 \le B \le E/2 }} \;\;
	\sum_{\alpha \,:\, G_\alpha \in \cG(A, B, E)} 
    \Paren{\E_{Y\sim \planted}\hermitepoly{Y}{\alpha}}^2 \\
    &\le
    \sum_{2 \le E \le D} \; \sum_{\substack{1 \le A \le E/2 \\ 1 \le B \le E/2 }} \;\;
    \Paren{\frac{{{60\cdot \sqrt{n} \sqrt{E} \cdot \beta}}}
	{\min\Big\{{\sqrt{E\cdot d}, \;
			{k \cdot \Paren{\Abs{\ln{\Paren{{d\cdot E}/{k^2}}}} + 1}}}\Big\}}}^E
            \\
      &\le
	\sum_{2 \le E \le D} \frac{E^2}{4}\cdot
     \Paren{\frac{{{60 \beta}}}
     	{\min\Big\{{\sqrt{d/n}, \;
     			{\Paren{k/\sqrt{En}} \cdot \Paren{\Abs{\ln{\Paren{{d\cdot E}/{k^2}}}} + 1}}}\Big\}}}^E
     \\
      &\le
\sum_{2 \le E \le D}
\Paren{\frac{{{120 \beta}}}
	{\min\Big\{{\sqrt{d/n}, \;
			{\Paren{k/\sqrt{En}} \cdot \Paren{\Abs{\ln{\Paren{{d\cdot E}/{k^2}}}} + 1}}}\Big\}}}^E\,.
   \end{align*}
   If
	 \[
\beta = \varepsilon\cdot \min\Big\{\sqrt{\frac{d}{n}},\;
\frac{k}{\sqrt{Dn}}\cdot
\Paren{\,\Abs{\,\ln\Paren{\frac{D\cdot d}{k^2}}\,}+1}\Big\}\,,
\]
   where $0 < \varepsilon < \frac{1}{1000}$, then
   \[
   	\sum_{0 < |\alpha| \le D} \Paren{\E_{Y\sim \planted}\hermitepoly{Y}{\alpha}}^2 
   	\le 1000^2\cdot\varepsilon^2\,,
   \]
   and using Corollary \ref{corollary:optimal_polynomial_statistic_hermite} we get the desired conclusion.
\end{proof}

\begin{proof}[Proof of Theorem \ref{theorem:spike-information-theoretical}]
    If $E/2 \le n/R^4$, we get the same bound as in the theorem \ref{theorem:spike-general}. So we can assume that $E/2 \ge n/R^4$. Thus
    \[
    \Paren{\frac{n}{AR^4}}^{A}
    \le 
    e^{n/R^4}
    \le 
    e^{E/2}
    \,.
    \]
	By Lemma \ref{lemma:fixed_graphs},
	\begin{align*}
\sum_{0 < |\alpha| \le D} \Paren{\E_{Y\sim \planted}\hermitepoly{Y}{\alpha}}^2 
&=
\sum_{2 \le E \le D} \; \sum_{\substack{1 \le A \le E/2 \\ 1 \le B \le E/2 }} \;\;
\Paren{\frac{{{60 \sqrt{e} \cdot AR^2\cdot \beta}}}
	{\min\Big\{{\sqrt{E\cdot d}, \;
			{k \cdot \Paren{\Abs{\ln{\Paren{{d\cdot E}/{k^2}}}} + 1}}}\Big\}}}^E
\\
&\le
\sum_{2 \le E \le D} \frac{E^2}{4}\cdot
\Paren{\frac{{{100 R^2\beta}}}
	{\min\Big\{{\sqrt{E\cdot d}/A, \;
			{\Paren{k/A} \cdot \Paren{\Abs{\ln{\Paren{{d\cdot A}/{k^2}}}} + 1}}}\Big\}}}^E
\\
&\le
\sum_{2 \le E \le D}
\Paren{\frac{{{200 R^2\beta}}}
	{\min\Big\{{\sqrt{d/n}, \;
			{\Paren{k/n} \cdot \Paren{\Abs{\ln{\Paren{{d\cdot n}/{k^2}}}} + 1}}}\Big\}}}^E\,.
\end{align*}
If
\[
\beta = \frac{\varepsilon}{R^2}\cdot \min\Big\{\sqrt{\frac{d}{n}},\;
\frac{k}{n}\cdot
\Paren{\,\Abs{\,\ln\Paren{\frac{n\cdot d}{k^2}}\,}+1}\Big\}\,,
\]
   where $0 < \varepsilon < \frac{1}{1000}$, then
\[
\sum_{0 < |\alpha| \le D} \Paren{\E_{Y\sim \planted}\hermitepoly{Y}{\alpha}}^2 
\le 1000^2\cdot\varepsilon^2\,,
\]
and using Corollary \ref{corollary:optimal_polynomial_statistic_hermite} we get the desired conclusion.
\end{proof}

\begin{@empty}
\end{@empty}

\subsection{Almost Gaussian vector in random subspace (proof)}\label{section:lowerbound-subspace-proof}
\label{sec:lowerbounds-subspace}
In this section we focus on Problem \ref{problem:almost_gaussian_model_special_case} and prove that if $d$ is significantly less than $n^{s/2+1}$, no low degree multilinear polynomial can distinguish between the planted and the null distribution.

The proof of Theorem \ref{theorem:lower_bound_main} relies on key lemmata which we provide below. The proof itself is then presented  at the end of the section.

\begin{lemma}\label{lemma:decomposing_expectation_hermite}
	Let $\alpha$ be a multiindex over $[n]\times [d]$ such that $\hermitepoly{Y}{\alpha}\in\Hermitemlpolys{D}$. Then 
	\[
	\Ep \hermitepoly{Y}{\alpha}	= 
	\E\Paren{\underset{i \in I_{\alpha}}{\tprod} \sigma_i^{\card{J_{\alpha,i}}}}
	\underset{j \in J_\alpha}{\tprod}
	\E\Brac{\underset{i \in I_{\alpha,j}}{\tprod}\Paren{z_{ij} + \tfrac{1}{\sqrt{n}}\paren{v_j-\tfrac{1}{\sqrt{n}}\underset{l \in [n]}{\sum}z_{lj}}}}\;,
	\]
	where $j \in [d],i \in [n]$, $\sigma_i:=\sqrt{n}u_i$ and $z_{ij}:=\sigma_i w_{ij}$.
	\begin{proof}
		We drop the subscript $\alpha$ for the exposition of the proof.
		\begin{align*}
		\Ep \hermitepoly{Y}{\alpha}	&= \Ep \underset{j \in J}{\tprod}\underset{i \in I_j}{\tprod} \yij\\
		&= \E\underset{j \in J}{\tprod}\underset{i \in I}{\tprod}\Brac{w_{ij}+u_i\Paren{v_j-\iprod{u,w_j}}}\\
		&=\E\underset{j \in J}{\tprod}\underset{i \in I_j}{\tprod}\Brac{w_{ij}+
			\sigma_i\Paren{\tfrac{1}{\sqrt{n}}v_j-\tfrac{1}{n}\underset{l \in [n]}{\sum}z_{lj}}}\\
		&=\E \underset{j \in J}{\tprod}\underset{i \in I_j}{\tprod}\sigma_i\Paren{z_{ij} + \tfrac{1}{\sqrt{n}}v_j-\tfrac{1}{n}\underset{l \in [n]}{\sum}z_{lj}}
		\tag{as $\sigma_iw_{ij}=\frac{w_{ij}}{\sigma_i}$}\\
		&=\E\Brac{ \Paren{\underset{j \in J}{\tprod}\underset{i \in I_j}{\tprod} \sigma_i}\underset{j \in J}{\tprod}\underset{i \in I_j}{\tprod}\Paren{z_{ij} + \tfrac{1}{\sqrt{n}}v_j-
				\tfrac{1}{n}\underset{l \in [n]}{\sum}z_{lj}}}\\
		&=\E\Paren{\underset{j \in J}{\tprod}\underset{i \in I_j}{\tprod} \sigma_i}\E\Brac{\underset{j \in J}{\tprod}\underset{i \in I_j}{\tprod}\Paren{z_{ij} + \tfrac{1}{\sqrt{n}}\paren{v_j-\tfrac{1}{\sqrt{n}}\underset{l \in [n]}{\sum}z_{lj}}}}\\
		&=
		\E\Paren{\underset{i \in I}{\tprod} \sigma_i^{\card{J_i}}}\underset{j \in J}{\tprod}
		\E\Brac{\underset{i \in I_j}{\tprod}\Paren{z_{ij} + \tfrac{1}{\sqrt{n}}\paren{v_j-\tfrac{1}{\sqrt{n}}\underset{l \in [n]}{\sum}z_{lj}}}}\,.
		\end{align*}
	\end{proof}
\end{lemma}

An immediate consequence of Lemma \ref{lemma:decomposing_expectation_hermite} is the following statement:
\begin{corollary}\label{corollary:expectation_odd_deg_hermite}
	Let $\alpha$ be a multiindex over $[n]\times [d]$ such that $\hermitepoly{Y}{\alpha}\in\Hermitepolys{D}$. If there exists $j \in J_\alpha$ (or $i \in I_\alpha$) such that $\card{I_{\alpha,j}}$ (respectively, $\card{J_{\alpha, i}}$) is odd, then $\Ep \hermitepoly{Y}{\alpha}=0$.
\end{corollary}

In the following lemma we use the fact that first $s$ moments of coordinates of $v$ coincide with Gaussian moments.

\begin{lemma}\label{lemma:expectation_low_deg_hermite}
	Let $s$ be the parameter of the planted distribution, let $\alpha$ be a multiindex over $[n]\times [d]$. Suppose that there exists $j_0\in J_{\alpha}$ such that $\card{I_{j_0}}\leq s$.  Then $\Ep \hermitepoly{Y}{\alpha}= 0$.
	\begin{proof}
		For simplicity we will the subscript $\alpha$. 
		If
		$\E\Paren{\underset{i \in I}{\tprod} \sigma_i^{\card{J_{i}}}} = 0$, the statement is obviously true. Assume that $\E\Paren{\underset{i \in I}{\tprod} \sigma_i^{\card{J_{i}}}} = 1$ (notice that this expectation can be only $0$ or $1$). Thus
		\begin{align*}
		\Ep \hermitepoly{Y}{\alpha}
		&= 
		\prod_{j\in J}
		\E \underset{i \in I_j}{\tprod}\Paren{z_{ij} + \tfrac{1}{\sqrt{n}}\paren{v_j-\tfrac{1}{\sqrt{n}}\underset{l \in [n]}{\sum}z_{lj}}}
		\\&=
		\E \underset{i \in I_{j_0}}{\tprod}\Paren{z_{ij_0} + \tfrac{1}{\sqrt{n}}\paren{v_{j_0}-\tfrac{1}{\sqrt{n}}\underset{l \in [n]}{\sum}z_{lj_0}}}
		\cdot
		\prod_{j\in J\setminus{\set{j_0}}}
		\E \underset{i \in I_j}{\tprod}\Paren{z_{ij} + \tfrac{1}{\sqrt{n}}\paren{v_j-\tfrac{1}{\sqrt{n}}\underset{l \in [n]}{\sum}z_{lj}}}\;.
		\end{align*}
		Since first $s$ moments of $v_{j_0}$ coincide with Gaussian moments,
		\[
		\E \underset{i \in I_{j_0}}{\tprod}\Paren{z_{ij_0} + \tfrac{1}{\sqrt{n}}\paren{v_{j_0}-\tfrac{1}{\sqrt{n}}
				\underset{l \in [n]}{\sum}z_{lj_0}}}
		=
		\E \underset{i \in I_{j_0}}{\tprod}
		\Paren{z_{ij_0} + \tfrac{1}{\sqrt{n}}\paren{\zeta-\tfrac{1}{\sqrt{n}}
				\underset{l \in [n]}{\sum}z_{lj_0}}}\,,
		\]
		where $\zeta$ is a standard Gaussian variable that is independent from all $z_{ij_0}$.
		Let $\xi_i = {z_{ij_0} + \tfrac{1}{\sqrt{n}}\paren{\zeta-\tfrac{1}{\sqrt{n}}
				\underset{l \in [n]}{\sum}z_{lj_0}}}$. Let's show that  $\xi \sim N(0, \Id_n)$. $\xi\in\R^n$ is a linear transformation of $\zeta, z_{1j_0},\ldots, z_{nj_0}$:
		\[
		\xi = A\,
		\begin{pmatrix}
		\zeta \\
		z_{1j_0}\\
		\vdots \\
		z_{nj_0}
		\end{pmatrix}\;,
		\]
		where $A$ is an $n\times(n+1)$ matrix with rows 
		$\transpose{A_i} = (\frac{1}{\sqrt{n}}, \frac{1}{n},\ldots,\frac{1}{n},\underbrace{(1-\frac{1}{n})}_{i+1},\frac{1}{n},\ldots, \frac{1}{n})$. The rows of $A$ are orthonormal: for all $i\in [n]$
		\[
		(A\transpose{A})_{ii} = \frac{1}{n} + (1-\frac{1}{n})^2 +\frac{n-1}{n^2} = 1 - \frac{2}{n} + \frac{1}{n^2}  + \frac{1}{n} + \frac{1}{n} - \frac{1}{n^2} = 1\,,
		\]
		and for all different $i, l\in [n]$
		\[
		(A\transpose{A})_{il} = \frac{1}{n} - \frac{2}{n}(1-\frac{1}{n}) +\frac{n-2}{n^2} = \frac{1}{n} - \frac{2}{n} +\frac{2}{n^2}  + \frac{1}{n} - \frac{2}{n^2} = 0\,.
		\]
		Hence $A\transpose{A} = \Id_n$ and $\xi \sim N(0, \Id_n)$. Therefore,
		\[
		\E \underset{i \in I_{j_0}}{\tprod}\Paren{z_{ij_0} + \tfrac{1}{\sqrt{n}}\paren{v_{j_0}-\tfrac{1}{\sqrt{n}}
				\underset{l \in [n]}{\sum}z_{lj_0}}}
		= 
		\E \underset{i \in I_{j_0}}{\tprod} \xi_i = 0\,.
		\]
	\end{proof}
\end{lemma}

\begin{lemma}\label{lemma:expectation_circle}
	Let $s,\delta,\lambda$ be the same as in the statement of Theorem \ref{theorem:lower_bound_main}. Let $j\in[d], I_j \subseteq [n]$ with even cardinality $\card{I_j} > s$.  Then, if $\card{I_j}  \le \frac{\lambda^2}{100}$,
	\[
	\E\Brac{\underset{i \in I_j}{\tprod}\Paren{z_{ij} + \tfrac{1}{\sqrt{n}}\paren{v_j-\tfrac{1}{\sqrt{n}}\underset{l \in [n]}{\sum}z_{l}}}}
	\le
	     \delta\Paren{\frac{2^{20}\cdot s\cdot \lambda}{\sqrt{n}}}^{\card{I_j}}\,,
	\]
	and if $\card{I_j} > \frac{\lambda^2}{100}$,
	\[
	\E\Brac{\underset{i \in I_j}{\tprod}\Paren{z_{ij} + \tfrac{1}{\sqrt{n}}\paren{v_j-\tfrac{1}{\sqrt{n}}\underset{l \in [n]}{\sum}z_{l}}}}
	\le
		\Paren{\frac{100\sqrt{\card{I_j}}}{\sqrt{n}}}^{\card{I_j}}\,.
	\]
	\begin{proof}
		We drop the subscript $j$ to simplify the notation 
		(in particular, in this proof we denote $v_j$ by $v$). 
		By symmetry of the Gaussian distribution, opening up the product we see that in order for a monomial to have non-zero expectation, for any left end term $z_i$ there must be a corresponding right term $\tfrac{1}{\sqrt{n}}\paren{v-\tfrac{1}{\sqrt{n}}\underset{l \in [n]}{\sum}z_{l}}$.
		Hence:
		\begin{align*}
		\E\Brac{\underset{i \in I}{\tprod}\Paren{z_{i} + \tfrac{1}{\sqrt{n}}\paren{v-\tfrac{1}{\sqrt{n}}\underset{l \in [n]}{\sum}z_{l}}}} 
		&=
		\underset{r=0}{\overset{\card{I}/2}{\sum}}
		\binom{\card{I}}{2r}\binom{2r}{r}
		\E\Brac{
			\Paren{\underset{i\in \brac{r}}{\tprod}z_i} 
			\frac{1}{n^{\card{I}/2-r/2}} 
			\Paren{v-\tfrac{1}{\sqrt{n}}\underset{l \in [n]}{\sum}z_{l}}^{\card{I}-r}}
		\\
		&=
	     \frac{1}{n^{\card{I}/2}} 
		\underset{r=0}{\overset{\card{I}/2}{\sum}}
		\binom{\card{I}}{2r}\binom{2r}{r}
		\E\Brac{
			\Paren{\underset{i\in \brac{r}}{\tprod}z_i^2} 
			\Paren{v-\tfrac{1}{\sqrt{n}}\underset{l \in [n]}{\sum}z_{l}}^{\card{I}-2r}} \,.
		\end{align*}
		Since $v$ is symmetric:
		\[
		\E\Brac{
			\Paren{\underset{i\in \brac{r}}{\tprod}z_i^2} 
			\Paren{v-\tfrac{1}{\sqrt{n}}\underset{l \in [n]}{\sum}z_{l}}^{\card{I}-2r}}
		=
		\sum_{m=0}^{\card{I}/2 - r} 
		\E\brac{v^{\card{I} -2r -2m}}
		\E\Brac{
			\Paren{\underset{i\in \brac{r}}{\tprod}z_i^2} \cdot 
			\Paren{\frac{1}{\sqrt{n}}\underset{l \in [n]}{\sum}z_{l}}^{2m}}\,.
		\]
		By Cauchy--Schwarz:
		\[
		\E\Brac{
			\Paren{\underset{i\in \brac{r}}{\tprod}z_i^2}
			\Paren{\frac{1}{\sqrt{n}}\underset{l \in [n]}{\sum}z_{l}}^{2m}} \le 
		\Paren{\E\underset{i\in \brac{r}}{\tprod}z_i^4}^{1/2} 
		\Paren{\E\;\Paren{\frac{1}{\sqrt{n}}\underset{l \in [n]}{\sum}z_{l}}^{4m}}^{1/2} \le 
		3^{r/2} \cdot \Paren{2m}^{m}\,.
		\]
		Hence,
		\begin{align*}
		\E\Brac{
			\Paren{\underset{i\in \brac{r}}{\tprod}z_i^2} 
			\Paren{v-\tfrac{1}{\sqrt{n}}\underset{l \in [n]}{\sum}z_{l}}^{\card{I}-2r}}
		&\le
		\sum_{m=0}^{\card{I}/2 - r} 
		\E\brac{v^{\card{I} -2r -2k}}\cdot 3^{r/2} \cdot \Paren{2m}^{m} \\
		&\le
		3^{r/2} \sum_{m=0}^{\card{I}/2 - r} 
		\Paren{\delta  \lambda^{\card{I} - 2r - 2m} + (10\sqrt{s\ln s})^{\card{I} - 2r - 2m}}\cdot \Paren{2m}^{m}\\
		&\le
		3^{\card{I}/4} \sum_{m=0}^{\card{I}/2} 
		\Paren{\delta  \lambda^{\card{I} - 2m} + (10\sqrt{s\ln s})^{\card{I} - 2m}}\cdot \Paren{2m}^{m}\,.
		\end{align*}
		
		Let $M = \max\set{\delta \lambda^{\card{I}}, \card{I}^{\card{I}/2}, (10\sqrt{s\ln s})^{\card{I}}}$. Thus $2M\ge \Paren{\delta  \lambda^{\card{I} - 2m} + (10\sqrt{s\ln s})^{\card{I} - 2m}}\cdot \Paren{2m}^{m}$.
		We get:
		\begin{align*}
		\E\Brac{\underset{i \in I}{\tprod}\Paren{z_{i} + \tfrac{1}{\sqrt{n}}\paren{v-\tfrac{1}{\sqrt{n}}\underset{l \in [n]}{\sum}z_{l}}}} 
		&\le
		\frac{1}{n^{\card{I}/2}}
		\underset{r=0}{\overset{\card{I}/2}{\sum}}
		\binom{\card{I}}{2r}\binom{2r}{r}
		3^{\card{I}/4}\cdot \card{I} \cdot 2M
		\\
		&\le
		\frac{1}{n^{\card{I}/2}}\cdot
		2^{\card{I}}\cdot 2^{\card{I}}\cdot 3^{\card{I}/4}\cdot 2^{\card{I}/2}
		\cdot
		M
		\\
		&\le
		\Paren{\frac{10}{\sqrt{n}}}^{\card{I}} \cdot M
		\end{align*}
		
		Consider the case $\card{I} > \frac{\lambda^2}{100}$. In this case, 
		$M \le 10^{\card{I}}\cdot\card{I}^{\card{I}/2}$. Hence
		\begin{align*}
		\E\Brac{\underset{i \in I}{\tprod}\Paren{z_{i} + \tfrac{1}{\sqrt{n}}\paren{v-\tfrac{1}{\sqrt{n}}\underset{l \in [n]}{\sum}z_{l}}}} 
		&\le
		\Paren{\frac{100\sqrt{\card{I}}}{\sqrt{n}}}^{\card{I}}\,.
		\end{align*}
		
		Now consider the case $\card{I} \le \frac{\lambda^2}{100}$. 
		If $\card{I} \ge 10s$, then
		$\delta\lambda^{\card{I}} \ge \Paren{\lambda/2}^{\card{I}-s} \ge \card{I}^{\card{I}/2}$.
		Indeed, the inequality holds if
		$\card{I} \le \Paren{\lambda/2}^{1.8}$, and if $\card{I} >  \Paren{\lambda/2}^{1.8}$, then $\Paren{\lambda/2}^{s/\card{I}}\sqrt{\card{I}}$ is monotone in $I$, so 
		\[
		\Paren{\lambda/2}^{s/\card{I}}\sqrt{\card{I}} \le 
		0.1 \cdot \lambda \cdot \Paren{\lambda/2}^{100s/\lambda^2} 
		\le 0.1\cdot \lambda \cdot \Paren{\lambda/2}^{1/\ln{\lambda}} 
		\le \frac{1}{2}\lambda
		\,,
		\]
		since $\lambda^2 \ge10000s\ln s$.
		If $\card{I} < 10s$, then $\card{I}^{\card{I}/2} <(10\sqrt{s\ln s})^{\card{I}}$. Therefore,
	\begin{align*}
	\E\Brac{\underset{i \in I}{\tprod}\Paren{z_{i} + \tfrac{1}{\sqrt{n}}\paren{v-\tfrac{1}{\sqrt{n}}\underset{l \in [n]}{\sum}z_{l}}}} 
	&\le
	\Paren{\frac{10}{\sqrt{n}}}^{\card{I}} \cdot \max\set{\delta\lambda^{\card{I}},
		 (10\sqrt{s\ln s})^{\card{I}}}
   \\
   &\le
   \Paren{\frac{10}{\sqrt{n}}}^{\card{I}} \delta\lambda^{s+2}\max\set{\lambda^{\card{I}-(s+2)}, \frac{1}{\delta\lambda^{s+2}}(10\sqrt{s\ln s})^{\card{I}}}
      \\
   &\le
      \Paren{\frac{10}{\sqrt{n}}}^{\card{I}} \delta\lambda^{s+2}\max\set{\lambda^{\card{I}-(s+2)}, 2^{10s}(10\sqrt{s\ln s})^{\card{I}}}
    \\
      &\le
     \delta\Paren{\frac{2^{20}\cdot \sqrt{s\ln s}\cdot \lambda}{\sqrt{n}}}^{\card{I}}\,.
	\end{align*}		
		
	\end{proof}
\end{lemma}

We are now ready to prove Theorem \ref{theorem:lower_bound_main}.
{
	\renewcommand*{\proofname}{Proof of Theorem \ref{theorem:lower_bound_main}}
	\begin{proof}
		For all positive integers $A$, $B$, $B'$, $E$ and $E'$ consider the set $\cG_s(A, B, B', E, E')$ of bipartite graphs $G_\alpha$  such that  $\card{I_\alpha} =  A$, 
		$\card{J_\alpha} = B$ and $\card{\alpha} = E$, $B'=\Card{\set{j\in J_\alpha\;|\; \card{I_j}\le \frac{\lambda^2}{100}}}$, $E'$ is a number of edges adjacent to $\set{j\in J_\alpha\;|\; \card{I_j}\le \frac{\lambda^2}{100}}$,  and all vertices of $G_\alpha$ have even degree strictly greater than $s$. Let $B'' = B-B'$ and $E'' = E-E'$.

		By lemma \ref{lemma:expectation_circle},
		\begin{align*}
		\underset{\hermitepoly{Y}{\alpha}\in \Hermitemlpolys{D}}{\sum} \Paren{\Ep\hermitepoly{Y}{\alpha}}^2
		&\le
		\sum\limits_{2(s+2) \le E \le D}\;
		\sum_{\substack{A, B, B', E' \\\cG_s(A,B,B',E,E')\neq\emptyset}}
		\begin{aligned}[t]
		&
		\binom{n}{A}
		\binom{d}{B}
		{\frac{\Paren{A B}^E}{E!}}
		\\&
		\cdot
		\delta^{2B'}
		\Paren{\frac{2^{20} \cdot \sqrt{s\ln s}\cdot \lambda}{\sqrt{n}}}^{2E'}
		\Paren{\frac{100 \sqrt{E}}{\sqrt{n}}}^{2E''}
		\end{aligned}
		\\
		&\le
		\sum\limits_{2(s+2) \le E \le D}\;
		\sum_{\substack{A, B, B', E' \\\cG_s(A,B,B',E,E')\neq\emptyset}}
		\begin{aligned}[t]
		&
		\Paren{\frac{e n}{A}}^A
		\Paren{\frac{e A B}{E}}^E
		\\&
		\cdot
		d^B
		\Paren{\delta^2}^{B'}
		\Paren{\frac{2^{20}  \cdot \sqrt{s\ln s} \cdot \lambda}{\sqrt{n}}}^{2E'}
		\Paren{\frac{100 \sqrt{E}}{\sqrt{n}}}^{2E''}\,.
		\end{aligned}
		\end{align*}
		Since $A \le D/2 = o(n)$, 	$\Paren{\frac{e n}{A}}^A$ is monotone in $A$. 
		Also notice that  if $\cG_s(A, B, B', E, E') \neq \emptyset$, $B' \le E'/(s+2)$ and $B'' \le 100\cdot E''/\lambda^2 \le E''/(s+2)$. Let $\phi(B',B'',E',E'')$ be an zero-one indicator that is one if and only if there exists $A$ such that $\cG_s(A,B'+B'',B',E'+E'',E')\neq\emptyset$.

		Consider the case $\delta^2 d \ge 1$. Assume that 
		$d = o\Paren{\frac{1}{\lambda^4}\cdot\Paren{\frac{n}{2^{120}\ln^2 s  D}}^{(s+2)/2}}$.  
		Since $D \le n^{0.33}$ and $\lambda^2 \ge 100000s\ln s$,  $d = o\Paren{\Paren{\frac{n}{10^{20} D^3}}^{\lambda^2/200}}$. Hence
		\begin{align*}
		\underset{\hermitepoly{Y}{\alpha}\in \Hermitemlpolys{D}}{\sum} \Paren{\Ep\hermitepoly{Y}{\alpha}}^2
		&\le
\sum_{\substack{0\le B', B'' \le D/2}}\;
\sum_{\substack{0 \le E', E'' \le D}}
		\begin{aligned}[t]
		&
		\phi(B',B'',E',E'')
		\Paren{\frac{n}{E}}^{E/2}
		\Paren{\frac{E}{s+2}}^{E}
		\\
		&\cdot
				\Paren{\delta^2 d}^{B'}
		\Paren{\frac{2^{30} \cdot \sqrt{s\ln s} \cdot \lambda}{\sqrt{n}}}^{2E'}
		d^{B''}
		\Paren{\frac{10^5 \sqrt{E}}{\sqrt{n}}}^{2E''}
		\end{aligned}
		\\
		&\le
\sum_{\substack{0\le B', B'' \le D/2}}\;
\sum_{\substack{0 \le E', E'' \le D}}
		\begin{aligned}[t]
		&
		\phi(B',B'',E',E'')
		\Paren{\delta^2 d}^{B'}
		\Paren{\frac{2^{120} \lambda^4 D \ln^2 s }{n}}^{E'/2}
		\\
		&\cdot
		d^{B''}
		\Paren{\frac{10^{20} D^3}{n}}^{E''/2}
		\end{aligned}
		\\
		&\le 
		2\sum_{B'=1}^{\infty}
\Paren{\delta^2 d
	\Paren{\frac{2^{120} \lambda^4 D\ln^2 s}{n}}^{(s+2)/2}}^{B'} +
2\sum_{B''=1}^{\infty}
\Paren{d
	\Paren{\frac{10^{20} D^3}{n}}^{\lambda^2/200}}^{B''}
		\\
		&\le
		2\sum_{B'=1}^{\infty}
		\Paren{\lambda^4 d
			\Paren{\frac{2^{100} D\ln^2 s}{n}}^{(s+2)/2}}^{B'} + o(1)
		\\
		&\le
		o(1)\,.
		\end{align*}
		
		Now condiser the case $\delta^2 d < 1$. Since $\lambda \ge 2^{10}$, $\lambda^{2s+2} > \frac{1}{\delta}$ and 
		\[
		(2s+2)\ln\lambda > \ln\Paren{\frac{1}{\delta} }\ge \ln\Paren{\frac{1}{\delta}} + \ln\Paren{\frac{1}{\delta d}} = \ln\Paren{\frac{1}{\delta^2 d}}\,.
		\]
		Since $\lambda^2 > 100000s\ln s$, $\lambda^2/100 \ge \ln\Paren{\frac{1}{\delta^2 d}}$ and $B''\le E/\abs{\ln\Paren{\delta^2 d}}$. Let $M = \max\set{B', B''}$. Recall that 
		$\lambda^4 Ds^2\ln^2 s  = o\Paren{n\log^2\Paren{\delta^2d}}$. It follows that
		\begin{align*}
		\underset{\hermitepoly{Y}{\alpha}\in \Hermitemlpolys{D}}{\sum} \Paren{\Ep\hermitepoly{Y}{\alpha}}^2
		&\le
\sum_{\substack{0\le B', B'' \le D/2}}\;
\sum_{\substack{0 \le E', E'' \le D}}
		\begin{aligned}[t]
		&
		\phi(B',B'',E',E'')
		\Paren{\frac{n}{E}}^{E/2}
		M^E
		\\
		&\cdot
		\Paren{\delta^2 d}^{B'}
		\Paren{\frac{2^{30} \cdot \sqrt{s\ln s} \cdot \lambda}{\sqrt{n}}}^{2E'}
		d^{B''}
		\Paren{\frac{10^5 \sqrt{E}}{\sqrt{n}}}^{2E''}
		\end{aligned}
		\\
		&\le
\sum_{\substack{0\le B', B'' \le D/2}}\;
\sum_{\substack{0 \le E', E'' \le D}}
		\begin{aligned}[t]
		&
		\phi(B',B'',E',E'')
		\Paren{\Paren{\delta^2 d}^{\frac{2B'}{E'}}M\cdot\frac{2^{120} \lambda^4 s^2\ln^2 s }{n}}^{E'/2}
		\\
		&\cdot
		d^{B''}
		\Paren{\frac{10^{20} D^3}{n}}^{E''/2}
		\end{aligned}
		\\
		&\le
\sum_{\substack{0\le B'' \le D/2}}\;
\sum_{\substack{0 \le E', E'' \le D}}\;
\sum_{0\le B' \le E'}
		\begin{aligned}[t]
		&
		\phi(B',B'',E',E'')
		\Paren{\frac{2^{120} \lambda^4 Ds^2\ln^2 s }{\ln^2\Paren{\delta^2d}n}}^{E'/2}
		\\
		&\cdot
		d^{B''}
		\Paren{\frac{10^{20} D^3}{n}}^{E''/2}
		\end{aligned}
		\\
		&\le
		\sum_{E'=1}^{\infty}
		\Paren{\frac{2^{130} \lambda^4 Ds^2\ln^2 s }{\ln^2\Paren{\delta^2d}n}}^{E'/2} +
		2\sum_{B''=1}^{\infty}
		\Paren{d
			\Paren{\frac{10^{20} D^3}{n}}^{\lambda^2/200}}^{B''}
		\\
		&\le
		o(1)\,.
		\end{align*}
		By Corollary \ref{corollary:optimal_polynomial_statistic_hermite}, we get the desired colclusion.
	\end{proof}
}

\section{Polynomial-based algorithm with the right log factor}\label{sec:polynomials}

In this section we will prove the following theorem.

\begin{theorem}\label{theorem:polynomial-algorithm-main}
	Let $d^{1-o(1)} \le k^2 \le o(d)$ and $n \ge \Omega\Paren{\log^5 d}$ as $d\to \infty$.
	Let $Y= \sqrt{\beta}u_0\transpose{v_0} + W$, where $\beta > 0$ and
		\begin{itemize}
		\item 	$u_0\in \R^n$ is a random vector with independent entries such that for all $i\in[n]$, $\E u_0(i)^2 = 1$ and $\E u_0(i)^4 \le O(1)$.
		\item $W\in\R^{n\times d}$ 
		is a random matrix with independent entries such that for all $i\in n$ and $j\in [d]$, 
		$\E W_{ij} = 0$ and  $\E W_{ij}^2 = 1$, and $W$ and $u_0$ are independent.
		\item $v_0\in \R^d$ is a (non-random) unit vector with entries in $\set{0, \pm 1/\sqrt{k}}$.
			
		\end{itemize}

	Suppose that 
	\[
	\beta \ge C^*\frac{k}{\sqrt{n}}\sqrt{\log\frac{d}{k^2} + \frac{\log d}{\log n}}
	\]
	for some large enough  constant $C^*$. 
	Then there exists a probabilistic algorithm that given $Y$ as input, in time $\Paren{nd}^{O(1)}$ outputs a unit vector $\hat{v}\in\R^d$ such that
	\[
	1 - \iprod{\hat{v},v_0}^2 \le o(1)
	\]
	with probability $1-o(1)$ (with respect to the distribution of $Y$ and the randomness of the algoruthm).
\end{theorem}
\begin{remark}
	The algorithm also works for $\Omega(d) \le k^2 \le d/2$ in the sense that $1 - \iprod{\hat{v},v_0}^2$ is small (we can make it arbitrarily small by increasing $C^*$), but $\iprod{\hat{v},v_0}^2$ might not tend to one in this regime.
\end{remark}

The advantage of algorithm \ref{theorem:polynomial-algorithm-main} compared to Covariance Thresholding (and other algorithms) is that it works for 
$\beta = o\Paren{\min\Set{\frac{k}{\sqrt{n}}\log d}, \sqrt{\frac{d}{n}}}$ and small $n$ (for example, $n = d^{0.99}$, or  $n = d^{0.01}$), while Covariance Thresholding can work with $\beta = o\Paren{\min\Set{\frac{k}{\sqrt{n}}\log d}, \sqrt{\frac{d}{n}}}$ only if $n > d^{1-o(1)}$ (see \cref{theorem:small-sample-ct}). Another advantage is that other algorithms for sparse PCA use many properties of Gaussian distribution (for example, they use $\chi^2$ tail bounds),
while this algorithm requires only assumptions on first two moments of $W$.

The algorithm \ref{theorem:polynomial-algorithm-main} will use low degree polynomials to estimate the entries of $v_0\transpose{v_0}$.
We give a precise description of polynomials that we use in the following subsection.

\subsection{Low degree polynomials as estimators}
To work with polynomials we introduce the following notation:

For a multi-index $\alpha$ over $[n]\times [d]$, let $I_\alpha:=\{ i\in[n] : \paren{i,j}\in\alpha \text{ for some }j\in[d] \}$  and similarly  $J_\alpha:=\{ j\in[d] : \paren{i,j}\in\alpha \text{ for some }i\in[n] \}$. For $j \in [d]$, let $I_{\alpha,j}:=\{ i\in[n] : \paren{i,j}\in \alpha \}$, and similarly let $J_{\alpha,i}:=\{j \in [d]:\paren{i,j}\in \alpha \}$. 
For a matrix $X\in\R^{n\times d}$ denote
$X^{\alpha} :=  \prod_{(i ,j)\in \alpha}  X_{ij}^{\alpha_{ij}}$.
Note that every multi-index $\alpha$ over $[n]\times [d]$
can be represented as a bipartite multigraph 
$G_\alpha = \Paren{I_\alpha \bigcup J_\alpha, E_\alpha}$ 
such that each edge $\{i,j\}$ has multiplicity $\alpha_{ij}$. In this representation
the set $J_{\alpha,i}$ corresponds to the neighborhood of the vertex $i$ and the set ${I_{\alpha,j}}$ corresponds to the neighborhood of $j$. If $\alpha$ is multilinear, 
$G_\alpha$ is just a graph (i.e. multiplicity of each edge is 1).

Now we define the graphs which represent the monomials that we will use.
\begin{definition}
Let $l\in \N$ and let $b\in \N$ be an odd number, and fix two different $j_0\in [d]$ and $j_l\in [d]$.
We define $\cG_{j_0 j_l}(b,l)$ to be the set of bipartite graphs with $2bl$ edges such that any $i\in I_\alpha$ has degree $2$, $J_\alpha = \set{j_0, j_1,\ldots, j_{l-1}, j_l}$ for distinct $j_1,\ldots, j_{l-1} \in [d]$ (different from $j_0$ and $j_l$), and for any $1 \le s \le l$ there are exactly $b$ different vertices $i_{s1},\ldots,i_{sb}\in I_\alpha$ that are adjacent to both $j_{s-1}$, $j_{s}$ (see \cref{fig:graph}).
\end{definition}

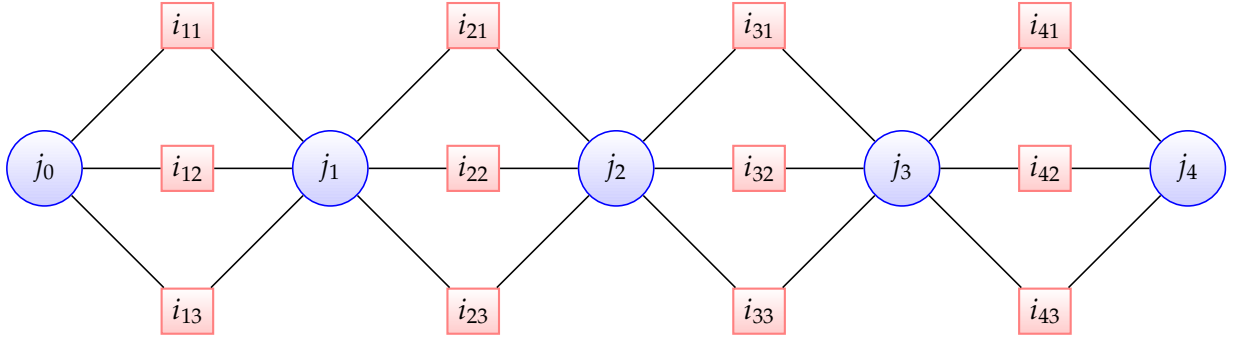
\begin{figure}[ht]
\centering
		\begin {tikzpicture}[-latex ,auto ,node distance =1.9 cm and 1.9cm ,on grid , semithick]
		\tikzstyle{column}=[circle, top color =white , bottom color = blue!20, draw,blue , text=black , minimum width =1 cm]
		\tikzstyle{row}=[rectangle,thick,top color =white , bottom color = red!20, draw,red!50, text = black, minimum size=6mm]
		
		\node[row] (i13) {$i_{13}$};
		\node[column] (j0) [above left=of i13] {$j_0$};
		\node[row] (i12) [right =of j0] {$i_{12}$};
		\node[row] (i11) [above right =of j0] {$i_{11}$};
		\path (j0) edge [-] node[above] {} (i11);
		\path (j0) edge [-] node[above] {} (i12);
		\path (j0) edge [-] node[above] {} (i13);
		
		\node[column] (j1) [right=of i12] {$j_1$};
		\node[row] (i21) [above right =of j1] {$i_{21}$};
		\node[row] (i22) [right =of j1] {$i_{22}$};
		\node[row] (i23) [below right =of j1] {$i_{23}$};
		\path (j1) edge [-] node[above] {} (i11);
		\path (j1) edge [-] node[above] {} (i12);
		\path (j1) edge [-] node[above] {} (i13);
		\path (j1) edge [-] node[above] {} (i21);
		\path (j1) edge [-] node[above] {} (i22);
		\path (j1) edge [-] node[above] {} (i23);
		
		\node[column] (j2) [right=of i22] {$j_2$};
		\node[row] (i31) [above right =of j2] {$i_{31}$};
		\node[row] (i32) [right =of j2] {$i_{32}$};
		\node[row] (i33) [below right =of j2] {$i_{33}$};
		\path (j2) edge [-] node[above] {} (i21);
		\path (j2) edge [-] node[above] {} (i22);
		\path (j2) edge [-] node[above] {} (i23);
		\path (j2) edge [-] node[above] {} (i31);
		\path (j2) edge [-] node[above] {} (i32);
		\path (j2) edge [-] node[above] {} (i33);
		
		\node[column] (j3) [right=of i32] {$j_3$};
		\node[row] (i41) [above right =of j3] {$i_{41}$};
		\node[row] (i42) [right =of j3] {$i_{42}$};
		\node[row] (i43) [below right =of j3] {$i_{43}$};
		\path (j3) edge [-] node[above] {} (i31);
		\path (j3) edge [-] node[above] {} (i32);
		\path (j3) edge [-] node[above] {} (i33);
		\path (j3) edge [-] node[above] {} (i41);
		\path (j3) edge [-] node[above] {} (i42);
		\path (j3) edge [-] node[above] {} (i43);
		
		\node[column] (j4) [right=of i42] {$j_4$};
		\path (j4) edge [-] node[above] {} (i41);
		\path (j4) edge [-] node[above] {} (i42);
		\path (j4) edge [-] node[above] {} (i43);
	\end{tikzpicture}
	\caption{A graph from $\cG_{j_0 j_l}(b,l)$ for $b = 3$ and $l=4$.}\label{fig:graph}

\end{figure}

Now we are ready to define the polynomials that we will use.

\begin{definition}
Let $b$ be the smallest odd number that is greater than $C^{**}\cdot \Paren{\log\frac{d}{k^2} + \frac{\log d}{\log n}}$ for some constant $C^{**}\ge 100$, and let $l$ be the smallest integer such that $bl\ge \log d$ and let $j_0 < j_l$.
We define
\[
P_{j_0j_l}(Y) = \frac{1}{\kappa}\sum_{\alpha \in \cG_{j_0 j_l}(b,l)} Y^{\alpha}\,,
\]
where
\[
\kappa =  k\cdots\Paren{k - l + 2}\cdot \binom{n}{b}\cdots\binom{n-(l-1)b}{b} \Paren{\frac{\beta}{k}}^{bl}\cdot k\,
\]
is a normalization factor so that $\E P_{j_0j_l}(Y) = v_{j_0}v_{j_l}$.
\end{definition}

Note that under assumptions of \cref{theorem:polynomial-algorithm-main},  
$b = o(\log d)$, so $l \to \infty$.

The following lemma shows that the expectation of $P_{j_0j_l}(Y)$ is indeed $v_0(i)v_0(j)$.

\begin{lemma}\label{lem:polynomial-algorithm-expectation}
	\[
	\E P_{j_0j_l}(Y) =  v_0(j_0)v_0(j_l)\,,
	\]

	\begin{proof}
		By construction of $\alpha$,
		\[
		\E Y^{\alpha} = \frac{1}{\kappa}\beta^{bl}\Paren{\frac{1}{k^b}}^{l-1} v_0(j_0)^bv_0(j_l)^b = 
		 \frac{1}{\kappa}\Paren{\frac{\beta}{k}}^{l} k\cdot v_0(j_0)v_0(j_l)\,.
		\]
		If $v_0(j_0)v_0(j_l)=0$, the statement is true. Assume that it is not zero.
		 Number of nonzero terms is $k\cdots\Paren{k - l + 2}\cdot \binom{n}{b}\cdots\binom{n-(l-1)b}{b}$, so we get the desired equality.
	\end{proof}
\end{lemma}

Now let's bound the variance of $P_{j_0j_l}(Y)$.
\begin{lemma}\label{lem:polynomial-algorithm-variance}
	Suppose that $\beta = C\cdot  \frac{k}{\sqrt{n}}\cdot \sqrt{b}$ for some constant $C \ge 100$.
	Then
	\[
		\Var_\mu P_{j_0j_l}(Y)  \le 
o\Paren{v_0(j_0)^{2}v_0(j_l)^{2} + \frac{1}{d}\Paren{v_0(j_0)^{2} + v_0(j_l)^{2}} + \frac{1}{d^2}}\,.
	\]
\end{lemma}
\begin{proof}

For simplicity we will write $\cG$ instead of $\cG_{j_0 j_l}(b,l)$. 

We need to bound $\E P^2_{j_0j_l}(Y) = \frac{1}{\kappa^2}\sum_{\alpha,\alpha'\in \cG} \E Y^{\alpha}Y^{\alpha'}$. Note that $\E Y^{\alpha}Y^{\alpha'}$ cannot be negative.

The terms in $\E P^2_{j_0j_l}(Y)$ that correspond to non-intersecting graphs have contribution at most $\Paren{\E P_{j_0j_l}(Y)}^2$. Indeed,
\begin{align*}
\Sdiff := \sum_{\substack{I_\alpha \cap I_{\alpha'} = \emptyset \\
		J_\alpha\cap J_{\alpha'} = \set{j_0, j_l}}}\;
\E Y^{\alpha}Y^{\alpha'} =
\sum_{\substack{I_\alpha \cap I_{\alpha'} = \emptyset \\
		J_\alpha\cap J_{\alpha'} = \set{j_0, j_l}}}\;
\E Y^{\alpha}\E Y^{\alpha'} 
\le \Paren{\E_\mu p(Y)}^2\,.
\end{align*}

To bound the other terms will need the following lemma:
\begin{lemma}\label{lem:polynomial-algorithm-contraction}
	Let $S$ be a set of pairs $(i,j)$ such that that for $\alpha, \alpha'\in \cG$
	\[
	Y^{\alpha}Y^{\alpha'} = \prod\limits_{(i,j)\in S}Y_{ij}^2g(Y)\,,
	\] 
	where $g(Y)$ is some monomial. Then
	\[
	\E_\mu Y^{\alpha}Y^{\alpha'}  = \Paren{1 + o(1)} \E_\mu g(Y)\,.
	\]
	\begin{proof}
		Assume that $S = \set{(i',j')}$ (i.e. $\card{S} = 1$). Since 
		$\beta = O\Paren{\frac{k}{\sqrt{n}}\sqrt{b}}$,
		\[
		\E_\mu Y^{\alpha}Y^{\alpha'}  = 
		\E_\mu\Paren{w_{i'j'} + \sqrt{\beta}u_{i'}v_{j'}}^2g(Y) =
		\E_\mu g(Y) + \beta\E_\mu u_{i'}^2v^2_{j'} g(Y) 
		= \E_\mu g(Y) + O\Paren{\sqrt{\frac{b}{n}}}\E_\mu g(Y)\,.
		\]
		Hence for arbitrary $S$
		\[
		\E_\mu Y^{\alpha}Y^{\alpha'}  = 
		\Paren{1+O\Paren{\sqrt{\frac{b}{n}}}}^{\card{S}}\E_\mu g(Y)\,.
		\]
		Since $\card{S} \le O(bl) \le O(\log d) \le o(\sqrt{\frac{n}{b}})$, we get the desired bound.
	\end{proof}
\end{lemma}

\cref{lem:polynomial-algorithm-contraction} implies that 
$\sum_{\alpha} \E_\mu \Paren{Y^{\alpha}}^2 \le (1+o(1))\card{\cG} $.
Note that 
\[
\kappa = (1-o(1))\binom{n}{b}^l k^{l} \Paren{\frac{\beta}{k}}^{bl}\,,
\]
since $l \le o(\sqrt{k})$ and  $bl \le O\Paren{\log d}\le o\Paren{\sqrt{n}}$. Similarly,
$\card{\cG} = \Paren{1-o(1)}\cdot d^{l-1}\binom{n}{b}^l$. Hence
\[
\sum_{\alpha} \E_\mu \Paren{Y^{\alpha}}^2
\le (1+o(1))\cdot d^{l-1}\binom{n}{b}^l
\le (1+o(1))\kappa^2 \Paren{\frac{1}{d}}
\Paren{\Paren{\frac{k^2 b}{\beta^2 n}}^{b} \Paren{\frac{d}{k^2}}}^l 
\le \frac{\kappa^2}{d^{10}}\,.
\]
Here we used $bl \ge \log d$.

To bound the other terms in $\E P^2_{j_0j_l}(Y)$  we define some notions related to the multigraphs which correspond to 
$ Y^{\alpha}Y^{\alpha'}$.

We will call $j\in J_\alpha$ \emph{circles} and $i\in I_\alpha$ \emph{boxes}.
A circle $j$ is called \emph{blocked} if  each box that is adjacent to $j$ has two parrallel edges to $j$. Equivalently, if $v_j$ appears in $ Y^{\alpha}Y^{\alpha'}$ only in squared parentheses, 
i.e. only in parantheses of the form $(w_{ij}+ \sqrt{\beta}u_i v_j)^2$.

If two blocked circles are adjacent to the same box (which means that they are adjacent to exactly $b$ same boxes), we call such circles \emph{consecutive}.
A maximal sequence of consecutive blocked circles is called a \emph{blocked segment}. The \emph{endpoints} of a blocked segment are circles from this segment that are either $j_0$, $j_l$, or  circles that share a box with a non-blocked circle. 

A block segment is called \emph{closed} if the boxes that are adjacent to the endpoints of this block are adjacent to exactly two circles. If this condition is satisfied only for one endpoint, we call such segment \emph{half-open}, and if it is not satisfied for both endpoints, we call it open.

A block segment that contains $j_0$ is called the \emph{leftmost} and the segment that contains $j_l$ is called the \emph{rightmost}. Other segments are called \emph{intermediate}.

We can group the terms different from $\Paren{Y^{\alpha}}^2$ in the following way: Let 

\begin{itemize}
	\item $q_1$ be the number of circles in the leftmost block segment
	\item $q_2$ be the number of circles in the rightmost block segment
	\item $q_3$ be the number of circles in intermediate block segments
	\item $s_c$ be the number of closed intermediate segments
	\item $s_o$ be the number of open intermediate segments 
	\item $s_h$ be the number of half-open intermediate segments
	\item $r$ be the number of circles of degree $4b$ that are not blocked and do not share any box with blocked circles.
	\item $m_a$ (for $a\in\set{3,4}$) be the number of boxes that are adjacent to exactly $a$ circles such that these circles are \emph{not} blocked. 
	\item $m$ be the number of boxes of degree $4$ which are adjacent to exactly two non-blocked circles and  at least one of these circles doesn't share a box with blocked circles.
\end{itemize}
	Then we group terms in a way such that these parameters $q_1$, $q_2$, $q_3$, $s_c$, $s_o$, $s_h$, $r$, $m_2$, $m_3$, $m_4$ are equal for all terms $Y^{\alpha}Y^{\alpha'}$ inside one group.
	
	Let's fix the parameters 
	($q_1$, $q_2$, $q_3$, $s_c$, $s_o$, $s_h$, $r$, $m_2$, $m_3$, $m_4$) and compute the contribution of nonzero terms $\E Y^{\alpha}Y^{\alpha'}$ which correspond to graph with these parameters. 
	
	We will use the following way of counting: for every $Y^{\alpha}Y^{\alpha'}$ such that $\alpha$ and $\alpha'$ have common box/circle of some type (different from $j_0$, $j_l$) we consider $\alpha''\in\cG$ obtained from $\alpha'$ by replacing each box/circle of this type by another box/circle that is not in $\alpha$.
	
	We start with $m_4 > 0$. For every nonzero $Y^{\alpha}Y^{\alpha'}$ with $m_4 > 0$
	let $M_4(\alpha,\alpha')$ be the set of boxes that are adjacent to $4$ non-blocked circles. 
	Consider all $\alpha''\in \cG$ that are obtained from $\alpha'$ by replacing each  box from $M_4(\alpha,\alpha')$ by some box that is not in $\alpha$.
	Then $Y^{\alpha}Y^{\alpha''}$ has parameter $m_4=0$.
	Recall that there exists a constant $C'$ such that for all $i\in[n]$, $\E u_i^4 \le C'$.
	Then $\E Y^{\alpha}Y^{\alpha''} \ge \Paren{\frac{1}{C'}}^{m_4} \E Y^{\alpha}Y^{\alpha'}$. 
	Number of such $\alpha''$ 
	 (for fixed $\alpha, \alpha'$)  is  is at least $(1-o(1)) n^{m_4}$, 
	while number of different $Y^{\alpha}Y^{\alpha'}$ 
	such that $ Y^{\alpha}Y^{\alpha''}$ could be obtain from them 
	using the procedure described above is is at most $\Paren{2bl}^{2m_4}$. 
	Hence the contribution of terms for which $m_4 = 0$ is larger  than the contribution of terms with $m_4 > 0$ by a factor
	$\Paren{\frac{n}{3b^2l^2C'}}^{m_4}$ .
	Note that this factor is $\omega(1)$ since $n \ge \Omega(\log^4 d)$.
	
	Similarly, for every nonzero $Y^{\alpha}Y^{\alpha'}$ with $m_3 > 0$ 
	let $M_3(\alpha,\alpha')$ be the set of boxes that are adjacent to $3$ non-blocked circles. 
	Consider all $\alpha''\in \cG$ that are obtained from $\alpha'$ by replacing each  box from $M_4(\alpha,\alpha')$ by some box that is not in $\alpha$.	
	Then $\E Y^{\alpha}Y^{\alpha''}\ge (1-o(1))\Paren{\frac{\beta}{k}}^{m_3}\E Y^{\alpha}Y^{\alpha'}$. 
	Number of such $\alpha''$  (for fixed $Y^{\alpha}Y^{\alpha'}$)  is  is at least $(1-o(1)) \Paren{n}{b}^{m_3}$, 
	while number of different $Y^{\alpha}Y^{\alpha'}$ 
	such that $Y^{\alpha}Y^{\alpha''}$ 
	could be obtain from them  is at most $\Paren{2bl}^{2m_3}$. 
	Hence the contribution of terms for which $m_3 = 0$ is larger 	than the contribution of terms with $m_3 > 0$ by a factor
	\[
	\Paren{\frac{\beta n}{3k b^3l^2}}^{m_3} \ge \Paren{\frac{n}{10b^5 l^4}}^{m_3/2}\,.
	\]  
     Note that this factor is $\omega(1)$ since $n \ge \Omega(\log^5 d)$.
	
	 For every nonzero $Y^{\alpha}Y^{\alpha'}$ with $m > 0$ and $m_3 = m_4 = 0$,
	 let $M(\alpha,\alpha')$ be the set of these $m$ boxes and let $R(\alpha,\alpha')$ be the set of non-blocked circles of degree $4b$ which do not share any box with blocked circles.
	
	Consider all $\alpha''\in \cG$ that are obtained from $\alpha'$ by replacing each box from $M(\alpha,\alpha')$ by some box that is not in $\alpha$ and each circle from $R(\alpha,\alpha')$ by some circle from the support of $v_0$ that is not in $\alpha'$.
    It follows that  
	$\E Y^{\alpha}Y^{\alpha''} \ge 
	(1-o(1))\Paren{\frac{\beta^2}{k^2}}^{m}\E Y^{\alpha}Y^{\alpha'}$.
	Number of such $Y^{\alpha}Y^{\alpha''}$ (for fixed $Y^{\alpha}Y^{\alpha'}$)  is at least $(1-o(1)) 	\Paren{\frac{n}{b}}^m k^{r}$.
	Number of different $Y^{\alpha}Y^{\alpha'}$ 
	such that $Y^{\alpha}Y^{\alpha''}$ could be obtain from them  is at most
	$(2l)^{2r}m^r\Paren{4bl}^{2m}\Paren{\frac{r}{m}}^{m}$. 
	Indeed, 
	number of ways to choose circles which replace $R(\alpha,\alpha')$ is bounded by $(2l)^{2r}$. 
	Once these circles are chosen, number of ways to choose boxes which replace $M(\alpha,\alpha')$  is bounded by the number of ways to choose the numbers $\tilde{m}_t$ 
	for each chosen circle $j_t$ such that $\sum_{t=1}^r \tilde{m}_t = m$, which is bounded by $(2b)^r$, and the product (over all $t\in [r]$) of numbers of possibilities to choose $2$ subsets of size $\tilde{m}_t$ in a set of size $4b$. This product is bounded by 
	\[
	\frac{\Paren{4b}^{2m}}{\Paren{\tilde{m}_1!}^2\cdots \Paren{\tilde{m}_r!}^2}\le \Paren{4b}^{2m} \Paren{\frac{r}{m}}^{2m}\,.
	\]
	Hence the contribution of terms with $m = 0$ is larger
	than the contribution of terms with $m > 0$  by a factor
	\[
	(1-o(1))\Paren{\frac{\beta^2n}{k^2b}}^{m}k^{r} (2l)^{-2r}(2b)^{-2r}
	\Paren{4b}^{-2m} \Paren{\frac{m}{r}}^{2m}
	\ge
	C^m\Paren{\frac{k}{5bl} \Paren{\frac{m}{rb}}^{2m/r}}^r
	\ge C^{m} \Paren{\frac{k}{5lb}  2^{-O(b)}}^r\,.
	\]   
	Note that this factor is $\omega(1)$ since $b=o(\log d)$ and $bl \le \log d \le O(\log k)$.
	
	Similarly, the contribution of terms with $r = 0$ is larger
	than the contribution of terms with $r > 0$ and $m=0$ by a factor $k^{r} (2l)^{-2r}$ which is $\omega(1)$.
	
	For every nonzero $Y^{\alpha}Y^{\alpha'}$ with $m_4 = m_3 = m = r = 0$ and $q_3>0$ let $Q(\alpha,\alpha')$ be a set of blocked circles in intermediate blocked segments and $M'(\alpha,\alpha')$ be a set of boxes adjacent to circles from $Q(\alpha,\alpha')$ . Let $R'(\alpha,\alpha')$ be a set of non-blocked circles of degree $4b$ which share a box with some $j\in Q$ (which means that it shares exactly $b$ boxes with some $j\in Q$). Let $M''(\alpha,\alpha')$ be a set of boxes of degree $4$ which are not from $M'(\alpha,\alpha')$ and which are adjacent to circles from $R'(\alpha,\alpha')$. We denote $m'' = \card{M''(\alpha,\alpha')}$.
	Note that $\card{Q(\alpha,\alpha')} = q_3$, $\card{M'(\alpha,\alpha')} = bq_3 + bs$, 
	$\card{R'(\alpha,\alpha')} = s_h + 2s_c$.
	
	Consider all $\gamma'\in \cG$ that are obtained from $\alpha'$ using the folowing procedure:   each box from $M'(\alpha,\alpha')\cup  M''(\alpha,\alpha')$ is replaced by some box that is not in $\alpha$,  each circle from $Q(\alpha, \alpha') \cup R'(\alpha,\alpha')$ is replaced by some circle from the support of $v_0$ that is not in $\alpha'$. Also consider all $\gamma$ that are obtained from $\alpha$ by replacing circles from $Q$ by a circle from the support of $v_0$ that is not in $\gamma'$.
	
	It follows that
	\[
	\E Y^{\gamma}Y^{\gamma'} \ge 
	(1-o(1))\Paren{\frac{\beta}{k}}^{2bq_3 + 2bs_c + bs_h + 2m''} \E Y^{\alpha}Y^{\alpha'}\,.
	\]
	Number of different $Y^{\gamma}Y^{\gamma'}$  (for fixed $Y^{\alpha}Y^{\alpha'}$) 
	is $(1-o(1)) \Paren{\frac{n}{b}}^{bq_3 + bs+m''} k^{2q_3 + s_h + 2s_c}$. Note that $\alpha$ and $\alpha'$ might contain now circles that are not from the support of $v_0$ (they should be in $Q$). By similar argument as for the case $m>0$, the number of different $Y^{\alpha}Y^{\alpha'}$ 
	such that $Y^{\gamma}Y^{\gamma'}$ could be obtain from them is at most
	\[
	(10l)^{4s} d^{q_3}(4b)^{2m''}\Paren{\frac{s_h + 2s_c}{m''}}^{2m''}\,.
	\]
	Hence the contribution of terms with $q_3 = 0$ is larger 
	than the contribution of terms with $q_3 > 0$  by a factor
	\[
	(1-o(1))\Paren{\frac{\beta}{k}}^{2bq_3 + 2bs_c + bs_h + 2m''} 
	\Paren{\frac{n}{b}}^{bq_3 + bs+m''} k^{2q_3 + s_h + 2s_c}
		(10l)^{-4s} d^{-q_3}(4b)^{-2m''}\Paren{\frac{s_h + 2s_c}{m''}}^{-2m''}\,,
	\]
	which is at least
	\[
	 \Paren{C^b\cdot \frac{k^2}{d}}^{q_3} C^{bs_c + m''} \Paren{\frac{n}{(10l)^{4}b}}^{bs_0} \Paren{\frac{\beta n}{k b}}^{bs_h}
	\Paren{\frac{k\cdot 2^{-O(b)}}{(10l)^{4}}}^{2s_c + s_h}\,.
	\]
	Since $C^b > 10^{b}\frac{d}{k^2}$, this factor is $\omega(1)$.
	
	Hence we can conclude that if $q_1=q_2 = 0$, then the contribution of the terms that are different from $\Sdiff$ is $o(\Sdiff)$.
	
	Now consider the case when $q_1 > 0$ and the other parameters are zero. There are two cases, the first is when the leftmost blocked segment is half-open and the second when it is closed. In the first case the contribution of such terms is bounded by
	\begin{align*}
     (1+o(1))\cdot d^{q_1 - 1} \Paren{\frac{n}{b}}^{bq_1} \cdot k^{2l-2-2(q_1- 1)} 
     \cdot \Paren{\frac{\beta n}{k b}}^{b(2l-2q_1+1)} kv_0(j_l)^{2} &\le 
     2\Paren{\Paren{\frac{k^2b}{\beta^2n}}^b{\frac{d}{k^2}}}^{q_1-1} 
     \Paren{\frac{k}{\beta}}^{b} \frac{\kappa^2}{k} v_0(j_l)^{2}
     \\
     &\le 2^{-b(q_1-1)} d^{-10}\cdot \frac{\kappa^2}{k} v_0(j_l)^{2}\,,
	\end{align*}
	where we used $(k/\beta)^b = \Paren{\frac{n}{Cb}}^{b/2} \ge d^{10}$.
	In the second case, we get
	\begin{align*}
	(1+o(1))\cdot d^{q_1 - 1} \Paren{\frac{n}{b}}^{bq_1} \cdot k^{2l-2q_1 - 1} 
	\cdot \Paren{\frac{\beta n}{k b}}^{2b(l-q_1)} kv_0(j_l)^{2} 
	&\le 
	2\Paren{\Paren{\frac{k^2b}{\beta^2n}}^b{\frac{d}{k^2}}}^{q_1-1}
	\Paren{\frac{k^2b}{\beta^2n}}^b
	\frac{\kappa^2}{k^2} v_0(j_l)^{2}
	\\
	&\le 10^{-bq_1}
	\frac{\kappa^2}{d} v_0(j_l)^{2}\,,
	\end{align*}
	where we used $C^b > 10^{b} \frac{d}{k^2}$.
	
	Similarly, in the case when $q_2 > 0$ and the other parameters are zero, the contribution is bounded by $10^{-bq_2}\frac{\kappa^2}{d}v_0(j_0)^2$.

	Similar computations show that if $q_1 > 0$ and $q_2 > 0$ (and the other parameters are zero), then the contribution is bounded by $o\Paren{\frac{\kappa^2}{d^2}}$. Therefore, dividing by $\kappa^2$, we get
	\[
			\Var_\mu P_{j_0j_l}(Y)  \le 
	o\Paren{v_0(j_0)^{2}v_0(j_l)^{2} + \frac{1}{d}\Paren{v_0(j_0)^{2} + v_0(j_l)^{2}} + \frac{1}{d^2}}\,.
	\]
\end{proof}

\subsection{Computation in polynomial time}

Since for $j_0\neq j_l$, $P_{j_0j_l}(Y)$ has degree $\Theta(\log d)$, simple evaluation takes time $d^{\Theta\Paren{\log d}}$. However, we can use a color coding technique to (approximately) evaluate $P$ in time $d^{O(1)}$. 

Let $P(Y)$ be a matrix such that for all $j_0 < j_l \in [d]$, its entries are just
$P_{j_0j_l}(Y)$, $P_{j_0j_l} = P_{j_0j_l}$ for $j_0 > j_l$.

\begin{lemma}\label{lem:polynomial-color-coding}
	Suppose that the conditions of \cref{theorem:polynomial-algorithm-main} are satisfied.
	There exists a probabilistic algorithm that given $Y$ as input, 
	in time $(nd)^{O(1)}$ outputs a matrix $\hat{P}\in\R^{d\times d}$ such that
	\[
	\normf{\hat{P} - P}^2 \le o(1)\,,
	\]
	with probability $1-o(1)$  (with respect to the distribution of $Y$ and the randomness of the algoruthm).
\end{lemma}

\begin{proof}

Let's fix $j,j'\in [d]$ such that $j < j'$ and let $a: [d] \to [l+1]$ and $c: [n] \to [bl]$ be fixed colorings. Note that we can compute 

\[
p_{acjj'}(Y) =\frac{1}{\kappa}\cdot \frac{(l+1)^{l+1} \Paren{bl}^{l}}{(l+1)!\Paren{bl}!}\sum_{\alpha\in\cG_{jj'}(b,l)} 
\textbf{1}\Brac{a(J_\alpha) = [l+1]} \cdot 
\textbf{1}\Brac{c(I_\alpha) = [bl]}\cdot Y^{\alpha}
\]

in time $d^{O(1)}2^{O(bl)}$. 

To do this, we compute a matrix whose rows and columns are indexed by $(j, S, M)$, where $j\in [d]$, $S\subseteq [l]$, $M\subseteq [bl]$, such that an entry $\Brac{\Paren{j,S,M}, \Paren{j',S',M'}}$ is not $0$ if and only if $S\subset S'$, $\card{S'\setminus S} = 1$, $M\subset M'$, $\card{M'\setminus M} = b$. If the entry is not zero, this is equal to
\[
\sum_{B\in\binom{n}{b}} 
\textbf{1}\Brac{c(B) = M'\setminus M}
\prod_{i\in B} Y_{ij}Y_{ij'}\,.
\]

Now, denote $M'\setminus M = \set{m_1,\ldots,m_b}$, where $m_1 <\cdots < m_b$. To compute the entry  we can compute for all $1 \le r \le b$
\[
T\Brac{r} = \sum_{B_r\in\binom{n}{r}}
\textbf{1}\Brac{c(B_r) = \set{m_1\ldots,m_r}}\prod_{i\in B_r} Y_{ij}Y_{ij'}\,.
\]

Note that
\[
T\Brac{r+1}= \sum_{i\in[n]:\;c(i) = m_{r+1}} T\Brac{r} Y_{ij}Y_{ij'}\,,
\]
so we can compute the entry of the matrix $T\Brac{b}$ in  time $O(nb)$, and all entries can be computed in time $nb d^{O(1)}2^{O(bl)} = d^{O(1)}$.

If we then compute the $l$-th power of this matrix, which takes time $d^{O(1)}$, the entry $\Brac{\Paren{j_0, \emptyset, \emptyset,},\Paren{j_l, [l], [bl]}}$ of the resulting matrix contains 
$\kappa\frac{(l+1)!\Paren{bl}!}{(l+1)^{l+1} \Paren{bl}^{l}}p_{acjj'}(Y)$.

Denote by $\mu$ the distribution of $Y$.
Let $a$, $c$ be independent random colorings ($a$ is sampled from uniform distributions over colorings of $[d]$ in $[l+1]$ colors and $c$  from uniform distributions over colorings of $[n]$ in $bl$ colors). Thus ${\E_{ac}p_{acjj'}(Y)}= P_{jj'}(Y)$. Note that since for any $\alpha, \alpha' \in \cG_{jj'}(b,l)$, $\E_\mu Y^{\alpha} Y^{\alpha'} \ge 0$, and
\[
\E_\mu\Var_{ac}p_{acjj'} \le 
\E_\mu\E_{ac}p_{acjj'}^2(Y) \le 
\frac{1}{\kappa^2}\Paren{\frac{(l+1)!\Paren{bl}!}{(l+1)^{l+1} \Paren{bl}^{l}}}^2 
\sum_{\alpha, \alpha' \in \cG} \E_\mu Y^{\alpha} Y^{\alpha'}\le 
d^{5} \E_\mu P_{jj'}^2(Y)\,,
\]
since $\Paren{\frac{(l+1)!\Paren{bl}!}{(l+1)^{l+1} \Paren{bl}^{l}}}^2 \le e^{4bl} \le d^5$.

Hence with probability at least $1-d^{-5}$ (with respect to $\mu$), $\Var_{ac}p_{ac}\le d^{10} \E_\mu  P_{jj'}^2(Y)$.

Let $\hat{P}_{jj'}(Y) = \frac{1}{\card{S}} \sum_{(a,c)\in S}{p}_{acjj'}(Y)$, where $S$ is a set of $d^{20}$ independent random colorings $(a,c)$. Thus
$\E_{S} \hat{P}_{jj'}(Y)  = P_{jj'}(Y)$ and with probability at least $1-d^{-5}$, 
\[
\Var_{S} \hat{P}_{jj'}(Y)  \le \frac{1}{d^{10}}\E_\mu P_{jj'}^2(Y) \le \frac{1}{d^{10}}\,.
\]
hence with probability at least $1-d^{-3}$ (with respect to $S$)
\[
\normf{\hat{P} - P}^2 \le \frac{1}{d^{3}}\,,
\]
where $\hat{P}$ is a symmetric matrix that is zero on the diagonal and whose entries for all $j < j'$
are $\hat{P}_{jj'}$. Therefore, with probability $1-o(1)$
\[
\normf{\hat{P} - P}^2 \le o(1)\,.
\]
\end{proof}

\begin{proof}[Proof of \cref{theorem:polynomial-algorithm-main}]
	By \cref{lem:polynomial-algorithm-expectation} and \cref{lem:polynomial-algorithm-variance},
	\[
	\E\normf{P-v_0\transpose{v_0}}^2 
	= \sum_{j\neq j'} \Var P_{jj'} + \sum_{j\in[d]} v_0(j)^4
	\le o(1) + \frac{1}{k}  \le o(1)
	\,.
	\]
	Hence by Markov's inequality
	\[
	\normf{P-v_0\transpose{v_0}}^2 \le o(1)
	\]
	with probability  $1-o(1)$.
	By \cref{lem:polynomial-color-coding} we can compute in time $(nd)^{O(1)}$ a matrix $\hat{P}$ such that
	\[
	\normf{\hat{P} - P}^2 \le o(1)\,.
	\]
	Hence
	\[
	\norm{\hat{P}-v_0\transpose{v_0}}^2 \le \normf{\hat{P}-v_0\transpose{v_0}}^2 \le o(1)\,.
	\]
	Therefore, by \cref{lem:linear-algebra-correlation-eigenverctor-large-quadratic-form}, the top eigenvector $\hat{v}$ of $\hat{P}$ satisfies
	\[
	1 - \iprod{\hat{v},v_0}^2 \le o(1)
	\]
	with probability $1-o(1)$.
	\end{proof}

    \begin{remark}
    	Note that the proof of \cref{theorem:polynomial-algorithm-main} also shows that in the case $\Omega(d) \le k^2 \le d/2$, $1 - \iprod{\hat{v},v_0}^2 \le 2^{-b} $. Hence, if the constant $C^*$ from the theorem statement is at least $10^5$, we can choost $C^{**}\ge 1000$ and get $1 - \iprod{\hat{v},v_0}^2 \le 2^{-1000}$, but as long as $n\ge d^{\Omega(1)}$, the error doesn't tend to zero as $d\to\infty$.
    \end{remark}

\section{Fast Spectral Algorithms for Recovery}\label{sec:detection-vs-estimation}

\DeclarePairedDelimiterX{\infdivx}[2]{(}{)}{%
	#1\;\delimsize\|\;#2%
}

One key limitation of Algorithm \ref{alg:sparse-pca-program} is the reliance on solving large semidefinite programs, something that is often computationally too expensive to do in practice for the large-scale problems that arise in machine learning.
So, inspired by the SoS program used in \ref{alg:sparse-pca-program}, in this section we present a \textit{fast} spectral algorithm which recovers the sparse vector $v_0$ in time $O\Paren{nd \log n}$.
Our algorithm which we call SVD-$t$, for $t\in \Set{2,4,6}$, is a slight modification of a fast spectral algorithm presented in \cite{DBLP:conf/stoc/HopkinsSSS16}.
Such algorithm recovers a sparse vector planted in a random subspace.
The algorithm was also based on the analysis of a degree-$4$ Sum-of-Squares algorithm introduced in \cite{DBLP:conf/stoc/BarakKS14}.
We remark that for $t=2$ \cref{alg:sparse-pca-sos-practical} corresponds to the SVD with thresholding algorithm outlined in \cref{sec:introduction-new}.
In non-robust settings, as well as in the adversarial model \ref{problem:almost_gaussian_model_special_case}, the algorithm achieves high correlation under  conditions similar (up to logarithmic terms) to those of the Sum-of-Squares algorithm \ref{alg:sparse-pca-program} (of degree $2,4$ and $6$).

\begin{mdframed}[nobreak=true]
	\begin{algorithm}[SVD-$t$: Sparse Vector Recovery]\label{algorithm:svd-t}
		\label[algorithm]{alg:sparse-pca-sos-practical}\mbox{}
		\begin{description}
			\item[Given:]
			Sample matrix $Y\in \R^{n \times d}$, let $y_1,\ldots,y_d\in \R^n$ be its columns. Degree $j \in \{2,4,6\}$.
			\item[Estimate:]
			The sparse vector $v_0$.
			\item[Operation:]\mbox{}
			\begin{enumerate}
				\item 
				Compute the top eigenvector $\hat{u}$ of the matrix
				\begin{equation*}
				A := \underset{i \in \Brac{d}}{\sum} c_j(y_i, n)\cdot y_i \transpose{y_i}
				\end{equation*}
				where for $x \in \R^n, t \in \R$, $c_2(x,t):=1$, $c_4(x,t):=\Paren{\Snorm{x}-(t-1)}$ , $c_6(x,t):=\Paren{c_4(x)^2- 2(t-1)}$.
				\item
				Compute $\hat{v} =\transpose{\hat{u}}Y$. 
				\item \label{alg-step:thresholding}
				Threshold the vector $\hat{v}$ in the following way (for some fixed $\tau\ge 0$):
				\begin{equation*}
				\forall i \in [d], \eta(\hat{v})_i =
				\begin{cases}
				\hat{v}_i,&  \text{if $\abs{\hat{v}_i}\ge \frac{\tau}{\sqrt{k}}$}\\
				0,& \text{otherwise}
				\end{cases}
				\end{equation*}
				\item Output the thresholded vector $\eta(\hat{v})$.
			\end{enumerate}
		\end{description}    
	\end{algorithm}
\end{mdframed}
\begin{remark}[Running Time of the Algorithm]
	For $j \in \Set{2,4,6}$, the terms $m_j(y_1,n),\ldots, m_j(y_d,n)$ are computable in time $O(nd)$.  Correctness of the algorithm will be proved showing that $A$ has at least constant spectral gap.  This means that  we can compute  the top eigenvalue with power iteration using $O(\log n)$ matrix-vector multiplications . A matrix multiplication requires computing $m_i= c_6(y_i,n)\iprod{y_i,z}$ for each $i$ and then taking the sum $\underset{i \in [d]}{\sum}m_i a_i$. Both operations take time $O(nd)$. Then, $\hat{v}$ can be computed in time $O(nd)$ and $\eta(\hat{v})$ in time $O(d)$. In conclusion the algorithm runs in time $O(nd \log n)$.
\end{remark}

 To get an intuition on the algorithm, consider SVD-6 and the simpler adversarial model 
 $Y =\Paren{\Id-\frac{1}{\Snorm{u}}u \transpose{u}}W+\sqrt{\beta}u\transpose{v}$.That is, the Single Spike Model \ref{problem:spiked_covariance_model_lb} with the noise projected into the space orthogonal to $u$.\footnote{Note that the estimate $\hat{u}$ obtained by SVD-2 is the same returned by the standard SVD. } Now for $i \in \supp\Set{v}$,
 \begin{align*}
 	\Norm{\Brac{\Paren{\Norm{y_i}^2-n}^2-2n}y_i \transpose{y_i}}\approx \frac{\beta^3n^3}{k^3}\,,
 \end{align*} while for $i \in [d]\setminus \supp\Set{v}$, 
 \begin{align*}
 	\Norm{\E\Brac{\Paren{\Norm{y_i}^2-n}^2-2n}y_i \transpose{y_i}}=O(1)\,.
 \end{align*}Indeed, the coefficient $c_6(y_i,n)$ has the effect of "killing" the expectation for Gaussian vectors. Then for $d\gg n^3$, the sum $\underset{i \in [d]\setminus \supp\Set{v}}{\sum}c_6(y_i,n)y_i\transpose{y_i}$ will be concentrated around its expectation $d$, while on the other hand 
 \begin{align*}
 	\Norm{\underset{i \in \supp\Set{v}}{\sum}\Brac{\Paren{\Norm{y_i}^2-n}^2-2n}y_i \transpose{y_i}}\approx \frac{\beta^3n^3}{k^2}\,.
 \end{align*} Hence, for $\beta\gtrsim \frac{k}{n}\Paren{\frac{d}{k}}^{1/3}$ the leading eigenvector of $A$ will be highly correlated with $u$.\\

We remark that it is an open question how these ideas could be generalize to construct an algorithm that works for  $\beta \gtrsim \frac{k}{n}\Paren{\frac{d}{k}}^{1/t}$ and $d\gg n^t$. Here, as a proof of concept, we show that SVD-6 succeed under the planted model in \ref{problem:almost_gaussian_model_special_case} when $d\gg n^3$ and $\beta \gg \frac{k}{n}\Paren{\frac{d}{k}}^{1/3}$.  
In order to define the adversarial perturbations, we will use the notation introduced for Problem \ref{problem:almost_gaussian_model_special_case}, we recall that with high probability $\lambda = (1\pm o(1))\sqrt{\frac{\beta n}{k}}$. 

\begin{theorem}\label{theorem:recovery-fast-algorithms}
	Consider a matrix of the form,
	\begin{align*}
	Y=W+\lambda u\transpose{v}+ u\transpose{\Paren{v'-\transpose{W}u}}
	\end{align*}
	for a Gaussian matrix $W\sim N(0,1)^{n \times d}$,  a random unit vector $u$, a $k$-sparse vector $v$ with entries in $\Set{0,\pm 1}$ and a vector $v'$ as defined in \ref{problem:almost_gaussian_model_special_case}. For $d\gtrsim n^3\log d\log n$, $\lambda \gtrsim \frac{\sqrt{\log d}}{\tau}$ and $k\geq n \log n$, Algorithm \ref{algorithm:svd-t} with degree $6$ returns a vector $\eta(\hat{v})$ such that
	\begin{align*}
	\Norm{\eta(\hat{v})-v}\leq O\Paren{\frac{d}{k\lambda^6}+\tau}\cdot \sqrt{k}
	\end{align*}
	with probability at least $0.99$.
	Furthermore, for $\frac{d}{k\lambda^6}+\tau\leq 1$ and $\beta =\frac{\lambda^2 k}{n}$,
	\begin{align*}
	1-\frac{\iprod{\eta(\hat{v}),v}^2}{\Snorm{\eta(\hat{v})}\cdot \Snorm{v}}\leq\Paren{ \frac{k}{n\beta}\Paren{\frac{d}{k}}^{1/3}+\tau^2}.
	\end{align*}
\end{theorem}
We remark that the second inequality of the theorem follows from the first by direct substitution and using the fact that $\eta(\hat{v})$ is close to a unit vector.
Comparing this result with Theorem \ref{thm:algorithm-robust-spca-technical}  we see that both SVD-6 \ref{algorithm:svd-t} and degree-6 SoS \ref{alg:sparse-pca-program} need $\beta \gtrsim \frac{k}{n}\frac{d}{k}^{1/3}$ in order to achieve correlation $0.9$ with the sparse vector.

To prove Theorem \ref{theorem:recovery-fast-algorithms}, we will first show that the vector $\hat{u}$ computed by the algorithm is close to the true vector $u$. Then, thresholding the vector $\transpose{\hat{u}}Y$ we will obtain a vector close to $v$. Concretely, we will prove two results. First,

\begin{lemma}\label{lemma:svd-t-recovery-u}
	Consider a matrix of the form,
	\begin{align*}
	Y=W+\lambda u\transpose{v}+ u\transpose{\Paren{v'-\transpose{W}u}}
	\end{align*}
	for a Gaussian matrix $W\sim N(0,1)^{n \times d}$,  a random unit vector $u$, a $k$-sparse vector $v$ with entries in $\Set{0,\pm 1}$ and a vector $v'$ as defined in \ref{problem:almost_gaussian_model_special_case}. Let $\hat{u}\in \R^n$ be the top eigenvector of the matrix 
	\begin{align*}
	\underset{i \in \Brac{d}}{\sum} c_6(y_i, n)\cdot y_i \transpose{y_i}.
	\end{align*}Then for $d\geq C^*n^3\log d \log n$, $n\geq 10 \log d$
	\begin{align*}
		\Norm{u-\hat{u}}\leq O\Paren{\frac{d}{k\lambda^6} +\frac{1}{\lambda} + \frac{\sqrt{n \log n}}{\lambda\sqrt{k}}}
	\end{align*}
	with probability at least $0.999$, where $C^*$ is a universal constants.
\end{lemma}

Second,

\begin{lemma}\label{lemma:svd-t-recovery-v}
	Let $\hat{u}$ be a vector such that $\norm{\hat{u}-u}\leq \varepsilon$ 
	for some $0 \le \varepsilon \le \frac{1}{10}$ and let $\hat{v}=\frac{1}{\lambda \sqrt{k}}\transpose{\hat{u}}Y$. If $\lambda \gtrsim \frac{\sqrt{\log d}}{\tau}$, then with probability at least $1-\exp(-n)$
	\[
	\norm{\hat{v} - v} \lesssim \Paren{\varepsilon +  \tau}\sqrt{k}\,,
	\]
	where $\eta(\hat{v})  \in \R^d$ is the vector with coordinates 
	\begin{equation}
	\eta(\hat{v})_i = 
	\begin{cases}
	\hat{v}_i , &\text{if}\;\; \Abs{\hat{v}_i} \geq  \tau\\
	0, &\text{otherwise.}
	\end{cases}
	\end{equation}
\end{lemma}

It is easy to see how the two results immediately imply Theorem \ref{theorem:recovery-fast-algorithms}.

Lemma \ref{lemma:svd-t-recovery-u} is proved in Section \ref{subsection:fast-algorithm-recover-u}, in Section \ref{subsection:fast-algorithm-recover-v} we prove Lemma \ref{lemma:svd-t-recovery-v}.

\subsection{Algorithm recovers \textit{u} with high probability}\label{subsection:fast-algorithm-recover-u}
The goal of this Section is to prove Lemma \ref{lemma:svd-t-recovery-u}.

By rotational symmetry of the Gaussian distribution, we may assume without loss of generality that $u=e_1$. Now, for vectors  $v,z\in \R^n$, define $M(v,z):=\Brac{\Paren{\Norm{v+z}^2-(n-1)}^2-2(n-1)}v \transpose{v}$. Recall that the adversarial vector $v'$ is, by construction, orthogonal to the sparse vector $v$. Hence our strategy will be the following, first we bound the contribution of terms of the form $M(w,\gamma e_1)$ and $M(v'(i)u,w)$. Note that the first type of terms arise due to the noise,  the second ones due to the adversarial distribution.  Then, lower bounding $M(\lambda u, w)$, we will be able to show that $\Norm{\underset{i \in \supp \Set{v}}{\sum}M(\lambda u, w) }\gg \Norm{\underset{i \in [d]}{\sum} M(w_i,\gamma_i e_1)}+\Norm{\underset{i \in [d]\setminus\supp\Set{v}}{\sum} M(v'(i)u,w_i)  }$ with high probability. Cross-terms will play a minor role.

First we bound the contribution of the Gaussian part. We will use Bernstein Inequality, the next  results act as building blocks for the bound, which is then shown in Lemma \ref{lemma:svd-t-bound-gaussian-contribution}.

\begin{fact}\label{fact:svd-t-expectation-norm-times-entry}
	Let $x\sim N(0,\Id_n)$,
	\begin{align*}
	\E \Snorm{x}x_i^2&=n+2\\
	\E \Norm{x}^4x_i^2&=n^2 + 6n +8.
	\end{align*}
	\begin{proof}
		\begin{align*}
		\E \Snorm{x}x_i^2 &= \underset{j \in [n],j\neq i}{\sum}\E x_i^2x_j^2+ \E x_i^4 = n+2\\
		\E \Norm{x}^4x_i^2 &= \underset{\substack{j,k \in [n]\\j\neq i,k\neq i,j\neq k}}{\sum}	\E x_i^2x_j^2x_k^2 + \underset{\substack{j,k \in [n]\\j = k\neq i}}{\sum}\E x_j^4x_i^2 + 2\underset{\substack{j,k \in [n]\\j\neq k= i}}{\sum} \E x_i^4x_j^2 + \E x_i^6\\
		&=(n-1)(n-2)+ 3(n-1)+6(n-1)+15\\
		&=n^2 + 6n +8.
		\end{align*}
	\end{proof}
\end{fact}

We bound the spectral norm of the expectation of the  terms $M(w,\gamma e_1)$.

\begin{lemma}\label{lemma:svd-t-bound-expectation-gaussian-with-spike}
	Let $w\sim N(0,\Id_n-e_1\transpose{e_1})$, $\gamma \in \R$. Then
	\begin{align*}
	\Norm{\E M(w,\gamma e_1)}=\gamma^4 +8 \gamma^2 + 8.
	\end{align*}
	\begin{proof}
		We need only to look into diagonal entries. By construction, $\E \Snorm{w}= n-1=:m$,
		\begin{align*}
		\E \Brac{\Paren{\Snorm{w+\gamma e_1}-m}^2-2m}w_i^2=& \E \Brac{\Paren{\Snorm{w}+\gamma^2 -m}^2-2m}w_i^2\\
		=& \E \Paren{\Norm{w}^4+\gamma^4+m^2 +2\Snorm{w}\gamma^2-2m\Snorm{w}-2\gamma^2m-2m}w_i^2
		\end{align*}
		Applying Fact \ref{fact:svd-t-expectation-norm-times-entry},
		\begin{align*}
		\E &\Brac{\Paren{\Snorm{w+\gamma e_1}-m}^2-2m}w_i^2 \\
		&=m^2 + 6m +8 + \gamma^4+m^2+2\gamma^2m+4\gamma^2 -2m^2-4m-2\gamma^2m-2m \\
		&= \gamma^4+8\gamma^2+8.
		\end{align*}
	\end{proof}
\end{lemma}

The second property we need is a high probability bound on the maximum value of $\Norm{M(w,\gamma e_1)}$.

\begin{lemma}\label{lemma:svd-t-max-deviation-gaussian-with-spike}
	Let $w\sim N(0,\Id_n-e_1\transpose{e_1})$, $\gamma \in \R$. Then for any $q\geq 1$, with probability at least  $1-2e^{-q}$,
	\begin{align*}
	\Norm{M(w,\gamma e_1)}\leq C \Paren{\gamma^4 n+n\max\Set{nq,q^2}},
	\end{align*}
	where $C$ is a universal constant.
	\begin{proof}
		For simplicity of the notation let $m=n-1$, and let $p=\max\Set{q,\sqrt{mq}}$. By Fact \ref{fact:chi-squared-tail-bounds}, 
		\begin{align*}
		\bbP\Paren{\Snorm{w}\notin\Brac{m-10p, m+10p}}\leq 2e^{-q}.
		\end{align*}
		Hence with probability at least $1-2e^{-q}$,
		\begin{align*}
		\Abs{\Brac{\Paren{\Snorm{w+\gamma e_1}-m}^2-2m}}=& \Abs{ \Brac{\Paren{\Snorm{w+\gamma e_1}-m}^2-2m}}\\
		=& \Abs{\Brac{\Paren{\Snorm{w}+\gamma^2 -m}^2-2m}}\\
		\leq & \Brac{\Paren{\gamma^2+10p}^2-2m}\\
		\leq & C\Paren{\gamma^4+p^2}
		\end{align*}
		for some universal constant $C>0$. The result follows.
	\end{proof}
\end{lemma}

And finally, the last ingredient we need for our Bernstein inequality is a bound on the variance.

\begin{lemma}\label{lemma:svd-t-bound-variance-gaussian-with-spike}
	Let $w\sim N(0,\Id_n-e_1\transpose{e_1})$, $\gamma \in \R$. Then
	\begin{align*}
	\Norm{\E M(w,\gamma e_1)^2}\leq C \Paren{\gamma^{8}n + n\max\Set{\log^4 n \gamma, n^2 \log n \gamma}},
	\end{align*}
	for a universal constant $C>0$.
	\begin{proof}
		For simplicity of the notation let $m=n-1$. Fix $q=50\log m\gamma$ and $p=\max\Set{q, \sqrt{mq}}$. Define the event $\cE=\Set{\Snorm{w}\notin\Brac{m-10p, m+10p}}$, which happens with probability at least $1-2e^{-q}$. Then,
		\begin{align*}
		\Brac{\Paren{\Snorm{w+\gamma e_1}-m}^2-2m}^2\leq  C \Paren{\gamma^8+p^4}.
		\end{align*}
		By triangle inequality,
		\begin{align*}
		\Norm{\E M(w,\gamma e_1)^2}=& \Norm{\bbP\Paren{\cE}\E\Brac{M(w,\gamma e_1)^2\given \cE }+ \bbP\Paren{\bar{\cE}}\E\Brac{M(w,\gamma e_1)^2\given \bar{\cE}}}\\
		\leq & \Norm{\bbP\Paren{\cE}\E\Brac{M(w,\gamma e_1)^2\given \cE }} + \Norm{\bbP\Paren{\bar{\cE}}\E\Brac{M(w,\gamma e_1)^2\given \bar{\cE}}}.
		\end{align*}
		We bound the first term,
		\begin{align*}
		\Norm{\bbP\Paren{\cE}\E\Brac{M(w,\gamma e_1)^2\given \cE }}\leq &\Norm{O\Paren{\Paren{\gamma^8+p^4}m}\E \Brac{w\transpose{w}\given \cE}}\\
		\leq & O\Paren{\gamma^{8}m+mp^4}\Norm{\E \Brac{w \transpose{w}\given \cE}}\\
		\leq &O\Paren{\gamma^{8}m+mp^4}\Norm{\E \Brac{w \transpose{w}}}\\
		\leq  &O\Paren{\gamma^{8}m+mp^4}.
		\end{align*}
		To bound the second term, observe that $\Norm{M(w,\gamma e_1)^2}\leq O\Paren{m^{12}+\gamma^{12}+\Norm{w}^{12}}$ for any $\gamma, w$, $m\geq 1$. For $i\in \N$, define the event
		\begin{align*}
		\cE_{q i}:=\Set{\Snorm{w}\in \Brac{m-2\sqrt{mq\cdot (i+1)} -2q(i+1), m+2\sqrt{mq\cdot(i+1)}+2q(i+1)   }}\\
		\cap \Set{\Snorm{w}\notin \Brac{m-2\sqrt{mq{i}}-2qi, m+2\sqrt{mq{i}} +2qi } }.
		\end{align*}
		By construction $\bbP(\cE_{q_i})\leq 2\max\Set{e^{-\frac{q_i^2}{4m}}, e^{-q_i/4}}$ and $\bar{\cE}\subseteq\underset{i \in \N}{\bigcup}\cE_{q i}$. By choice of $q$, it follows that
		\begin{align*}
		\Norm{\bbP\Paren{\bar{\cE}}\E\Brac{M(w,\gamma e_1)^2\given \bar{\cE}}}\leq &\underset{i\in \N}{\sum}\Norm{\bbP\Paren{\cE_{q i}}\E\Brac{M(w,\gamma e_1)^2\given \cE_{q i}}}\\
		\leq & O(1),
		\end{align*} concluding the proof.

	\end{proof}
\end{lemma}

We can now apply Bernstein Inequality \ref{theorem:matrix-bernstein}:

\begin{lemma}\label{lemma:svd-t-bound-gaussian-contribution}
	Let $w_1,\ldots,w_l\sim N(0,\Id_{n}-e_1\transpose{e_1})$, let $\Abs{\gamma_1},\ldots,\Abs{ \gamma_l}\leq\gamma\in \R$. Then for $l\geq C^*\cdot \max\Set {n^3\log (l+\gamma n), n \log^3(l+\gamma n)}$,
	\begin{align*}
	\Norm{\sum_{i\in [l]} M(w_i,\gamma_i e_1)}\leq l(\gamma^4+8)+C^*\gamma^4\sqrt{ln\log n}
	\end{align*}
	with probability at least $1-2l^{-10}-n^{-10}$,  where  $C^*$ is a universal constant.
	\begin{proof}
		By triangle inequality,
		\begin{align*}
		\Norm{\underset{i \in [l]}{\sum}M(w_i,\gamma_i e_1)}\leq \Norm{\underset{i \in [l]}{\sum}\E M(w_i,\gamma_i e_1)} + \Norm{\underset{i \in [l]}{\sum}M(w_i,\gamma_i e_1) - \underset{i \in [l]}{\sum}\E M(w_i,\gamma_i e_1)}.
		\end{align*}
		By Lemma \ref{lemma:svd-t-bound-expectation-gaussian-with-spike},
		\begin{align*}
		\Norm{\underset{i \in [l]}{\sum}\E M(w_i,\gamma_i e_1)}\le 8l+8l\gamma^2+l\gamma^4.
		\end{align*}
		Let $q=100\log(l+m\gamma)$ and $p:=\max\Set{q, \sqrt{mq}}$. Define the event $\cE=\Set{\Snorm{w}\notin\Brac{m-10p, m+10p}}$, which happens with probability at least $1-2e^{-q}$. By Lemma \ref{lemma:svd-t-max-deviation-gaussian-with-spike}, with probability at least $1-2l^{-10}$, for each $i \in [l]$,
		\begin{align*}
		\Norm{M(w_i,\gamma_i e_1) - \E M(w_i,\gamma_i e_1)}\leq C\Paren{8+8\gamma^2+\gamma^4 +\gamma^4 n+n\max \Set{n\log l,\log^2 l}},
		\end{align*}
		for a constant $C>0$.
		Hence, by Lemma \ref{lemma:svd-t-bound-variance-gaussian-with-spike}, applying Bernstein Inequality \ref{theorem:matrix-bernstein}
		\begin{align*}
		\Norm{\underset{i \in [l]}{\sum}\Paren{M(w_i,\gamma_i e_1) - \E M(w_i,\gamma_i e_1)}}\leq C'\cdot t {\sqrt{ln\log n}\cdot\gamma^4 }
		\end{align*}
		with probability at least $1-2l^{-10}-e^{-(t-1)\log n}$, where $C'$ is a universal constant.
	\end{proof}
\end{lemma}

The next lemma will be used to bound the contribution of the adversarial vector $v'$. 
\begin{lemma}\label{lemma:svd-t-upper-bound-adversarial-contribution}
	Let $\abs{a_1},\ldots,\abs{a_l}\leq a\in \R$. Let $w\sim N(0,\Id_{n}-e_1\transpose{e_1})$. Then, with probability at least $1-e^{-(t-1) \log n}-2l^{-10}$
	\begin{align*}
	\Norm{\underset{i \in [l]}{\sum}M(a_ie_1,w)}\leq C \cdot t\cdot \sqrt{l \log n} a^2\Paren{a^4 +\max\Set{\log l, \sqrt{n \log l}}}
	\end{align*}
	where $t\geq 1$ and $C>0$ is a universal constant.
	\begin{proof}
		For simplicity let $m=n-1$ and $q=10\log l$ and $p:=\max\Set{q, \sqrt{mq}}$. By Fact \ref{fact:chi-squared-tail-bounds}, 
		\begin{align*}
		\bbP\Paren{\Snorm{w}\notin\Brac{m-10p, m+10p}}\leq 2e^{-q}.
		\end{align*}
		Hence, as in Lemma \ref{lemma:svd-t-max-deviation-gaussian-with-spike}
		\begin{align*}
		\Abs{\Brac{\Paren{\Snorm{w+a_i e_1}-m}^2 -2m  }}\leq O\Paren{a_i^4 +p^2}\leq O\Paren{a^4+p^2}.
		\end{align*}
		This implies,
		\begin{align*}
		\Norm{M(a_ie_1,w)^2}\leq O\Paren{a^{12}+p^4a^4}.
		\end{align*}
		We have everything we need to apply Hoeffding Inequality \ref{theorem:matrix-hoeffding}
		\begin{align*}
		\bbP\Paren{\Norm{\underset{i \in [l]}{\sum} M(a_ie_1,w)} \geq C\cdot t \cdot \sqrt{l\log n}\Paren{a^6 +p^2a^2} }\leq e^{-(t-1)\log n},
		\end{align*}
		for $p\geq 1$ and a universal constant $C$.
	\end{proof}
\end{lemma}

The last  intermediate result, is a high probability lower bound on the spectral norm of the matrix $\underset{i \in \supp \Set{v}}{\sum}M(\lambda e_1,w)$, that is, the matrix corresponding to the sum of the columns that contain the spike.

\begin{lemma}\label{lemma:svd-t-lower-bound-contribution-u}
	Let $\zeta_1,\ldots,\zeta_l\in \Set{-1,+1}$ and $w_1,\ldots,w_l \sim N(0,\Id_n-e_1\transpose{e_1})$. Let $\gamma\in \R$, for $t\geq 1$ and auniversal constant $C>0$, suppose $l \geq C\cdot t\log  n \cdot \max\Set{n^3 \log l, \log^2 l}$.
	Then 
	\begin{align*}
		\Norm{\underset{i \in [l]}{\sum}M(\zeta_i \gamma e_1, w_i)} \geq \frac{l \gamma^6}{2},
	\end{align*}
	with probability at least $1-2e^{-t\log n}-2 l^{-10}$.
	\begin{proof} Now,
		\begin{align*}
			\underset{i \in [l]}{\sum}M(\zeta_i \gamma e_1, w_i) =& \Paren{\underset{i \in [l]}{\sum} \Brac{\Paren{\Snorm{w_i+\zeta_i \gamma e_1}-m}^2 -2m  } }\gamma^2 e_1 \transpose{e_1}\\
			=& \Brac{\gamma^4l+m^2l-2ml-2\gamma^2ml + \Paren{\underset{i \in [l]}{\sum} \Norm{w_i}^4-2m\Snorm{w_i}+2\gamma^2\Snorm{w_i}  }}\gamma^2 e_1 \transpose{e_1}.
		\end{align*}
		We bound the terms in the parenthesis. 
		Recall that by construction
		\begin{align*}
		\E\Brac{m^2l-2ml-2\gamma^2ml + \Paren{\underset{i \in [l]}{\sum} \Norm{w_i}^4-2m\Snorm{w_i}+2\gamma^2\Snorm{w_i}  }} = 0.
		\end{align*}
		For $q:=10\log l$ and $p=\max \Set{\sqrt{mq},q}$, we can condition on the event, 
		\begin{align*}
			\cE:= \Set{\forall i \in [l]\given \Snorm{w}\in\Brac{m-10p, m+10p}},
		\end{align*}
		which happens with probability at least $1- 2e^{-q}$.
		Then, by Hoeffding Inequality \ref{theorem:matrix-hoeffding},
		\begin{align*}
			\Abs{\underset{i \in [l]}{\sum} \Brac{m^2l-2ml-2\gamma^2ml + \Paren{\underset{i \in [l]}{\sum} \Norm{w_i}^4-2m\Snorm{w_i}+2\gamma^2\Snorm{w_i}  }}}\\
			\geq C\cdot  t \cdot \sqrt{l} \Paren{p^2+mp+\gamma^2p} 
		\end{align*}
		with probability at most $2e^{-t}$,	for $t\geq 1$ and a universal constant $C$.
		By assumption on $\gamma,l$ and $n$, it follows that 
		\begin{align*}
			\Norm{\underset{i \in [l]}{\sum}M(\zeta_i \gamma e_1, w_i) \gamma^2e_1 \transpose{e_1}}\geq \Norm{\frac{l \gamma^6}{2}e_i \transpose{e_1}} = \frac{l \gamma^6}{2}.
		\end{align*}
	\end{proof}
\end{lemma}

We are now ready to prove the main result of the section.  Combining Lemma \ref{lemma:svd-t-decomposition-of-matrix} with  Lemma \ref{lemma:svd-t-lower-bound-contribution-u} and an application of Lemma \ref{lemma:top-eigenvecor-bound-general} we immediately get Lemma \ref{lemma:svd-t-recovery-u}.

\begin{lemma}\label{lemma:svd-t-decomposition-of-matrix}
	Let $Y$ be defined as in Theorem \ref{theorem:recovery-fast-algorithms}, let $d\geq C\cdot n^3 \log d \log n\geq 100$. For $k\lambda^6 \geq C^* d$, $n\geq \log d$ and large enough constants $C, C^*$, with probability at least $0.999$,
	\begin{align*}
		\underset{i \in [d]}{\sum}\Brac{\Paren{\Snorm{y_i}-m}^2-2m}y_i \transpose{y_i}  = \underset{i \in \supp\Set{v}}{\sum}M(v(i)\lambda e_1,w) + M,
	\end{align*}
	where $M$ is a matrix such that 
	\begin{align*}
		\Norm{M}\leq O\Paren{d+ k\lambda^5 +\sqrt{kn\log n}\lambda^5 }.
	\end{align*}

	\begin{proof}
		Let $m=n-1$. Recall the notation used in the algorithm with $c_6(y_i,n)=\Brac{\Paren{\Snorm{y_i}-m}^2-2m}$. Then we can rewrite the matrix $A$ computed by SVD-6 as,
		\begin{align*}
		\underset{i \in [d]}{\sum}\Brac{\Paren{\Snorm{y_i}-m}^2-2m}y_i \transpose{y_i} =& \underset{i \in \supp\Set{v}}{\sum}M(v(i)\lambda e_1,w) + \underset{i \in \supp\Set{v}}{\sum}M(w_i,v(i)\lambda e_1) \\
		&+ \underset{i \in [d]\setminus \supp \Set{v}}{\sum}M(w_i,v'(i)e_1) + \underset{i \in [d]\setminus \supp \Set{v}}{\sum}M(v'(i)e_1,w)\\
		&+\underset{i \in \supp\Set{v}}{\sum}c_6(w_i + v(i)\lambda e_1,n)\Paren{v(i)w_i\transpose{e_1}+v(i)e_1\transpose{w_i}}\\
		&+ \underset{i \in [d]\setminus \supp \Set{v}}{\sum} c_6(w_i+v'(i)e_1)\Paren{v'(i)w_i \transpose{e_1} + v'(i)e_1\transpose{w_i}}.
		\end{align*}
		We first bound the cross-terms,
		\begin{align*}
		&\Norm{\underset{i \in [d]\setminus \supp \Set{v}}{\sum} c_6(w_i+v'(i)e_1)v'(i)w_i \transpose{e_1}}\\ &\leq  \Norm{\underset{i \in [d]\setminus \supp \Set{v}}{\sum} c_6(w_i+v'(i)e_1)v'(i)^2e_1\transpose{e_1}}^{1/2}\Norm{\underset{i \in [d]\setminus \supp \Set{v}}{\sum} c_6(w_i+v'(i)e_1)v'(i)^2w_i\transpose{w_i}}^{1/2}\\
		&\leq \Norm{\underset{i \in [d]\setminus \supp \Set{v}}{\sum} c_6(w_i+v'(i)e_1)v'(i)^2e_1\transpose{e_1}}+\Norm{\underset{i \in [d]\setminus \supp \Set{v}}{\sum} c_6(w_i+v'(i)e_1)v'(i)^2w_i\transpose{w_i}}
		\end{align*}
		And
		\begin{align*}
		&\Norm{\underset{i \in \supp\Set{v}}{\sum}c_6(w_i + v(i)\lambda e_1,n)v(i)w_i\transpose{e_1}}\\
		&\leq \Norm{\underset{i \in \supp\Set{v}}{\sum}c_6(w_i + v(i)\lambda e_1,n)v(i)^2e_1\transpose{e_1}}^{1/2}\Norm{\underset{i \in \supp\Set{v}}{\sum}c_6(w_i + v(i)\lambda e_1,n)w_i\transpose{w_i}}^{1/2}
		\end{align*}
		Observe that, by construction of the vector $v'$ in Model \ref{problem:almost_gaussian_model_special_case} and since $k\lambda^6\geq  C^*d$, for a large enough constant $C^*$, we get that for all $i \in [d]$, $\Abs{v'(i)}\leq 100$.
		Moreover we get that with probability at least $0.999$, all the following inequalities hold.
		
		By Lemma \ref{lemma:svd-t-lower-bound-contribution-u},
		\begin{align*}
		\Norm{\underset{i \in \supp \Set{v}}{\sum}M(v(i)\lambda e_1, w_i)}\geq \frac{k \lambda^6}{2}.
		\end{align*}
		By Lemma \ref{lemma:svd-t-bound-gaussian-contribution},
		\begin{align*}
		\Norm{\underset{i \in \supp\Set{v}}{\sum}M(w_i,v(i)\lambda e_1)}&\leq O \Paren{k \lambda^4+\sqrt{kn \log n}\lambda^4}\\
		\Norm{\underset{i \in [d]\setminus \supp \Set{v}}{\sum}M(w_i,v'(i)e_1)}&\leq O\Paren{d+\sqrt{dn\log n}}\leq O(d).
		\end{align*}
		By Lemma \ref{lemma:svd-t-upper-bound-adversarial-contribution},
		\begin{align*}
		\Norm{\underset{i \in [d]\setminus \supp \Set{v}}{\sum}M(v'(i)e_1,w)}\leq O\Paren{\sqrt{d\log n} \max \Set{\log d, \sqrt{n \log d}}}.
		\end{align*}
		All in all we get,
		\begin{align*}
		&\Norm{\underset{i \in [d]}{\sum}\Brac{\Paren{\Snorm{y_i}-m}^2-2m}y_i \transpose{y_i} -\underset{i \in \supp \Set{v}}{\sum}M(v(i)\lambda e_1, w_i)}\\
		&\leq O\Paren{d+\sqrt{dn\log n}+  \sqrt{dn}\log d +\sqrt{d}\log d + k\lambda^5 +\sqrt{kn\log n}\lambda^5 }.
		\end{align*}
		The result follows.
	\end{proof}
\end{lemma}

\subsection{Algorithm recovers \textit{v} with high probability}\label{subsection:fast-algorithm-recover-v}
We now show how to obtain a good estimate of the sparse vector $v$ from $\hat{u}$. Since several algorithms try to recover first $u$ and then the sparse vector (e.g. SVD with thresholding) we turn back to the model  \ref{def:wishart-matrix-model}.
A corollary for model \ref{problem:almost_gaussian_model_special_case} is presented at the end of the section.
So, for the rest of the section, let $Y=\sqrt{\beta}u\transpose{v}  + W+ E$, where $v\in \R^d$ is a $k$-sparse unit vector, $u\in\R^{n}$ is a vector such that $\norm{u} \ge 0.9\sqrt{n}$, and $W\sim N(0,1)^{n\times d}$.
We also assume that $n \le k \le d$.

The first observation is that on one hand  the vector $Yv$ is close to  $\sqrt{\beta}u$ with high probability.
On the other hand,  the vector $\transpose{Y}u$ may be far from the sparse vector; that is, even knowing exactly $u$, the thresholding step is required to recover $v$.
The next theorem provides guarantees on the achievable correlation with the sparse vector given a vector close to $u$.
Theorem \ref{theorem:estimation_lower_bound} shows in which sense in which these guarantees are information theoretically tight. 

\begin{theorem}\label{theorem:from_u_to_v}
	Let $\hat{u}$ be a vector such that $\norm{\hat{u}-u}\leq \varepsilon \sqrt{n}$ 
	for some $0 \le \varepsilon \le \frac{1}{10}$,
	and let 
	\[
	\hat{v}=\frac{1}{\sqrt{\beta}\cdot\norm{\hat{u}}^2 }\transpose{\hat{u}}Y\,.
	\]
	If $\beta \gtrsim \frac{k}{\tau^2n} \Paren{\log d+ \norm{E}_{1\to 2}^2}$ 
	for some $0< \tau \le 1$, then
	with probability at least $1 - 10\exp\Paren{-n}$, 
	\[\norm{\eta(\hat{v})-v}\lesssim \varepsilon + \tau\,,\]
	where $\eta(\hat{v})  \in \R^d$ is the vector with coordinates 
	\begin{equation*}
	\eta(\hat{v})_i = 
	\begin{cases}
	\hat{v}_i , &\text{if}\;\; \Abs{\hat{v}_i} \geq  \tau/\sqrt{k}\\
	0, &\text{otherwise}
	\end{cases}
	\end{equation*}
	\begin{proof}
		Assume that $\beta \ge 10^4\cdot  \frac{k}{\tau^2n} \Paren{\log d+ \norm{E}_{1\to 2}^2}$.
		Let's rewrite $Y=\sqrt{\beta}\hat{u}\transpose{v}  + W + Z + E$ for a matrix 
		$Z=\sqrt{\beta}\Paren{u - \hat{u}}\transpose{v}\in \R^{n\times d}$.  
		Then for $i \in \Brac{d}$:
		\[
		\Abs{\Paren{\transpose{\hat{u}}Z}_i}=\Abs{\sqrt{\beta}\iprod{\hat{u},u - \hat{u}}v_i}
		\leq \varepsilon\sqrt{\beta n} \cdot \norm{\hat{u}} \cdot |v_i|\,.
		\]
		Let $S = \set{i \;|\; v_i = 0}$, $T = \set{i\;|\;\hat{v}_i >  \tau/\sqrt{k}}$, 
		$A = \set{i\;|\; \abs{v_i}\le  2\tau/\sqrt{k}}$ and 
		$B = \set{i\;|\; (\hat{u}W)_i \ge  10\norm{\hat{u}}\sqrt{\log d}}$.  By Lemma \ref{lemma:number-of-large-coordinates-in-row-span-of-gaussian}, with probability at least $1-2\exp\Paren{-n}$, $\card{B} \le  n$.
		Consider some $i \in S\cap T$. Since $v_i =0$,
		\[
		(\hat{u}W)_i = \sqrt{\beta}\cdot \norm{\hat{u}}^2\cdot\hat{v}_i - (\hat{u}E)_i \ge
		100\cdot \frac{\norm{\hat{u}}}{\sqrt{n}}\cdot
		\norm{\hat{u}}\Paren{\sqrt{\log d} + \norm{E}_{1\to 2}} - 
		\norm{\hat{u}} \cdot \norm{E}_{1\to 2} 
		\ge	10\norm{\hat{u}}\sqrt{\log d}\,,
		\]
		which means that $ S\cap T \subseteq B$. Hence $\card{T\setminus B} \le \card{\overline{S}} = k$.
		Note that since $\varepsilon \le \frac{1}{10}$ and $\norm{u} \ge 0.9\sqrt{n}$,
		$\norm{\hat{u}}\ge 0.8\sqrt{n}$. Hence
		\begin{align*}
		\sum_{i\in T\setminus B} (\eta(\hat{v})_i - v_i)^2 
		&= 
		\sum_{i\in T\setminus B} \Paren{\hat{v}_i-v_i}^2
		\\&\le 
		2\sum_{i\in T\setminus B} \frac{n}{\norm{\hat{u}}^2}\varepsilon^2 v_i^2 + 
		2\sum_{i\in T\setminus B} \frac{1}{\beta \norm{\hat{u}}^4} 
		\Paren{(\transpose{\hat{u}}W)_i^2 +  (\transpose{\hat{u}}E)_i^2}		
		\\&\le 
		4\varepsilon^2 + 
		4\sum_{i\in T\setminus B} \frac{1}{\beta n} 
		\Paren{100\log d + \norm{E}_{1\to 2}^2}
		\\&\le 4\varepsilon^2 +  \tau^2\,.
		\end{align*}
		Note that since $\card{B} \le n$, By Theorem \ref{theorem:k-sparse-norm-gaussian}, with probability at least $1-\exp\Paren{-n}$,
		\[
		\sum_{i\in B} (\transpose{\hat{u}}W)_i^2 \le 100\norm{\hat{u}}^2 \cdot  n \log d\,.
		\]
		Hence
		\begin{align*}
		\sum_{i\in T\cap B} (\eta(\hat{v})_i - v_i)^2 
		&= \sum_{i\in T\cap B} \Paren{\hat{v}_i-v_i}^2 
		\\&\le 
		2\sum_{i\in B} \frac{n}{\norm{\hat{u}}^2}\varepsilon^2 v_i^2  + 
		2\sum_{i\in B} \frac{1}{\beta \norm{\hat{u}}^4} {(\transpose{\hat{u}}W)_i^2} +  
		2\sum_{i\in B} \frac{1}{\beta \norm{\hat{u}}^4} (\transpose{\hat{u}}E)_i^2
		\\&\le 
		4\varepsilon^2  + 
		400 \frac{\log d}{\beta} + 
		4\frac{\norm{E}_{1\to 2}^2}{\beta}
		\\&\le 
		4\varepsilon^2 + 
		\tau^2\,.
		\end{align*}
		If $i\in S\setminus T$, then $\eta(\hat{v})_i = v_i = 0$. 
		If $i \in \overline{S}\cap \overline{T} \cap\overline{A}$, then
		\[
		\frac{\tau}{\sqrt{k}}\ge
		\abs{\hat{v}_i} \ge 
		\Paren{1-\frac{\sqrt{n}}{\norm{\hat{u}}}{\varepsilon} }\abs{v_i} - 
		\Abs{\frac{(\transpose{\hat{u}}W)_i}{\sqrt{\beta}\norm{\hat{u}}^2}} -
		\Abs{\frac{(\transpose{\hat{u}}E)_i}{\sqrt{\beta}\norm{\hat{u}}^2}}
		\ge
		1.8\frac{\tau}{\sqrt{k}} - 
		\Abs{\frac{(\transpose{\hat{u}}W)_i}{\sqrt{\beta}\norm{\hat{u}}^2}}
		- 0.1\frac{\tau}{\sqrt{k}}
		\,,
		\]
		hence in this case 
		$\Abs{(\hat{u}W)_i} > 0.7\cdot 0.8 \cdot 100\sqrt{\log d} \ge 10 \sqrt{\log d}$, so $i\in B$.
		Moreover,
		\[
		\abs{v_i} \le \frac{1.1\tau}{0.9\sqrt{k}} + 2\Abs{\frac{1}{\sqrt{\beta}n}(\transpose{\hat{u}}W)_i}\,.
		\]
		Therefore
		\[
		\sum_{i\in \overline{S}\cap \overline{T} \cap\overline{A}} (\eta(\hat{v})_i - v_i)^2 
		= 
		\sum_{i\in \overline{S}\cap \overline{T} \cap\overline{A}} v_i^2 
		\le 
		2\sum_{i\in B} \frac{2\tau^2}{k} + 4\sum_{i\in B} \frac{1}{\beta n^2} (\hat{u}W)_i^2
		\le 4\tau^2 + \tau^2 = 5\tau^2\,.
		\]
		It follows that
		\[
		\norm{\eta(\hat{v})-v}^2 \le 
		\sum_{i\in T} (\eta(\hat{v})_i - v_i)^2 + 
		\sum_{i\in \overline{S}\cap \overline{T} \cap A} v_i^2 +  
		\sum_{i\in \overline{S}\cap \overline{T} \cap\overline{A}} v_i^2
		\le 8\varepsilon^2 + 2\tau^2 + 4\tau^2 + 5\tau^2
		=8\varepsilon^2 + 11\tau^2\,.
		\]
		Hence
		with probability at least $1 - 3\exp\Paren{-n}$, 
		\[\norm{\eta(\hat{v})-v}\lesssim \varepsilon + \tau\,.\]
	\end{proof}
\end{theorem}

An immediate consequence is the following corollary.

\begin{corollary}\label{cor:recovering-v-from-u-lambda}
	Consider a matrix of the form,
	\begin{align*}
	Y=W+\lambda u\transpose{v}+ u\transpose{\Paren{v'-\transpose{W}u}}
	\end{align*}
	for a Gaussian matrix $W\sim N(0,1)^{n \times d}$,  a random unit vector $u$, a $k$-sparse vector $v$ with entries in $\Set{0,\pm 1}$ and a vector $v'$ as defined in \ref{problem:almost_gaussian_model_special_case}.
	Let $\hat{u}$ be a vector such that $\norm{\hat{u}-u}\leq \varepsilon$ 
	for some $0 \le \varepsilon \le \frac{1}{10}$. If $\lambda \gtrsim \frac{\sqrt{\log d}}{\tau}$, then we can compute in time $O(n d)$ an estimator $\hat{v}$  such that
	with probability at least $1-\exp(-n)$
	\[
	\norm{\hat{v} - v} \lesssim \Paren{\varepsilon +  \tau}\sqrt{k}\,.
	\]
\end{corollary}

\section{Experiments}\label{sec:experiments}
\newcommand{\namedtheorem}{}
\newtheorem*{genericremark*}{\namedtheorem}
\newenvironment{namedremark*}[1]
{\renewcommand{\namedtheorem}{#1}%
	\begin{genericremark*}}
	{\end{genericremark*}}
In this section we compare the performance of Diagonal Thresholding and SVD of degree 2, 4, 6  as in \ref{algorithm:svd-t} on practical instances. The table below explains the regimes of the figures presented. We refer to Robust Sparse PCA as model \ref{def:wishart-matrix-model} where the adversarial matrix $E$ follows the distribution shown  in \ref{problem:almost_gaussian_model_special_case}. Appendix \ref{section:experimental-setup} contains a detailed report of the experimental setup.
\begin{table}[h]
	\centering
	\begin{tabular}{|m{2cm}|m{6cm}|m{6cm}|}
		\hline
		&\textbf{Standard Sparse PCA} & \textbf{Robust Sparse PCA}\\\hline
		$k\geq \sqrt{d}$&Figure \ref{subfigure:svd-2} for $\beta \geq \sqrt{\frac{d}{n}}$& Figure \ref{subfigure:svd-4} for $\beta\geq \frac{k}{n}\Paren{\frac{d}{k}}^{1/2}$ \qquad Figure \ref{subfigure:svd-6} for $\beta\geq \frac{k}{n}\Paren{\frac{d}{k}}^{1/3}$\\\hline
		$k\leq \sqrt{d}$&Figure \ref{figure:dt-settings} for $\beta\geq \frac{k}{\sqrt{n}}\sqrt{\log \frac{d}{k}}$&  \\\hline
	\end{tabular}
  \caption{Plots}
\end{table}

\begin{figure}[!ht]%
	\centering
	\subfloat[Standard Sparse PCA, with $k\geq \sqrt{d}$, $\beta \geq \sqrt{\frac{d}{n}}$]{{\includegraphics[width=7.8cm]{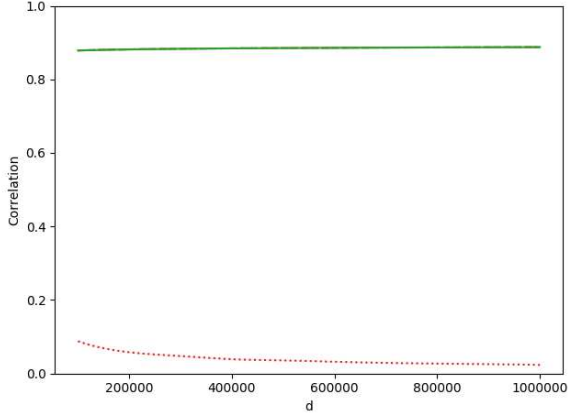} }\label{subfigure:svd-2}}%
	\quad
	\subfloat[Robust Sparse PCA with $k\geq \sqrt{d}$, $\beta \geq \frac{k}{n}\Paren{\frac{d}{n}}^{1/2}$ ]{{\includegraphics[width=7.8cm]{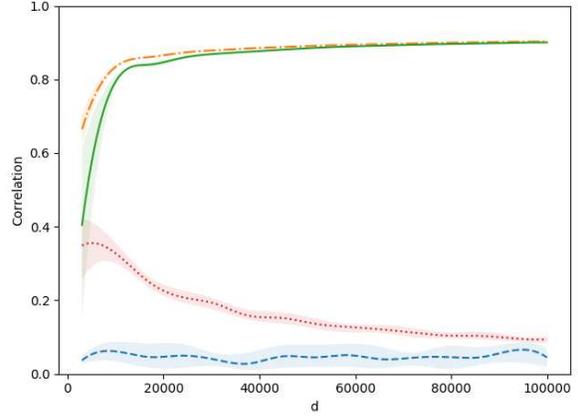} }\label{subfigure:svd-4}}%
	
	\subfloat[Robust Sparse PCA with $k\geq \sqrt{d}$, $\beta \geq \frac{k}{n}\Paren{\frac{d}{n}}^{1/3}$]{{\includegraphics[width=8.3cm]{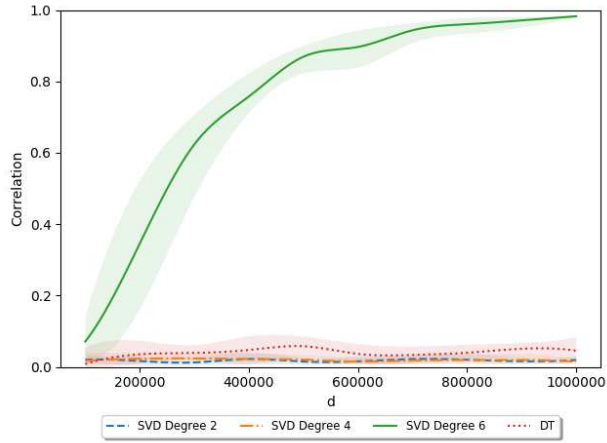} }\label{subfigure:svd-6}}%
	\caption{Forthe single spiked covariance model with $k> \sqrt{d}$, Figure \ref{subfigure:svd-2} shows how the SVD algorithms works (with information theoretically optimal guarantees) and Diagonal Thresholding fails.
    Figures \ref{subfigure:svd-4}, \ref{subfigure:svd-6} show however how adversarial noise immediately breaks  SVD with thresholding.
    In Figure \ref{subfigure:svd-4} $\beta\gtrsim\frac{k}{n}\Paren{\frac{d}{k}}^{1/2}$, hence as $d$ increases and becomes larger than $n^2$, SVD-4 returns a good estimate.
    We point out how how  SVD-6 performs well even for $d \ll n^3$ when the signal is much larger than $\frac{k}{n}\Paren{\frac{d}{k}}^{1/3}$.
    Finally, Figure \ref{subfigure:svd-6} shows how DT, SVD-4 and SVD-2 fails for $\beta =\Theta \Paren{\frac{k}{n}\Paren{\frac{d}{k}}^{1/3}}$, but as $d$ grows towards $n^3$, SVD-6 approaches correlation 1.}%
	\label{figure:svd-settings}
\end{figure}

\begin{figure}[!ht]%
	\centering
	\includegraphics[width=8.3cm]{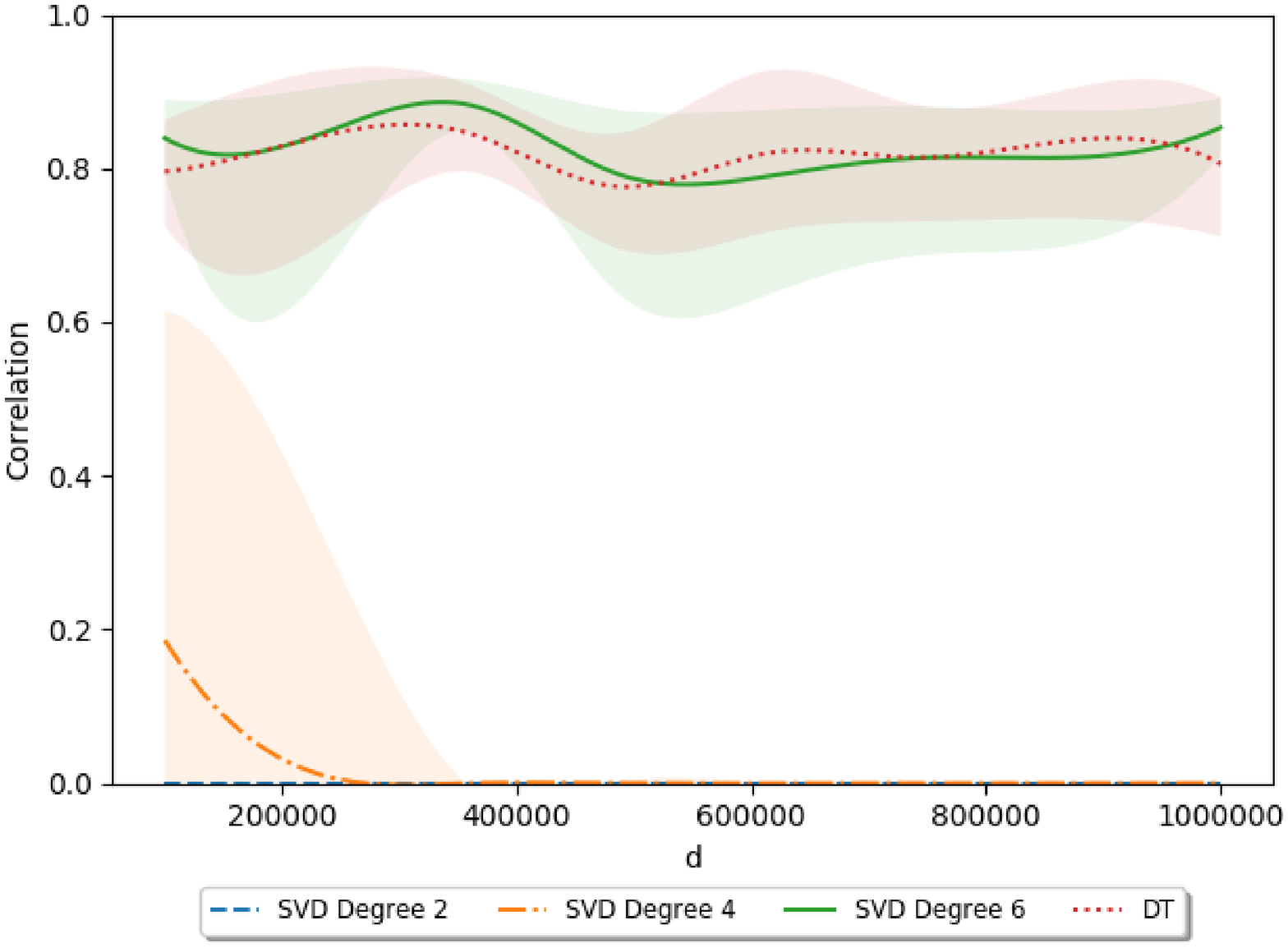}
	\caption{The figure shows settings in which $k\leq \sqrt{d}$. In this regime, among the algorithms considered, Diagonal Thresholding achieves asymptotically the most correlation. In practical settings however it is often the case that $\frac{k}{\sqrt{n}}\sqrt{\log d}\geq \Omega \Paren{\frac{k}{n}\Paren{\frac{d}{k}}^{1/3}}$ and hence also SVD-6 can accurately recover the signal.}%
	\label{figure:dt-settings}
\end{figure}

\clearpage


\phantomsection
\addcontentsline{toc}{section}{References}
\bibliographystyle{amsalpha}
\bibliography{bib/mathreview,bib/dblp,bib/custom,bib/scholar}

\appendix


\section{Relationship with Clustering mixture of subgaussians} \label{sec:relationship-gmm}
The canonical version of the Gaussian Mixture Model consists of $d$ samples $y_1,\ldots,y_d\in \R^n$ from a mixture of $k$ Gaussian probability distributions $G_1,\ldots,G_k$ with means $\mu_1,\ldots,\mu_k$ and covariances $\Sigma_1,\ldots\Sigma_k$ close to the identity in the spectral sense $\Norm{\Id-\Sigma_i}\lesssim 1$, such that $\forall i,j \in[k]$, $\Norm{\mu_i-\mu_j}\geq \Delta$ for some parameter $\Delta$ of the problem. The goal is then to partition the samples in clusters $S_1,\ldots, S_k$ such that $\forall y_j \in S_i$, $y_j\sim G_i$. Notice that the input can be written as $Y=W+X$ where $X$ is a matrix with only $k$ distinct columns $\mu_1,\ldots,\mu_k$ and $W$ is a matrix with independent columns following distributions  $N(0,\Sigma_1),\ldots,N(0,\Sigma_k)$.

There is a reach literature concerning this problem both from a statistical and computational perspective (e.g. see \cite{pearson, DBLP:conf/focs/Dasgupta99, DBLP:conf/focs/MoitraV10, DBLP:conf/focs/VempalaW02, DBLP:conf/focs/DiakonikolasKK016,  DBLP:conf/focs/0001V17, DBLP:conf/stoc/Hopkins018,  DBLP:journals/corr/abs-1711-11581}). A simple greedy algorithm called \textit{single-linkage clustering} can be designed observing that with high probability, whenever $\Delta\gtrsim n^{1/4}$, pairs of samples from the same cluster are closer in Euclidean distance closer than pairs of samples from different clusters. Furthermore, for $k< n$, as the centers $\mu_1,\ldots,\mu_k$ lives in a $k$ dimensional subspace and  this space is close to the span of the top $k$ singular vectors of $Y$, the bound can be pushed down to $\Delta \gtrsim k^{1/4}$ projecting the points into this low-dimensional space.

This algorithmic barrier of $\Delta\gtrsim k^{1/4}$ was broken only recently. Independently \cite{DBLP:conf/stoc/Hopkins018,  DBLP:journals/corr/abs-1711-11581} provided SoS algorithms able to reconstruct the clusters for $\Delta\gtrsim k^{\gamma}$, for any $\gamma>0$, in time $(dk)^{O(1/\gamma^2)}$ using $kd^{1/\gamma}$ samples. This new approach (which we refer to as the \textit{moments method}) is based on the following intuition: given a set $S_i=\Set{y_1,\ldots,y_{d^{1/\gamma}}}$ of $d^{1/\gamma}$ samples, if most of the samples come form the same distribution $G_i$, then the first $1/\gamma$ moments of the empirical distribution will satisfy the subgaussian bound $\E_{S_i}\Brac{y^t}\lesssim t^{t/2}$. Conversely, if samples actually belongs to different clusters, the $t$-th moment will not be Gaussian.

For our discussion, an important observation is that the same algorithms works for mixtures of subgaussians.

It is insightful to see the Sparse PCA problem in this perspective. Rewrite the matrix $Y=W+\sqrt{\beta}u_0 \transpose{v_0}$ as $Y=W+\lambda \bar{u}\transpose{\sigma}$ where $\bar{\sigma}\in\R^d$ is the vector with support $\supp\Set{v_0}$ and entries $\sigma_i=\sign(v_{0,i})$, $\bar{u}$ is the vector $\bar{u}:=\frac{u_0}{\Norm{u_0}}$ and  $\lambda = \sqrt{\frac{\beta n}{k}}$. The two models are equivalent as with high probability $\sqrt{\frac{\beta}{k}}\Norm{u_0}=\Theta(\lambda)$. Furthermore, notice that the matrix $\lambda \bar{u}\transpose{\sigma}$ has only three distinct columns: $-\bar{u}$, $\bar{u}$ and the zero vector. That is, we could see an instance of the canonical Sparse PCA problem as a non-uniform mixture of three Gaussian distributions with separation $\lambda$. In this formulation,  Diagonal and Covariance Thresholding recovers the sparse vector (and hence separates the cluster centered at the origin from the others) for $\lambda \geq \tilde{O}\Paren{n^{1/4}}$, the same bound as the single-linkage algorithm.

Consider now the Wishart model \ref{def:wishart-matrix-model} with adversarial perturbations: $Y=W+\sqrt{\frac{\beta}{k}}\norm{u_0}\bar{u}\cdot \transpose{\bar{\sigma}}+E$. Fix $D=1/\gamma$, the adversarial matrix $E_D$ will be the one described in Section \ref{sec:lower-bounds-robust-settings} and \cref{sec:lowerbounds}, which can be written as $E_D=\bar{u}\transpose{\Paren{v' -\transpose{W}\bar{u}}}$ for some vector $v'\in \R^d$ with support $[d]\setminus \supp\Set{v_0}$ and such that its non-zero entries are distributionally independent. All in all, this leads us to the formulation 
\begin{align*}
Y=W+E_D+\lambda \bar{u}\cdot\transpose{\bar{\sigma}}=W'+X,
\end{align*} with $W'=W+E_D$ being a matrix with independent columns and subgaussian moments. Again the matrix $X$ has three distinct columns: $-\bar{u}, \bar{u}$ and the zero vector. 

Similarly to the non-robust settings, this formulation can be seen as an non-uniform mixture of three subgaussian distributions with separation $\Delta \gtrsim \lambda$. By the argument shown in \cref{sec:lowerbounds} any algorithm that can be described as a low-degree polynomial \textit{and} that tries to cluster these points using the moments method will be able to detect that $Y$ is not a single subgaussian distribution (and hence it is not a good cluster) only using at least $d^{D}\gtrsim kd^{D}$ samples.

\section{Comparison with the Wigner model}\label{sec:wigner-model}
The Wigner model presents some differences in the robust settings. Here we consider a matrix 
$Y = \beta v_0\transpose{v_0} + W+ E,$
where $W\sim N(0,1)^{d \times d}$, $v_0$ is a $k$-sparse unit vector with entries in $\Set{\pm \frac{1}{\sqrt{k}},0}$ and the jointly-distributed random variables $W$, $v_0$ are independent. The matrix $E$ has norm $\Normi{E}\leq b^2$ for some $b \in \R$. An analysis similar to the one made for the Wishart model in \cref{sec:perturbation-resilience-overview} shows that in order to have an algorithm that outputs a vector $\hat{v}$ such that $\iprod{\hat{v},v_0}^2$ is bounded away from zero with high probability, the adversarial matrix needs to satisfy the bound $\Normi{E}\leq \frac{\beta}{k}$. 

In these settings the simple PCA approach of computing the top eigenvector of $Y$ and removing all but the top $k$ entries fails. Indeed it suffices to plant a matrix $E=z z^\top$ where $z$ is the vector with entries $z_i=b$ if $i \in[d]\supp\Set{v_0}$ and $0$ otherwise. For $bd\gtrsim \beta$ the top eigenvector of $Y$ is almost orthogonal to $v_0$, and so is its projection to the top $k$ coordinates.

The Covariance Thresholding algorithm also can be easily fooled by an adversary with the same approach  to the one previously shown: simply let $E=z z^\top$ be a  rank $1$ matrix with $\supp\Set{z}\cap \supp\Set{v_0}=\emptyset$.

The difference appears in the Diagonal Thresholding algorithm, which turns out to be perturbation resilient. Indeed as $\Normi{E}\lesssim \frac{\beta}{k}$, for $i \in \supp\Set{v_0}$ diagonal entries will have value $\Abs{Y_{ii}}\gtrsim 1+\beta-\Normi{E}\gtrsim 1+\beta$. Conversely diagonal entries indexed by $j\notin \supp\Set{v_0}$ will have value bounded by $\tilde{O}\Paren{1+\Normi{E}}\lesssim 1+\beta$. The reason behind this diversity is that, in model \ref{def:wishart-matrix-model} the adversarial perturbation exploits the large norm of the columns of $W$.

\section{Thresholding Algorithms are Fragile}
\label{sec:thresholding-algorithms}
\label{sec:non-robust-algorithms}

In this section we formalize the discussions of the introduction and show that SVD with Thresholding, Diagonal Thresholding and Covariance Thresholding are indeed not resilient to adversarial perturbations.

\subsection{SVD with Thresholding is Fragile }\label{sec:fool-svd} The polynomial-time algorithm presented in \cref{sec:introduction-new} for the strong-signal regime is highly sensitive to small adversarial perturbations.
Concretely, this can be shown constructing $E$ with entries bounded \(\tilde O(1/\sqrt {n\,})\) so that eigenvectors of $\transpose{Y}Y$ cannot be used to recover $u_0$.

Consider $E = -\gamma u_0\transpose{u_0}W$ for some $0 < \gamma < \snorm{u_0}$ 
that we will choose later.
Then $Y = \sqrt{\beta} u_0\transpose{v_0} + \Paren{\Id - \gamma u_0\transpose{u_0}}W$ and
\begin{align*}
Y\transpose{Y} =& \beta u_0\transpose{u_0} + 
\Paren{\Id - \gamma u\transpose{u_0}}W\transpose{W}\Paren{\Id - \gamma u_0\transpose{u_0}}\\
&+\sqrt{\beta}\Paren{u_0\transpose{v_0}W\Paren{\Id - \gamma u_0\transpose{u_0}} + 
	\Paren{\Id - \gamma u_0\transpose{u_0}}\transpose{W}v_0\transpose{u_0} }
\end{align*}
Hence with high probability,
\[
\frac{1}{\snorm{u_0}}\snorm{\transpose{u_0}Y} =
z  + \beta\snorm{u_0} + \gamma^2\snorm{u_0}d - 2\gamma\norm{u_0} d
+ \tilde{O}\Paren{\sqrt{\beta/n}}
\,, 
\]
where $z$ has a  $\chi^2$-distribution with $d$ degrees of freedom. On the other hand notice that for a unit vector $x$ orthogonal to $u_0$ and independent of $W$, we get $\snorm{\transpose{x}Y}= \snorm{\transpose{x}W}$ which has the same distribution as $z$. 
So our claim follows choosing  $\gamma$ so that 
$2\gamma\norm{u_0} - \gamma^2\snorm{u_0} = \beta \cdot\frac{\snorm{u_0}}{d}+ \tilde{O}\Paren{\frac{1}{d}\sqrt{\beta/n}}$. Indeed then $u_0Y$ has the same distribution as $z$. Now, since with high probability $\snorm{u_0} \le 2n$, if $d/n \gtrsim \beta$, such a $\gamma$ exists.

\subsection{Diagonal Thresholding is Fragile}\label{sec:fooling-diagonal-threshilding}
Recall that Diagonal Thresholding finds the top $k$ diagonal entries of the covariance matrix and output a top eigenvector of the corresponding $k\times k$ principal submatrix.  We shows here that a simple adversary can make diagonal entries in $[d]\setminus \supp\Set{v_0}$ larger than diagonal entries in $\supp\Set{v_0}$, hence leading the algorithm to choose a submatrix which contain no information about the sparse vector.  

Concretely, the algorithm can be written as follows:

\begin{mdframed}[nobreak=true]
	\begin{algorithm}[Diagonal Thresholding]
		\label[algorithm]{alg:dt}\mbox{}
		\begin{description}
			\item[Given:]
			Sample matrix $Y$ of  form \ref{def:wishart-matrix-model} where $v_0$ is a flat vector.
			\item[Estimate:] The sparse vector $v_0$.
			\item[Operation:]\mbox{}
			\begin{enumerate}
				\item Let $S:=\Set{i_1,\ldots,i_k}\subseteq [d]$ be the set of indexes denoting the $k$ largest diagonal entries of $\transpose{Y}Y$.
				\item Output a top eigenvector of $\transpose{Y}Y \Brac{S\times S}$.
			\end{enumerate}
		\end{description}    
	\end{algorithm}
\end{mdframed}

We start by defining  the adversarial matrix.

\begin{definition}\label{def:adversary-dt}
	Let $b\in R$ and denote with $W_1,\ldots,W_d$ the columns of $W$. Define $E$ to be the matrix with columns
	\begin{align*}
	E_i =\begin{cases}
	\frac{b}{\Norm{W_i}}W_i &\text{ if} i \in [d]\setminus \supp\Set{v_0}\\
	0 &\text{otherwise.}
	\end{cases}
	\end{align*}
\end{definition}
The result is shown in  the theorem below.

\begin{theorem}\label{thm:fooling-dt}
	Let $n \geq \omega\Paren{\log d}$, $\beta = o \Paren{k}$. Let $Y$ be sampled according to \ref{def:wishart-matrix-model} where $v_0$ is a flat vector. Let $E$ be as defined in \ref{def:adversary-dt} and $\Norm{E}_{1\rightarrow 2}\gtrsim \frac{\beta \sqrt{n}}{k}+\sqrt{\log d}$. Then for each $i \in [d]\supp\set{v_0}$ and $j \in \supp\Set{v_0}$
	\begin{align*}
	\Snorm{Ye_i}\geq \Snorm{Ye_j}
	\end{align*}
	with probability at least $0.99$. 
\end{theorem}

Notice how, for $\beta=\Theta \Paren{\frac{k}{\sqrt{n}}\sqrt{\log d}}$ the theorem implies that an adversary with $\Norm{E}_{1\rightarrow 2}\gtrsim \sqrt{\log d}$ suffices to fool Diagonal Thresholding. A perturbation resilient algorithm would succeed as long as $\Norm{E}_{1\rightarrow 2}\lesssim \min \Set{\Paren{n \log d}^{1/4}, \sqrt{k\log d}}$.

\begin{remark}
	The same adversary also fools the limited exhaustive search algorithm from \cite{conf/innovations/BandeiraKW20} that runs in time $n^{O(t)}$ up to some very large $t$ (say, up to some $t = n^{\Omega(1)}$).
\end{remark}

\begin{proof}[Proof of Theorem \ref{thm:fooling-dt}]
	Let $b= \Norm{E}_{1\rightarrow 2}$. We condition our analysis on the event that
	\begin{align*}
	\forall i \in [d]\qquad &\Snorm{W_i}\in \Brac{n-40\sqrt{n \log d},n+ 40\sqrt{n \log d}}.\\
	&\Snorm{u_0}\leq n+100\sqrt{n}\\
	\forall i \in [d]\qquad &\iprod{u_0,W_i}\leq 10\sqrt{n \log d}
	\end{align*}
	which happen with probability at least $0.99$ by Fact \ref{fact:chi-squared-tail-bounds}. Denote with $e_1,\ldots,e_d$ the standard basis vectors in $\R^d$. Notice that, by construction of $E$, for $i \in [d]\setminus\supp\Set{v_0}$, $\frac{Ee_i}{\Norm{Ee_i}}=\frac{We_i}{\Norm{We_i}}$. Thus,
	\begin{align*}
	\Snorm{Ye_i} &=\Snorm{\Paren{W+E+\sqrt{\beta}u_0\transpose{v_0}}e_i } \\
	&= \Snorm{\Paren{1+\frac{b}{\Norm{W_i}}}W_i}\\
	&= \Snorm{W_i}+b^2+2b\Norm{W_i}\\
	&\geq n +b^2+b\sqrt{n}-O\Paren{\sqrt{n \log d}}\\
	&\geq n+b\sqrt{n}-O\Paren{\sqrt{n \log d}}.
	\end{align*}
	On the other hand, for $j \in \supp\Set{v_0}$,
	\begin{align*}
	\Snorm{Ye_j}&=\Snorm{W_j} + \frac{\beta}{k}\Snorm{u_0}+\sqrt{\frac{\beta}{k}}\iprod{W_j,u_0}\\
	&\leq n +\frac{n\beta}{k}+ O\Paren{\sqrt{n\log d}+\sqrt{\frac{\beta n \log d}{k}}}\\
	&\leq n +\frac{n\beta}{k}+ O\Paren{\sqrt{n\log d}}
	\end{align*}
	where the last step follows as $\beta =o(k)$. Combining the two inequalities,
	\begin{align*}
	\Snorm{Ye_i}-\Snorm{Ye_j}&\geq b\sqrt{n} -\frac{n\beta}{k}-O\Paren{\sqrt{n\log d}}
	\end{align*}
	which is larger then zero whenever,
	\begin{align*}
	b\geq O\Paren{\frac{\beta \sqrt{n}}{k}+\sqrt{\log d}}.
	\end{align*}
\end{proof}

\subsection{Covariance Thresholding is Fragile}\label{sec:fooling-covariance-thresholding}
In this section, we show how in model \ref{def:wishart-matrix-model} the Covariance Thresholding algorithm fails to output a good estimation of the vector $v_0$ in the presence of an adversarial distribution.  Specifically, we will show that the  algorithm fails for $k\geq \frac{\sqrt{n \log \frac{d}{k^2}}}{\Snorm{E}_{1\rightarrow 2}}$. This bound is significant in the sense that already for  $\Norm{E}_{1\rightarrow 2} = d^{o(1)}\sqrt{\frac{\beta n}{d}}$, the algorithm breaks.  We remark that a similar phenomenon can also be observed in the Wigner model, we omit this proof since it is simpler than in the Wishart model.

Recall that the central idea behind Covariance Thresholding is to  threshold entries of the empirical covariance matrix. The thresholding operation should remove noise while leaving the submatrix $\beta \Snorm{u_0}v_0 \transpose{v_0}$ untouched. The top eigenvector of $\eta_\tau \Paren{\transpose{Y}Y-n \Id}$ will then be close to the sparse vector. The key observation behind the adversary is that it is possible to plant a matrix $E$ with small norm $\Norm{E}_{1\rightarrow 2}$ such that the thresholded covariance matrix $\eta \Paren{\transpose{Y}Y-n \Id}$ has many large eigenvalues with eigenspace far from $v_0$. 

Consider the Covariance Thresholding algorithm:

\begin{mdframed}
	\begin{algorithm}[Standard Covariance Thresholding]
		\label[algorithm]{alg:covariance-thresholding}\mbox{}
		\begin{description}
			\item[Input:]
			Threshold $\tau$, sample matrix $Y=\sqrt \beta \cdot u_0v_0^T +W+E \in \R^{n \times d}$ where $v_0$ is $k$-sparse, $u_0$ and $W$ have i.i.d subgaussian entries of mean $0$ and variance $1$ and $E$ has  column norms bounded by $b$.
			\item[Estimate:]
			The sparse vector $v_0$.
			\item[Operation:]\mbox{}
			\begin{enumerate}
				\item 
				Compute the thresholded matrix $\eta_\tau \Paren{\transpose{Y}Y-n \Id}$.
				\item Output a top eigenvector $\hat{v}$ of $\eta_\tau \Paren{\transpose{Y}Y-n \Id}$. 
			\end{enumerate}
		\end{description}    
	\end{algorithm}
\end{mdframed}

The main result of the section is the Theorem below. Its significance is to be read under this perspective: it shows that there exists an adversary that can plant several (i.e. $\omega(\log d)$) large eigenvalues, as a consequence the top eigenvectors of  $\eta_\tau \Paren{\transpose{Y}Y-n \Id}$ will not be correlated with $v_0$.

\begin{theorem}\label{thm:fooling-covariance-thresholding-informal}
	Suppose that $k\leq \sqrt{d}$ and $\log^{10} d \le n \le d$.
	Let $Y$ be of the form \ref{def:wishart-matrix-model} for a flat vector $v_0$. 
	Let  $r \in [n]$ be such that $\omega\Paren{\log d} \le r \le d^{o(1)}$ and $\tau\in \R$ be such that $2\sqrt{n} \le \tau \le o\Paren{\sqrt{n\log d}}$ as $d\to \infty$.
	
	Then with probability at least $1-o(1)$ (as $d\to \infty$) 
	there exists an adversarial matrix $E$ with maximal column norm
	$\Norm{E}_{1\rightarrow 2}\le  d^{o(1)} \sqrt{\frac{\beta n}{d}}$ 
	and orthogonal  vectors $z^1,\ldots,z^r$ such that
	\begin{align*}
	\forall i \in [r], \qquad
	\frac{1}{\Snorm{z^i}}\cdot \transpose{\Paren{z^i}}\eta_\tau\Paren{\transpose{Y}Y-n \Id}z^i\ge
	\transpose{v_0}\eta_\tau\Paren{\transpose{Y}Y-n \Id} v_0
	\end{align*}
	and $\iprod{z^i,v_0}=0$.
\end{theorem}

The theorem shows that with these adversarial perturbations the first $r$ eigenvectors of the thresholded covariance matrix are uncorrelated with the sparse vector $v_0$. Notice that for $\beta \geq 1$ a perturbation resilient algorithm should succeed with perturbations bounded by $\sqrt{\frac{\beta n}{k}}$, that is, much larger (in absolute value) than the ones used to fool Covariance Thresholding. In particular, for $\beta = \Theta\Paren{ \frac{k}{\sqrt{n}}\sqrt{\log\frac{d}{k^2}}}$ \cref{thm:fooling-covariance-thresholding-informal} implies that already with perturbations satisfying 
$\Norm{E}_{1\rightarrow 2}\leq d^{-1/4 + o(1)} n^{1/4}$ the algorithm fails, while a perturbation resilient algorithm would succeed for $\Norm{E}_{1\rightarrow 2}\leq \tilde{O}\Paren{n^{1/4}}$.

Before showing the proof, we provide some intuition.

\paragraph{Algorithm Intuition}\label{par:ct-algorithm-intuition}
Let's ignore cross-terms for a moment and consider the Wishart model with no adversarial distribution. Then the centered empirical Covariance Matrix looks like
\begin{align*}
\transpose{Y}Y- n \Id \approx \transpose{W}W -n \Id + \beta n v_0 \transpose{v_0}.
\end{align*}
If we set the threshold $\tau = C\sqrt{n \log \frac{d}{k^2}}$ for some large enough constant $C>0$,
then $d^2\exp\Brac{-\Theta\Paren{\tau^2/n}} \approx k^4$ entries in 
$\Paren{\transpose{W}W- n\Id}$ will be larger than $\tau$.
\footnote{To see this, recall that in a $d\times d$ Gaussian matrix, with high probability there are at most $k^4$ entries larger than $\sqrt{\log \frac{d^2}{k^4}}$. While entries in $\transpose{W}W-n \Id$ are dependent, a similar bound will hold. } On the other hand, for $\beta \gtrsim \frac{k}{\sqrt{n}}\sqrt{\log\frac{d}{k^2}}$ as $\Abs{\Paren{\beta n v_0 \transpose{v_0}}_{\ij}}\geq \tau$ whenever $i,j \in \supp \Set{v_0}$, many entries of $\beta n v_0 \transpose{v_0}$ will survive the thresholding. This means that,
\begin{align*}
\eta_\tau \Paren{\transpose{Y}Y- n \Id} \approx \Paren{\transpose{W}W-n \Id}\brac{S}+ \beta n v_0 \transpose{v_0}
\end{align*}
where $S\subseteq [d]\times [d]$ has cardinality approximately $k^4$. If the entries were independent, since the fourth moment of each entry is not much larger than the second moment, standard spectral matrix bounds suggest
\[
\Norm{\Paren{\transpose{W}W-n \Id}\Brac{S}} \le O\Paren{\sigma\sqrt{d}}\,,
\]
where $\sigma \le \tau\exp\Brac{-{\frac{C\tau^2}{10n}}} \le \sqrt{n} \cdot \frac{k}{\sqrt{d}}$ is a standard deviation of each entry.
Hence we get
\[
\Norm{\Paren{\transpose{W}W-n \Id}\Brac{S}} \le O\Paren{k\sqrt{n}}\,,
\]
and
\[
\Norm{\beta n v_0 \transpose{v_0}}= \beta n.
\]
In conclusion, for $\beta \gtrsim\frac{k}{\sqrt{n}}\sqrt{\log \frac{d}{k^2}}$ the top eigenvector of $\eta_\tau \Paren{\transpose{Y}Y- n \Id}$ will be close to $v_0$.

The main technical difficulty here is that the entries of $\transpose{W}W$ are not independent. In \cite{DBLP:conf/nips/DeshpandeM14} the authors provide a method to bound the spectral norm of the thresholded matrix\footnote{Formally, in \cite{DBLP:conf/nips/DeshpandeM14} the authors provided a proof for a matrix obtained applying \textit{soft-thresholding}. As we will see these can easily be extended to the hard-thresholded matrix $\eta\Paren{\transpose{W}W-n \Id}$.}.

\paragraph{Adversarial Strategy}
Now we provide intuition on how to choose $E$ such that with constant probability there exists a vector $z$ orthogonal to $v_0$ for which
\begin{align*}
\frac{\transpose{z}}{\Norm{z}}\eta_\tau \Paren{\transpose{Y}Y- n \Id}\frac{z}{\Norm{z}}\gtrsim \transpose{v_0}\eta_\tau\Paren{\transpose{Y}Y- n \Id}v_0.
\end{align*}
Let $x\in \R^n$ be a randomly chosen unit vector orthogonal to $u_0$, let $z$ be a vector such that $\supp\Set{z}=[d]\setminus \supp \Set{v_0}$ and for $i \in \supp \Set{z}$, $z_i= \sigma _ib$ for some $b \in \R_+$ to be set later and $\sigma_i\sim \Set{\pm 1}$. We define the adversarial matrix as $E:=x\transpose{z}$, notice that $\Normi{E}\leq \tilde{O}\Paren{b/\sqrt{n}}$. For $i,j\in \supp \Set{z}$, consider the entry $\ij$ of the centered empirical covariance matrix $\Paren{\transpose{Y}Y-n \Id}$,
\begin{align*}
\Abs{\Paren{\transpose{Y}Y-n \Id}_\ij} = \Abs{\iprod{w_i,w_j}+\iprod{w_i,x} +\iprod{x,w_j}+z_iz_j} \lesssim \Abs{\iprod{w_i,w_j} + z_iz_j},
\end{align*}
by construction of $z$, the term $z_iz_j$ is symmetric and bounded by $b^2$. Hence for $b^2 = o(\sqrt{n})$, the thresholding of entry $\Paren{\transpose{Y}Y- n \Id}_{\ij}$ will depend almost only on the Gaussian contribution  $\Paren{\transpose{W}W -  n \Id}_{\ij}$. Let $S\subseteq [d]\times [d]$ be the set of non-zero entries in $\eta_\tau \Paren{\transpose{Y}Y-n \Id}$. By independence of $z$ and $W$, and since $S$ dependence of $z$ is very limited, we expect,  as in our previous discussion, $\Card{S}\gtrsim k^4$. Now consider the quadratic form
\begin{align*}
\frac{\transpose{z}}{\Norm{z}}\eta_\tau\Paren{\transpose{Y}Y-n \Id}\Brac{S}\frac{z}{\Norm{z}}\approx \frac{\transpose{z}}{\Norm{z}}\Paren{z \transpose{z}}\Brac{S}\frac{z}{\Norm{z}} + \frac{\transpose{z}}{\Norm{z}}\Paren{\transpose{W}W-n \Id}\Brac{S}\frac{z}{\Norm{z}}.
\end{align*}
As argued in the previous paragraph \ref{par:ct-algorithm-intuition}, $\frac{\transpose{z}}{\Norm{z}}\Paren{\transpose{W}W  - n \Id}\Brac{S}\frac{z}{\Norm{z}}\le O\Paren{k\sqrt{n}}$. On the other hand,
\begin{align*}
\frac{\transpose{z}}{\Norm{z}}\Paren{z \transpose{z}}\Brac{S}\frac{z}{\Norm{z}}= \frac{1}{\Snorm{z^2}}\underset{(i,j)\in S}{\sum } z_i^2z_j^2 = \frac{\card{S}}{\Snorm{z}} b^4 \gtrsim \frac{k^4}{d}b^2\gtrsim d^{1-o(1)}b^2.
\end{align*}
For the signal we instead have $\transpose{v_0}\eta_\tau\Paren{\transpose{Y}Y-n \Id}\Brac{S}v_0 \lesssim \beta n.$ It follows that setting $b\gtrsim \sqrt{\frac{\beta n}{d^{1-o(1)}}}$ the top eigenvector of $\eta \Paren{\transpose{Y}Y-n \Id}$ will not achieve constant correlation with $v_0$. Recall now that $n\ge d^{1-o(1)}$ and that $\Normi{E}\lesssim \tilde{O}(b/\sqrt{n})$. Hence for $\beta \approx \frac{k}{\sqrt{n}}\sqrt{\log\frac{d}{k^2}} \le n^{o(1)}$, 
adversarial perturbations are bounded by ${n^{o(1)}}/{\sqrt{n}}$ are enough to fool the algorithm.
\begin{remark}
	While this adversarial matrix is enough to break Covariance Thresholding it also allows an easy fix. Indeed, although the top eigenvector is now almost uncorrelated with $v_0$, the eigenspaces spanned by two largest eigenvectors contain a vector close to $v_0$ and a brute-force search over such space can be performed in polynomial time. The same approach however can be used to build an adversarial matrix $E$ such that there exist vectors $z^1,\ldots,z^r$  for which, with constant probability
	\begin{align*}
	i \in [r]\qquad \frac{\transpose{z^i}}{\Norm{z^i}}\eta_\tau \Paren{\transpose{Y}Y- n \Id}\frac{z^i}{\Norm{z^i}}\gtrsim \transpose{v_0}\eta_\tau\Paren{\transpose{Y}Y- n \Id}v_0.
	\end{align*}
	The idea is to chose  $x^1,\ldots,x^r$ to be orthonormal vectors orthogonal to $u_0$, and $z^1,\ldots,z^r$ with non-intersecting supports and the same structure as before. This latter choice of $E$ implies that the space containing eigenvectors associated with large eigenvalues has now dimension at least $\Omega(r)$.  For $r\geq \omega\Paren{\log d}$, brute-force search of a vector close to $v_0$ in this space requires super-polynomial time.
\end{remark}

\subsubsection{Proving covariance thresholding fragile}

Now we formally prove the theorem. First we define the adversarial matrix.

\begin{definition}[Adversarial matrix]\label{def:adversarial-matrix-covariance-thresholding} 
	For $b\ge1, r \in \N$, $W\sim N(0,1)^{n \times d}$, $u_0\sim N(0,\Id_n)$ and $v_0$ k-sparse, the adversarial matrix is built as follows. Let $x^1,\ldots,x^r\in \R^n$  be unit vectors that are independent of $W$ such that  for distinct $i,j \in [r]$, $\iprod{x^i,x^j}=0$. Partition the set $[d]\setminus \supp\Set{v_0}$ in sets $Z_1,\ldots,Z_r$ of cardinality $\frac{d-\Card{\supp\Set{v_0}}}{r}$. For each $i \in [r]$, let $z^i$ be the vector with support $Z_i$ such that:
	\begin{align*}
	\forall l \in Z_i, \qquad z^i_l = 
	\begin{cases}
	b &\text{ if } \iprod{w_l,x^i}\geq 0\\
	-b &\text{ otherwise.}
	\end{cases}
	\end{align*}
	Then
	\begin{align*}
	E:=\underset{i \in [r]}{\sum}x^i\transpose{z^i}.
	\end{align*}
\end{definition}

Notice that $\Norm{E}_{1\rightarrow 2} = b\sqrt{r}$.

Theorem \ref{thm:fooling-covariance-thresholding-informal} follows immediately combining Theorem \ref{thm:fooling-covariance-thresholding-adversarial-vector}, and Lemma \ref{lem:fooling-covariance-thresholding-sparse-vector}. 

\begin{theorem}\label{thm:fooling-covariance-thresholding-adversarial-vector} 
	Let $Y$ be of the form \ref{def:wishart-matrix-model} with $E$ constructed as in definition \ref{def:adversarial-matrix-covariance-thresholding} with $\omega(\log d)\le r\le d^{o(1)}$ and $b\le \sqrt[4]{n}$. 
	Assume that $d\ge n \ge \log^{10} d$ and that $k \le \sqrt{d}$.
	Let $2\sqrt{n} \le \tau \le o\Paren{\sqrt{n\log d}}$ as $d\to \infty$.
	Then with probability at least $1-2d^{-\Omega(1)}$ there exists a subset $R\subseteq [r]$ of size at least $\frac{r}{10}$ such that
	\begin{align*}
	\forall i \in R, \qquad
	&\frac{1}{\Snorm{z^i}}\cdot \transpose{\Paren{z^i}}\eta_\tau\Paren{\transpose{Y}Y-n \Id}z^i
	\geq
	b^2\cdot \frac{d^{1-o(1)}}{r}\,.
	\end{align*}
	
\end{theorem}

\begin{lemma}\label{lem:fooling-covariance-thresholding-sparse-vector}
	Suppose the conditions of \cref{thm:fooling-covariance-thresholding-adversarial-vector} are satisfied and that the entries of $v_0$ are from $\set{0,\pm 1/\sqrt{k}}$ and $n\ge \omega(\log d)$ as $d\to \infty$.
	Then with probability $1-O(d^{-10})$
	\[
	\Abs{\transpose{v_0}\eta_\tau (Y^TY - n\Id)v_0} \le O\Paren{k\sqrt{n\log d} + \beta n}\,.
	\]
\end{lemma}

\begin{proof}[of Lemma \ref{lem:fooling-covariance-thresholding-sparse-vector}]
	With probability $1-O(d^{-10})$  the entries of $\eta_\tau (Y^TY - n\Id)$  are bounded by 
	\[
	O\Paren{\frac{\beta n}{k} + \sqrt{n\log d} + \sqrt{\frac{\beta n\log d}{k}}} \le 
	O\Paren{\frac{\beta n}{k} + \sqrt{n\log d}}\,.  
	\]
	Since $v_0$ has at most $k$ nonzero entries, 
	\[
	\Abs{\transpose{v_0}\eta_\tau (Y^TY - n\Id)v_0} \le 
	k \cdot \normi{\eta_\tau (Y^TY - n\Id)}
	\le O\Paren{\beta n + k\sqrt{n\log d} }\,.  
	\]
\end{proof}

Tp prove Theorem \ref{thm:fooling-covariance-thresholding-adversarial-vector} we make use of intermediate steps  \ref{lemma:ct-lower-bound-large-entries}-\ref{lem:ct-bound-cross-terms-adversary}. Our plan is to show that many entries of $z^i \transpose{z^i}$ survive the thresholding due to the contribution of $\transpose{W}W-n \Id$. So, we start our analysis lower bounding the number of entries of $\transpose{W}W-n \Id$ that are above the threshold.

The following lemma shows that for each vector $z^i$, many entries in $\supp\Set{z^i}\times \supp\Set{z^i}$ will survive the thresholding.

\begin{lemma}\label{lemma:ct-lower-bound-large-entries}
	For any $b,r \in \R$ consider $Y$ sampled from model \ref{def:wishart-matrix-model} with $E$ as in \ref{def:adversarial-matrix-covariance-thresholding}. For some $10\le q \le d^{o(1)}$ let $\tau =\sqrt{n \log q}$. For $i \in [r]$ define the set 
	\begin{align*}
	S_i:= \Set{(j,l) \in \supp\Set{z^i}\times\supp\Set{z^i}
		\suchthat j\neq l, \Paren{\transpose{Y}Y-n \Id}_{jl}\geq  \tau} . 
	\end{align*}
	Then with probability at least $1 - \exp\Paren{d^{1-o(1)}}$,
	\begin{align*}
	\card{S_i} \ge \frac{d^2}{1000r^2q^{10}}\,.
	\end{align*}
	\begin{proof}
		Consider an off diagonal entry $jl$ of $\eta_\tau \Paren{\transpose{Y}Y-n\Id}$ such that $j,l \in \supp\Set{z^i}$ for some $i\in [r]$. Since with probability at least $1-2\exp\Brac{-\Omega(n^{0.2})} \ge 1-2d^{\Omega(1)}$, 
		$\iprod{w_j, x_l} \le n^{0.1}$, we get
		\begin{align*}
		\bbP \Paren{\Abs{\iprod{w_j,w_l}+\iprod{w_j,x}z^i_l+\iprod{w_l,x}z^i_j + z^i_jz^i_l} \geq \tau} 
		&\ge
		\bbP \Paren{\tfrac{1}{\sqrt{n}}\Abs{\iprod{w_j,w_l}}\geq 2\sqrt{\log q}} - d^{-\Omega(1)}
		\\&\ge
		\frac{1}{10q^{10}}\,.
		\end{align*} 
		For fixed $z^i$ and fixed row $j \in [d]$, the $\iprod{w_j, w_l}$ (for different $l \in \supp\set{z^i}$) are independent from each other. Since $r \le d^{o(1)}$, with probability $1-\exp\Brac{\frac{d}{100rq^{10}}} = 1 - \exp\Paren{d^{1-o(1)}}$ number of different $l$ such that  $\Paren{\transpose{Y}Y-n \Id}_{jl}\geq  \tau$ is at least $\frac{d}{1000rq^{10}}$. Hence if  with probability at least $1 - \exp\Paren{d^{1-o(1)}}$, for each $z^{i}$, $S_i \ge \frac{d^2}{1000r^2q^{10}}$.
	\end{proof}
\end{lemma}

The last ingredient needed for Theorem \ref{thm:fooling-covariance-thresholding-adversarial-vector} is a proof that the cross-terms in the quadratic form $\transpose{z^i}\eta_\tau\Paren{\transpose{Y}Y-n \Id}z^i$ do not remove the contribution of the adversarial vector.

\begin{lemma}\label{lem:ct-bound-cross-terms-adversary}
	Let $Y$ be sampled from model \ref{def:wishart-matrix-model} with $E$ as in Definition \ref{def:adversarial-matrix-covariance-thresholding}. 
	Let $S_i$ be as in Lemma \ref{lemma:ct-lower-bound-large-entries}
	Then with probability at least $\frac{1}{2}$,
	\begin{align*}
	\transpose{\Paren{z^i}}\Paren{\transpose{W}W + \transpose{W}x^i\transpose{z^i}+ z^i \transpose{x^i}W}\Brac{S_i}\;z^i\geq 0.
	\end{align*}
	\begin{proof}
		For simplicity of the notation we will refer to $x^i,z^i$ simply as $x,z$. 
		Opening up the sum,
		\begin{align*}
		\transpose{{z}}\Paren{\transpose{W}W + \transpose{W}x\transpose{z}+ z \transpose{x}W}\Brac{S_i}\;z
		&=
		2\underset{(j,l)\in S_i}{\sum}\iprod{w_j,w_l}z_jz_l+ \iprod{w_j,x}b^2z_j+ \iprod{w_l,x}b^2z_l
		\\&\geq 
		2\underset{(j,l)\in S_i}{\sum}\iprod{w_j,w_l}z_jz_l\\
		&\geq 
		2\underset{(j,l)\in S_i}{\sum}\iprod{w_j,\Paren{\Id-x\transpose{x}}w_l}z_jz_l,
		\end{align*}
		using the fact that by construction $\iprod{w_j,x}b^2z_j\geq 0,\iprod{w_l,x}b^2z_l\geq 0$.
		So it is enough to prove that
		\begin{align*}
		\bbP \Paren{\underset{j=1}{\overset{d'}{\sum}}\underset{(j,l)\in S_i}{\sum}
			\iprod{w_j,\Paren{\Id-x\transpose{x}}w_l}z_jz_l \geq 0}\geq \frac{1}{2}.
		\end{align*}
		
		Let 
		\[
		a_{jl} = \iprod{w_j,\Paren{\Id-x\transpose{x}}w_l}z_jz_l = 
		\Paren{\iprod{w_j, w_l} - \iprod{w_j,x}\iprod{w_l, x}}\cdot z_jz_l\,
		\] 
		and 
		\begin{align*}
		p_{jl} 
		&=
		\Paren{\iprod{w_j,x}z_l+ \iprod{w_l,x}z_j + \iprod{w_j, x}\iprod{w_i, x} + z_jz_l}\cdot z_jz_l 
		\\&=  
		\iprod{w_j,x}b^2z_j+ \iprod{w_l,x}b^2z_l + \iprod{w_j, x}\iprod{w_i, x}z_jz_l + b^4\,.
		\end{align*}
		Notice that $(j,l) \in S_i$ if and only if $j\neq l$ and $\Abs{a_{jl} + p_{jl}} \ge b^2\tau$. Also notice that $p_{jl} \ge 0$ and that with probability at least $1-2\exp\Brac{-\Omega(n^{0.2})}$, $p_{jl} < b^2\tau$.
		
		Since $\card{\supp\set{z}} = d'$, without loss of generality assume $\supp\set{z}=[d']$.
		For $q\in [d'-1]$ define
		\[
		T^*_q 
		:=
		\underset{j=q}{\overset{d'-1}{\sum}}\;\;
		\underset{\substack{j < l\leq d'\text{ s.t.}\\\Abs{a_{jl}}\geq b^2\tau}}{\sum}
		\iprod{\Paren{\Id-x\transpose{x}}w_j,\Paren{\Id-x\transpose{x}}w_l}z_jz_l 
		=
		\underset{j=q}{\overset{d'-1}{\sum}}\;\;
		\underset{\substack{j < l\leq d'\text{ s.t.}\\\Abs{a_{jl}}\geq b^2\tau}}{\sum}
		a_{jl}
		\]
		and
		\[
		T_q 
		:= 
		\underset{j=q}{\overset{d'-1}{\sum}}\;
		\underset{\substack{j < l\leq d'\text{ s.t.} \\ \Abs{a_{jl} + p_{jl}} \ge b^2\tau}}{\sum}
		\iprod{\Paren{\Id-x\transpose{x}}w_j,\Paren{\Id-x\transpose{x}}w_l}z_jz_l
		=
		\underset{j=q}{\overset{d'-1}{\sum}}\;
		\underset{\substack{j < l\leq d'\text{ s.t.} \\ \Abs{a_{jl} + p_{jl}} \ge b^2\tau}}{\sum}
		a_{jl}\,.
		\]
		Let $T_{d'} = T_{d'} = 0$. For $j\in[d'-1]$ consider
		\[
		T^*_j-T^*_{j+1}
		=
		\underset{\substack{j < l\leq d'\text{ s.t.}\\\Abs{a_{jl}}\geq b^2\tau}}{\sum}
		\iprod{\Paren{\Id-x\transpose{x}}w_j,\Paren{\Id-x\transpose{x}}w_l}z_jz_l\,. 
		\]
		$\Paren{\Id-x\transpose{x}}w_j$ is symmetric around zero and independent from all $z_j$ and all $w_l$ for $l > w_j$. Moreover, the sign of $\Paren{\Id-x\transpose{x}}w_j$ does not influence on the condition $\abs{a_{jl}}\geq b^2\tau$.
		It follows that the conditional disribution of $T^*_j$ given $z_j$, $z_l$, $w_l$ for $l > j$ is symmetric around $T^*_{j+1}$ and thus by induction $T^*_1$ is symmetric around zero. It remains to show that $\bbP \Paren{T_1\geq 0}\geq \bbP \Paren{T^*_1 \geq 0}$, which is true since if $T^*_1 \ge 0$, then $T^*_1 \geq T_1$. 
		Indeed, if $T^*_1 \ge 0$, then any $a_{jl} \ge 0$ such that $\abs{a_{jl}} \ge b^2\tau$ 
		satisfies $\abs{a_{jl} + p_{jl}} \ge b^2\tau$, 
		and any $a_{jl} < 0$ such that  $\abs{a_{jl} + p_{jl}} \ge b^2\tau$ 
		satisfies $\abs{a_{jl}} \ge b^2\tau$.
	\end{proof}
\end{lemma}

We are now ready to prove Theorem \ref{thm:fooling-covariance-thresholding-adversarial-vector}.
\begin{proof}[Proof of Theorem \ref{thm:fooling-covariance-thresholding-adversarial-vector}]
	Let $10 \le q \le d^{o(1)}$ so that $\tau = \sqrt{n\log q}$.
	By construction of $z^i$,
	\begin{align*}
	\transpose{z^i}\Paren{\transpose{W} u \transpose{v}}\Brac{S_i}z^i&=0\\
	\transpose{z^i}\Paren{\transpose{E} u \transpose{v}}\Brac{S_i}z^i&=0\\
	\transpose{z^i}\Paren{\transpose{E}E}\Brac{S_i}z^i&=\transpose{z^i}\Paren{z^i \transpose{z^i}}\Brac{S_i}z^i.
	\end{align*}
	
	With probability $1-d^{\Omega(1)}$ sum over diagonal entries is bounded by:
	\[
	\sum_{j} \Paren{z^i_j}^2 \Paren{\snorm{w_j} - n} \le O\Paren{b^2 d\sqrt{n\log d}} 
	\le b^2d^{1.5 + o(1)}\,.
	\]
	
	Notice that for different $i, m\in[r]$ the events $\transpose{\Paren{z^i}}\Paren{\transpose{W}W + \transpose{W}x^{i}\transpose{z^i}+ z^i \transpose{x^i}W}\Brac{S_i}\;z^i\ge 0$ and $\transpose{\Paren{z^m}}\Paren{\transpose{W}W + \transpose{W}x^m\transpose{z^m}+ z^m \transpose{x^{m}}W}\Brac{S_m}\;z^m \ge 0$ are independend.
	Hence, by Lemma \ref{lem:ct-bound-cross-terms-adversary} with probability at least $1-2^{-0.1r}$ for at least $r/10$ different $i\in[r]$,
	\begin{align*}
	\transpose{\Paren{z^i}}&\eta_\tau\Paren{\transpose{Y}Y-n \Id}\;z^i
	\\
	=&  
	\transpose{\Paren{z^i}}\Paren{z^i \transpose{z^i}}\Brac{S_i}\;z^i + 
	{\transpose{z^i}}\Paren{\transpose{W}W + \transpose{W}x\transpose{z^i}+ z^i 
		\transpose{x^i}W}\Brac{S_i}\;z^i
	\\\geq& 
	\transpose{\Paren{z^i}}\Paren{z^i \transpose{z^i}}\Brac{S_i}\;z^i\\
	=&{b^4}\Card{S_i}.
	\end{align*}
	By Lemma \ref{lemma:ct-lower-bound-large-entries} $\Card{S_i}\geq \frac{d^2}{1000r^2q^{10}}$. The theorem follows  observing that $\Snorm{z^i}\leq \frac{d b^2}{r}$.

\end{proof}

\section{Covariance Thresholding doesn't work with large signal and small sample size}\label{sec:covariance-thresholding-fails-small-n}
We show here a formal argument that proves Covariance Thresholding doesn't work for $n\leq k^2$.
The lower bound \ref{theorem:spike-general} shows that the assumption $k^2\ge d^{1-o(1)}$ is important, since if $k^2 \le d^{1-\Omega(1)}$, the conjecture \cref{con:low-degree-polynomials} implies that it is unlikely that there exists a polynomial time algorithm with asymptotically better guarantees than Diagonal Thresholding. So in this section we assume that $k^2\ge d^{1-o(1)}$ and that the thresholding parameter $\tau$ satisfies $\tau\le o(\sqrt{n\log d})$ (since otherwise Covariance Thresholding doesn't have asymptotically better guarantees than Diagonal Thresholding).

Notice that the assumptions $n \ge k^2$ and $k^2 \ge d^{1-o(1)}$ imply $n \ge d^{1-o(1)}$.

\begin{theorem}\label{theorem:small-sample-ct}
	Let $x^i$ be a unit vector in the direction of the $i$-th row of $W$. Assume that $k \le \sqrt{d}$, $d^{\Omega(1)} \le n \le d^{1 - \Omega(1)}$, $\beta \le d^{o(1)}\cdot\frac{k}{\sqrt{n}}$ and $\sqrt{n} \le \tau \le o\Paren{\sqrt{n\log d}}$ as $d\to \infty$. Also assume that the entries of $v_0$ are from $\set{0,\pm 1/\sqrt{k}}$.
	Then with probability $1-O(d^{-10})$
	\[
	\transpose{\paren{x^i}}\zeta_\tau (Y^TY - n\Id)x^i\ge d^{1-o(1)}\,.
	\]
\end{theorem}

\cref{theorem:small-sample-ct} immediately follows from \cref{lemma:small-sample-ct-main} and \cref{lemma:small-sample-ct-cross}.

\begin{lemma}\label{lemma:small-sample-ct-main}
	Let $S_0$ be a set of pairs $(j,l)\in[d]^2$ such that any $(j,l)\in S_0$ satisfies
	$j\notin\supp\set{v_0}$, $l \notin\supp\set{v_0}$, and $j\neq l$.
	Then with probability $1-O(d^{-20})$
	\[
	\transpose{\paren{x^i}}\zeta_\tau (Y^TY - n\Id)[S_0]x^i\ge d^{1-o(1)}\,.
	\]
\end{lemma}

\begin{proof}
	Without loss of generality assume that $i = 1$ and denote $x = x^1$. Let's denote the squared norm of the first row of $W$ by $s^2$.
	Notice that 
	$\iprod{w_i,w_j} = w_{1i}w_{1j} + z_{ij}$, where $z_{ij}$ is independent of $x$.  Hence
	\begin{align*}
	x^T \zeta_\tau (Y^TY - n\Id)[S_0] x 
	&= 
	\sum_{(ij)\in S_0} \ind{\abs{\iprod{w_i, w_j}} > \tau} \Paren{\iprod{w_i, w_j} 
		- \sign\Paren{\iprod{w_i, w_j}}\tau} x_ix_j
	\\&=
	\frac{1}{s^2}\sum_{(ij)\in S_0} \ind{\abs{ w_{1i}w_{1j} + z_{ij}} > \tau} 
	\Paren{w_{1i}w_{1j} + z_{ij} -  \sign\Paren{\iprod{w_i, w_j}}\tau} w_{1i}w_{1j}\,.
	\end{align*}
	Notice that with probability $1-O(d^{-30})$ for every survived $(i,j)$, $\sign(z_{ij}) =  \sign\Paren{\iprod{w_i, w_j}}$, since $\tau > \sqrt{n}$, while $w_{1i}w_{1j} < O(\log d)$ with  $1-O(d^{-30})$. Hence
	\begin{align*}
	x^T \zeta_\tau (Y^TY - n\Id)[S_0] x 
	&=
	\frac{1}{s^2}\sum_{(ij)\in S_0} \ind{\abs{ w_{1i}w_{1j} + z_{ij}} > \tau} 
	\Paren{w_{1i}w_{1j} + z_{(ij)\in S_0}	-  \sign\Paren{z_{ij}}\tau} w_{1i}w_{1j}
	\\&\ge
	- \frac{100\log d}{s^2} \sum_{(ij)\in S_0}
	\ind{\Abs{z_{ij} - \sign\Paren{z_{ij}}\tau} < 100\log d}
	\Abs{z_{ij} -  \sign\Paren{z_{ij}}\tau}
	\\& \quad
	+\frac{1}{s^2}
	\sum_{(ij)\in S_0}\ind{\abs{z_{ij}} \ge \tau + 100\log d} 
	\Paren{w_{1i}w_{1j} + z_{ij} - \sign\Paren{z_{ij}}\tau} w_{1i}w_{1j}.
	\end{align*}
	Notice that $\Pr\Brac{\Abs{z_{ij}- \sign\Paren{z_{ij}}\tau} < 100\log d} < \frac{1000\log d}{\sqrt{n}} =: p$.
	If we fix $j'$, then for different $i$, $z_{ij'}$ are independent, so the number of $z_{ij'}$ such that $\Abs{z_{ij'}- \sign\Paren{z_{ij'}}\tau} < 100\log d$ is bounded by $10pd$ with probability at least $1-2^{-pd}$. Hence by union bound with high probability $1-d\cdot 2^{-pd}$ there are at most $10pd^2$ such $z_{ij}$ for all $(i, j) \in S_0$. Notice that with probability at least $1-\exp\Brac{-\Omega(n)}$, $s^2 = \Theta(n)$. So with probability at least $1 - O(d^{-20})$ the contribution of the first term is $-\frac {O(d\log^2 d)}{\sqrt{n}}$.
	
	Now, by Bernstein inequality \ref{theorem:matrix-bernstein}, for any fixed $j$
	\[
	\Abs{\sum_{i}\ind{\abs{z_{ij}} \ge \tau + 100\log d} \Paren{z_{ij} - \sign\Paren{z_{ij}}\tau} w_{1i}w_{1j}} \le O\Paren{\sqrt{nd\log d}}
	\]
	with probability at least $1 - O(d^{-30})$.
	
	Furthermore, since $z_{ij}$, $w^{1i}$ and $w_{1j}$ are independent, for any $j$ such that $w_{1j}^2 \ge 1$
	\[
	\frac{1}{\norm{w^1}^2}
	\sum_{i}\ind{\abs{z_{ij}} \ge \tau + 100\log d} 
	w_{1i}^2w_{1j}^2 
	\ge 
	\frac{1}{\norm{w^1}^2}
	\sum_{i}\ind{\abs{z_{ij}} \ge 2\tau} 
	w_{1i}^2
	\ge
	\exp\Paren{-\frac{O\Paren{\tau^2}}{n}} \ge d^{-o(1)}
	\]
	with probability at least $1 - \exp\Paren{-d^{1-o(1)}}$. Hence
	\[
	\frac{1}{\norm{w^1}^2}
	\sum_{ij}\ind{\abs{z_{ij}} \ge \tau + 100\log d} 
	w_{1i}^2w_{1j}^2
	\ge
	d^{1-o(1)}\,,
	\]
	and therefore, with probability $1 - O(d^{-20})$
	\[
	x^T \zeta_\tau (Y^TY - n\Id)[S_0] x  \ge d^{1-o(1)} - d\cdot\frac {O(\log^2 d)}{\sqrt{n}} - O\Paren{\sqrt{nd\log d}} \ge d^{1-o(1)}\,.
	\]
	
\end{proof}

\begin{lemma}\label{lemma:small-sample-ct-cross}
	Let $S_1 = [d]^2\setminus S_0 $.
	Then with probability  $1-O(d^{-20})$,
	\[
	\Abs{\transpose{\paren{x^i}}\zeta_\tau (Y^TY - n\Id)[S_1]x^i} \le d^{1/2 + o(1)}\cdot\sqrt{n}\,.
	\]
\end{lemma}
\begin{proof}
	With probability $1-O(d^{-20})$ the entries of $\zeta_\tau (Y^TY - n\Id)[S_1]$  are bounded by 
	\[
	O\Paren{\frac{\beta n}{k} + \sqrt{n\log d} + \sqrt{\frac{\beta n\log d}{k}}} \le 
	d^{o(1)}\Paren{\sqrt{n} + \sqrt[4]{n}} \le d^{o(1)}\cdot \sqrt{n}\,.  
	\]
	
	Number of nonzero entries in $\zeta_\tau (Y^TY - n\Id)[S_1]$ is at most $d + k^2+ 2kd\le O\Paren{d^{3/2}}$. With probability $1-O(d^{-20})$  entries of $x^i$ are bounded by $O\Paren{\sqrt{\frac{\log d}{d}}}$. Hence with probability $1-O(d^{-20})$
	\[
	\Abs{\transpose{\paren{x^i}}\zeta_\tau (Y^TY - n\Id)[S_1]x^i} \le d^{o(1)}\sqrt{n} \cdot d^{3/2} \cdot\frac{1}{d} \le d^{1/2 + o(1)}\cdot\sqrt{n}\,.
	\]
\end{proof}

To conclude that under assumptions of \cref{theorem:small-sample-ct} 
Covariance Thresholding doesn't work, we need the following lemma: 

\begin{lemma}\label{lem:thresholded-matrix-with-signal}
	Assume that the entries of $v_0$ are from $\set{0,\pm 1/\sqrt{k}}$ and $n\ge \omega(\log d)$ as $d\to \infty$.
	Then with probability $1-O(d^{-10})$
	\[
	\Abs{\transpose{v_0}\zeta_\tau (Y^TY - n\Id)v_0} \le O\Paren{k\sqrt{n\log d} + \beta n}\,.
	\]	
\end{lemma}

The proof of Lemma  \ref{lem:thresholded-matrix-with-signal} is the same as the proof of Lemma  \ref{lem:fooling-covariance-thresholding-sparse-vector}.

\section{Statistical Lower bound for Recovery}\label{section:estimation-lower-bound}
\providecommand{\Vol}{\text{Vol}}
\providecommand{\ham}[1]{d_H\Paren{#1}}
\providecommand{\given}{\:\vert\:}

In this Section we provide an information theoretic lower bound for recovery of the sparse vector, we will show that for $\beta \ll \frac{k}{n}\log\frac{d}{k}$ no estimator can achieve correlation $0.9$.  We remark that the bound was previously known, we include it for completeness. Formally, we prove the following statement.

\begin{theorem}\label{theorem:estimation_lower_bound}
	Given an  $n$-by-$d$ matrix $Y$ of the form $Y=W+ \sqrt{\beta}u_0\transpose{v_0}$ for a $k$-sparse unit vector $v_0\in\R^d$, a standard Gaussian vector $u_0\sim N(0,\Id_n)$, and a Gaussian matrix $W\sim N(0,1)^{n\times d}$.
	For any estimator $\hat{v}:\R^{n\times d} \rightarrow \R^d$, there exists a $k$-sparse vector $v\in \{0,\pm 1/\sqrt{k}\}^d$ such that
	\begin{equation*}
	1-\E \iprod{\hat{v}(Y),v}^2 \geq 0.2 - \frac{5 n \beta}{k \log \frac{d}{k}}\,.
	\end{equation*}
\end{theorem}

Observe how the theorem compare with the distinguishing lower bound \ref{theorem:spike-information-theoretical}. While the top eigenvector of the covariance matrix can distinguish the Gaussian distribution from the planted distribution if $\beta \gtrsim \sqrt{\frac{d}{n}}$, Theorem \ref{theorem:estimation_lower_bound} shows that in \textit{any} settings  it is also required to have $\beta \gtrsim \frac{k}{n}\log \frac{d}{k}$ in order to obtain correlation $0.9$. In other words, for $n \lesssim \frac{k^2}{d}\log^2\frac{d}{k}$ simple polynomial time algorithms can distinguish, but it is information theoretically impossible to have an estimator achieving correlation $0.9$.

A standard technique to prove such result is bounding the minimax risk. This can be done observing that, given an appropriate well-separated set of candidate vectors, any estimator will erroneously guess which is the true planted vector with large enough probability. We  introduce some standard notions that will be used in the proof, the proof itself can be found at the end of the section. We follow closely  \cite{wainwright_2019}.

Consider the following notation. For vectors $v_1,v_2\in \R^d$ let $\ham{v_1,v_2}$ denote their Hamming distance. For $b\in \R^d, 0 < t \leq d$,  define $\cB(b,t,l):=\{y \in \{0,\pm t\}^d | \ham{y,b} \leq l \}$ and $\cS(b,t,l):=\{y \in \{0,\pm t\}^d| \ham{y,b} = l \}$. Given a random variable  $x$, denote with $\Sigma(x)$ its covariance and with $\cX$ its support.

\begin{definition}
	Let $\rho$ be a metric. A $\delta$-packing of a set $T$ with respect to a metric $\rho$ is a subset $\Paren{z_1,\ldots,z_M}\subset T$ such that $\rho\Paren{z_i-z_j}>\delta$ for all $i,j\in [M], i \neq j$. A $\delta$-covering of a set $T$ with respect to $\rho$ is a subset $\Paren{z_1,\ldots,z_M}\subset T$ such that $\forall z \in T$, $\exists j \in [M]$ with $\rho\Paren{z-z_j}<\delta$.
\end{definition}

The following Lemma lower bounds the size of the largest $\sqrt{\delta}$ packing.

\begin{lemma}\label{lemma:packing-hamming-set}
	Let $\delta \in (0,1)$. There exists a $\sqrt{\delta}$-packing of $\cS(0,1/\sqrt{k},k)$ with respect to $\norm{\cdot}_2$ of cardinality at least $\Paren{\frac{d}{k}}^{k\Paren{1-\delta}}\Paren{\frac{\delta}{e}}^{\delta k}$.
	\begin{proof}
		Since $\Vol\Paren{\cB(0,1/\sqrt{k},\delta k)}\leq \binom{d}{\delta k}2^k$, a $\delta k$-covering of $\cS(0,1/\sqrt{k},k)$ with respect to $\ham{\cdot}$ has cardinality at least $\binom{d}{k}2^k / \binom{d}{\delta k}2^k\geq \Paren{\frac{d}{k}}^{k\Paren{1-\delta}}\Paren{\frac{\delta}{e}}^{\delta k}$. This is a lower bound for the $\delta k$-packing number of $\cS(0,1/\sqrt{k},k)$. The lemma follows since for $y_1,y_2 \in \{0,\pm1/\sqrt{k} \}^d$, $\ham{y_1-y_2}\geq l $ implies $\Norm{y_1-y_2}_2\geq \sqrt{l/k}$.
	\end{proof}
\end{lemma}

The main tool used in the Lemma will be the well-known Fano's Inequality.

\begin{definition}
	Let $z,j$ be random variables. Then their mutual information is 
	\begin{equation*}
	I(z,j):= H(z)+H(j)-H(z,j)
	\end{equation*}
	where for a r.v. $x$ with density function $p(x)$, $H(x):=-\E \Brac{\log p(x)}$ is the Shannon Entropy.
\end{definition}

\begin{lemma}\label{lemma:mutual-information-bound}
	Let $z,j$ be random variables. Suppose $j$ is uniformly distributed over $[M]$ and that $z$ conditioned on $j=J$, for $j\in [M]$ has a Gaussian distribution with covariance $\Sigma(z | J)$. Then
	\begin{align*}
	I(z,j)\leq \frac{1}{2}\Paren{\log\det \Sigma(z)-\frac{1}{M}\underset{J \in [M]}{\sum}\log\det \Sigma(z |J)}
	\end{align*}
\end{lemma}

\begin{lemma}[Fano's Inequality]\label{lemma:fano-inequality}
	Let $\rho$ be a metric and $\Phi:[0,\infty)\rightarrow [0,\infty)$ an increasing function.
	Given a family of distributions $\cD$ and a functional $v:\cD \rightarrow \Omega$, consider a $M$-ary hypothesis testing problem between distributions $\{D_1,\ldots,D_M \}\subseteq \cD$ where $x$ is uniformly distributed over $[M]$ and $(z| x=i)\sim D_i$. If for all $i,j \in[M]$ with $i \neq j$, $\rho(v(D_i),v(D_j))>2\delta$, then for any estimator $\hat{v}:\cZ \rightarrow \Omega$:
	\begin{align*}
	\underset{i \in [M]}{\sup}\E_{D_i}\Brac{\Phi\Paren{\rho(\hat{v},v(D_i))}}\geq \Phi(\delta)\Paren{1-\frac{I(z,x)+\log2}{\log M}}
	\end{align*}
	The \textit{minimax risk} is defined as $\mathfrak{m} \Paren{v(\cD),\Phi \circ \rho}=\underset{\hat{v}}{\inf}\underset{i \in [M]}{\sup}\E_{D_i}\Brac{\Phi\Paren{\rho(\hat{v},v(D_i))}}$.
\end{lemma}

We are now ready to prove the bound, we provide a slightly more general version which immediately implies Theorem \ref{theorem:estimation_lower_bound}.

\begin{theorem}
	For any estimator $\hat{v}:\R^{n\times d} \rightarrow \R^d$ and $\delta \in \Paren{0,\frac{1}{4}}$, there exists a $k$-sparse vector $v\in \{0,\pm 1/\sqrt{k}\}^d$ such that
	\begin{equation*}
	\E\Snorm{\hat{v}(Y)-v}_2\geq \delta\Paren{1 - \frac{n\beta + 2\log 2}{k\Paren{\Paren{1-4\delta}\log \frac{d}{k} +4\delta \log\frac{4\delta}{e} }}}
	\end{equation*}
	where  $Y=\sqrt{\beta}u\transpose{v}+W$ for $W\sim N(0,1)^{n\times d}$ and $u\sim N(0,\Id_n)$.
	\begin{proof}
		By Lemma \ref{lemma:packing-hamming-set} there exists a $2\sqrt{\delta}$-packing $\{\Delta_1,\ldots,\Delta_M \}$ of $\cS(0,1/\sqrt{k},k)$ with size $M\geq \Paren{\frac{d}{k}}^{k\Paren{1-4\delta}}\Paren{\frac{4\delta}{e}}^{4\delta k}$. Denote by $\cU\subseteq \R^{d \times d}$ the set of orthonormal matrices corresponding to a permutation of the columns along with flip of signs. Notice that for $U \in \cU$,  $\norm{\Delta_i-\Delta_j}_2=\norm{U\Delta_i-U\Delta_j}_2$ and $U\Delta_i\in \cS(0,1/\sqrt{k},k)$, define the corresponding family of vectors $v_j(U)=U\Delta_j$ for $j \in [M]$. For a fixed $U\in \cU$, let $Y(U)\in \R^{n \times d}$ be the random variable generated picking $j$ uniformly at random from $[M]$ and then sampling $Y(U)=W + \sqrt{\beta}u\transpose{ v_j(U)}$.  For $n=1$, we denote it by $y(U)$. 
		By Lemma \ref{lemma:mutual-information-bound},
		\begin{align*}
		\E_U\Brac{I(y(U);J)}&\leq \frac{1}{2}\Brac{\E_U \log \det \Sigma\Paren{y(U)} -\frac{1}{M}\underset{J \in \Brac{M}}{\sum}\E_U\log\det\Sigma(y(U)|J) }\\
		&=  \frac{1}{2}\Brac{\E_U \log \det \Sigma\Paren{y(U)} -\frac{1}{M \Card{\cU}}\underset{J \in \Brac{M},U\in \cU}{\sum}\log\det\Sigma(y(U)|J) }\\
		&= \frac{1}{2}\Brac{\E_U \log \det \Sigma\Paren{y(U)} -\log(1+\beta)}\\
		&\leq \frac{1}{2}\Brac{ \log \det \E_U\Sigma\Paren{y(U)} -\log(1+\beta)}\\
		&= \frac{1}{2}\Brac{d\log\Paren{1+\frac{\beta}{d}}-\log \Paren{1+\beta}}\\
		&\leq \frac{\beta}{2}
		\end{align*}
		using concavity of the log-determinant and the matrix-determinant Lemma.
		Applying Fano's Inequality \ref{lemma:fano-inequality}, for any estimator $\hat{v}: \R^{n\times d} \rightarrow \R^d$
		\begin{align*}
		\underset{j \in [M]}{\sup}\E_{D_j} \Brac{\Snorm{\hat{v}-v_j}_2}&\geq \delta \Paren{ 1 - \frac{\E_U\Brac{I(Y(U);J)}+\log 2}{\log M}}\\
		&\geq\delta\Paren{ 1 - \frac{n\E_U\Brac{I(y(U);J)}+\log 2}{\log M}}\\
		&\geq \delta\Paren{1 -\frac{n\beta+2\log 2}{2\log M}}\\
		&\geq \delta\Paren{1 - \frac{n\beta + 2\log 2}{k\Paren{\Paren{1-4\delta}\log \frac{d}{k} +4\delta \log\frac{4\delta}{e} }}}
		\end{align*}
		where we used the fact that by independence of the rows of $Y(U)$, $I(Y(U);J)\leq n I(y(U);J)$.
	\end{proof}
\end{theorem}

\section{Existence of the Adversarial Distribution of Model \ref{problem:almost_gaussian_model_special_case}}
Let $\R[x]_{\le s}$ be the space of one variable polynomials of degree at most $s$. 
To construct the desired distribution we will need the following theorem.
\begin{theorem}[Theorem 1.26 in \cite{konrad}]\label{Tchalaloff}
	Suppose that $m_1,\ldots,m_s \in \R$ and $K \subseteq \R$ is compact. Consider a linear functional $\mathcal{L} : \R[x]_{\le s} \to \R$ such that $\mathcal{L}(1) = 1$ and
	\[
	\mathcal{L}(x^r) = m_r, \quad 1\le r \le s\,.
	\]
	If $\mathcal{L}(p) \ge 0$ for every $p \in \R[x]_{\le s}$ that is nonnegative on $K$, then there exists a finitely supported probability distribution $\eta$ such that $\supp (\eta) \subseteq K$ and
	$\E_{x\sim\eta} x^r = m_r$ for $1 \le r \le s\,.$
\end{theorem}
Let's take the maximal even number $s$ such that $\delta\lambda^s \le 2^{-10s}$.
We will show that there exists a distribution with compact support such that with probability $\delta$ it takes values $\pm \lambda$ and its first $s$ moments coincide with the first $s$ Gaussian moments. Such a distribution is a mixture $\eta = (1-\delta) \eta_0 + \delta \eta_1$, 
where $\eta_1$ takes values $\pm \lambda$ with probability $\frac{1}{2}$ each, and $\eta_0$ has particular moments up to $s$.
\begin{proposition}\label{proposition:existence_of_distribution}
	Suppose that $s \ge 2$ is even, $0<\delta<1$, $\lambda \ge 2$ and $\delta \lambda^s \le 2^{-10s}$. Then there exists a finitely supported probability distribution $\eta_0$ such that  $\supp(\eta_0) \subseteq [-10\sqrt{s\ln s},10\sqrt{s\ln s}]$ 
	and $\E_{x\sim\eta_0} x^r = M_r$ , where
	\[
	M_r =
	\begin{cases}
	0, & \text{if $r$ is odd}, \\
	\frac{1}{1-\delta}\big((r-1)!! - \delta \lambda^r\big), & \text{if $0\le r \le s$ and $r$ is even}. 
	\end{cases}
	\]
\end{proposition}
\begin{proof}
	Consider a linear functional $\mathcal{L} : \R[x]_{\le s} \to \R$ such that $\mathcal{L}(1)=1$ and $\mathcal{L}(x^r) = M_r$ for $1 \le r \le s$. We need to show that $\mathcal{L}(p) \ge 0$ for every polynomial $p \in \R[x]_{\le s}$ that is nonnegative on $[-10\sqrt{s\ln s},10\sqrt{s\ln s}]$. 
	Notice that for any polynomial $p\in\R[x]_{\le s}$
	\[
	(1-\delta)\cdot\mathcal{L}(p) = \E_{x\sim \mathcal{N}(0,1)}p(x) - \frac{\delta}{2}\big(p(\lambda) + p(-\lambda)\big)\,.
	\]
	Consider an arbitrary polynomial $p(x) = \sum_{r=0}^s p_r x^r$ that is nonnegative on $ [-10\sqrt{s\ln s},10\sqrt{s\ln s}]$. 
	If $p = 0$, then obviously $\mathcal{L}(p) = 0$. 
	So we can assume that $p \neq 0$ and without loss of generality 
	$\max\limits_{0\le r \le s}\{|p_r|\} = 1$.
	Since $p$ is nonnegative on $ [-10\sqrt{s\ln s},10\sqrt{s\ln s}]$,
	\[
	\E_{x\sim \mathcal{N}(0,1)}p(x) \ge \frac{1}{\sqrt{2\pi}}\int\limits_{-1}^{1} p(x) e^{-x^2/2} + 
	\frac{1}{\sqrt{2\pi}}\int\limits_{|x| > 10\sqrt{s\ln s}}^{} p(x) e^{-x^2/2}dx\,.
	\]
	The second integral can be bounded as follows
	\[
	\bigg|\int\limits_{|x| > 10\sqrt{s\ln s}}^{} p(x) e^{-x^2/2}\bigg| \le 
	\sum_{r=0}^s |p_r| \int\limits_{|x| > 10\sqrt{s\ln s}}^{}  \abs{x}^r e^{-x^2/2}dx \le
	(s+1)\int\limits_{|x| > 10\sqrt{s\ln s}}^{} x^s e^{-x^2/2}dx\,.
	\]
	Notice that since the function $s\ln x + 10s - 0.4x^2$ is monotone for $|x| \ge 10\sqrt{s\ln s}$,
	\[
	x^s e^{-x^2/2}\le e^{-10s - x^2/10}
	\]
	for all $x$ such that $|x| \ge 10\sqrt{s\ln s}$.
	Hence
	\[
	\bigg|\int\limits_{|x| > 10\sqrt{s\ln s}} p(x) e^{-x^2/2}dx\bigg| \le 
	(s+1)\cdot e^{-10s} \int\limits_{|x| > 10\sqrt{s\ln s}}  e^{-x^2/10} dx \le 
\sqrt{10}(s+1)\cdot e^{-10s} \le e^{-8s}\,.
	\]
	Let's bound $\int_{-1}^{1} p(x) \exp\Big(-\frac{x^2}{2}\Big)dx$. 
	Since $p(x)$ is nonnegative on $[-1,1]$,
	\[
	\int\limits_{-1}^{1} p(x)e^{-x^2/2}dx \ge 
	\int\limits_{-1}^{1} \frac{p^2(x)}{\max\limits_{|x|\le1} p(x)}e^{-x^2/2} dx \ge 
	\frac{e^{-1/2}}{\sum_{r=0}^s{|p_r|}} \int\limits_{-1}^{1} p^2(x) dx \ge 
	\frac{1}{2(s+1)} \int\limits_{-1}^{1} p^2(x) dx\,.
	\]
	To bound $\int_{-1}^{1} p^2(x) dx$ we can use Legendre polynomials (see for example \cite{legendre}). The degree $j$ Legendre polynomial is
	\[
	L_j(x) = \sum_{r=0}^j L_{j,r} x^r = 
	\sum_{r=0}^j \sqrt{\frac{2j + 1}{2}}\cdot2^j \binom{j}{r} \binom{\frac{j+r-1}{2}}{j} x^r\,.
	\]
	They form an orthonormal system on $[-1,1]$ with respect to the unit weight. Hence there exist coefficients $c_0,\ldots, c_s$ such that $p(x) = \sum_{j=0}^s c_s L_j(x)$ and
	\[
	\int\limits_{-1}^{1} p^2(x) dx = \sum_{j=0}^s c_j^2\,.
	\]
	Recall that by assumption $\max\limits_{0\le r \le s}\{|p_r|\} = 1$, so there exists some $r$ such that $|p_r| = 1$. Thus
	\[
	1 = |p_r| = \Big|\sum_{j=r}^s c_j L_{j,r}\Big| \le \sum_{j=r}^s |c_j| |L_{j,r}| \le 
	\max\limits_{r\le j \le s}{|L_{j,r}|} \sqrt{(s+1)} \sqrt{ \tsum_{j=0}^s c_j^2}\;.
	\]
	Notice that $|L_{j,r}| \le \sqrt{s+1}\cdot2^{2s}$ for $0\le r \le  j \le s$. Hence we get a bound
	\[
	\int\limits_{-1}^{1} p(x) e^{-x^2/2}dx \ge \frac{1}{3s}  \sum_{j=0}^s c_j^2 \ge  
	\frac{1}{2(s+1)^3}  2^{-4s} \ge 2^{-7s}\,,
	\]
	and
	\[
	\E_{x\sim \mathcal{N}(0,1)}p(x) \ge \frac{1}{\sqrt{2\pi}}2^{-7s}  - e^{-8s} \ge 2^{-8s}\,.
	\]
	Notice that
	\[
	\frac{\delta}{2}\big(p(\lambda) + p(-\lambda)\big) \le \delta \sum_{r = 0}^s |\lambda|^r 
	\le 2\delta \lambda^s \le 2^{-9s}\,.
	\]
	Hence finally we get
	\[
	(1-\delta)\cdot\mathcal{L}(p) = 
	\E_{x\sim \mathcal{N}(0,1)}p(x) - \frac{\delta}{2}\big(p(\lambda) + p(-\lambda)\big)  
	\ge 2^{-8s} - 2^{-9s}  > 0\,.
	\]
	Therefore by Theorem \ref{Tchalaloff} there exists a finitely supported probability distribution $\eta_0$ with moments 
	$M_1,\ldots, M_s$ such that $\supp (\eta_0) \subseteq [-10\sqrt{s\ln s}, 10\sqrt{s\ln s}]$. 
\end{proof}
We can assume that $\eta_0$ is symmetric (since if $z\sim \eta_0$ and  $w\sim N(0,1)$ are independend, $zw/\abs{w}$ is symmetrically distributed and has the same first $s$ moments as $z$).
Thus the mixture distribution 
$\eta = (1-\delta)\eta_0 + \delta\eta_1$ (where $\eta_1$ takes values $\pm \lambda$ with probability $\frac{1}{2}$ each) is symmetric and has Gaussian moments up to $s+1$:
\[
\E_{x\sim\eta} x^r = \E_{x\sim\mathcal{N}(0,1)} x^r\,, \quad \text{if $0 \le r \le s + 1$}\,,
\]
and its higher moments satisfy
\[
\delta\lambda^r \le \E_{x\sim\eta} x^r \le \delta\lambda^r + (10s)^r\,, \quad \text{if $r > s$ is even.}
\]

\section{Matrix concentration bounds}
In this section, we use standard tools to establish some matrix concentration inequalities that are essential to our main results. 

Our key tools will be the following general result by Rudelson showing convergence of empirical covariances of random variables.
\begin{fact}[Theorem 1, \cite{RUDELSON199960}] \label{fact:cov-estimation} 
Let $Y$ be a random vector in the isotropic position. Let $Y_1, Y_2, \ldots,Y_q$ be $q$ independent copies of  $Y$. Then, for some absolute constant $C> 0$,

\[
\E \Norm{ \frac{1}{q} \sum_{i = 1}^q Y_i Y_i^{\top}- I} \leq C\frac{ \sqrt{\log q} }{\sqrt{q}} \cdot \E (\norm{Y}^{\log q})^{1/\log q}\mper
\]
\end{fact}

We will also use the following simple computation of variances of low-degree polynomials of product subgaussian random vectors.

\begin{lemma}[Variance of Polynomials of Independent Subgaussians] \label{lem:variances-of-polynomials}
Let $Y$ be a product random variable on $\R^n$ with coordinates of mean $0$, variance $1$ satisfying $\E \iprod{Y,u}^{2t} \leq C^t(2t)^t$ for every unit vector $u$ for some absolute constant $C >0$. Let $p = \sum_{S:|S| \leq k} p_S y_S$  be a polynomial in $y \in \R^n$ of degree $k$ where the sum ranges of multisets $S \subseteq [n]$ of size at most $k$. Then, $\sum_{S:|S|\leq k} p_S^2 \leq \E p^2(Y) \leq C^t(2t)^t \sum_{S:|S|\leq k} p_S^2$.
\end{lemma}
\begin{proof}
For any polynomial $p$, we write $\norm{p}_2^2$ to denote the sum of squares of its coefficients in the monomial basis.
For any multilinear polynomial $p$, observe that $\E p^2 = \norm{p}_2^2$. 
For a non-multilinear $p$, we write $p = \sum_{S: |S| \leq k/2} y_S^2 q_S$ such that $q_S$ is a multilinear polynomial of degree at most $k-2|S|$. Observe that $\norm{p}_2^2 = \sum_{S} \norm{q_S}_2^2$. Further, $\E y_S^2 y_{S'}^2 q_S q_{S'} = 0$ whenever $S \neq S'$. Now, $\E p^2 = \sum_{S: |S| \leq k/2} \E y_S^2 q_S^2$. Since $\E y_S^2 \geq 1$ for any $S$, $\E y_S^2 q_S^2 \geq \norm{q_S}_2^2$. Thus, $\E p^2 \geq \sum_{S} \norm{q_S}_2^2 = \norm{p}_2^2$. On the other hand, $\E y_S^2 q_S^2 \leq \norm{q_S}_2^2 \cdot \max_{|S| \leq k} \E y_S^2 \leq C^k(2k)^k$. 
\end{proof}

\begin{lemma} \label{lem:matrix-concentration}
Let $Y$ be a random vector in $\R^n$ with independent coordinates of mean $0$ and variance $1$ satisfying $\E \iprod{Y,u}^{2t} \leq C^t(2t)^t$ for some absolute constant $C> 0$. 
Then, with probability at least $0.99$ over the draw of $Y_1, Y_2,\ldots, Y_d$ i.i.d. copies of $Y$, 
\[
 \Norm{\frac{1}{d} \sum_i \paren{Y_i^{\otimes t}} \paren{Y_i^{\otimes t}}^{\top}- \E_{Y \sim D} \Paren{Y^{\otimes t}} \Paren{Y^{\otimes t}}^{\top}}  \leq 
 \frac{n^{t/2} \log^{(t+1)/2}{(n)} \Paren{C't}^{t}}{\sqrt{d}}\mcom
\]
for some absolute constant $C' >0$.
\end{lemma}

\begin{proof}
Let $M = \E \paren{Y^{\otimes t}} \paren{Y^{\otimes t}}^{\top}$. 
Then, quadratic forms $\iprod{u,Mu}$ is the variance of polynomial $p = \iprod{u,Y^{\otimes t}}$ of degree at most $t$. 
Thus, using Lemma~\ref{lem:variances-of-polynomials}, we have that $\norm{u}_2^2 \leq \iprod{u,Mu} \leq \norm{u}_2^2 C^t (2t)^t$.
Thus, all eigenvalues of $M$ are between $1$ and $C^t (2t)^t$. 

We will now apply Fact~\ref{fact:cov-estimation} to the isotropic random vectors $M^{-1/2} Z_i$ for $Z_i = Y_i^{\otimes t}$ for $1 \leq i \leq d$. Then, we obtain:

\[
\E \Norm{\dyad{M^{-1/2} Z}-I} \leq C\frac{\sqrt{\log d}}{\sqrt{d}} \Paren{\E \Norm{M^{-1/2}Z}^{\log d}_2}^{1/\log d} \mper
\]

To finish, we compute $\E \Norm{M^{-1/2} Z}^{\log d} \leq \Norm{M^{-1/2}}^{\log d} \E \Norm{Z}^{\log d}$. 
Next, $\E \Norm{Z}^{\log d} = \E \Norm{Y}^{t \log d} \leq n^{(t/2) \log d} C^{(t/2) \log d} ((t/2)\log d)^{(t/2) \log d}$. 
Using $\Norm{M^{-1/2}} \leq 1$, 
we obtain: $\Paren{\E \Norm{M^{-1/2} Z}^{\log d}}^{1/\log d} \leq n^{t/2} C^{t/2} (t \log d)^{t/2}$. 

Thus, for  using $n\ge \log d$
and Fact~\ref{fact:cov-estimation}, $\E \Norm{\dyad{M^{-1/2} Z}-I} \leq  \frac{n^{t/2} \log^{(t+1)/2}{(n)} \Paren{10Ct}^{t}}{\sqrt{d}}$. Applying Markov's inequality completes the proof.

\end{proof}

We also state here some standard concentration bounds used in the proofs.

\begin{fact}\cite{laurent2000}\label{fact:chi-squared-tail-bounds}Let $X\sim \chi^2_m$,  $x>0$, then
	\begin{align*}
	\bbP \Paren{X-m \geq 2x+2\sqrt{mx}} &\leq e^{-x}\\
	\bbP \Paren{m-X\geq x}&\leq e^{-\frac{x^2}{4m}}
	\end{align*}
\end{fact}

\begin{fact}\cite{wainwright_2019}\label{fact:epsilon-net-sphere}
	Let $0 < \varepsilon < 1$. The $n-1$-dimensional Euclidean sphere has an $\varepsilon$-net of size $\Paren{\frac{3}{\varepsilon}}^n$. That is,
	there exists a set $N_\varepsilon$ 
	of unit vectors in $\R^n$ of size at most $\Paren{\frac{3}{\varepsilon}}^n$ such that for any unit vector $u\in \R^n$ there exists some $v\in N_\varepsilon$ such that $\norm{v-u} \le \varepsilon$.
\end{fact}

\begin{theorem}\cite{wainwright_2019}\label{thm:bound-covariance-gaussian}
	Let $W\sim N(0,1)^{n \times d}$. Then with probability $1-\exp\Paren{-t/2}$,
	\[
	\Norm{W}\le \sqrt{n}  + \sqrt{d} + \sqrt{t}\,
	\]
	and
	\[
	\Norm{\transpose{W}W-n \Id} \le d + 2\sqrt{dn} + t + 4\sqrt{t(n+d)} \,.
	\]
\end{theorem}

\begin{theorem}[Matrix Bernstein \cite{journals/focm/Tropp12}] \label{theorem:matrix-bernstein}Consider a finite sequence $\Set{Z_k}$ of independent, random, self-adjoint matrices in $\R^{d_1\times d_2}$. Assume that each random matrix satisfies
	\begin{align*}
	\E Z_k = 0 \text{ and } \Norm{Z_k}\leq R \text{ almost surely.}
	\end{align*}
	Define
	\begin{align*}
	\sigma^2 := \max \Set{\Norm{\sum_k \E Z_k \transpose{Z_k}}, \Norm{\sum_k \E \transpose{Z_k}Z_k} }.
	\end{align*} Then, or all $t\geq 0$,
	\begin{align*}
	\bbP \Paren{\Norm{\sum_k Z_k}\geq t}\leq (d_1+d_2)\exp \Set{\frac{-t^2/2}{\sigma^2+Rt/3}}.
	\end{align*}
\end{theorem}

\begin{theorem}[Matrix Hoeffding \cite{journals/focm/Tropp12}] \label{theorem:matrix-hoeffding}Consider a finite sequence $\Set{Z_k}$ of independent, random, self-adjoint matrices in $\R^{d \times d}$. Assume that each random matrix satisfies
	\begin{align*}
	\E Z_k = 0 \text{ and } Z^2_k \preceq  A^2_k \text{ almost surely.}
	\end{align*}
	Then, for all $t\geq 0$,
	\begin{align*}
		\bbP \Paren{\Norm{\sum_k Z_k}\geq t}\leq d\exp\Set{-\frac{t^2}{8\sigma^2}}
	\end{align*}
	where $\sigma^2:= \Norm{\sum_k A^2_k}$.
\end{theorem}

\begin{theorem}[$k$-sparse norm of a Gaussian matrix]\label{theorem:k-sparse-norm-gaussian}
	Let $W\sim N(0,1)^{n\times d}$ be a Gaussian matrix. Let $1\le k \le d$. Then with probability at least $1-\Paren{\frac{k}{ed}}^k$
	\[
	  \max_{\substack{u\in\R^n\\ \norm{u}=1}}\;\;
	  \max_{\substack{\text{$k$-sparse }v\in\R^d\\ \norm{v}=1}} \transpose{u}Wv \le
	  \sqrt{n} + 3\sqrt{k\ln\Paren{\frac{ed}{k}}}
	  \,.
	\]
	\begin{proof}
		Let $v$ be some $k$-sparse unit vector that maximizes the value, and let $S(v)$ be the set of nonzero coordinates of $v$.
		Consider some fixed (independend of $W$) unit $k$-sparse vector $x\in \R^d$ and the set $S(x)$ of nonzero coordinates of $x$. If we remove from $W$ all the rows with indices not from $S(x)$, we get an
		$n\times k$ Gaussian matrix $W_{S(x)}$. By \cref{thm:bound-covariance-gaussian} norm of this matrix is bounded by $\sqrt{n}+\sqrt{k} + \sqrt{t}$
		with probability at least $\exp\Paren{-t/2}$. Number of all subsets $S\subseteq [d]$ of size $k$ is $\binom{d}{k}$.
		By the union bound, the probability that the
		norm of $W_{S(v)}$ is greater than $\sqrt{n}+\sqrt{k} + \sqrt{t}$ is at most
		\[
		\binom{d}{k}\cdot \exp\Paren{-t/2} \le \exp\Paren{k\log_2\Paren{ed/k} - t/2}\,.
		\]
		Taking $t = 4{k\ln\Paren{ed/k}}$, we get the desired bound.
	\end{proof}
\end{theorem}

\begin{lemma}\label{lemma:large-coordinates-gauusian-vector}
	Let $w\sim N(0,1)^d$ be a Gaussian vector and let $1\le k \le d$. Let $S_k$ be the set of $k$ largest coordinates of $w$. Then with probability $1-\Paren{\frac{k}{ed}}^k$,
	$\sum_{i\in S_k} w_i^2 \le 10k\ln\Paren{ed/k}$.
\end{lemma}
\begin{proof}
	Let $S$ be any fixed subset of $[d]$ of size $k$. Then $w$ restricted on $S$ is a $k$-dimensional Gaussian vector and by \cref{fact:chi-squared-tail-bounds}, 
	$\Pr\Paren{\sum_{i\in S} w_i^2 \ge 2x + 2\sqrt{kx}}\le e^{-x}$. By a union bound over all $\binom{d}{k}$ subsets of $[d]$ of size $k$, we get
	\[
	\Pr\Paren{\sum_{i\in S} w_i^2 \ge k + 2x + 2\sqrt{kx}} \le e^{-x + k\log_2\Paren{ed/k}}\,.
	\]
	Taking $x = 4k\ln(ed/k)$ we get the desired bound.
\end{proof}

\begin{lemma}\label{lemma:number-of-large-coordinates-in-row-span-of-gaussian}
	For large enough $n$ and $d$ such that $n \le d$, let $W\sim N(0,1)^{n\times d}$ be a Gaussian matrix and let $u\in\R^n$ be an arbitrary
	unit vector (which can possibly depend on $W$). For any $t \ge 0$ let
	$S_t = \set{i\in[d]\;|\;\abs{(\transpose{u}W)_i} \ge t}$. Also let $B \ge 1$.
	
	Then, for any $t \ge 3\sqrt{B\ln d}$, $\card{S_t}\le n/B$  with probability at least $1-2\exp\Paren{-n}$.
	\begin{proof}
		Let $t \ge 3\sqrt{B\ln d}$. For any fixed (independend of $W$) unit vector  $x\in \R^n$, $(\transpose{x}W)_i$ are i.i.d. standard Gaussian variables. For large enough $d$, 
		\[
		\Pr\Brac{\abs{(\transpose{x}W)_i} \ge t-1} \le \exp{\Paren{-t^2/3}}\,.
		\]
		Hence the probability that there are $3\le k \le d$ coordinates that are larger than $t-1$ is at most 
		\[
		\binom{d}{k}  \exp{\Paren{-k\cdot t^2/3}} \le 
		\exp\Brac{k\Paren{1 + \ln{\frac{d}{k}}- t^2/3}}\le 		
		\exp\Brac{k\Paren{\ln d- t^2/3}} \le
		\exp\Paren{-\frac{2}{3}  \cdot k t^2\ln d}\,.
		\]
		If for unit vectors $x,y\in \R^n$, $\norm{x - y} \le \varepsilon$, then 
		$\norm{Wx - Wy} \le \varepsilon\norm{W}$. 
		By \cref{thm:bound-covariance-gaussian}, with probability at least $1-\exp\Paren{-n}$, $\norm{W}\le 10\sqrt{d}$. Hence if $\varepsilon\le\frac{1}{10\sqrt{d}}$,  $\Abs{(\transpose{x}W)_i - (\transpose{y}W)_i} \le 1$.
		By \cref{fact:epsilon-net-sphere}, for any $0 < \varepsilon < 1$, 
		for $\varepsilon = \frac{1}{10\sqrt{d}}$ there exists an $\varepsilon$-net in $n-1$-dimensional sphere of size
		$\exp\Paren{\frac{n}{2}\log d + n\log 100} \le \exp\Paren{n\log d}$ (for large enough $d$). By the union bound, the probability that $\card{S_t} > n/B$ is at most
		\[
		\exp\Paren{n\ln d -\frac{2}{3}\card{S_t}t^2} \le \exp\Paren{-n}\,.
		\]
	\end{proof}
\end{lemma}

The next lemma is the main technical challenge of \cref{sec:basic-sdp}.

\begin{theorem}\cite{DBLP:conf/nips/DeshpandeM14}\label{lem:spectral-norm-thresholded-gaussian}
	Let $W\sim N(0,1)^{n\times d}$, where $n \ge \omega(\log d)$ as $d\to \infty$. Let $0 \le \tau \le o\Paren{\sqrt{n\log d}}$ and let $N$ be the matrix whose diagonal entries are zeros and each non-diagonal entry $N_{ij}$ is
	\begin{align*}
	 N_{ij}=\begin{cases}
		\Paren{\transpose{W}W}_\ij - \sign{\Paren{\transpose{W}W}_\ij}\cdot \tau &\text{ if }\;\Abs{\Paren{\transpose{W}W}_\ij}\geq \tau\\
		0& \text{ otherwise}
		\end{cases}
	\end{align*}
	Then there exists an absolute constant $C \ge 1$ such that with probability $1-o(1)$
	\begin{align*}
		\Norm{N}\leq C\Paren{d + \sqrt{dn}}\exp\Brac{-{\frac{\tau^2}{Cn}}}\,.
	\end{align*}
\end{theorem}

\section{Linear Algebra}
\begin{lemma}\label{top-eigenvector-bound-rank1}
	Let $v$ and $u$ be unit vectors such that 
	$\norm{v\transpose{v} -u\transpose{u}} \leq \varepsilon$.
	Then $\iprod{v,u}^2 \ge 1 - 2\varepsilon^2$.
\end{lemma}\label{lemma:top-eigenvector-bound-rank1}
\begin{proof}
	Let $w$ be a unit vector orthogonal to $u$ such that $v = \rho u + \sqrt{1-\rho^2}w$ for some positive $\rho \le 1$. Then 
	\[
	v\transpose{v} - u\transpose{u} = 
	\Paren{\rho^2-1} u \transpose{u} + \rho\sqrt{1-\rho^2}u\transpose{w} + 
	\rho\sqrt{1-\rho^2}w\transpose{u} + \Paren{1-\rho^2}w\transpose{w}.
	\]
	Since $v\transpose{v} - u\transpose{u}$ has rank $2$, its Frobenius norm is bounded by $2\varepsilon$, hence
	\[
	4\varepsilon^2 \ge \norm{v\transpose{v} - u\transpose{u}}_F^2 = 
	2\Paren{1-\rho^2}^2 + 2\rho^2\Paren{1-\rho^2} = 2\Paren{1-\rho^2}.
	\]
	It follows that
	\[
	\iprod{v,u}^2 = \rho^2 \ge 1 - 2\varepsilon^2\,.
	\]
\end{proof}

\begin{lemma}\label{lemma:top-eigenvecor-bound-general}
	Let $M$ be a symmetric matrix such that $\norm{M-u\transpose{u}}\le \varepsilon < \frac{1}{2}\;$ for some unit vector $u$. Then the top eigenvalue $\lambda_1$ of $M$ satisfies 
	$\abs{\lambda_1 - 1}\le \varepsilon$ and the top eigenvector $v_1$ of $M$ satisfies  $\iprod{v_1,u}^2 \ge 1 - 100\varepsilon^2$.
\end{lemma}

\begin{proof}
	Consider an eigenvalue decomposition of $M$:
	\[
	M = \sum_{j=1}^d \lambda_j v_j \transpose{v_j},
	\]
	where $|\lambda_1| \ge |\lambda_2|\ge\ldots\ge |\lambda_d|$ and $\{v_j\}_{j=1}^d$ is an orthonormal basis in $\R^d$.
	By triangle inequality
	\[
	\norm{v_1\transpose{v_1} -u\transpose{u}} \le \norm{M - v_1\transpose{v_1}} + \norm{M-u\transpose{u}} \le \norm{M - v_1\transpose{v_1}} +\varepsilon.
	\]
	Let's bound $\norm{M - v_1\transpose{v_1}}$:
	\[
	\norm{M - v_1\transpose{v_1}} \le \max\{\abs{1-\lambda_1},\abs{\lambda_2}\}.
	\]
	Since $\norm{M-u\transpose{u}}\le \varepsilon$, $\transpose{u}Mu \ge 1-\varepsilon$, hence $|\lambda_1| \ge 1-\varepsilon > \frac{1}{2}$.  Notice that
	\[
	\abs{\lambda_1 - \iprod{v_1,u}^2} = \abs{\transpose{v_1}Mv_1 - \iprod{v_1,u}^2} \le \varepsilon,
	\]
	hence $\lambda_1\ge -\varepsilon > -\frac{1}{2}$, so $\lambda_1 > 0$, and $\lambda_1 \le  \iprod{v_1,u}^2 + \varepsilon \le 1 + \varepsilon$, so $\abs{\lambda_1-1}\le \varepsilon$. 
	
	By triangle inequality $\abs{\iprod{v_1,u}^2 - 1}\le 2\varepsilon$.
	By Pythagorean theorem $\sum_{j=1}^d \iprod{v_j, u}^2 = 1$, hence 
	\[
	\sum_{j=2}^d \iprod{v_j, u}^2 \le 2\varepsilon,
	\] 
	so $\iprod{v_2, u}^2 \le 2\varepsilon$.
	Now let's bound $\abs{\lambda_2}$:
	\[
	\abs{\lambda_2 - \iprod{v_2,u}^2} = \abs{\transpose{v_2}Mv_2 - \iprod{v_2,u}} \le \varepsilon,
	\]
	hence $\abs{\lambda_2} \le 3\varepsilon$. Therefore
	\[
	\norm{v_1\transpose{v_1} -u\transpose{u}} \le 4\varepsilon.
	\]
	By lemma \ref{top-eigenvector-bound-rank1},  $\iprod{v,u}^2 \ge 1 - 32\varepsilon^2$.
\end{proof}

\begin{lemma}\label{lem:linear-algebra-correlation-eigenverctor-large-quadratic-form}
	Let $M\in \R^{d\times d}$, $M\sge 0$, $\Tr M = 1$ and let $z\in \R^d$ be a unit vector such that $\transpose{z}Mz\geq 1-\eps$. Then the top eigenvector $v_1$ of $M$ satisfies $\iprod{v_1,z}^2\geq 1-O\Paren{\eps}$.
	\begin{proof}
		Write $z=\alpha v_1+\sqrt{1-\alpha^2}v_{\bot}$ where $v_\bot$ is a unit vector orthogonal to $v_1$. 
		\begin{align*}
		\transpose{z}Mz &= \alpha^2 \transpose{v_1}Mv_1+\Paren{1-\alpha^2}\transpose{v_\bot} M v_\bot\\
		&=\alpha^2 \Paren{\lambda_1-\transpose{v_\bot}Mv_\bot}+\transpose{v_\bot}M v_\bot \\
		&\geq 1-\eps
		\end{align*}
		As $\transpose{v_1}Mv_1\geq \transpose{z}Mz$ and $\transpose{v_\bot}Mv_\bot\leq \eps$, rearranging
		\begin{align*}
		\alpha^2 \geq \frac{1-\eps-\transpose{v_\bot}M v_\bot}{\lambda_1-\transpose{v_\bot}M v_\bot}\geq 1-2\eps.
		\end{align*}
	\end{proof}
\end{lemma}

\begin{fact}\label{fact:product-of-psd-matrices}
Let $A, B \in \R^{d\times d}$, $A,B\sge0$. Then $\iprod{A,B}\geq 0$.
\end{fact}

\begin{lemma}\label{lem:basic-sdp-iprod-matrix-spectral-norm}
	Let $X\in \R^{d\times d}$ be a positive semidefinite matrix. Then for any $A\in \R^{d\times d}$,
	\begin{align*}
	\Abs{\iprod{A,X}}\leq \norm{A}\cdot\Tr{X}\,.
	\end{align*}
	\begin{proof}
		Since $X$ is positive semidefinite, $X = \sum_{i=1}^d \lambda_i z_i\transpose{z_i}$ for unit vectors $z_i$ such that  $\lambda_i\ge 0$ and $ \sum_{i=1}^d \lambda_i = \Tr X$. Hence
		\begin{align*}
		\Abs{\iprod{A,X}} = \Abs{\Tr \transpose{X}A} = 
		\Abs{\sum_{i=1}^d \lambda_i \Tr  z_i \transpose{z_i} A}
		= \Abs{\sum_{i=1}^d \lambda_i \Tr  \transpose{z_i}A z_i } \le
		\sum_{i=1}^d \lambda_i \norm{A} = \norm{A} \cdot\Tr{X}\,.
		\end{align*}
	\end{proof}
\end{lemma}

\begin{lemma}\label{lem:basic-sdp-c-s}
	Let $X\in \R^{d\times d}$ be a positive semidefinite matrix. Then for any $a,b\in \R^d$,
	\begin{align*}
	\iprod{a\transpose{b}, X}^2 \le \iprod{a\transpose{a}, X}\cdot \iprod{b\transpose{b}, X}\,.
	\end{align*}
	\begin{proof}
		By \cref{fact:product-of-psd-matrices}, 
		$\iprod{a\transpose{a}, X} \ge 0$ and $\iprod{b\transpose{b}, X}\ge 0$.
		Notice that if the inequality is true for some $a,b\in \R^d$, it is also true for $c_1 a, c_2 b$ for all positive numbers $c_1,c_2$. So we can assume without loss of generality that $\iprod{a\transpose{a}, X}= \iprod{b\transpose{b}, X} = 1$. Consider  $\iprod{\Paren{a+b}\transpose{\Paren{a+b}}, X} \ge 0$ and $\iprod{\Paren{a-b}\transpose{\Paren{a-b}}, X} \ge 0$ . We get
		\[
		2\iprod{a\transpose{b}, X} \le \iprod{a\transpose{a}, X} + \iprod{b\transpose{b}, X}\le 2\,
		\]
		and
		\[
		-2\iprod{a\transpose{b}, X} \le \iprod{a\transpose{a}, X} + \iprod{b\transpose{b}, X}\le 2\,,
		\]
		hence $\iprod{a\transpose{b}, X}^2 \le 1$.
	\end{proof}
\end{lemma}

\section{Experimental Setup}\label{section:experimental-setup}
In the experiments, the instances were sampled from the planted distributions of models \ref{problem:spiked_covariance_model_lb} and \ref{problem:almost_gaussian_model_special_case} with the difference that $u\sim N(0, \Id_n)$ and $v$ is a $k$-sparse unit vector obtained sampling a random $k$-subset $S\subseteq[d]$  and then a unit vector with support $S$. 
All the algorithms returned the top $k$ coordinates of their estimation vector. 
Figure \ref{figure:svd-settings}, \ref{figure:svd-settings} plot the absolute correlation between $v$ and its estimate. Each plot was obtained averaging multiple independent runs on the same parameters, for each algorithm the shadowed part corresponds to the interval containing 50\% of the results, the line corresponds to the mean of the results in such interval. \\
In Figure \ref{subfigure:svd-4}, the adversarial matrix $E$ is sampled according to model $\ref{problem:almost_gaussian_model_special_case}$ for $s=2$, that is the first $2$ moments of $Y$ are Gaussian. Similarly, in Figure \ref{subfigure:svd-6}, $E$ is sampled according to model $\ref{problem:almost_gaussian_model_special_case}$ for $s=4$, so the first $4$ moments of $Y$ are Gaussian.

Experiments were done on a laptop computer with a 3.5 GHz Intel Core i7 CPU and 16 GB of RAM, random instances were obtained using \texttt{Numpy} pseudo-random generator.

\end{document}